\documentclass[nohyperref]{article}

\usepackage{microtype}
\usepackage{graphicx}
\usepackage{booktabs} 

\usepackage{hyperref}

\usepackage[accepted]{icml2022}

\usepackage{amsmath}
\usepackage{amssymb}
\usepackage{mathtools}
\usepackage{amsthm}
\usepackage{slashbox}

\usepackage{bm}

\usepackage{multirow}
\usepackage{graphicx,subcaption}
\usepackage{caption}
\usepackage{amsmath}
\usepackage{amssymb}
 \usepackage{mathrsfs}
\usepackage{tabularx}
\usepackage{capt-of}

 \usepackage{pifont}

\usepackage{sidecap}
\usepackage{enumitem}

\usepackage{amsthm}
\usepackage{thmtools,thm-restate}
\usepackage{multicol}

\newtheorem{induction}{Induction hypothesis}[section]

\usepackage{wrapfig}

\DeclareRobustCommand{\rchi}{{\mathpalette\irchi\relax}}
\newcommand{\irchi}[2]{\raisebox{\depth}{$#1\chi$}}

\newtheorem{theorem}{Theorem}[section]
\newtheorem{definition}{Definition}[section]

\newtheorem{remark}{Remark}
\newtheorem{lemma}{Lemma}[section]

\newtheorem{parametrization}{Parametrization}[section]
 
\icmltitlerunning{Towards understanding how momentum improves generalization in deep learning}

\begin{document}

\twocolumn[
\icmltitle{Towards understanding how momentum improves generalization in deep learning}




\begin{icmlauthorlist}
\icmlauthor{Samy Jelassi}{yy}
\icmlauthor{Yuanzhi Li}{yyy}
\end{icmlauthorlist}

\icmlaffiliation{yy}{Princeton University, NJ, USA.}
\icmlaffiliation{yyy}{Carnegie Mellon University, PA, USA.}

\icmlcorrespondingauthor{Samy Jelassi}{sjelassi@princeton.edu}

\icmlkeywords{Machine Learning, ICML}

\vskip 0.3in
]



\printAffiliationsAndNotice{} 

\begin{abstract}

    Stochastic gradient descent (SGD) with momentum is widely used for training modern deep learning architectures. While it is well-understood that using momentum can lead to faster convergence rate in various settings, it has also been observed that momentum yields higher generalization. Prior work argue that momentum stabilizes the SGD noise during training and this leads to higher generalization. In this paper, we adopt another perspective  and first empirically show that gradient descent with momentum (GD+M) significantly improves generalization compared to gradient descent (GD) in some deep learning problems. From this observation, we formally study how momentum improves generalization. We devise a binary classification setting where a one-hidden layer (over-parameterized) convolutional neural network trained with GD+M provably generalizes better than the same network trained with GD, when both algorithms are similarly initialized. 
    The key insight in our analysis is that momentum is beneficial in datasets where the examples share some feature but differ in their margin. Contrary to GD that memorizes the small margin data, GD+M still learns the feature in these data thanks to its historical gradients. Lastly, we empirically validate our theoretical findings. 

\end{abstract}

\section{Introduction}\label{sec:intro}

It is commonly accepted that adding momentum to an optimization algorithm is required to optimally train a large-scale deep network. Most of the modern architectures maintain during the training process a heavy momentum close to 1 \citep{krizhevsky2012imagenet,simonyan2014very,he2016deep,zagoruyko2016wide}. Indeed, it has been empirically observed that architectures trained with momentum outperform those which are trained without \citep{sutskever2013importance}. Several papers have attempted to explain this phenomenon.
From the optimization perspective, 
\cite{defazio2020understanding} assert that momentum yields faster convergence of the training loss since, at the early stages, it cancels out the noise from the stochastic gradients. On the other hand, \cite{leclerc2020two} empirically observes that momentum yields faster training convergence only when the learning rate is small. 
While these works shed light on how momentum acts on neural network training, 
they fail to capture the generalization improvement induced by 
momentum~\citep{sutskever2013importance}. Besides, the noise reduction property of momentum advocated by~\cite{defazio2020understanding} contradicts the observation that, in deep learning, having a large noise in the training improves generalization \citep{li2019towards,haochen2020shape}. To the best of our knowledge, there is no existing work which \emph{theoretically} explains how momentum improves generalization in deep learning. 
Therefore, this paper aims to close this gap and addresses the following question:

\vspace{-.2cm}

\begin{center}
     \emph{ Why does momentum improve generalization?  What is the underlying mechanism of momentum improving generalization in deep learning?}
\end{center}

\begin{figure*}[t]
\centering
\begin{subfigure}[t]{0.48\textwidth}
 \includegraphics[width=.9\linewidth]{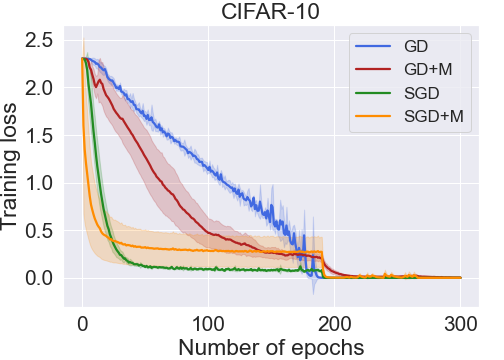}
 \vspace{-2mm}
 \caption{}\label{fig:traincifar10}
\end{subfigure}
\begin{subfigure}[t]{0.48\textwidth}
 \includegraphics[width=.9\linewidth]{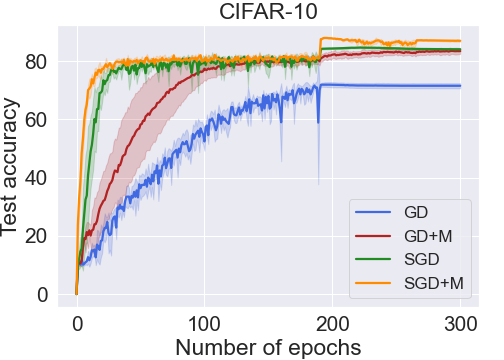}
 \vspace{-2mm}
\caption{}\label{fig:testcifar10}
\end{subfigure}
\vspace{-.3cm}
\caption{\small Training loss (a) and test accuracy (b) obtained with VGG-19   trained with SGD, SGD+M, GD and GD+M on CIFAR-10. The model is trained for 300 epochs to ensure zero training error. To isolate the effect of momentum, we \textit{turn off} data augmentation and batch normalization (see \autoref{sec:num_exp} for further implementation details). GD and SGD respectively refer to stochastic gradient descent with batch sizes $50k$ (full batch) and $128$. 
 Results are averaged over 5 seeds.}\label{tab:cifar} 
\end{figure*}

\begin{table*}[tbp]
\vspace{-.3cm}
 \center
 \resizebox{1.5\columnwidth}{!}{%
 \begin{tabular}{|l||*{5}{c|}}\hline
\backslashbox{Student}{Teacher}
&\makebox[3em]{Linear} &\makebox[3em]{1-MLP}
&\makebox[3em]{2-MLP}&\makebox[3em]{1-CNN}&\makebox[3em]{2-CNN}\\\hline\hline
1-MLP & 93.48/93.25 &  92.32/92.18 & 84.3/83.68  & 94.18/94.12  &  76.04/76.12 \\\hline
2-MLP & 93.45/92.85  & 91.02/91.78    &  83.82/83.25 & 94.14/94.20    & 75.50/75.56  \\\hline
1-CNN & 92.21/92.34  & 92.31/92.33   &83.39/83.44 & 94.39/94.39    &  79.44/78.32 \\\hline
2-CNN & 91.04/91.22  & 91.51/91.56   &  82.44/82.12  &  93.91/93.79  & 80.86/78.56  \\\hline
\end{tabular}
}
 
\vspace{-.2cm} 
 
\caption{\small Test accuracy obtained using GD/GD+M on a Gaussian synthetic dataset trained using neural network with ReLU activations. The training dataset consists in 500 points in dimension 30 and test set in 5000 points. The student networks are  trained for 1000 epochs to ensure that the loss stays constant. Results are averaged over 3 seeds and we only report the mean (see \autoref{sec:app_exps} for full table). }
\label{tab:gauss}

\vspace{-5mm}

\end{table*}

\vspace{-.2cm}

In computer vision, practitioners usually train their architectures with stochastic gradient descent with momentum (SGD+M). It is therefore natural to investigate whether the generalization improvement induced by momentum is tied to the  stochasticity of the gradient.  We train a VGG-19 \citep{simonyan2014very} using SGD, SGD+M, gradient descent (GD) and GD with momentum (GD+M) on the CIFAR-10 image classification task. To further isolate the regularization effect of momentum, we turn off data augmentation and batch normalization. \autoref{tab:cifar} displays the training loss and test accuracy of the four models. Not only momentum improves generalization in the full batch setting but the generalization improvement increases as the batch size is larger. Motivated by this empirical observation, we focus on the contribution of momentum in gradient descent. We emphasize that this setting allows to isolate the contribution of  momentum on generalization since the stochastic gradient noise influences generalization \citep{li2019towards,haochen2020shape}.

Given the success of momentum in different deep learning tasks such as image classification \citep{simonyan2014very,he2016deep} or language modelling \citep{vaswani2017attention,devlin2018bert}, we start our investigation by raising the following question:  
\vspace*{-.25cm}
\begin{center}
    \emph{Does momentum \textbf{unconditionally} improve generalization in deep learning?} 
\end{center}
\vspace{-.25cm}
We respond in the negative to this question  through the following synthetic binary classification example. We consider a Gaussian dataset where each data-point is sampled from a standard normal distribution. We generate the labels using multiple teacher networks. Starting from the same initialization, we train several student networks on this dataset using GD and GD+M and compare their test accuracies in \autoref{tab:gauss}. Whether the target function is simple (linear) or complex (neural network), momentum does not improve generalization for any of the student networks. The same observation holds for SGD/SGD+M as shown in  \autoref{sec:app_exps}. Therefore, momentum \textit{does not} always lead to a higher generalization in deep learning. Instead, such benefit seems to heavily depend on both the  \textit{structure of the data} and the \textit{learning problem}.


Motivated by the aforementioned observations, this paper aims to determine the underlying mechanism produced by momentum to improve generalization. Our work is a first step to formally understand the role of momentum in deep learning. Our contributions are divided as follows: 

\begin{itemize}[leftmargin=*, itemsep=1pt, topsep=1pt, parsep=1pt]
    \item[--] In \autoref{sec:num_exp}, we empirically confirm that momentum consistently improves generalization when using different architectures on a wide range of batch sizes and datasets. We also observe that as the batch size increases, momentum contributes more significantly to generalization.
    \item[--] In \autoref{sec:setup}, we introduce our synthetic data structure and learning problem to theoretically study the contribution of momentum to generalization.
    \item[--] In \autoref{sec:results}, we present our main theorems along with the intermediate lemmas.  We theoretically show that a 1-hidden layer neural network trained with GD+M on our synthetic dataset is able to generalize better than the same model trained with GD. Above all, we rigorously characterize the mechanism by which momentum improves generalization. A sketch of the proof is presented in \autoref{sec:GD} and \autoref{sec:GDM}. 
\end{itemize}

\begin{figure}[t]
\hspace*{1.1cm}
\begin{subfigure}[t]{0.48\textwidth}
         \includegraphics[width=.7
        \linewidth]{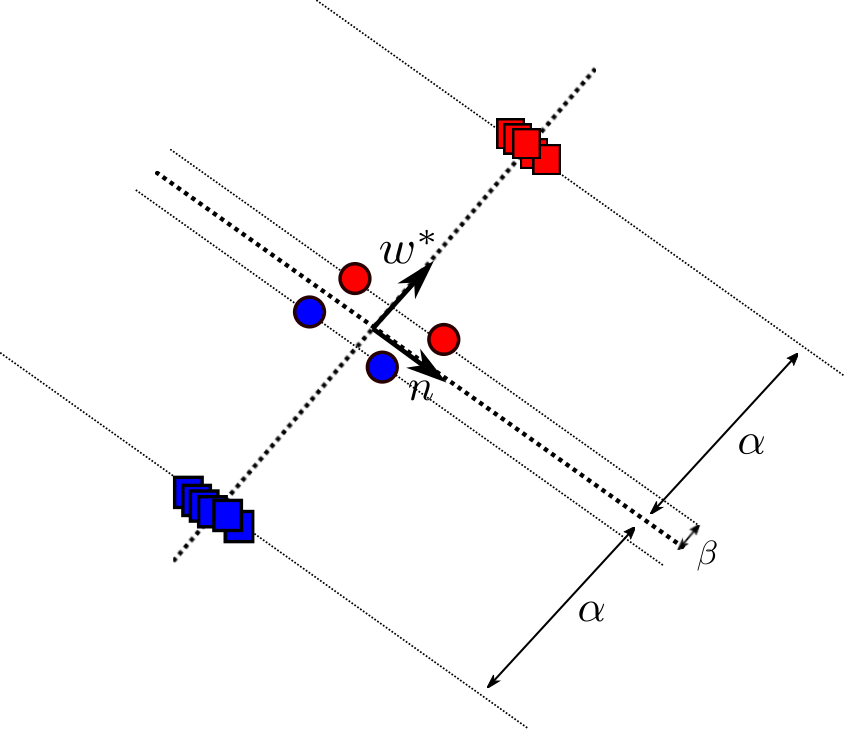}
\end{subfigure}    
    \vspace{-.5cm}
\caption{\small Our synthetic dataset in 2D. Each data-point is $\bm{X}_i=[c_i\cdot \bm{w^*},d_i\cdot \bm{n}]\in \mathbb{R}^4$ for some $c_i,d_i\in \mathbb{R}.$ We project these points in the 2D space $(\mathrm{span}(\bm{w^*}),\mathrm{span}(\bm{n})).$ The feature is $\bm{w^*}$ and the noisy patch is in  $\mathrm{span}(\bm{n})$. 
The large margin data (squares) have large component along $\bm{w^*}$ and relatively small noise component. The  small margin data (circles) have relatively large noise component and thus, these data are well-spread on $\mathrm{span}(\bm{n})$. 
}\label{fig:dataset}
\vspace{-.4cm}
 \end{figure}

 \vspace{-.3cm}
 
\paragraph{Insights on the setting.} The previous experiments suggest that momentum improves generalization in CIFAR-10 while it does not for Gaussian datasets. This means that this generalization improvement must be specific to the data structure and the learning problem. In Section 3, we devise a binary classification problem where   the data are linearly separated by a hyperplane directed by the vector $\bm{w^*}$ as depicted in \autoref{fig:dataset}.  We refer to this vector as the feature and the goal is to learn it. Each data-point is a vector constituted of a single signal patch equal to $\theta \bm{w^*}$ and of multiple noise patches. For $\mu \ll 1$, we assume that with probability $1-\mu$, the sampled data-point has large margin i.e. $\theta=\alpha\gg 1$ while it has small margin i.e. $\theta=\beta \ll 1$ with probability $\mu$. The noise patches are Gaussian random vectors with small variance. We underline that all the examples share the \textit{same feature} but differ in their margins. Our dataset can be viewed as an extreme simplification of real-world  object-recognition datasets with data of different level of difficulty. Indeed, images are divided into signal patches that are helpful for the classification such as the nose of a dog and noise patches e.g. the background of an image that are uninformative. Besides, the signal patch may be strong i.e.\ the feature is clearly visible or weak when the feature is indistinguishable e.g.\ in a car image, the wheel feature is more or less visible  depending on the orientation of the car.

\paragraph{Why does GD+M generalize better than GD?} This paper proposes a theory to explain why momentum improves generalization. The following informal theorems characterize the generalization of a 1-hidden layer convolutional neural network trained with GD and GD+M on the aforedescribed dataset. They dramatically simplify \autoref{thm:GD} and \autoref{thm:GDM} but highlight the intuitions.

\begin{theorem}[Informal, GD]\label{thm:gd_inf} There exists a dataset of size $N$  such that a 1-hidden layer (over-parameterized) convolutional network trained with GD: 

\vspace*{-.3cm}

\begin{enumerate}
    \item initially only learns the $(1-\mu)N$ large margin data.
    \item has small gradient after learning these data.
    \item memorizes the remaining small margin data from the $\mu N$ examples.
\end{enumerate}

\vspace*{-.3cm}

The model thus reaches zero training loss and well-classifies the large margin data at test. However, it fails to classify the small margin data because of the memorization step during training.
\end{theorem}

\begin{theorem}[Informal, GD+M]\label{thm:gdm_inf} There exists a dataset of size $N$ such that a one-hidden layer (over-parameterized) convolutional network trained with GD+M: 

\vspace*{-.4cm}

\begin{enumerate}
    \item initially only learns the $(1-\mu)N$ large margin data.
    \item has large historical gradients that contain the feature $\bm{w^*}$ present in small margin data. 
    \item keeps learning the feature in the small margin data using its momentum historical gradients. 
\end{enumerate}

\vspace*{-.4cm}

The model thus reaches zero  training error and perfectly classify large and small margin data at test.

\end{theorem}

\autoref{thm:gd_inf} and \autoref{thm:gdm_inf} indicate that since the large margin data are dominant, the two models learn in priority these examples to decrease their training losses. Since the training loss is the logistic one, this implies that the gradient terms stemming from the large margin data thus become negligible. Consequently, the current gradient becomes a sum of the small margin data gradients. Thus, it is in the direction of $\beta \bm{w^*}$ (signal patch) and Gaussian vectors $\mathbf{g}$ (noise patches).  Since $\|\beta\bm{w^*}\|_2\ll \|\mathbf{g}\|_2 $,  the current gradient is noisy. Therefore, the GD model keeps decreasing its training loss and memorizes the small margin data. On the other hand, contrary to GD, GD+M updates its weights using a weighted average of the \emph{historical} gradients. In particular, it has large past gradients (stemming from large margin data) that are in the direction $\alpha \bm{w^*}$. Therefore, even though the current gradient is noisy, the GD+M uses its historical gradients to learn the small margin data \textit{since all the examples share the same feature}. We name this process \textit{historical feature amplification} and believe that it is key to understand why momentum improves generalization. 

\begin{figure*}[tbp]  
\vspace{-.3cm}
\begin{subfigure}{0.33\textwidth}
 \includegraphics[width=1\linewidth]{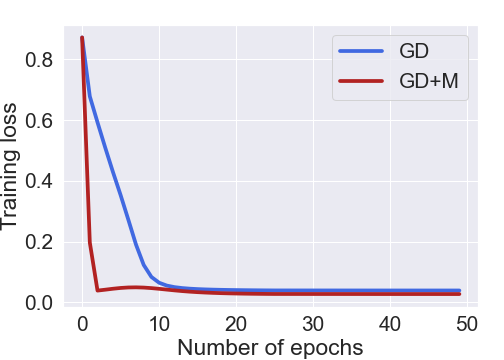}
 \vspace{-5mm}
 \caption{}\label{fig:trainsyn}
\end{subfigure}
\begin{subfigure}{0.33\textwidth}
\hspace{.1cm}
 \includegraphics[width=1\linewidth]{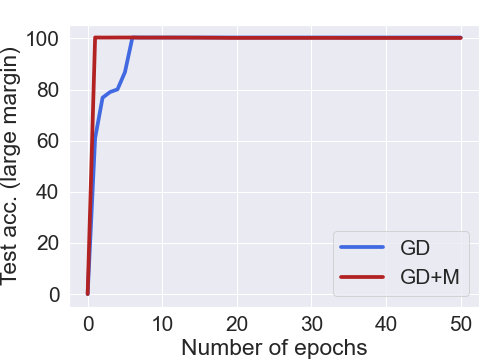}
  \vspace{-5mm}
\caption{}\label{fig:testloss}
\end{subfigure}
\begin{subfigure}{0.33\textwidth}
\hspace{.1cm}
\includegraphics[width=1\linewidth]{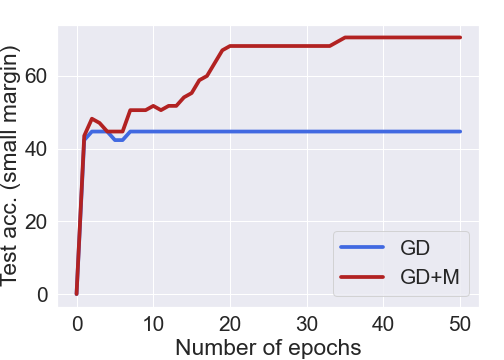}
 \vspace{-5mm}
\caption{ }\label{fig:testlosssmall}
\end{subfigure}
\vspace{-4mm}
\caption{\small  Training loss (a),  accuracy on the large margin (b) and the small margin test data (c) in the setting described in \autoref{sec:setup}. While GD and GD+M get zero training loss, GD+M  generalizes better on small margin data than GD. 
}\label{fig:theory}
\end{figure*}

\begin{figure*}[t]  
\vspace*{-.3cm}
\begin{subfigure}{0.48\textwidth}
 \hspace*{1.cm}\includegraphics[width=.8\linewidth]{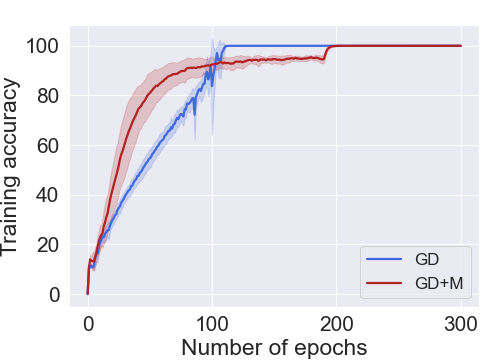}
 \vspace{-2mm}
 \caption{}\label{fig:acctrainsmallmg}
\end{subfigure}
\begin{subfigure}{0.48\textwidth}
 \vspace{-5mm}
 \includegraphics[width=.8\linewidth]{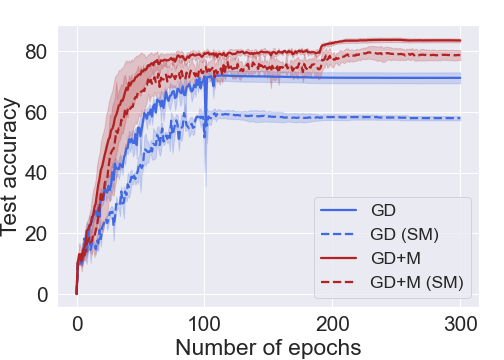}
  \vspace{-6mm}
\caption{}\label{fig:testsmallmg}
\end{subfigure}
\begin{subfigure}{\textwidth}
 \mbox{\includegraphics[width=.31\linewidth]{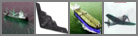}
  \hspace{.5cm}  \includegraphics[width=.31\linewidth]{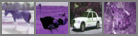}  \hspace{.5cm} \includegraphics[width=.31\linewidth]{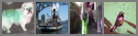} }
\caption{}\label{fig:imgssmallmg}
\end{subfigure}
\vspace{-4mm}
\caption{\small 
 Training (a) and test (b) accuracy obtained with VGG-19 on the artificially modified CIFAR-10 dataset with small margin data (c). The architectures are trained using GD/GD+M for $300$ epochs to ensure zero training error. Data augmentation and batch normalization are turned off. (SM) stands for the test accuracy obtained by the algorithm on the small margin data.  Results are averaged over 5 runs with best scheduled learning rate and weight decay for each individual algorithm separately.}\label{fig:smallmg}
 \vspace{-.5cm}
 
\end{figure*}


\paragraph{Numerical validation of the theory.} Our theory relies on the ability of momentum to well-classify small margin data. We first perform experiments in our theoretical  setting described in \autoref{sec:setup}. We set the dimension to $d=30$, the number of training examples to $N=20000$, the test examples to $2000$. Regarding the architecture, we set the number of neurons to $m=5$ and the number of patches to $P=5$. The parameters $\alpha,\beta,\mu$ are set as in \autoref{sec:setup}.  We refer to stochastic gradient descent optimizer with full batch size as GD/GD+M. Note that for each optimizer, we grid-search over stepsizes to find the best one in terms of test accuracy. We trained the models for 50 epochs. We set the momentum parameter to 0.9. We apply a linear decay learning rate scheduling during training.  \autoref{fig:theory} shows that the models trained with GD and GD+M get zero training loss and  well-classify large-margin data at test time. Contrary to GD, GD+M well-classifies small margin data. 


\paragraph{Small-margin data in CIFAR-10.} To further validate our theory, we artificially generate small-margin data in CIFAR-10. We first randomly sample 10\% of the training and test images. As displayed in \autoref{fig:imgssmallmg}, for each image, we randomly shuffle the RGB channels. We train a VGG-19  without data augmentation nor batch normalization. While the GD and GD+M models reach 100\% training accuracy, \autoref{fig:smallmg} shows that GD+M gets higher test accuracy than GD. Above all, GD+M generalizes better than GD on small-margin data as the accuracy drop factor for GD+M is $79.47/53.30=1.49$ while for GD, this drop factor is $68.33/34.80=1.96$.


\vspace{-.3cm}

\subsection*{Related Work}

\paragraph*{Non-convex optimization with momentum.} A long line of work consists in understanding the convergence speed of momentum methods when optimizing non-convex functions. \citep{mai2020convergence,liu2020improved,cutkosky2020momentum,defazio2020understanding} show that SGD+M reaches a stationary point as fast as SGD under diverse assumptions. Besides, \cite{leclerc2020two} empirically shows that momentum accelerates neural network training for small learning rates and slows it down otherwise. Our paper differs from these works as we work in the batch setting and 
theoretically investigate the generalization benefits brought by momentum (and not the training ones). 


\vspace{-.5cm}

\paragraph*{Generalization with momentum.}  Momentum-based methods such as SGD+M, RMSProp \citep{tieleman2012lecture} and Adam \citep{kingma2014adam} are standard in deep learning training since the seminal work of \cite{sutskever2013importance}. 
Although it is known that momentum improve generalization in deep learning, only a few works formally investigate the role of momentum in generalization. \cite{leclerc2020two} \emph{empirically} report that momentum   yields higher generalization when using a large learning rate. However, they assert that this benefit can be obtained by 
applying an even larger learning rate on vanilla SGD. We suspect that this is due to \textit{data augmentation}   and \textit{batch normalization} \citep{ioffe2015batch}  which are known to  bias the algorithm's generalization \citep{bjorck2018understanding}. To our knowledge, our work is the first that \emph{theoretically} investigates the generalization of momentum in deep learning. 

\begin{figure*}[tbp] 
\vspace{-.2cm}
\begin{subfigure}{0.48\textwidth}
\hspace{.7cm}\includegraphics[width=.8\linewidth]{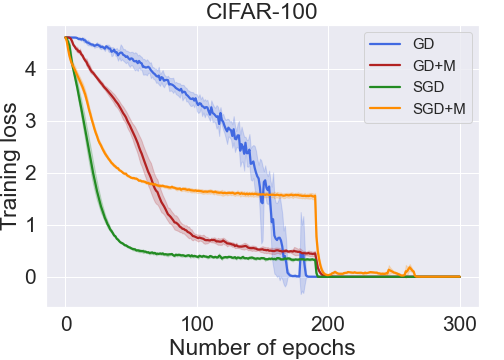}
 \vspace{-1.5mm}
 \caption{}\label{fig:traincifar100}
\end{subfigure}
\begin{subfigure}{0.48\textwidth}
\hspace{.5cm}
 \includegraphics[width=.8\linewidth]{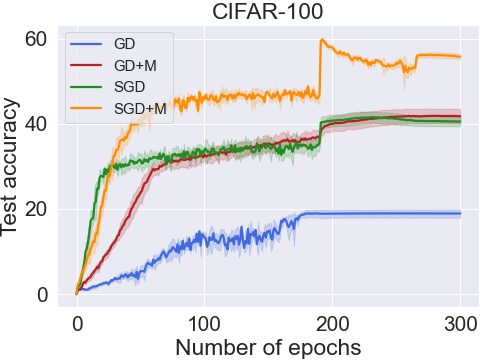}
  \vspace{-1.5mm}
\caption{ }\label{fig:testcifar100}
\end{subfigure}
\vspace{-4mm}
\caption{\small Training loss (a) and test accuracy (b) obtained with VGG-19  trained with SGD, SGD+M, GD and GD+M on CIFAR-100. 
Data augmentation and batch normalization are turned off. Momentum significantly improves generalization whether in the stochastic case (SGD) or in the full batch setting (GD).}\label{fig:cifar100}
\end{figure*}

\begin{figure*}[t] 
\vspace{-.2cm}
\begin{subfigure}{0.48\textwidth}
\hspace{.5cm}\includegraphics[width=.8\linewidth]{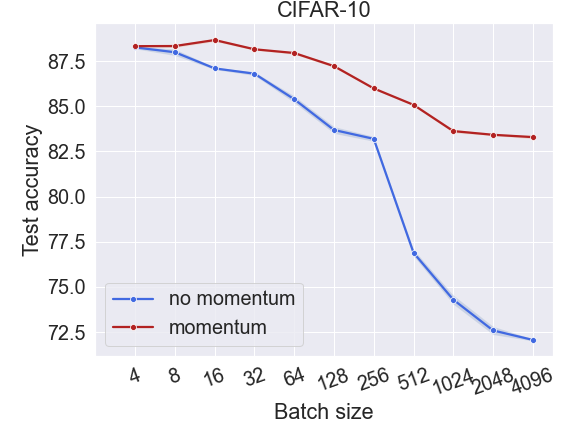}
 \vspace{-2mm}
 \caption{}\label{fig:batchcifar10}
\end{subfigure}
\begin{subfigure}{0.48\textwidth}
 \hspace{.5cm}\includegraphics[width=.81\linewidth]{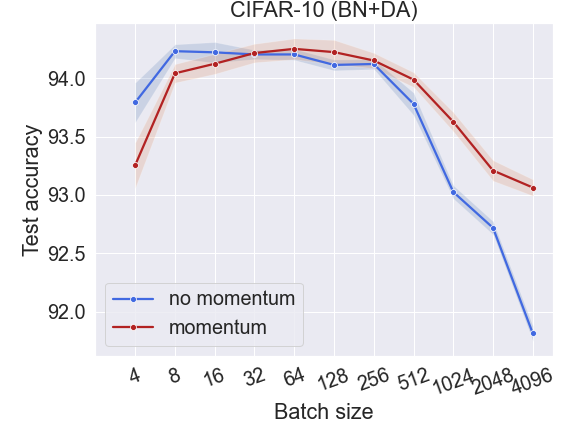}
  \vspace{-2mm}
\caption{ }\label{fig:bndacifar10}
\end{subfigure}
\vspace{-4mm}
\caption{\small Test accuracy obtained with a VGG-19 using the stochastic gradient descent optimizer on CIFAR-10 when batch normalization and data augmentation are turned off (a) and on (b). In (a), as the batch size increases, the generalization improvement induced by momentum gets larger. When the network is trained with batch normalization and data augmentation, momentum slightly improves generalization for large batch sizes.
}\label{fig:bnda}
\vspace{-.5cm}
\end{figure*} 
\section{Numerical performance of momentum}\label{sec:num_exp}

To evaluate the contribution of momentum to generalization, we conducted extensive experiments on CIFAR-10 and CIFAR-100  \citep{krizhevsky2009learning}.  We used  VGG-19 \citep{simonyan2014very} and Resnet-18 \citep{he2016deep} as  architectures.  In this section, we only present the plots obtained with VGG-19 and invite the reader to look at \autoref{sec:app_exps} for the Resnet-18 experiments.

In all of our experiments, we refer to the stochastic gradient descent optimizer with batch size 128 as SGD/SGD+M and the optimizer with full batch size as GD/GD+M. We turn off data augmentation and batch normalization to isolate the contribution of momentum to the optimization. Note that for each algorithm, we grid-search over stepsizes and momentum parameter to find the best one in terms of test accuracy. We train the models for 300 epochs. The stepsize is  decayed by a factor 10 at epochs 190 and 265 during training. All the results are averaged over 5 seeds.


\textbf{Momentum improves generalization.} \autoref{fig:cifar100} shows the performance of GD, GD+M, SGD  and SGD+M when training a VGG-19 on CIFAR-100. We observe that  GD+M/SGD+M consistently outperform GD/SGD. Besides, we highlight that the generalization improvement induced is more significant for GD than for SGD. Similar observations hold for Resnet-18 (see \autoref{sec:app_exps}).

\textbf{Influence of batch size.}  \autoref{fig:batchcifar10} shows the test accuracy of a VGG-19 trained on CIFAR-10 with the stochastic gradient descent optimizer on a wide range of batch sizes. We compare the generalization obtained with momentum and without. We remark that momentum does not improve generalization  when the batch size is tiny. However, as the batch size increases, the gap between the momentum curve and the no momentum one widens.  


\textbf{Batch normalization and data augmentation.} Practitioners usually add batch normalization and data augmentation when training their architectures. \autoref{fig:bndacifar10} displays the test accuracy obtained when training a VGG-19 with these two regularizers. We remark that they \emph{inhibit} the generalization improvement of momentum for small and middle range batch sizes. For large batch sizes, momentum slightly improves generalization. Additional experiments on the influence of batch normalization and data augmentation are in \autoref{sec:app_exps}.

\section{Setting and algorithms}\label{sec:setup}
In this section, we introduce our theoretical setting to analyze the implicit bias of momentum. We first formally define the data distribution sketched in the introduction and the neural network model we use to learn it. We finally present the GD and GD+M algorithms.

\vspace{-.1cm}
 
\paragraph*{General notations.} 
For a matrix $\bm{W}\in\mathbb{R}^{m\times d}$, we denote by $\bm{w}_r$ its $r$-th row. For a function $f\colon \mathbb{R}^{m\times d} \rightarrow \mathbb{R}$, we denote by $\nabla_{\bm{w}_r}f(W)$ the gradient of $f$ with respect to $w_r$ and $\nabla f(\bm{W})$ the gradient with respect to $\bm{W}.$ For an optimization algorithm updating a vector $\bm{w}$,  $\bm{w}^{(t)}$ represents its iterate at time $t$. We use $\mathbf{I}_d$ for the $d\times d$ identity matrix and $\mathbf{1}_m$ the all-ones vector of dimension $m.$  Finally, we use the asymptotic complexity notations when defining the different constants in the paper. We use $\tilde{O}, \tilde{\Theta}, \tilde{\Omega}$ to hide logarithmic dependency on $d$.

\paragraph*{Data distribution.}\label{sec:data} We define a data distribution $\mathcal{D}$ where each sample consists in an input $\bm{X}$ and a label $y$ such that:

\vspace{-.1cm}

\begin{enumerate}[leftmargin=*, itemsep=1pt, topsep=1pt, parsep=1pt]
    \item Uniformly sample the label $y$ from $\{-1,1\}.$
    \item  $\bm{X}=(\bm{X}[1],\dots,\bm{X}[P])$ where each patch $\bm{X}[j]\in\mathbb{R}^d.$
    \item Signal patch: one patch $P(X)\in[P]$  satisfies\\
    \vspace{-.6cm}
    \begin{align*}
        \hspace*{-.2cm}
    \bm{X}[P(\bm{X})]=c\bm{w^*}, \text{where } c\in\mathbb{R}, \bm{w^*}\in\mathbb{R}^d, \|\bm{w^*}\|_2=1.
   \end{align*}
   
   \vspace*{-.3cm}
   
    \item 
    $c$ is distributed as $c=\alpha y$ with probability $1-\mu$\\
    and $c=\beta y$ otherwise.
    \item Noisy patches:  $\bm{X}[j]\sim\mathcal{N}(0,(\mathbf{I}_d-\bm{w^*}\bm{w^{*\top}})\sigma^2 ),$\\
    for  $j\in [P]\backslash \{P(\bm{X})\}$.
\end{enumerate}

\vspace{.1cm}

To keep the analysis simple, the noisy patches are sampled from the orthogonal complement of $\bm{w^*}$  and the parameters are set to $\beta=d^{-0.251}$, $\alpha=\mathrm{polylog(d)}\sqrt{d}\beta$, $\sigma=\frac{1}{d^{0.509}}$ and $P \in [2,  \mathrm{polylog(d)} ]$.

Using this model, we generate a training dataset $\mathcal{Z}=\{(\bm{X}_i,y_i)\}_{i\in [N]}$ where $\bm{X}_i=(\bm{X}_{i}[j])_{j\in [P]}.$ We set $\mu=1/\mathrm{poly}(d)$ and $N = \Theta\left( \frac{\log\log(d)}{\mu} \right)$. We let $\mathcal{Z}$ to be partitioned in two sets $\mathcal{Z}_1$ and $\mathcal{Z}_2$ such that $\mathcal{Z}_1$ gathers the large margin data while $\mathcal{Z}_2$ the small margin ones. Lastly, we define $\hat{\mu}=\frac{|\mathcal{Z}_2|}{N}$ the fraction of small margin data.

\paragraph*{Learner model.} We use a 1-hidden layer convolutional neural network with cubic activation to learn the training dataset $\mathcal{Z}$.  The cubic is the smallest polynomial degree that makes the network non-linear and compatible with our setting. Indeed, the quadratic activation would only output positive labels and mismatch our labeling function. 
The first layer weights are  $\bm{W}\in\mathbb{R}^{m\times d}$ and the second layer is fixed to $\mathbf{1}_m.$ Given a input data $\bm{X}$, the output of the model is
\vspace{-.7cm}

\begin{align}\label{eq:model}\tag{CNN}
    f_{\bm{W}}(\bm{X})=\sum_{r=1}^m\sum_{j=1}^P \langle \bm{w}_r,\bm{X}[j]\rangle^3.
\end{align} 
The number of neurons is set as $m= \mathrm{polylog}(d)$ to ensure that \eqref{eq:model} is mildly over-parametrized.  
\paragraph*{Training objective.} We solve the following logistic regression problem for $\lambda \in [0,  1/\mathrm{poly}(d) N ]$,

\begin{align}\label{eq:optim_pb}\tag{P}
 \resizebox{.9\hsize}{!}{$\displaystyle \vspace{-1cm}\min_{\bm{W}} \hspace{-.1cm} \frac{1}{N}\sum_{i=1}^N\log(1+e^{-y_i f_{\bm{W}}(\bm{X}_i)} ) + \frac{\lambda}{2} \| \bm{W} \|_2^2=\widehat{L}(\bm{W}).$}
\end{align}

\paragraph*{Importance of non-convexity.} When $\lambda > 0$, if the loss $\frac{1}{N}\sum_{i=1}^N\log\left(1+\exp\left(-y_i f_{\bm{W}}(\bm{X}_i) \right) \right) $ is convex, then there is a unique global optimal solution, so the choice of optimization algorithm \textit{does not matter}. In our case, due to the non-convexity of the training objective, GD+M converges to a different (approximate) global optimal compared to GD, with better generalization properties. 

\vspace{-3mm}
\paragraph*{Test error.} We assess the quality of a predictor $\bm{\widehat{W}}$ using the classical 0-1 loss used in binary classification. Given a sample $(\bm{X},y),$ the \emph{individual test (classification) error} is defined as $\mathscr{L}(\bm{X},y)=\mathbf{1}\{f_{\bm{\widehat{W}}}(\bm{X})y<0\}.$ While $\mathscr{L}$ measures the error of $f_{\bm{\widehat{W}}}$ on an individual data-point, we are interested in the \textit{test error} that measures the average loss over data points generated from $\mathcal{D}$ and defined as 


\vspace{-1cm}

\begin{align}\label{eq:testerr}\tag{TE}
\mathscr{L}(f_{\bm{\widehat{W}}}):=\mathbb{E}_{(\bm{X},y)\sim\mathcal{D}}[\mathscr{L}(f_{\bm{\widehat{W}}}(\bm{X}),y)].
\end{align}

\paragraph*{Algorithms.}\hspace{-.3cm} We solve the training problem \eqref{eq:optim_pb} using GD and GD+M. GD is defined for $t\geq 0$ by 
\begin{align}\label{eq:GD}\tag{GD}
    \bm{W}^{(t+1)}&=\bm{W}^{(t)}-\eta \nabla \widehat{L}(\bm{W}^{(t)}),
\end{align}
 where $\eta>0$ is the learning rate. On the other hand, GD+M is defined by the update rule
\begin{align}\label{eq:GDM}\tag{GD+M}
    \begin{cases}
       \bm{g}^{(t+1)}&= \gamma \bm{g}^{(t)}+(1-\gamma)\nabla \widehat{L}(\bm{W}^{(t)})\\
    \bm{W}^{(t+1)}&=\bm{W}^{(t)}-\eta \bm{g}^{(t+1)}
    \end{cases},
\end{align}
where $\bm{g}^{(0)}=\bm{0}_{m\times d}$ and $\gamma\in (0,1)$ is the momentum factor. We now detail how to set parameters in \ref{eq:GD} and \ref{eq:GDM}. 

 \begin{parametrization}\label{ass:paramsGDGDM} When running GD and GD+M on \eqref{eq:optim_pb},  the number of iterations is any $T \in \left[ \mathrm{poly(d) }N/\eta , d^{O(\log d)}/(\eta)\right]$. For both algorithms, the weights $\bm{w}_1^{(0)},\dots,\bm{w}_m^{(0)}$ are initialized using independent samples from a normal distribution $\mathcal{N}(0,\sigma_0^2\mathbf{I}_d)$ where $\sigma_0^2 = \frac{\mathrm{polylog}(d)}{d}.$ The learning rate is set as: 
 \vspace{-.2cm}
 
 \begin{enumerate}
  \item GD: the learning rate is reasonable $\eta\in (0,\tilde{O}(1)].$
  \item GD+M: the learning rate is large: $\eta=\tilde{\Theta}(1).$\footnote{This is consistent with the empirical observation that only momentum with large learning rate improves generalization~\citep{sutskever2013importance}}
 \end{enumerate}
Lastly, the momentum factor is set to  $\gamma = 1 - \frac{\mathrm{polylog}(d)}{d}$. 
 \end{parametrization}
Our \autoref{ass:paramsGDGDM} matches with the parameters used 
in practice as the weights are generally initialized from Gaussian with small variance and momentum is set close to 1 \citep{sutskever2013importance}.


\section{Main results}\label{sec:results}
We now formally state our main theorems regarding the generalization of models trained using \eqref{eq:GD} and \eqref{eq:GDM} in the setting described in \autoref{sec:setup}. We first introduce some notations.

\vspace{-.3cm}

\paragraph*{Main objects. } Let $r\in [m]$, $i\in [N]$, $j\in P\backslash \{P(\bm{X}_i)\} $ and $t\geq 0.$ Our analysis tracks $\bm{w}_r^{(t)}$ the $r$-th weight of the network, $\nabla_{\bm{w}_r}\widehat{L}(\bm{W}^{(t)})$ the gradient of $\widehat{L}$ with respect to $\bm{w}_r$, $\bm{g}_r^{(t)}$ the momentum gradient defined by $\bm{g}_r^{(t+1)} =\gamma \bm{g}_r^{(t)}+(1-\gamma)\nabla_{\bm{w}_r}\widehat{L}(\bm{W}^{(t)})$. We introduce the projection of these objects on the feature $\bm{w^*}$ and noise patches $\bm{X}_{i}[j]$: 


{\hspace{.2cm}-- Projection on $\bm{w^*}$: $c_r^{(t)}=\langle \bm{w}_r^{(t)},\bm{w^*}\rangle$.}


\mbox{\hspace{.2cm}--  Projection on $\bm{X}_{i}[j]:$ $\Xi_{i,j,r}^{(t)}=\langle \bm{w}_r^{(t)},\bm{X}_{i}[j]\rangle$.}

\mbox{\hspace{.2cm}-- Total noise: $\Xi_i^{(t)} = \sum_{r=1}^m\sum_{j\in [P]\backslash\{P(\bm{X}_i)\}} y_i (\Xi_{i,j,r}^{(t)})^3.$}

\vspace{-.1cm}

\mbox{\hspace{.2cm}-- Maximum signal: $c^{(t)}=\max_{r\in[m]}c_{r_{\max}}^{(t)}$}. 


    


Lastly, we define the negative sigmoid  $\mathfrak{S}(x)=1/(1+e^x).$

We now provide our first result which states that the learner model trained with GD does not generalize well on $\mathcal{D}$.

\begin{restatable}{theorem}{GDmain}\label{thm:GD} Assume that we run GD on \ref{eq:optim_pb} for $T$ iterations with parameters set  as in \autoref{ass:paramsGDGDM}. With high probability,  the weights learned by GD

\vspace{-.2cm}

 \begin{enumerate}[leftmargin=*, itemsep=1pt, topsep=1pt, parsep=1pt]
     \item  \mbox{partially learn $\bm{w^*}$: for $r\in [m]$,  $|c_r^{(T)}| \leq \tilde{O}(1/\alpha).$}
     \item memorize small margin data: for $i\in \mathcal{Z}_2,$ $\Xi_i^{(T)}\geq \tilde{\Omega}(1).$
\end{enumerate}

Consequently, the training error is smaller than ${O}(\mu/\mathrm{poly}(d))$ and the test error is \textbf{at least} $\tilde{\Omega}(\mu).$
\end{restatable}
{Intuitively, the training process of the GD model is described as follows. Given $|\mathcal{Z}_1|\gg|\mathcal{Z}_2|$ and our choice of parameters for $\alpha,\beta,\sigma$, the gradient points mainly in the direction of $\bm{w^*}$. }
Therefore, GD eventually learns the feature in $\mathcal{Z}_1$ (\autoref{lem:increase_signalGD}) and  the gradients from $\mathcal{Z}_1$  quickly become small. Afterwards, the gradient is dominated by the  gradients from $\mathcal{Z}_2$ (\autoref{lem:Z1derivative}). Because $\mathcal{Z}_2$ has small margin, the full gradient is now directed by the noisy patches. It implies that GD memorizes noise in $\mathcal{Z}_2$  (\autoref{lem:noise_dominates}). Since these gradients also control the amount of remaining feature to be learned (\autoref{eq:ct_update}), we conclude that the GD model partially learns the feature and introduces a huge noise component in the learned weights. We provide a proof sketch of \autoref{thm:GD} in \autoref{sec:GD}. On the other hand, the  model trained with GD+M generalizes well on $\mathcal{D}$.

\vspace{.1cm}

\begin{restatable}{theorem}{GDMmain}\label{thm:GDM}  Assume that we run GD+M on \eqref{eq:optim_pb} for $T$ iterations with parameters set  as in \autoref{ass:paramsGDGDM}. With high probability,  the weights learned by GD+M 
 \begin{enumerate}[leftmargin=*, itemsep=1pt, topsep=1pt, parsep=1pt]
 \item \mbox{at least one of them  is correlated with $\bm{w^*}$:  $c^{(T)} >\tilde{\Omega}(1/\beta).$}
 
\item are barely correlated with noise: for all $r\in [m]$, $i\in[N]$, $j\in [P]\backslash\{P(\bm{X}_i)\}.$ $|\Xi_{i,j,r}^{(T)}|\leq \tilde{O}(\sigma_0).$
\end{enumerate}


\mbox{The training loss and  test error are \textbf{at most} ${O}(\mu/\mathrm{poly}(d)).$} 
\end{restatable}
Intuitively, the GD+M model follows this training process. 
Similarly to GD, it first learns the feature in $\mathcal{Z}_1$  (\autoref{lem:increase_signalGDM}). 
Contrary to GD,  the momentum gradient is still highly correlated with $\bm{w^*}$ after this step  (\autoref{lem:gradlarge}).  Indeed, the key difference is that momentum accumulates historical gradients. Since these gradients were accumulated when learning $\mathcal{Z}_1$, the direction of momentum gradient is highly biased towards $\bm{w^*}$. Therefore, the GD+M model \textit{amplifies the feature} from these historical gradients to learn the feature in small margin data (\autoref{lem:ct_large_M}). Subsequently, the gradient becomes small (\autoref{lem:ZderivativeM}) and the GD+M model manages to ignore the noisy patches (\autoref{lem:noise_GDM1}) and learns the feature from both $\mathcal{Z}_1$ and $\mathcal{Z}_2.$ We provide a proof sketch of \autoref{thm:GDM} in \autoref{sec:GDM}. 

\paragraph*{Signal and noise iterates.} Our analysis is built upon a decomposition of the updates \eqref{eq:GD} and \eqref{eq:GDM} on  $\bm{w^*}$ and $\bm{X}_{i}[j]$. The projection of the vanilla and momentum gradients along these directions are


 \begin{itemize}[leftmargin=*, itemsep=1pt, topsep=1pt, parsep=1pt]
    \item[--] \mbox{ $\mathscr{G}_r^{(t)}=\langle \nabla_{\bm{w}_r}\widehat{L}(\bm{W}^{(t)}),\bm{w^*}\rangle$ and $\mathcal{G}_r^{(t)}=\langle \bm{g}_r^{(t)}, \bm{w^*}\rangle.$}
    \item[--] \mbox{$\texttt{G}_{i,j,r}^{(t)} =\langle \nabla_{\bm{w}_r} \widehat{L}(\bm{W}^{(t)}), \bm{X}_{i}[j]\rangle$ and $G_{i,j,r}^{(t)}=\langle \bm{g}_r^{(t)},\bm{X}_{i}[j]\rangle.$}
 \end{itemize}


  We now define the projected updates as follows: 

\vspace*{-.8cm}

\begin{multicols}{2}
  \begin{equation}
    \resizebox{.75\hsize}{!}{$\displaystyle \hspace{-.3cm}c_{r}^{(t+1)}=c_{r}^{(t)}-\eta\mathscr{G}_r^{(t)}$}\label{eq:GD_signal} 
  \end{equation}
  \begin{equation}
    \hspace{-.2cm}\resizebox{.85\hsize}{!}{$\displaystyle\Xi_{i,j,r}^{(t+1)}=\Xi_{i,j,r}^{(t)}-\eta\texttt{G}_{i,j,r}^{(t)}$}\label{eq:GD_noise} 
  \end{equation}
\end{multicols}

  \vspace*{-1.1cm}

\begin{multicols}{2}
  \begin{equation}
  \hspace*{-.3cm} 
  \begin{aligned}
       \resizebox{.87\hsize}{!}{$\displaystyle\mathcal{G}_{r}^{(t+1)}=\gamma \mathcal{G}_{r}^{(t)}+(1-\gamma)\mathscr{G}_r^{(t)}$}\\
        \resizebox{.85\hsize}{!}{$\displaystyle c_r^{(t+1)}=c_r^{(t)}-\eta\mathcal{G}_r^{(t+1)}$}\label{eq:GDM_signal} 
        \end{aligned}
  \end{equation}
  \begin{equation}
   \hspace*{-.5cm}
     \begin{aligned}\label{eq:GDM_noise} 
        \resizebox{.9\hsize}{!}{$\displaystyle G_{i,j,r}^{(t+1)}=\gamma G_{i,j,r}^{(t)}+(1-\gamma)\texttt{G}_{i,j,r}^{(t)}$}\\
          \resizebox{.88\hsize}{!}{$\displaystyle\Xi_{i,j,r}^{(t+1)}=\Xi_{i,j,r}^{(t)}-G_{i,j,r}^{(t+1)}$}
  \end{aligned}
  \end{equation}
\end{multicols}

 
We detail how to use these dynamics to analyze GD+M and GD in \autoref{sec:GD} and \autoref{sec:GDM}. 
 Our analysis depends on the gradients of $\widehat{L}$ 
which involve $\mathfrak{S}(x)=\left(1+e^{x}\right)^{-1}.$ We define the derivative of a data-point $i$ as $\ell_i^{(t)}=\mathrm{sigmoid}(y_if_{\bm{W}^{(t)}}(\bm{X}_i))$,
  derivatives $\nu_k^{(t)} = \frac{1}{N} \sum_{i\in \mathcal{Z}_k}\ell_{i}^{(t)}$ for $k\in \{1,2\}$ and full derivative $\nu^{(t)}=\nu_1^{(t)}+\nu_2^{(t)}$.  
  

\vspace{-.1cm}

 \section{Analysis of GD}\label{sec:GD}
 
In this section, we provide a proof sketch for \autoref{thm:GD} that reflects the behavior of GD with $\lambda = 0$. A more detailed proof extending to $\lambda > 0$ can be found in the Appendix.

\paragraph*{Step 1: Learning $\mathcal{Z}_1$.} 

At the beginning of the learning process, the gradient is mostly dominated by the gradients coming from the $\mathcal{Z}_1$ samples. Since these data have large margin, the gradient is thus highly correlated with  $\bm{w^*}$ and $c_r^{(t)}$ increases as shown in the following Lemma.


\begin{restatable}{lemma}{lemcincrGD}\label{lem:increase_signalGD}
For all $r\in [m]$ and  $t\geq 0$, \eqref{eq:GD_signal} is simplified as:
 \begin{align*}
     c_r^{(t+1)}\geq c_r^{(t)}+\Theta(\eta)\alpha^3(c_{r}^{(t)})^2\cdot\mathfrak{S}(\textstyle\sum_{s=1}^t \alpha^3 (c_{s}^{(t)})^3).
 \end{align*}
 Consequently, after $T_0=\tilde{\Theta}\left(\frac{1}{\eta \alpha^3\sigma_0 } \right)$ iterations, for all $t\in [T_0,T]$, we have $c^{(t)} \geq \tilde{\Omega}(1/\alpha)$.
\end{restatable}


Intuitively, the increment in the update in \autoref{lem:increase_signalGD} is non-zero when the sigmoid is not too small which is equivalent to $c^{(t)}\leq \tilde{O}(1/\alpha)$. Therefore, $c^{(t)}$ keeps increasing until reaching this threshold. After this step, the $\mathcal{Z}_1$ data have small gradient and therefore, GD has learned these data.
\begin{restatable}{lemma}{lemderivGD}\label{lem:Z1derivative}
Let $T_0=\tilde{\Theta}\left(\frac{1}{\eta \alpha^3\sigma_0 } \right)$. After $t\in [T_0,T]$ iterations, 
 $\nu_1^{(t)}$ is bounded as
 \mbox{$\nu_1^{(t)}\leq    \tilde{O}\left(\frac{1}{\eta (t-T_0+1) \alpha}\right)+ \tilde{O}\left(\frac{\beta^3}{\alpha}\right) \nu_2^{(t)} .$} 
 
\end{restatable}
By our choice of parameter, \autoref{lem:Z1derivative} indicates that the full gradient is dominated by the gradients from $\mathcal{Z}_2$ data after $T_0=\tilde{\Omega}\left(\frac{1}{\hat{\mu}\eta\alpha}\right).$  Consequently, $\nu_2^{(t)}$ also rules the amount of feature learnt by GD.  
\begin{restatable}{lemma}{lemctbdGD}\label{eq:ct_update}
 Let $T_0=\tilde{\Theta}\left(\frac{1}{\eta \alpha^3\sigma_0 } \right)$. For $t\in[T_0,T]$, \eqref{eq:GD_signal} becomes $c^{(t+1)}\leq \tilde{O}(1/\alpha)   + \tilde{O}(\eta\beta^3/\alpha)\sum_{\tau=T_0}^{t} \nu_2^{(\tau)}.$
 \vspace{-.2cm} 
\end{restatable}
\autoref{eq:ct_update} implies that quantifying the decrease rate of $\nu_2^{(t)}$ provides an estimate on the quantity of feature learnt by the model. We remark that $\nu_2^{(t)} = \mathfrak{S}(\beta^3 \sum_{s=1}^m (c_s^{(t)})^3 + \Xi_i^{(t)}) $ for some $i\in\mathcal{Z}_2$. We thus need to determine whether the feature or the noise terms dominates in the sigmoid. 

\paragraph*{Step 2: Memorizing $\mathcal{Z}_2$.} 
We now show that the total correlation between the weights and the noise in $\mathcal{Z}_2$ data increases until being large. 
\begin{restatable}{lemma}{lemxibdGD}\label{lem:noise_dominates} Let $i\in\mathcal{Z}_2$, $j\in[P]\backslash \{P(\bm{X}_i)\}$ and $r\in[m]$. For $t\geq 0$, \eqref{eq:GD_noise} is simplified as: 

\vspace*{-1cm}

  \begin{align*}
  y_i \Xi_{i,j,r}^{(t+1)}&\geq  y_i \Xi_{i,j,r}^{(0)}+\frac{\Tilde{\Theta}(\eta\sigma^2 d)}{N}\sum_{\tau=0}^t   (\Xi_{i,j,r}^{(\tau)})^2 \mathfrak{S}(\Xi_i^{(\tau)})\\ &-\Tilde{O}(P\sigma^2\sqrt{d}/\alpha).
  \end{align*}
 
 \vspace{-.2cm} 
  
Let $T_1=\tilde{\Theta}\left(\frac{N}{\sigma_0\sigma\sqrt{d}\sigma^2d}\right)$. Consequently, for $t\in [T_1,T]$, we have $\Xi_i^{(t)} \geq \tilde{\Omega}(1)$. Thus, GD memorizes.

 \vspace{-.2cm}

\end{restatable}
By \autoref{lem:noise_dominates}, the noise $\Xi_i^{(t)}$ dominates in $\nu_2^{(t)}$. 
Consequently, the algorithm memorizes the $\mathcal{Z}_2$ data which implies a fast decay of $\nu_2^{(t)}$. 
\begin{restatable}{lemma}{lemnutwo}\label{lem:Z2derivative}
  Let $T_1=\tilde{\Theta}\left(\frac{N}{\sigma_0\sigma\sqrt{d}\sigma^2d}\right)$. For $t\in [T_1,T]$, 
we have
$\sum_{\tau=0}^t\nu_2^{(\tau)}\leq \tilde{O}( 1/\eta \sigma_0) .$

\end{restatable}
 
Combining \autoref{lem:Z2derivative} and \autoref{eq:ct_update}, we prove that GD partially learns the feature. 
\begin{restatable}{lemma}{lemctfinGD}\label{lem:signal_final}
 For $t\leq T$, we have $c^{(t)}\leq \tilde{O}(1/\alpha)$. 
\end{restatable}
\autoref{lem:noise_dominates} and \autoref{lem:signal_final} respectively yield the first two items in \autoref{thm:GD}.  Bounds on the training loss and test errors are  obtained by plugging these results in \eqref{eq:optim_pb} and \eqref{eq:testerr}.

\section{Analysis of GD+M}\label{sec:GDM}

In this section, we provide a proof sketch for \autoref{thm:GDM} that reflects the behavior of GD+M with $\lambda = 0$. A proof  extending to $\lambda > 0$ can be found in the Appendix.

\paragraph*{Step 1: Learning $\mathcal{Z}_1$.}  
Similarly to GD, by our initialization choice, the early gradients and so, momentum gradients are large. They are in the span of $\bm{w^*}$ and therefore, the GD+M model also increases its correlation with $\bm{w^*}$.

\begin{restatable}{lemma}{ctalphGDM}\label{lem:increase_signalGDM}
 For all $r\in [m]$ and $t\geq 0$, as long as $c^{(t)} \leq \tilde{O}(1/\alpha)$, the momentum update \eqref{eq:GDM_signal} is simplified as:
 \begin{align*}
     -\mathcal{G}_r^{(t+1)} = -\gamma \mathcal{G}_r^{(t)}+(1-\gamma)\Theta(\alpha^3) (c_{r}^{(t)})^2
 \end{align*}
 Consequently, after $\mathcal{T}_0=\tilde{\Theta} \left( \frac{1}{\sigma_0\alpha^2}  + \frac{1}{1 - \gamma} \right)$ iterations, for all $t\in [\mathcal{T}_0,T]$, we have $c^{(t)} \geq \tilde{\Omega}(1/\alpha)$.
\end{restatable}

\paragraph*{Step 2: Learning $\mathcal{Z}_2$. } Contrary to GD, GD+M has a large momentum that contains $\bm{w^*}$ after Step 1.
\begin{restatable}{lemma}{grdlrg}\label{lem:gradlarge}
Let $\mathcal{T}_0=\tilde{\Theta} \left( \frac{1}{\sigma_0\alpha^3}  + \frac{1}{1 - \gamma} \right)$. Let $r_{\max}=\mathrm{argmax}_{r\in[m]}c_r^{(t)}$. 
For $t\in [\mathcal{T}_0,T]$, we have $ \mathcal{G}_{r_{\max}}^{(t)} \geq \tilde{\Omega}(\sqrt{1-\gamma}/\alpha).$
\end{restatable}
\autoref{lem:gradlarge} hints an important distinction between GD and GD+M: while the current gradient along $\bm{w^*}$ is small at time $\mathcal{T}_0,$ the momentum gradient stores historical gradients that are spanned by $\bm{w^*}$. It \textit{amplifies} the feature present in previous gradients to learn the feature in $\mathcal{Z}_2$.
\begin{restatable}{lemma}{ctlargGDM}\label{lem:ct_large_M}
Let $\mathcal{T}_0=\tilde{\Theta} \left( \frac{1}{\sigma_0\alpha^3}  + \frac{1}{1 - \gamma} \right)$. After $ \mathcal{T}_1 = \mathcal{T}_0+\tilde{\Theta} \left(  \frac{1}{1 - \gamma} \right)$ iterations, for $t\in[\mathcal{T}_1,T]$, we have $c^{(t)} \geq \tilde{\Omega} \left( \frac{1}{\sqrt{1 - \gamma   } \alpha} \right). $ Our choice of parameter in \autoref{sec:setup}, this implies $c^{(t)} \geq \tilde{\Omega}(1/\beta).$
\end{restatable}
\autoref{lem:ct_large_M} states that at least one of the weights that is highly correlated with the feature compared to GD where $c^{(t)} = \tilde{O}(1)$. This result implies that $\nu^{(t)}$ converges fast.

\begin{restatable}{lemma}{nubd}\label{lem:ZderivativeM}
Let $\mathcal{T}_0=\tilde{\Theta} \left( \frac{1}{\eta\sigma_0\alpha^3}  + \frac{1}{1 - \gamma} \right)$. After $ \mathcal{T}_1 = \mathcal{T}_0+\tilde{\Theta} \left(  \frac{1}{1 - \gamma} \right)$ iterations, for $t\in[\mathcal{T}_1,T]$, $\nu^{(t)}\leq    \tilde{O}\left(\frac{1}{\eta (t-\mathcal{T}_1+1) \beta} \right)$.   



 
\end{restatable}
With this fast convergence, \autoref{lem:ZderivativeM} implies that the correlation of the weights with the noisy patches does not have enough time to increase and thus, remains small.
\begin{restatable}{lemma}{lemnoisGDM}\label{lem:noise_GDM1} Let $i\in[N]$, $j\in [P]\backslash\{P(\bm{X}_i)\}$ and $r\in[m]$. For $t\geq 0,$  \eqref{eq:GDM_noise} can be rewritten as $|G_{i,j,r}^{(t+1)}|\leq \gamma |G_{i,j,r}^{(t)} | +(1-\gamma)\tilde{O}(  \sigma_0^2\sigma^4d^2) \nu^{(t)} .$
As a consequence, after $t\in [\mathcal{T}_1,T]$ iterations, 
we thus have $|\Xi_{i,j,r}^{(t)}| \leq \tilde{O}(\sigma_0\sigma\sqrt{d}).$
 \end{restatable}

\vspace{-.1cm}

\autoref{lem:ct_large_M} and \autoref{lem:noise_GDM1} respectively yield the two first items in \autoref{thm:GDM}.


\section{Discussion}\label{sec:ccl}


Our work is a first step towards understanding the algorithmic regularization of momentum and leaves room for improvements. 
We constructed a data distribution where historical feature amplification may explain the generalization improvement of momentum. However, it would be interesting to understand whether this phenomenon is the only reason or whether there are other mechanisms explaining momentum's benefits. 
  An interesting setting for this question is NLP where momentum is used to train large models as BERT \citep{devlin2018bert}. Lastly, our analysis is in the batch setting to isolate the generalization induced by momentum.  It would be interesting to understand how the stochastic noise and the momentum together contribute to the generalization of a neural network.



\bibliography{references}

\begin{thebibliography}{50}
\providecommand{\natexlab}[1]{#1}
\providecommand{\url}[1]{\texttt{#1}}
\expandafter\ifx\csname urlstyle\endcsname\relax
  \providecommand{\doi}[1]{doi: #1}\else
  \providecommand{\doi}{doi: \begingroup \urlstyle{rm}\Url}\fi

\bibitem[Allen-Zhu \& Li(2020)Allen-Zhu and Li]{allen2020towards}
Allen-Zhu, Z. and Li, Y.
\newblock Towards understanding ensemble, knowledge distillation and
  self-distillation in deep learning.
\newblock \emph{arXiv preprint arXiv:2012.09816}, 2020.

\bibitem[Arora et~al.(2018)Arora, Li, and Lyu]{arora2018theoretical}
Arora, S., Li, Z., and Lyu, K.
\newblock Theoretical analysis of auto rate-tuning by batch normalization.
\newblock \emph{arXiv preprint arXiv:1812.03981}, 2018.

\bibitem[Arora et~al.(2019)Arora, Cohen, Hu, and Luo]{arora2019implicit}
Arora, S., Cohen, N., Hu, W., and Luo, Y.
\newblock Implicit regularization in deep matrix factorization.
\newblock \emph{arXiv preprint arXiv:1905.13655}, 2019.

\bibitem[Bjorck et~al.(2018)Bjorck, Gomes, Selman, and
  Weinberger]{bjorck2018understanding}
Bjorck, J., Gomes, C., Selman, B., and Weinberger, K.~Q.
\newblock Understanding batch normalization.
\newblock \emph{arXiv preprint arXiv:1806.02375}, 2018.

\bibitem[Carbery \& Wright(2001)Carbery and Wright]{carbery2001distributional}
Carbery, A. and Wright, J.
\newblock Distributional and $ l^q$ norm inequalities for polynomials over
  convex bodies in $\mathbb{R}^n$.
\newblock \emph{Mathematical research letters}, 8\penalty0 (3):\penalty0
  233--248, 2001.

\bibitem[Chizat \& Bach(2020)Chizat and Bach]{chizat2020implicit}
Chizat, L. and Bach, F.
\newblock Implicit bias of gradient descent for wide two-layer neural networks
  trained with the logistic loss.
\newblock In \emph{Conference on Learning Theory}, pp.\  1305--1338. PMLR,
  2020.

\bibitem[Cutkosky \& Mehta(2020)Cutkosky and Mehta]{cutkosky2020momentum}
Cutkosky, A. and Mehta, H.
\newblock Momentum improves normalized sgd.
\newblock In \emph{International Conference on Machine Learning}, pp.\
  2260--2268. PMLR, 2020.

\bibitem[d'Aspremont(2008)]{d2008smooth}
d'Aspremont, A.
\newblock Smooth optimization with approximate gradient.
\newblock \emph{SIAM Journal on Optimization}, 19\penalty0 (3):\penalty0
  1171--1183, 2008.

\bibitem[Defazio(2020)]{defazio2020understanding}
Defazio, A.
\newblock Understanding the role of momentum in non-convex optimization:
  Practical insights from a lyapunov analysis.
\newblock \emph{arXiv preprint arXiv:2010.00406}, 2020.

\bibitem[Devlin et~al.(2018)Devlin, Chang, Lee, and Toutanova]{devlin2018bert}
Devlin, J., Chang, M.-W., Lee, K., and Toutanova, K.
\newblock Bert: Pre-training of deep bidirectional transformers for language
  understanding.
\newblock \emph{arXiv preprint arXiv:1810.04805}, 2018.

\bibitem[Devolder et~al.(2014)Devolder, Glineur, and
  Nesterov]{devolder2014first}
Devolder, O., Glineur, F., and Nesterov, Y.
\newblock First-order methods of smooth convex optimization with inexact
  oracle.
\newblock \emph{Mathematical Programming}, 146\penalty0 (1):\penalty0 37--75,
  2014.

\bibitem[Goyal et~al.(2017)Goyal, Doll{\'a}r, Girshick, Noordhuis, Wesolowski,
  Kyrola, Tulloch, Jia, and He]{goyal2017accurate}
Goyal, P., Doll{\'a}r, P., Girshick, R., Noordhuis, P., Wesolowski, L., Kyrola,
  A., Tulloch, A., Jia, Y., and He, K.
\newblock Accurate, large minibatch sgd: Training imagenet in 1 hour.
\newblock \emph{arXiv preprint arXiv:1706.02677}, 2017.

\bibitem[Gunasekar et~al.(2018)Gunasekar, Woodworth, Bhojanapalli, Neyshabur,
  and Srebro]{gunasekar2018implicit}
Gunasekar, S., Woodworth, B., Bhojanapalli, S., Neyshabur, B., and Srebro, N.
\newblock Implicit regularization in matrix factorization.
\newblock In \emph{2018 Information Theory and Applications Workshop (ITA)},
  pp.\  1--10. IEEE, 2018.

\bibitem[HaoChen et~al.(2020)HaoChen, Wei, Lee, and Ma]{haochen2020shape}
HaoChen, J.~Z., Wei, C., Lee, J.~D., and Ma, T.
\newblock Shape matters: Understanding the implicit bias of the noise
  covariance.
\newblock \emph{arXiv preprint arXiv:2006.08680}, 2020.

\bibitem[He et~al.(2016)He, Zhang, Ren, and Sun]{he2016deep}
He, K., Zhang, X., Ren, S., and Sun, J.
\newblock Deep residual learning for image recognition.
\newblock In \emph{Proceedings of the IEEE conference on computer vision and
  pattern recognition}, pp.\  770--778, 2016.

\bibitem[Hoffer et~al.(2017)Hoffer, Hubara, and Soudry]{hoffer2017train}
Hoffer, E., Hubara, I., and Soudry, D.
\newblock Train longer, generalize better: closing the generalization gap in
  large batch training of neural networks.
\newblock \emph{arXiv preprint arXiv:1705.08741}, 2017.

\bibitem[Hoffer et~al.(2019)Hoffer, Banner, Golan, and Soudry]{hoffer2019norm}
Hoffer, E., Banner, R., Golan, I., and Soudry, D.
\newblock Norm matters: efficient and accurate normalization schemes in deep
  networks, 2019.

\bibitem[Ioffe \& Szegedy(2015)Ioffe and Szegedy]{ioffe2015batch}
Ioffe, S. and Szegedy, C.
\newblock Batch normalization: Accelerating deep network training by reducing
  internal covariate shift.
\newblock In \emph{International conference on machine learning}, pp.\
  448--456. PMLR, 2015.

\bibitem[Ji \& Telgarsky(2019)Ji and Telgarsky]{ji2019implicit}
Ji, Z. and Telgarsky, M.
\newblock The implicit bias of gradient descent on nonseparable data.
\newblock In \emph{Conference on Learning Theory}, pp.\  1772--1798. PMLR,
  2019.

\bibitem[Keskar et~al.(2016)Keskar, Mudigere, Nocedal, Smelyanskiy, and
  Tang]{keskar2016large}
Keskar, N.~S., Mudigere, D., Nocedal, J., Smelyanskiy, M., and Tang, P. T.~P.
\newblock On large-batch training for deep learning: Generalization gap and
  sharp minima.
\newblock \emph{arXiv preprint arXiv:1609.04836}, 2016.

\bibitem[Kidambi et~al.(2018)Kidambi, Netrapalli, Jain, and
  Kakade]{kidambi2018insufficiency}
Kidambi, R., Netrapalli, P., Jain, P., and Kakade, S.
\newblock On the insufficiency of existing momentum schemes for stochastic
  optimization.
\newblock In \emph{2018 Information Theory and Applications Workshop (ITA)},
  pp.\  1--9. IEEE, 2018.

\bibitem[Kingma \& Ba(2014)Kingma and Ba]{kingma2014adam}
Kingma, D.~P. and Ba, J.
\newblock Adam: A method for stochastic optimization.
\newblock \emph{arXiv preprint arXiv:1412.6980}, 2014.

\bibitem[Krizhevsky et~al.(2009)Krizhevsky, Hinton,
  et~al.]{krizhevsky2009learning}
Krizhevsky, A., Hinton, G., et~al.
\newblock Learning multiple layers of features from tiny images.
\newblock 2009.

\bibitem[Krizhevsky et~al.(2012)Krizhevsky, Sutskever, and
  Hinton]{krizhevsky2012imagenet}
Krizhevsky, A., Sutskever, I., and Hinton, G.~E.
\newblock Imagenet classification with deep convolutional neural networks.
\newblock \emph{Advances in neural information processing systems},
  25:\penalty0 1097--1105, 2012.

\bibitem[Leclerc \& Madry(2020)Leclerc and Madry]{leclerc2020two}
Leclerc, G. and Madry, A.
\newblock The two regimes of deep network training.
\newblock \emph{arXiv preprint arXiv:2002.10376}, 2020.

\bibitem[Lessard et~al.(2015)Lessard, Recht, and Packard]{lessard2015analysis}
Lessard, L., Recht, B., and Packard, A.
\newblock Analysis and design of optimization algorithms via integral quadratic
  constraints, 2015.

\bibitem[Li et~al.(2019)Li, Wei, and Ma]{li2019towards}
Li, Y., Wei, C., and Ma, T.
\newblock Towards explaining the regularization effect of initial large
  learning rate in training neural networks.
\newblock \emph{arXiv preprint arXiv:1907.04595}, 2019.

\bibitem[Liu et~al.(2020)Liu, Gao, and Yin]{liu2020improved}
Liu, Y., Gao, Y., and Yin, W.
\newblock An improved analysis of stochastic gradient descent with momentum.
\newblock \emph{arXiv preprint arXiv:2007.07989}, 2020.

\bibitem[Lovett(2010)]{Lovett2010AnEP}
Lovett, S.
\newblock An elementary proof of anti-concentration of polynomials in gaussian
  variables.
\newblock \emph{Electron. Colloquium Comput. Complex.}, 17:\penalty0 182, 2010.

\bibitem[Lyu \& Li(2019)Lyu and Li]{lyu2019gradient}
Lyu, K. and Li, J.
\newblock Gradient descent maximizes the margin of homogeneous neural networks.
\newblock \emph{arXiv preprint arXiv:1906.05890}, 2019.

\bibitem[Mai \& Johansson(2020)Mai and Johansson]{mai2020convergence}
Mai, V. and Johansson, M.
\newblock Convergence of a stochastic gradient method with momentum for
  non-smooth non-convex optimization.
\newblock In \emph{International Conference on Machine Learning}, pp.\
  6630--6639. PMLR, 2020.

\bibitem[Nemirovskij \& Yudin(1983)Nemirovskij and
  Yudin]{nemirovskij1983problem}
Nemirovskij, A.~S. and Yudin, D.~B.
\newblock Problem complexity and method efficiency in optimization.
\newblock 1983.

\bibitem[Nesterov(1983)]{nesterov1983method}
Nesterov, Y.
\newblock A method for unconstrained convex minimization problem with the rate
  of convergence o (1/k\^{} 2).
\newblock In \emph{Doklady an ussr}, volume 269, pp.\  543--547, 1983.

\bibitem[Nesterov(2003)]{nesterov2003introductory}
Nesterov, Y.
\newblock \emph{Introductory lectures on convex optimization: A basic course},
  volume~87.
\newblock Springer Science \& Business Media, 2003.

\bibitem[Neyshabur et~al.(2015)Neyshabur, Salakhutdinov, and
  Srebro]{neyshabur2015path}
Neyshabur, B., Salakhutdinov, R., and Srebro, N.
\newblock Path-sgd: Path-normalized optimization in deep neural networks.
\newblock \emph{arXiv preprint arXiv:1506.02617}, 2015.

\bibitem[Polyak(1963)]{polyak1963gradient}
Polyak, B.~T.
\newblock Gradient methods for the minimisation of functionals.
\newblock \emph{USSR Computational Mathematics and Mathematical Physics},
  3\penalty0 (4):\penalty0 864--878, 1963.

\bibitem[Polyak(1964)]{polyak1964some}
Polyak, B.~T.
\newblock Some methods of speeding up the convergence of iteration methods.
\newblock \emph{Ussr computational mathematics and mathematical physics},
  4\penalty0 (5):\penalty0 1--17, 1964.

\bibitem[Schmidt et~al.(2011)Schmidt, Roux, and Bach]{schmidt2011convergence}
Schmidt, M., Roux, N.~L., and Bach, F.
\newblock Convergence rates of inexact proximal-gradient methods for convex
  optimization.
\newblock \emph{arXiv preprint arXiv:1109.2415}, 2011.

\bibitem[Simonyan \& Zisserman(2014)Simonyan and Zisserman]{simonyan2014very}
Simonyan, K. and Zisserman, A.
\newblock Very deep convolutional networks for large-scale image recognition.
\newblock \emph{arXiv preprint arXiv:1409.1556}, 2014.

\bibitem[Smith et~al.(2018)Smith, Kindermans, Ying, and Le]{smith2018dont}
Smith, S.~L., Kindermans, P.-J., Ying, C., and Le, Q.~V.
\newblock Don't decay the learning rate, increase the batch size, 2018.

\bibitem[Soudry et~al.(2018)Soudry, Hoffer, Nacson, Gunasekar, and
  Srebro]{soudry2018implicit}
Soudry, D., Hoffer, E., Nacson, M.~S., Gunasekar, S., and Srebro, N.
\newblock The implicit bias of gradient descent on separable data.
\newblock \emph{The Journal of Machine Learning Research}, 19\penalty0
  (1):\penalty0 2822--2878, 2018.

\bibitem[Srivastava et~al.(2014)Srivastava, Hinton, Krizhevsky, Sutskever, and
  Salakhutdinov]{srivastava2014dropout}
Srivastava, N., Hinton, G., Krizhevsky, A., Sutskever, I., and Salakhutdinov,
  R.
\newblock Dropout: a simple way to prevent neural networks from overfitting.
\newblock \emph{The journal of machine learning research}, 15\penalty0
  (1):\penalty0 1929--1958, 2014.

\bibitem[Sutskever et~al.(2013)Sutskever, Martens, Dahl, and
  Hinton]{sutskever2013importance}
Sutskever, I., Martens, J., Dahl, G., and Hinton, G.
\newblock On the importance of initialization and momentum in deep learning.
\newblock In \emph{International conference on machine learning}, pp.\
  1139--1147. PMLR, 2013.

\bibitem[Tieleman \& Hinton(2012)Tieleman and Hinton]{tieleman2012lecture}
Tieleman, T. and Hinton, G.
\newblock Lecture 6.5-rmsprop: Divide the gradient by a running average of its
  recent magnitude.
\newblock \emph{COURSERA: Neural networks for machine learning}, 4\penalty0
  (2):\penalty0 26--31, 2012.

\bibitem[Vaswani et~al.(2017)Vaswani, Shazeer, Parmar, Uszkoreit, Jones, Gomez,
  Kaiser, and Polosukhin]{vaswani2017attention}
Vaswani, A., Shazeer, N., Parmar, N., Uszkoreit, J., Jones, L., Gomez, A.~N.,
  Kaiser, {\L}., and Polosukhin, I.
\newblock Attention is all you need.
\newblock In \emph{Advances in neural information processing systems}, pp.\
  5998--6008, 2017.

\bibitem[Vershynin(2018)]{vershynin2018high}
Vershynin, R.
\newblock \emph{High-dimensional probability: An introduction with applications
  in data science}, volume~47.
\newblock Cambridge university press, 2018.

\bibitem[Wainwright(2019)]{wainwright2019high}
Wainwright, M.~J.
\newblock \emph{High-dimensional statistics: A non-asymptotic viewpoint},
  volume~48.
\newblock Cambridge University Press, 2019.

\bibitem[Wei et~al.(2020)Wei, Kakade, and Ma]{wei2020implicit}
Wei, C., Kakade, S., and Ma, T.
\newblock The implicit and explicit regularization effects of dropout, 2020.

\bibitem[Wilson et~al.(2017)Wilson, Roelofs, Stern, Srebro, and
  Recht]{wilson2017marginal}
Wilson, A.~C., Roelofs, R., Stern, M., Srebro, N., and Recht, B.
\newblock The marginal value of adaptive gradient methods in machine learning.
\newblock \emph{arXiv preprint arXiv:1705.08292}, 2017.

\bibitem[Zagoruyko \& Komodakis(2016)Zagoruyko and
  Komodakis]{zagoruyko2016wide}
Zagoruyko, S. and Komodakis, N.
\newblock Wide residual networks.
\newblock \emph{arXiv preprint arXiv:1605.07146}, 2016.

\end{thebibliography}
\bibliographystyle{icml2022}

\newpage
\appendix
\onecolumn
\section{Additional Experiments}\label{sec:app_exps}

In this section, we present additional experiments to further strengthen our empirical results. We verify on Resnet-18 that momentum induces generalization improvement when trained without batch normalization and data augmentation. We then check that when these two regualizers are used, momentum does not improve generalization.  We then confirm on Resnet-18 that the generalization improvement gets larger as the batch size increases. Then, we provide the performance of SGD and SGD+M on the Gaussian experiment introduced in the introduction. Lastly, we give additional plots showing that momentum allows to well-classify small margin data as mentioned at the end of the introduction.

\subsection{Experiments with Resnet-18}

\autoref{fig:cifar10r18} displays the training loss and test accuracy obtained by training a Resnet-18 on CIFAR-10 and CIFAR-100. Similarly to the case where we trained a VGG-19, momentum significantly improves generalization whether in the stochastic case (SGD) or in the full batch setting (GD).\newline

\begin{figure}[h!] 
\vspace{-.3cm}
\begin{subfigure}{0.48\textwidth}
\hspace{-1.5mm}
\includegraphics[width=.8\linewidth]{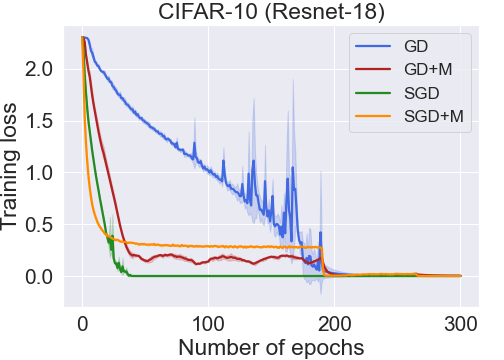}
 \vspace{-2.5mm}
 \caption{}\label{fig:traincifar10r18}
\end{subfigure}
\begin{subfigure}{0.48\textwidth}
\hspace{-6mm}
\vspace{3mm}
 \includegraphics[width=.8\linewidth]{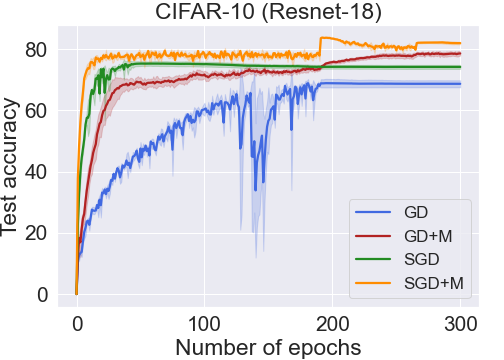}
  \vspace{-5mm}
\caption{ }\label{fig:testcifar10r18}
\end{subfigure}
\begin{subfigure}{0.48\textwidth}
\includegraphics[width=.8\linewidth]{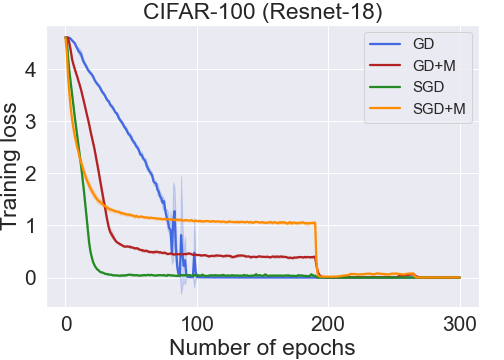}
 \vspace{-1.5mm}
 \caption{}\label{fig:traincifar100r18}
\end{subfigure}
\begin{subfigure}{0.48\textwidth}
\vspace{2mm}
 \includegraphics[width=.8\linewidth]{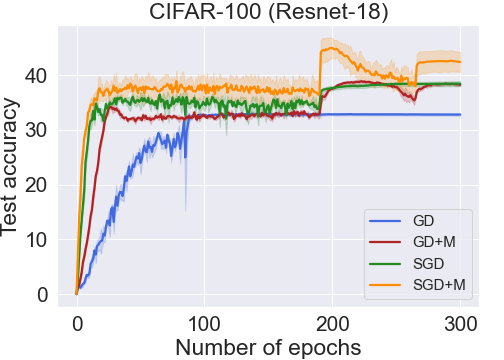}
  \vspace{-1.5mm}
\caption{ }\label{fig:testcifar100r18}
\end{subfigure}
\vspace{-4mm}
\caption{\small Training loss and test accuracy obtained with Resnet-18  trained with SGD, SGD+M, GD and GD+M on CIFAR-10 (a-b) and CIFAR-100 (c-d). 
Data augmentation and batch normalization are turned off. }\label{fig:cifar10r18}
\end{figure}

\subsection{Influence of the batch size}

\autoref{fig:btchresnet} shows the test accuracy obtained with a Resnet-18 using the stochastic gradient descent optimizer on CIFAR-10. Similarly to the VGG-19 experiment in \autoref{sec:num_exp}, the generalization improvement induced by momentum gets larger as the batch size increases.

\begin{figure}[h!] 
\vspace{-.2cm}
\centering
\includegraphics[width=.35\linewidth]{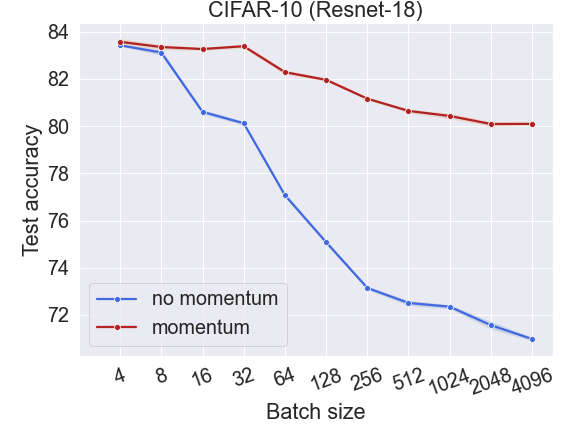}
 
\vspace{-4mm}
\caption{\small Test accuracy obtained with a Resnet-18 using the stochastic gradient descent optimizer on CIFAR-10 when batch normalization and data augmentation are turned off.}\label{fig:btchresnet}
\end{figure}

\subsection{Influence of batch normalization and data augmentation}

As mentioned in \autoref{sec:num_exp}, batch normalization and data augmentation significantly reduce the generalization improvement induced by momentum. We further confirm this  observation in \autoref{fig:vgg19bnda} and \autoref{fig:r18bnda}.

\begin{figure}[h!] 
\vspace{-.2cm}
\begin{subfigure}{0.48\textwidth}
\includegraphics[width=.8\linewidth]{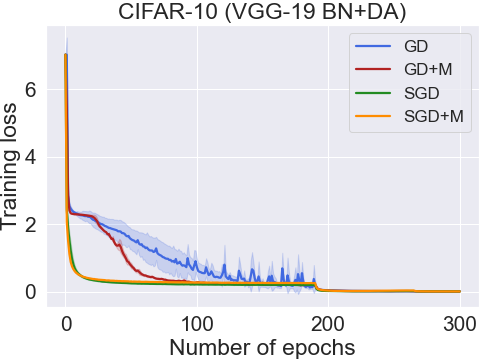}
 \vspace{-1.5mm}
 \caption{}
\end{subfigure}
\begin{subfigure}{0.48\textwidth}
\vspace{2mm}
 \includegraphics[width=.8\linewidth]{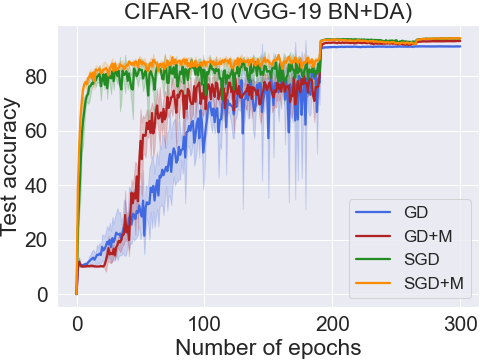}
  \vspace{-1.5mm}
\caption{ }\
\end{subfigure}
\begin{subfigure}{0.48\textwidth}
\includegraphics[width=.8\linewidth]{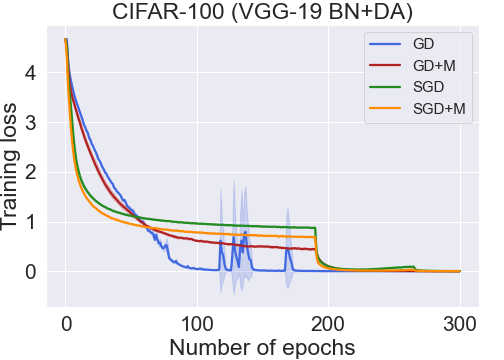}
 \vspace{-6.5mm}
 \caption{}\
\end{subfigure}
\begin{subfigure}{0.48\textwidth}
\vspace{2mm}
\hspace{.5cm} \includegraphics[width=.8\linewidth]{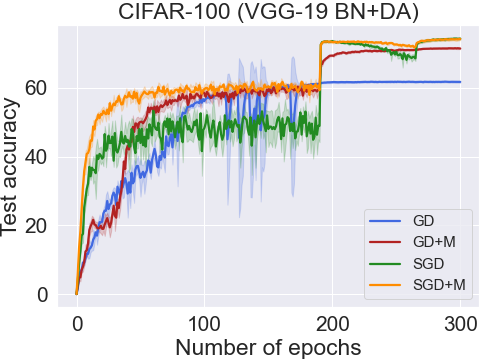}
  \vspace{-6.5mm}
\caption{ }\
\end{subfigure}
\vspace{-4mm}
\caption{\small Training loss and test accuracy obtained with VGG-19  trained with SGD, SGD+M, GD and GD+M on CIFAR-10 (a-b) and CIFAR-100 (c-d). 
Data augmentation and batch normalization are turned on.}\label{fig:vgg19bnda}
\end{figure} 

\newpage 

\begin{figure}[h!] 
\vspace{-.2cm}
\begin{subfigure}{0.48\textwidth}
\includegraphics[width=.8\linewidth]{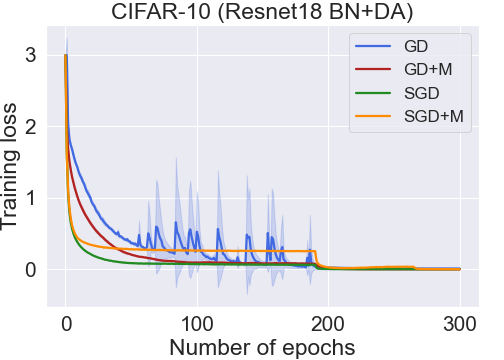}
 \vspace{-1.5mm}
 \caption{}\
\end{subfigure}
\begin{subfigure}{0.48\textwidth}
\vspace{2mm}
 \includegraphics[width=.8\linewidth]{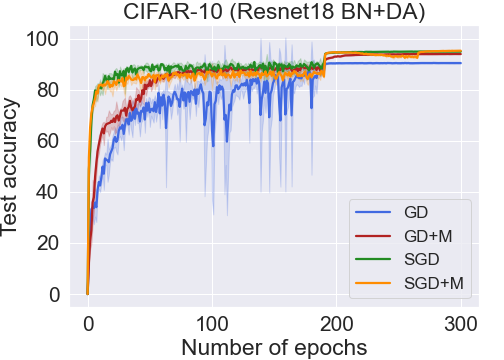}
  \vspace{-1.5mm}
\caption{ }\
\end{subfigure}
\begin{subfigure}{0.48\textwidth}
\includegraphics[width=.8\linewidth]{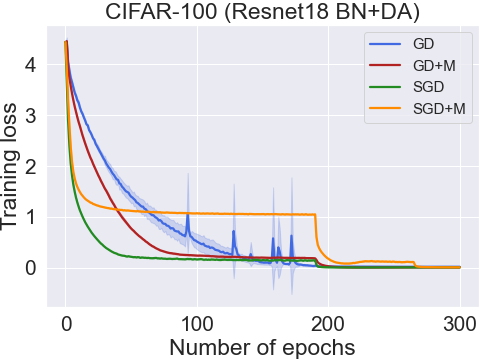}
 \vspace{-1.5mm}
 \caption{}\
\end{subfigure}
\begin{subfigure}{0.48\textwidth}
\vspace{0cm}
\hspace{.8cm}\includegraphics[width=.8\linewidth]{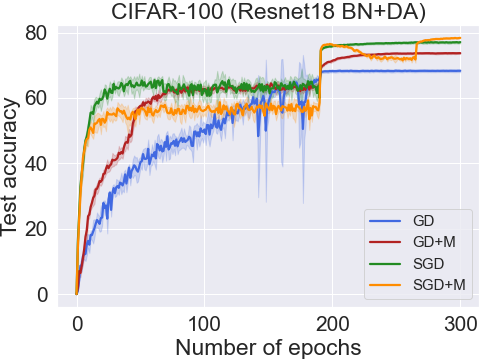}
  \vspace{-1.5mm}
\caption{ }\
\end{subfigure}
\vspace{-4mm}
\caption{\small Training loss and test accuracy obtained with Resnet-18  trained with SGD, SGD+M, GD and GD+M on CIFAR-10 (a-b) and CIFAR-100 (c-d). 
Data augmentation and batch normalization are turned on.}\label{fig:r18bnda}
\end{figure}

\subsection{Synthetic Gaussian data experiments}  

We provide a complete table with mean and standard deviations obtained by using different student networks to learn the Gaussian synthetic experiment mentioned in the introduction.

\begin{table}[h!]
    \centering
 \begin{tabular}{|l||*{5}{c|}}\hline
\backslashbox{Student}{Teacher}
&\makebox[3em]{Linear} &\makebox[3em]{1-MLP}
&\makebox[3em]{2-MLP}&\makebox[3em]{1-CNN}&\makebox[3em]{2-CNN}\\\hline\hline
1-MLP & $93.48\pm 0.13$ &  $92.32\pm 0.50$ & $84.30\pm 0.82$  & $94.18\pm 0.42$  &  $76.04\pm 0.29$ \\\hline
2-MLP & $93.45\pm 0.22$  & $91.02\pm 0.41$    &  $83.82\pm 0.43$ & $94.14\pm 0.47$   & $75.50\pm 0.35$  \\\hline
1-CNN & $92.21\pm 0.16$  & $92.31\pm 0.57$   & $83.39\pm 0.48$ & $94.39\pm 0.17$    &  $79.44\pm 0.58$ \\\hline
2-CNN & $91.04\pm 0.48$  & $91.51\pm 0.40$   &  $82.44\pm 0.45$  &  $93.91\pm 0.35$  & $80.86\pm 0.92$  \\\hline
\end{tabular}
\caption*{(a)}
\end{table}

\begin{table}[h!]
    \centering
 \begin{tabular}{|l||*{5}{c|}}\hline
\backslashbox{Student}{Teacher}
&\makebox[3em]{Linear} &\makebox[3em]{1-MLP}
&\makebox[3em]{2-MLP}&\makebox[3em]{1-CNN}&\makebox[3em]{2-CNN}\\\hline\hline
1-MLP & $93.25\pm 0.22$ &  $92.18\pm 0.53$ & $83.68\pm 0.74$  & $94.12\pm 0.43$  &  $76.12\pm 0.22$ \\\hline
2-MLP & $92.85\pm 0.34$  & $91.78\pm 0.62$    &  $83.25\pm 0.70$ & $94.20  \pm 0.13$& $75.56\pm 0.33$  \\\hline
1-CNN & $92.34\pm 0.21$  & $92.33\pm 0.64$   & $83.44\pm 0.52$ & $94.39\pm 0.15$    &  $78.32\pm 0.34$ \\\hline
2-CNN & $91.22\pm 0.39$  & $91.56\pm 0.52$   & $82.12\pm 0.55$  &  $93.79\pm 0.25$  & $78.56\pm 0.64$  \\\hline
\end{tabular}
\caption*{(b)}
\caption*{}
\end{table}

\begin{table}[h!]
    \centering
 \begin{tabular}{|l||*{5}{c|}}\hline
\backslashbox{Student}{Teacher}
&\makebox[3em]{Linear} &\makebox[3em]{1-MLP}
&\makebox[3em]{2-MLP}&\makebox[3em]{1-CNN}&\makebox[3em]{2-CNN}\\\hline\hline
1-MLP & $93.58\pm 0.32$ &  $92.56\pm 0.62$ & $85.74\pm 0.56$  & $94.18\pm 0.42$  &  $76.06\pm 0.39$ \\\hline
2-MLP & $93.51\pm 0.25$  & $91.82\pm 0.83$    &  $85.33\pm 0.81$ & $94.14\pm 0.33$   & $75.33\pm 0.47$  \\\hline
1-CNN & $92.42\pm 0.05$  & $92.03\pm 0.53$   & $84.57\pm 0.47$ & $94.22\pm 0.18$    &  $80.02\pm 0.45$ \\\hline
2-CNN & $91.54\pm 0.37$  & $92.04\pm 0.48$   &  $83.81\pm 0.47$  &  $93.95\pm 0.31$  & $82.86\pm 0.59$  \\\hline
\end{tabular}
\caption*{(c)}
\end{table}

\vspace*{-1.5cm}

\begin{table}[h!]
    \centering
 \begin{tabular}{|l||*{5}{c|}}\hline
\backslashbox{Student}{Teacher}
&\makebox[3em]{Linear} &\makebox[3em]{1-MLP}
&\makebox[3em]{2-MLP}&\makebox[3em]{1-CNN}&\makebox[3em]{2-CNN}\\\hline\hline
1-MLP & $93.56\pm 0.28$ &  $92.82\pm 0.26$ & $84.65\pm 0.45$  & $94.16\pm 0.42$  &  $76.01\pm 0.33$ \\\hline
2-MLP & $93.24\pm 0.34$  & $92.26\pm 0.76$    &  $84.27\pm 0.79$ & $94.24  \pm 0.40$& $75.04\pm 0.47$  \\\hline
1-CNN & $92.50\pm 0.05$  & $91.68\pm 0.72$   & $83.39\pm 0.44$ & $94.07\pm 0.035$    &  $78.92\pm 0.41$ \\\hline
2-CNN & $91.61\pm 0.41$  & $91.94\pm 0.54$   & $83.70\pm 0.37$  &  $93.89\pm 0.33$  & $80.50\pm 0.45$  \\\hline
\end{tabular}
\caption*{(d)}
\caption{Test accuracy obtained using GD (a), GD+M (b), SGD (c) and GD+M (d) on a Gaussian synthetic dataset trained using neural network with ReLU activations. The training dataset consists in 500 data points in dimension 50 and test set in 5000 points. The student networks are  trained for 1000 epochs to ensure zero training error. Results  averaged over 3 runs. }\label{tab:SGDgauss}
\end{table}

\newpage
 
\subsection{Additional justification for the theory}

In this section, we present further experiments to consolidate  the experiment on the artificially decimated CIFAR-10 dataset described in the introduction. 

\begin{figure}[h!]  
\begin{subfigure}{0.48\textwidth}
 \vspace{-1mm}
 \hspace{-.1cm}\includegraphics[width=.8\linewidth]{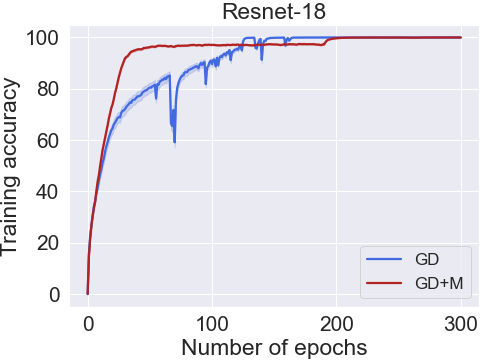}
 \vspace{-2mm}
 \caption{} \label{fig:testr18smallmg}
\end{subfigure}
\begin{subfigure}{0.48\textwidth}
 \vspace{-1mm}
 \includegraphics[width=.8\linewidth]{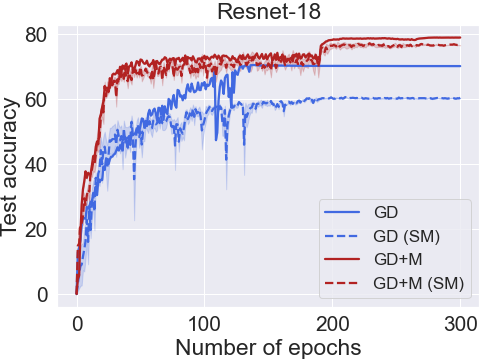}
  \vspace{-2mm}
\caption{} 
\end{subfigure}
\begin{subfigure}{0.48\textwidth}
 \vspace{-1mm}
 \hspace{-.1cm}\includegraphics[width=.8\linewidth]{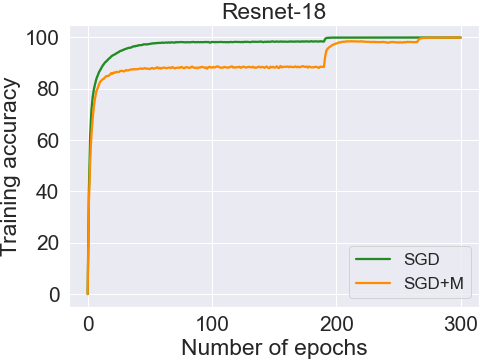}
 \vspace{-2mm}
 \caption{} 
\end{subfigure}
\begin{subfigure}{0.48\textwidth}
 \vspace{-2mm}
 \includegraphics[width=.8\linewidth]{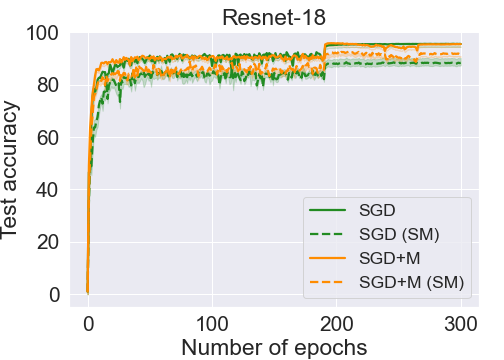}
  \vspace{-0mm}
\caption{} \label{fig:testr18smallmgsgd}
\end{subfigure}
\vspace{-2mm}
\caption{\small  Training and test accuracy obtained with Resnet-18 on the artificially modified CIFAR-10 dataset with small margin data. The architectures are trained using GD/GD+M (a-b) and SGD/SGD+M (c-d) for $300$ epochs to ensure zero training error. Data augmentation and batch normalization are turned off. 
 }\label{fig:resnetsmallmg}
\end{figure}

In \autoref{fig:testr18smallmg}, we observe that using a Resnet-18, momentum still improves generalization on the small margin images.In \autoref{fig:testr18smallmgsgd} and \autoref{fig:vgg19tstsmallmg}, we see that using stochastic updates lead SGD to classify small margin images as well as SGD+M. Lastly,  \autoref{fig:vgg19smallmgrbn} and \autoref{fig:resnetsmallmgrbn} show that batch normalization and data augmentation also reduce the generalization improvement of momentum: GD/SGD perform similarly as well as GD+M/SGD+M on the small margin data.

\begin{figure}[h!]  
\begin{subfigure}{0.48\textwidth}
 \hspace{-.1cm}\includegraphics[width=.8\linewidth]{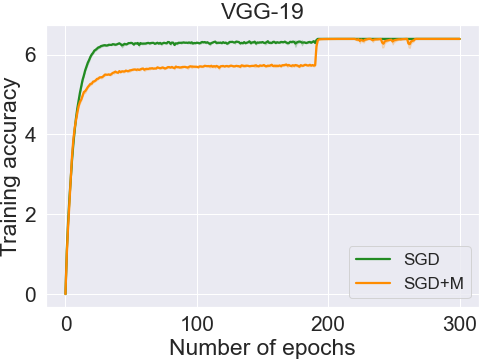}
 \vspace{-2mm}
 \caption{} 
\end{subfigure}
\begin{subfigure}{0.48\textwidth}
 \vspace{-5mm}
 \includegraphics[width=.8\linewidth]{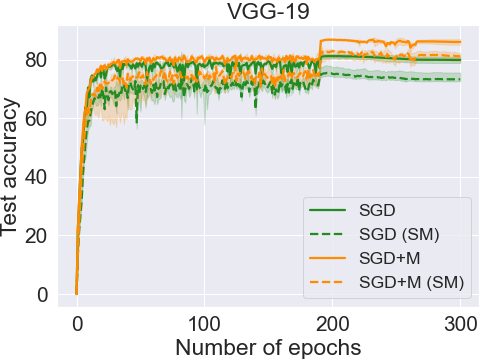}
  \vspace{-6mm}
\caption{} \label{fig:vgg19tstsmallmg}
\end{subfigure}
\vspace{-4mm}
\caption{\small 
 Training and test accuracy obtained with VGG-19 on the artificially modified CIFAR-10 dataset with small margin data. The architectures are trained using SGD/SGD+M (a-b).  Data augmentation and batch normalization are turned off. }\label{fig:vgg19sgdsmallmg}
\end{figure}

\begin{figure}[h!]  
\begin{subfigure}{0.48\textwidth}
 \vspace{-2mm}
 \hspace{-.1cm}\includegraphics[width=.8\linewidth]{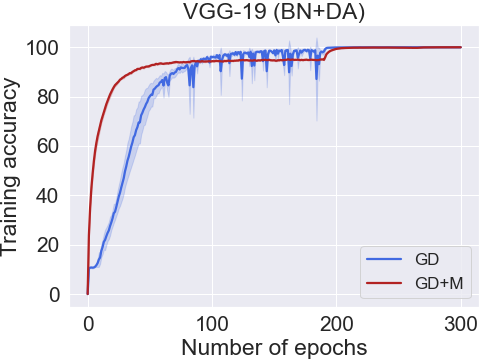}
 \vspace{-2mm}
 \caption{} 
\end{subfigure}
\begin{subfigure}{0.48\textwidth}
 \vspace{-5mm}
 \includegraphics[width=.8\linewidth]{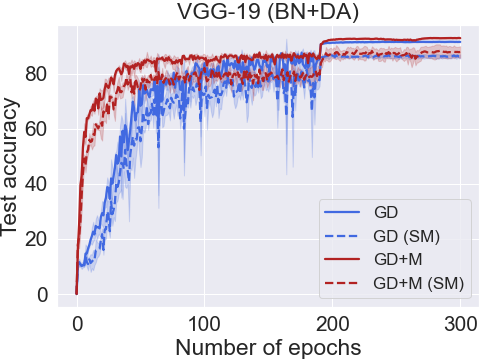}
  \vspace{-6mm}
\caption{} 
\end{subfigure}
\begin{subfigure}{0.48\textwidth}
 \vspace{-0mm}
 \hspace{-.1cm}\includegraphics[width=.8\linewidth]{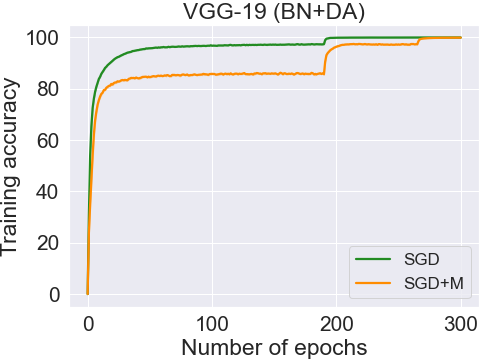}
 \vspace{-2mm}
 \caption{}
\end{subfigure}
\begin{subfigure}{0.48\textwidth}
 \vspace{-0mm}
 \includegraphics[width=.8\linewidth]{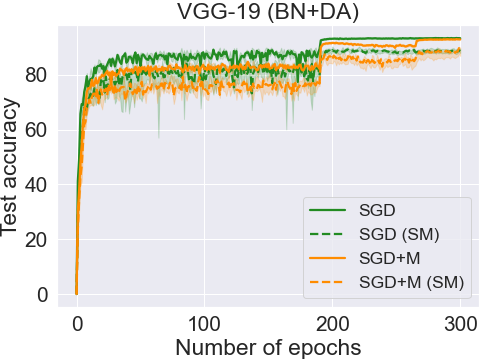}
  \vspace{-6mm}
\caption{}
\end{subfigure}
\vspace{-4mm}
\caption{\small 
 Training and test accuracy obtained with VGG-19 on the artificially modified CIFAR-10 dataset with small margin data. The architectures are trained using GD/GD+M (a-b) and SGD/SGD+M (c-d). Data augmentation and batch normalization are turned on.}\label{fig:vgg19smallmgrbn}
\end{figure}

\begin{figure}[h!]  
\begin{subfigure}{0.48\textwidth}
 \hspace{-.1cm}\includegraphics[width=.8\linewidth]{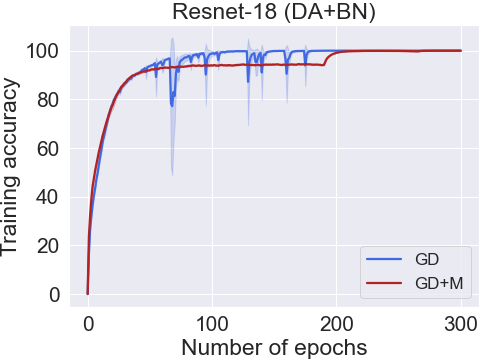}
 \vspace{-2mm}
 \caption{}
\end{subfigure}
\begin{subfigure}{0.48\textwidth}
 \vspace{-2mm}
 \includegraphics[width=.8\linewidth]{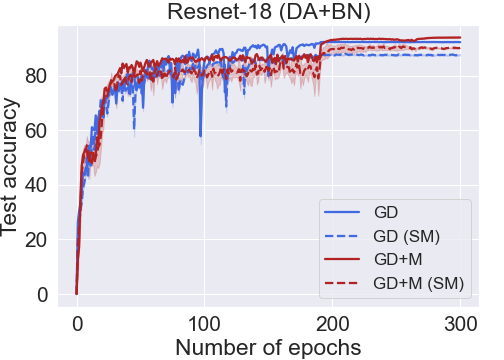}
  \vspace{-2mm}
\caption{}
\end{subfigure}
\begin{subfigure}{0.48\textwidth}
 \vspace{-0mm}
 \hspace{-.1cm}\includegraphics[width=.8\linewidth]{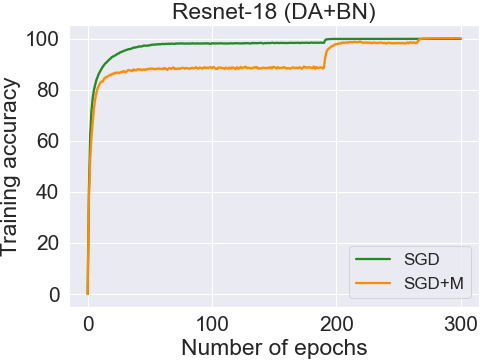}
 \vspace{-2mm}
 \caption{}
\end{subfigure}
\begin{subfigure}{0.48\textwidth}
 \vspace{-0mm}
 \includegraphics[width=.8\linewidth]{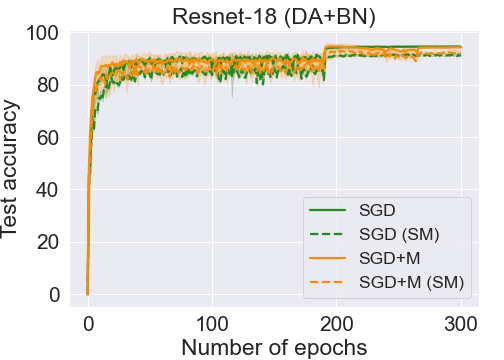}
  \vspace{-2mm}
\caption{}
\end{subfigure}
\vspace{-4mm}
\caption{\small 
 Training and test accuracy obtained with Resnet-18 on the artificially modified CIFAR-10 dataset with small margin data. The architectures are trained using GD/GD+M (a-b) and SGD/SGD+M (c-d). Data augmentation and batch normalization are turned on.}\label{fig:resnetsmallmgrbn}
\end{figure}

\newpage\phantom{blabla}
\newpage\phantom{blabla}

\section{Additional related work}

\paragraph*{Momentum in convex setting.} GD+M (a.k.a.\ heavy ball or Polyak momentum) consists in using an exponentially weighted average of the past gradients  to update the weights.
For convex functions near a strict twice-differentiable minimum, GD+M is optimal regarding local convergence rate \citep{polyak1963gradient,polyak1964some,nemirovskij1983problem,nesterov2003introductory}. However, it may fail to converge globally for general strongly convex twice-differentiable functions \citep{lessard2015analysis} and is no longer optimal for the class of smooth convex functions. 
In the stochastic setting, GD+M is more sensitive to noise in the gradients; that is, to preserve their improved convergence rates, significantly less noise is required
\citep{d2008smooth,schmidt2011convergence,devolder2014first,kidambi2018insufficiency}. Finally, 
other momentum methods are extensively used for convex functions such as Nesterov’s accelerated gradient \citep{nesterov1983method}. Our paper focuses on the use of GD+M and contrary to the aforementioned papers,  our setting is non-convex. Besides, we mainly focus on the generalization of the model learned by GD and GD+M when both methods converge to global optimal. Contrary to the non-convex case, generalization is disentangled from optimization for (strictly) convex functions. 

\paragraph{Algorithmic regularization.} The question we address concerns algorithmic regularization which characterizes the generalization of an optimization algorithm when multiple global solutions exist in over-parametrized  models \citep{soudry2018implicit,lyu2019gradient,ji2019implicit,chizat2020implicit,gunasekar2018implicit,arora2019implicit}. This regularization arises in deep learning mainly due to the \textit{non-convexity} of the objective function. Indeed, this latter potentially creates multiple global minima scattered in the space that vastly differ in terms of generalization. Algorithmic regularization is induced by and depends on many factors such as learning rate and batch size \citep{goyal2017accurate,hoffer2017train,keskar2016large,smith2018dont}, initialization \cite{allen2020towards}, adaptive step-size \citep{kingma2014adam,neyshabur2015path,wilson2017marginal}, batch normalization \citep{arora2018theoretical,hoffer2019norm,ioffe2015batch} and dropout \citep{srivastava2014dropout,wei2020implicit}. However, none of these works theoretically analyzes the regularization induced by momentum.

\section{Notations}

In this section, we introduce the different notations used in the proofs. We start by defining the notations that appear for GD and GD+M. We first consider the case when $\lambda = 0$, we will extend the proof to $\lambda > 0$ in section~\ref{sec:ext}

\subsection{Notations for GD and GD+M}

Our paper rely on the notions of signal and noise components of the iterates.

\vspace{.2cm}

\hspace{.2cm}-- Signal intensity: $\theta = \alpha$ if $i\in\mathcal{Z}_1$ and $\beta$ otherwise.

\hspace{.2cm}-- Signal: $c_{r}^{(t)} = \langle \bm{w^*} , \bm{w}_r^{(t)} \rangle$ for $r\in [m].$

\hspace{.2cm}-- Max signal: $c^{(t)} = c_{r_{\max}}^{(t)}$ where $r_{\max}\in\mathrm{argmax}_{r \in [m]} c_{r}^{(t)} $.

\hspace{.2cm}-- Noise: $\Xi_{i, j, r}^{(t)} = \langle \bm{w}_r^{(t)}, \bm{X}_i[j] \rangle$ for $i\in [N]$ and $j\in [P]\backslash \{P(\bm{X}_i)\}.$

\hspace{.2cm}-- Max noise: ${\Xi}_{\max}^{(t)} = \max_{i \in [N], j \not= P(\bm{X}_{i}), r \in [m]}  | \Xi_{i, j, r}^{(t)}|^2.$

\hspace{.2cm}-- Total noise: $ \Xi_i^{(t)} = \sum_{r \in [m], j \in [P], j \not= P(\bm{X}_{i})}y_i \left( \Xi_{i, j, r}^{(t)} \right)^3 .$

We also use the following notations when dealing with the loss function and its gradient. 
\vspace{.2cm}

\hspace{.2cm}--  Signal loss: $\widehat{\mathcal{L}}^{(t)} (a)= \log\left(1+\exp\left(-\sum_{r=1}^m (c_{r}^{(t)})^3 a^3  \right) \right)$ for $a\in\mathbb{R}$.

\hspace{.2cm}--  Noise loss: $\widehat{\mathcal{L}}^{(t)} (\Xi_i^{(t)})= \log\left(1+\exp\left(- \Xi_i^{(t)} \right) \right)$.

\hspace{.2cm}-- Negative sigmoid function: $\mathfrak{S}(x) = (1+\exp(x))^{-1}$, for $x\in \mathbb{R}.$

\hspace{.2cm}--  Signal derivative: $\widehat{\ell}^{(t)} (a)= \mathfrak{S}\left(\sum_{r=1}^m (c_{r}^{(t)})^3 a^3 \right), $ for $a\in\mathbb{R}.$

\hspace{.2cm}--  Noise derivative: $\widehat{\ell}^{(t)} (\Xi_i^{(t)})= \mathfrak{S}(\Xi_i^{(t)} ).$ 

\hspace{.2cm}-- Derivative: $\ell_{i}^{(t)}  = \mathfrak{S}\left(- \sum_{r=1}^m \sum_{j=1}^P y_i \langle \bm{w}_r^{(t)},\bm{X}_i[j]\rangle^3  \right)$, for $i\in [N].$

\hspace{.2cm}--  $\mathcal{Z}_k$ derivative: $ \nu_k^{(t)} = \frac{1}{N}\sum_{i\in\mathcal{Z}_k} \ell_i^{(t)}$ for $k\in\{1,2\}.$

\hspace{.2cm}--  Full derivative: $\nu^{(t)}=\nu_1^{(t)}+\nu_2^{(t)}.$

\hspace{.2cm}--   Gradient on signal:  $\mathscr{G}_r^{(t)}=\langle \nabla_{\bm{w}_r} \widehat{L}(\bm{W}^{(t)}),\bm{w^*}\rangle$ for $r\in [m].$

\hspace{.2cm}--  Gradient on noise: $\texttt{G}_{i,j,r}^{(t)}=\langle \nabla_{\bm{w}_r} \widehat{L}(\bm{W}^{(t)}),\bm{X}_i[j]\rangle$ for $i\in [N]$, $j\in [P]\backslash\{P(\bm{X}_i)\}$ and $r\in [m].$

\hspace{.2cm}--  Gradient on normalized noise: $\mathrm{G}_{r}^{(t)}=\left\langle \nabla_{\bm{w}_r}\widehat{L}(\bm{W}^{(t)}),\bm{\rchi} \right\rangle,$ for  $r\in [m]$, where $\bm{\rchi}=\frac{\frac{1}{N}\sum_{i\in\mathcal{Z}_2} \sum_{j\neq P(\bm{X}_i)} \bm{X}_i[j]}{\|\frac{1}{N}\sum_{i\in\mathcal{Z}_2} \sum_{j\neq P(\bm{X}_i)}\bm{X}_i[j]\|_2}.$

\subsection{Notations specific to GD+M}  
We now introduce the notations that only appear in the proofs involving GD+M.
\vspace{.2cm}

\hspace{.2cm}-- Momentum gradient oracle: $\bm{g}_r^{(t)} = \gamma \bm{g}_r^{(t-1)} +(1-\gamma) \nabla_{\bm{w}_r} \widehat{L}(\bm{W}^{(t)})$ for $r \in[m].$

\hspace{.2cm}-- Signal momentum: $\mathcal{G}_r^{(t)}:= \langle \bm{g}_r^{(t)},\bm{w^*}\rangle$  for $r \in[m].$
 
\hspace{.2cm}-- Max signal momentum: $ \mathcal{G}^{(t)} = \mathcal{G}_{r_{\max}}^{(t)}$, where $r_{\max}=\mathrm{argmax}_{r\in [m]} c_r^{(t)}.$

\hspace{.2cm}-- Noise momentum:  $G_{i,j,r}^{(t)}= \langle \bm{g}_r^{(t)},\bm{X}_i[j]\rangle$ for $i\in[N]$, $j\in [P]\backslash \{P(\bm{X}_i)\}$ and $r\in [m].$


\section{Induction hypotheses}\label{sec:indh}

We prove our main result using an induction. More specifically, we make the following assumptions for every time $t\leq T.$
\begin{induction}[Bound on the noise component for GD]\label{indh:noiseGD}
Throughout the training process using GD for $t \leq T$, we maintain that:
\begin{enumerate}
   \item (Large signal data have small noise component). For every $i \in \mathcal{Z}_1$, for every $j \in [P]\backslash \{ P(\bm{X}_i) \}$ and $r\in [m],$ we maintain: 
\begin{align}
|\Xi_{i, j, r}^{(t)} | \leq \tilde{O}(\sigma_0  \sigma \sqrt{d}).
\end{align}

   \item  (Small signal data have large noise component). For every $i \in \mathcal{Z}_2$, for every $j \in [P]\backslash \{ P(\bm{X}_i) \}$ and $r\in [m],$ we have:
   \begin{align}
|\Xi_{i, j, r}^{(t)} | \leq  \tilde{O}(1), \quad y_i \Xi_{i, j, r}^{(t)} \geq  -\tilde{O}(\sigma_0  \sigma \sqrt{d}). 
\end{align}
\end{enumerate}
\end{induction}
\begin{induction}[Bound on the signal component for GD]\label{indh:signGD}
Throughout the training process using GD for $t \leq T$, the signal component is bounded for every $r\in [m]$ as
\begin{align*}
 - \tilde{O}(\sigma_0) \leq c_{r}^{(t)}\leq\widetilde{O}(1/\alpha).
\end{align*}
\end{induction}
\begin{induction}[Max noise is bounded by max signal component]\label{indh:maxnoisesign}
Throughout the training process using GD for $t \leq T$, we maintain:
\begin{align*}
\alpha \min\{\kappa,\alpha^2(c^{(t)})^2\} \geq\tilde{\Omega} \left( {\Xi}_{\max}^{(t)}  \right),
\end{align*}
where $\kappa=\tilde{O}(1).$
\end{induction}

\begin{induction}[Bound on the noise component for GD+M]\label{indh:xi_mom}
Throughout the training process using GD+M for $t \leq T$, for every $i \in [N]$, for every $j \in [P]\backslash \{ P(\bm{X}_i) \}$, we have that: 
\begin{align}
|\Xi_{i, j, r}^{(t)} | \leq \tilde{O}(\sigma_0 \sigma \sqrt{d})
\end{align}
\end{induction}

\begin{induction}[Bound on the signal component for GD+M]\label{indh:signal_mom}
Throughout the training process using GD+M for $t \leq T$, for $r\in[m]$, we have that: 
\begin{align}
 - \tilde{O}(\sigma_0) \leq c_{r}^{(t)}\leq\widetilde{O}(1/\beta).
\end{align}
\end{induction}

In what follows, we assume these induction hypotheses for $t<T$ to prove our generalization results. We then prove these hypotheses for $t+1.$
\section{Gradients and updates }

In this section, we first derive the gradient of the loss $\widehat{L}$. We then provide its projection on $\bm{\bm{w^*}}$ (signal gradient) and on $\bm{X}_i[j]$ (noise gradient). We first derive the gradient of the loss $\widehat{L}.$
\begin{lemma}[Gradient of $\widehat{L}$]\label{lem:grad}
For $t\geq 0$ and $r\in[m]$, the gradient of the loss $\widehat{L}$ with respect to $\bm{w}_r$ is:  
\begin{align*}
    \nabla_{\bm{w}_r}\widehat{L}(\bm{W}^{(t)})&= -\frac{3}{N}\left[\left(\sum_{i\in\mathcal{Z}_1} \alpha^3\ell_i^{(t)} + \sum_{i\in\mathcal{Z}_2} \beta^3\ell_i^{(t)} \right)(c_r^{(t)})^2\bm{w^*} + \sum_{i=1}^N \sum_{j\neq P(X_i)} \ell_i^{(t)}(\Xi_{i,j,r}^{(t)})^2 \bm{X}_i[j] \right].
\end{align*}
\end{lemma}
\begin{proof}[Proof of \autoref{lem:grad}]. We derive $\widehat{L}$ with respect to $\bm{w}_r$ and obtain: 
\begin{align}\label{eq:gradwr}
    \nabla_{\bm{w}_r}\widehat{L}(\bm{W}^{(t)})&= -\frac{3}{N}\sum_{i=1}^N \sum_{j=1}^P \frac{y_i\langle \bm{w}_r^{(t)}, \bm{X}_i[j]\rangle^2}{1+\exp(f_{\bm{W}^{(t)}}(X_i))} \bm{X}_i[j]. 
\end{align}
By rewriting \eqref{eq:gradwr}, we obtain the desired result.
\end{proof}

\subsection{Signal gradient}

To track the signal learnt by our models, we compute the signal gradient which is the projection of the gradient on $\bm{w^*}.$

\begin{lemma}[Signal gradient]\label{lem:signgrad} For all $t\geq 0$ and $r\in[m]$, the signal gradient is: 
\begin{align*} 
   -\mathscr{G}_r^{(t)}&= \frac{3}{N}\left(\sum_{i\in\mathcal{Z}_1} \alpha^3\ell_i^{(t)} + \sum_{i\in\mathcal{Z}_2} \beta^3\ell_i^{(t)} \right)(c_r^{(t)})^2.
\end{align*}
\end{lemma}
\begin{proof}[Proof of \autoref{lem:signgrad}] We obtain the desired result by projecting the gradient from \autoref{lem:grad} on $\bm{w^*}$ and using $\bm{X}_i[j]\perp \bm{w^*}.$
\end{proof}

\subsection{Noise gradient} 

To prove the memorization of GD and the non-memorization of GD+M, we also need to compute the noise gradient which is the projection of the gradient $\nabla_{\bm{w}_r}\widehat{L}$ on $\bm{X}_i[j].$
\begin{lemma}[Noise gradient]\label{lem:noisegrad} For all $t\geq 0$, $i\in[N]$ and $j\in [P]\backslash\{P(X_i)\}$ and $r\in[m]$, the noise gradient is: 
\begin{align*} 
   -  \texttt{G}_{i,j,r}^{(t)} &= \frac{3}{N} \ell_i^{(t)}(\Xi_{i,j,r}^{(t)})^2\|\bm{X}_i[j]\|_2^2\\
   &+\frac{3}{N} \sum_{k\neq P(X_i)}\ell_i^{(t)}(\Xi_{i,k,r}^{(t)})^2 \langle \bm{X}_i[k],\bm{X}_i[j]\rangle\\
   &+\frac{3}{N}\sum_{a\neq i}\sum_{k\neq P(X_a)}\ell_a^{(t)}(\Xi_{a,k,r}^{(t)})^2 \langle \bm{X}_a[k],\bm{X}_i[j]\rangle.
\end{align*}
\end{lemma}
\begin{proof}[Proof of \autoref{lem:noisegrad}] Similarly to  \autoref{lem:signgrad}, we obtain the desired result by projecting the gradient from \autoref{lem:grad} on $\bm{X}_i[j]$ and using $\bm{X}_i[j]\perp \bm{w^*}.$
\end{proof}

\begin{remark}\label{remark}
The gradient in \autoref{lem:grad} involve sigmoid terms $\ell_i^{(t)}$. In several parts of the proof, we focus on the time where these terms are small. We consider that the sigmoid term is small for a $\kappa$ such that 
\begin{align}\label{eq:sum_sigmoid}
    \sum_{\tau=0}^T \frac{1}{1+\exp(\kappa)}&\leq \Tilde{O}(1) \implies  \kappa  \geq \log(\Tilde{\Omega}(T)) \iff \kappa \geq \Tilde{\Omega}(1).
\end{align}
Intuitively, \eqref{eq:sum_sigmoid} means that the sum of the sigmoid terms for all time steps is bounded (up to a logarithmic dependence).
\end{remark} 
\section{Learning with GD}\label{sec:app_GD}

In this section, we detail the proofs of the lemmas in \autoref{sec:GD} and \autoref{thm:GD}. We first characterize the dynamics of the signal $c_r^{(t)}$ in \autoref{sec:signgdapp}. We then  analyze the dynamics of the noise $\Xi_{i,j,r}^{(t)}$ in \autoref{sec:mem_gd} and show the memorization of the GD model. We finally prove \autoref{thm:GD} in \autoref{sec:gdthmprf} and the induction hypotheses in \autoref{sec:indhGDprf}.

\subsection{Learning signal with GD}\label{sec:signgdapp}

To track the amount of signal learnt by GD, we make use of the following update. 

\begin{lemma}[Signal update]\label{lem:signupdate} For all $t\geq 0$ and $r\in[m]$, the signal update \eqref{eq:GD_signal} is equal: 
\begin{align*} 
    c_r^{(t+1)}&=c_r^{(t)}+3\eta\left(  \alpha^3\nu_1^{(t)} + \beta^3 \nu_2^{(t)}\right)(c_r^{(t)})^2.
\end{align*}
Consequently, it satisfies:
\begin{align} \label{eq:necdcre}
   \tilde{\Theta}(\eta)(1-\hat{\mu}) \alpha^3\widehat{\ell}^{(t)}(\alpha) (c_r^{(t)})^2 \leq c_r^{(t+1)} -c_r^{(t)} \leq \tilde{\Theta}(\eta)\left((1-\hat{\mu}) \alpha^3\widehat{\ell}^{(t)}(\alpha)   +\beta^3 \nu_2^{(t)} \right)(c_r^{(t)})^2.
\end{align}
\end{lemma}
\begin{proof}[Proof of \autoref{lem:signupdate}] The signal update is obtained by using \eqref{eq:GD_signal} and the signal gradient (\autoref{lem:signgrad}). This yields
\begin{align}\label{eq:signupd}
    c_r^{(t+1)}&=c_r^{(t)}+\frac{3\eta}{N}\left(\sum_{i\in\mathcal{Z}_1} \alpha^3\ell_i^{(t)} + \sum_{i\in\mathcal{Z}_2} \beta^3\ell_i^{(t)} \right)(c_r^{(t)})^2.
\end{align}
To obtain the desired lower bound, we first drop the sum over $\mathcal{Z}_2$ in \eqref{eq:signupd}. Then, for $i\in\mathcal{Z}_1$, we apply \autoref{eq:derZ1_bd} to get $\ell_i^{(t)}  =\Theta(1)\widehat{\ell}^{(t)}(\alpha)$.

To obtain the desired upper bound, we apply the same reasoning as above to bound the $\mathcal{Z}_1$ term. 
\end{proof}

\subsubsection{Early stages of the learning process $t\in[0,T_0]$: learning $\mathcal{Z}_1$ data}\label{sec:earlysg_z1}

Since $\bm{w}_r^{(0)}\sim\mathcal{N}(0,\sigma_0^2\mathbf{I}_d)$ with $\sigma_0$ small, the sigmoid terms $\widehat{\ell}^{(t)}(\alpha)$ and $\ell_i^{(t)}$ in  the signal update are large at early iterations. As $c_r^{(t)}$ is non-decreasing (by \autoref{lem:signupdate}), $\widehat{\ell}^{(t)}(\alpha)$ eventually becomes small at a time $T_0>0.$ As mentioned in \autoref{remark},  the sigmoid term $\mathfrak{S}(x)$ is small when  $x\geq \kappa\geq \tilde{\Omega}(1)$. We therefore simplify \eqref{eq:signupd} for $t\in [0,T_0]$.

\begin{lemma}[Signal update at early iterations]\label{lem:signalearly} Let $T_0>0$ the time where there exists $s\in [m]$ such that $c_{s}^{(t)}\geq \tilde{\Omega}(1/\alpha)$. Then, for $t\in [0,T_0]$ and for all $r\in [m]$, the signal update is simplified as:
 \begin{align}\label{eq:signupdearlytrue}
  \Theta(\eta)(1-\hat{\mu}) \alpha^3  (c_r^{(t)})^2 \leq c_r^{(t+1)} -c_r^{(t)} \leq \Theta(\eta)\left((1-\hat{\mu}) \alpha^3  +\hat{\mu}\beta^3 \right)(c_r^{(t)})^2.
\end{align}

\end{lemma}
\begin{proof}[Proof of \autoref{lem:signalearly}] For $t\in [0,T_0]$, we know that for all $s\in[m]$, we have $c_{s}^{(t)}\leq \frac{\kappa}{m^{1/3}\alpha}$.
Therefore, we have 
\begin{align}\label{eq:hulf}
  \frac{1}{1+\exp(\tilde{\Omega}(1))} \leq  \widehat{\ell}^{(t)}(\alpha) = \frac{1}{1+\exp\left( \sum_{s=1}^m \alpha^3 (c_s^{(t)})^3 \right)}\leq 1.
\end{align}
From \autoref{remark}, the sigmoid is small when it is equal to $\frac{1}{1+\exp(\tilde{\Omega}(1))} $. Thus, for $t\in[T_0,T]$, we have:
\begin{align}\label{eq:hbvce}
    \widehat{\ell}^{(t)}(\alpha)=\Theta(1).
\end{align}
Plugging \eqref{eq:hbvce} in the left-hand side of \eqref{eq:necdcre} yields the desired lower bound. 

To get the upper bound, we start from the right-hand side of \eqref{eq:necdcre}. We upper bound  $\frac{1}{N}\sum_{i\in\mathcal{Z}_2}\ell_i^{(t)}\leq \hat{\mu}$ since $\ell_i^{(t)}\leq 1.$ Moreover, we use \eqref{eq:hbvce} to upper bound the $\widehat{\ell}^{(t)}(\alpha)$ term.
\end{proof}

We now prove \autoref{lem:increase_signalGD} that quantifies the amount of signal learnt by GD when the derivative is  large.

\lemcincrGD*   

\begin{proof}[Proof of \autoref{lem:increase_signalGD}] Let $r\in[m]$. From \autoref{lem:signalearly}, the signal update for $t\in [0,T_0]$ is 
 \begin{equation}
        \begin{aligned}\label{eq:ct_bds}
        \begin{cases}
             c_r^{(t+1)} \leq    c_r^{(t)} + A( c_r^{(t)} )^2\\
           c_r^{(t+1)} \geq    c_r^{(t)} + B (c_r^{(t)})^2
        \end{cases},
        \end{aligned}
    \end{equation}
where $A$ and $B$ are respectively defined as: 
\begin{align*}
    A&:=\tilde{\Theta}(\eta)\left((1-\hat{\mu}) \alpha^3  +\hat{\mu}\beta^3 \right) ,\\
    B&:=\tilde{\Theta}(\eta)(1-\hat{\mu}) \alpha^3 .
\end{align*}
Now, we would like to find the time $T_0$ where $c_r^{(t)}\geq  \tilde{\Omega}(1/\alpha).$ This time exists as $c_r^{(t)}$ is non-decreasing. To this end, we apply the Tensor Power method (\autoref{lem:pow_method}). This lemma only applies to non-negative sequences. Since we initialize the weights $\bm{w}_r^{(0)} \sim \mathcal{N}(0,\sigma_0^2 \mathbf{I}_d)$, we have $c_r^{(0)} \sim \mathcal{N}(0,\sigma_0^2).$ Since all the $w_r^{(0)}$'s are i.i.d. so do the $c_r^{(0)}$'s. Therefore, the probability that at least one of the $c_r^{(0)}$ is non-negative is $1-(1/2)^{m}=1-o(1).$ We thus conclude that with high probability, there exist an index $r\in [m]$ such that $c_r^{(0)}\geq 0.$  Among the possible indices $r$ that satisfy this inequality, we now focus on $r=r_{\max}$ where $r_{\max}\in \mathrm{argmax} \;  c_r^{(0)}$. 

Setting $\upsilon=\tilde{\Theta}(1/\alpha)$ in \autoref{lem:pow_method}, we deduce that the time $T_0$ is 
\begin{align*}
 T_0= \frac{\tilde{\Theta}(1)}{\eta \alpha^3\sigma_0}+\frac{\tilde{\Theta}(1)\left((1-\hat{\mu}) \alpha^3  +\hat{\mu}\beta^3 \right) }{(1-\hat{\mu}) \alpha^3 }\left\lceil \frac{-\log(\tilde{\Theta}(\sigma_0\alpha))}{\log(2)}\right\rceil
\end{align*}


\end{proof}

We now prove \autoref{lem:Z1derivative}. It states that since the signal  $c^{(t)}$ has significantly increased, the $\mathcal{Z}_1$ derivative $\nu_1^{(t)}$ is now small. Before proving this result, we introduce an auxiliary Lemma.

\begin{lemma}[Lower bound on the signal update]\label{prop:ct_lwbd} Run GD on the loss function $\widehat{L}(W).$ After $T_0=\tilde{\Theta}\left(\frac{1}{\eta \alpha^3\sigma_0 } \right)$ iterations, the signal update is satisfies for $t\geq t_0$
\begin{align*}
    c^{(t+1)}\geq c^{(t)}+\eta\tilde{\Omega}(\alpha)\nu_1^{(t)}.
\end{align*}
 \end{lemma}

\begin{proof}[Proof of \autoref{prop:ct_lwbd} ] 
From \autoref{lem:signupdate}, we know that
\begin{align}\label{eq:tmptcrt}
    c^{(t+1)}&\geq c^{(t)}+ \tilde{\Theta}(\eta) \nu_1^{(t)}\alpha^3 (c^{(t)})^2.
\end{align}
 Plugging $c^{(t)} \geq \tilde{\Omega}(1/\alpha)$ (\autoref{lem:increase_signalGD}) in \eqref{eq:tmptcrt}, we obtain the desired result.
\end{proof}

\lemderivGD*

\begin{proof}[Proof of \autoref{lem:Z1derivative}] From  \autoref{prop:ct_lwbd}, we deduce an upper bound on $\nu_1^{(t)}$:
\begin{align}\label{eq:nu1bdtmp}
    \nu_1^{(t)}\leq \tilde{O}\left(\frac{1}{\eta \alpha}\right)(c^{(t+1)}-c^{(t)}).
\end{align}
On the other hand, using \autoref{lem:signgrad}, the signal difference is bounded as:
\begin{align}
    c^{(t+1)}-c^{(t)}&\leq \sum_{r=1}^m c_r^{(t+1)}-c_r^{(t)}\nonumber\\
    &\leq (1-\hat{\mu}) \Theta(\eta\alpha)\sum_{r=1}^m (\alpha c_r^{(t)})^2\widehat{\ell}^{(t)}(\alpha)+\hat{\mu} \Theta(\eta\beta^3)\sum_{r=1}^m(c_r^{(t)})^2\nu_2^{(t)}.\label{eq:nefcreerwc}
\end{align}
By applying \autoref{indh:noiseGD} in \eqref{eq:nefcreerwc} and using $m=\tilde{\Theta}(1)$, we obtain:
\begin{align}\label{eq:ihrcehrcwed}
    c^{(t+1)}-c^{(t)}&\leq (1-\hat{\mu}) \Theta(\eta\alpha)\sum_{r=1}^m (\alpha c_r^{(t)})^2\widehat{\ell}^{(t)}(\alpha)+\hat{\mu}\tilde{O}(\eta\beta^3)\nu_2^{(t)}.
\end{align}
We now bound \eqref{eq:ihrcehrcwed} by a loss term by applying \autoref{lem:masterlogsigmdelta}. Using \autoref{lem:increase_signalGD} and \autoref{indh:signGD}, we have:
\begin{align}\label{eq:ceiec}
  0<\tilde{\Omega}(1/\alpha) \leq \tilde{\Omega}(1/\alpha)-m\tilde{O}(\sigma_0)\leq  c^{(t)} - \sum_{r\neq r_{\max}}c_r^{(t)}\leq  \sum_{r=1}^m\alpha c_r^{(t)} \leq m\tilde{O}(1)\leq\tilde{O}(1).
\end{align}
We can now apply \autoref{lem:masterlogsigmdelta} and get:
\begin{align}\label{eq:frejcjewoidc}
    \sum_{r=1}^m (\alpha c_r^{(t)})^2\widehat{\ell}^{(t)}(\alpha)\leq \frac{20m\alpha  e^{m\tilde{O}(\sigma_0)}}{\tilde{\Omega}(1)}\widehat{\mathcal{L}}^{(t)}(\alpha)\leq \tilde{O}(\alpha)\widehat{\mathcal{L}}^{(t)}(\alpha).
\end{align}
Plugging \eqref{eq:frejcjewoidc} in \eqref{eq:ihrcehrcwed} yields:
\begin{align}\label{eq:ncekrefrejwc}
    c^{(t+1)}-c^{(t)}&\leq (1-\hat{\mu}) \tilde{O}(\eta\alpha^2)\widehat{\mathcal{L}}^{(t)}(\alpha) +\hat{\mu}\tilde{O}(\eta\beta^3)\nu_2^{(t)}.
\end{align}
Combining \eqref{eq:nu1bdtmp} and \eqref{eq:ncekrefrejwc}, we thus obtain: 
\begin{align}\label{eq:ct_eq1}
   \nu_1^{(t)}\leq  \tilde{O}\left(\frac{1}{\alpha}\right)\left((1-\hat{\mu}) \tilde{O}(\alpha^2)\widehat{\mathcal{L}}^{(t)}(\alpha) +\hat{\mu}\tilde{O}(\beta^3)\nu_2^{(t)}\right).
\end{align}
From  \autoref{thm:convratez1}, we have the convergence rate of  $\widehat{\mathcal{L}}^{(t)}(\alpha)$. We use it to  bound  $\nu_1^{(t)}.$
  
 The bound on $\nu^{(t)}$ is obtained by using its definition $\nu^{(t)}=\nu_1^{(t)}+\nu_2^{(t)}$. 
\end{proof}

\subsubsection{Late stages of learning process $t\in[T_0,T]$: amount of learnt signal controlled by $\mathcal{Z}_2$ derivative}

We earlier proved that after $T_0$ iterations, the signal $c^{(t)}$ learnt by the GD model significantly increases until making $\nu_1^{(t)}$ small. We therefore need to rewrite the signal update in this case.
\begin{lemma}[Rewriting of signal update]\label{lem:lateiter}
For $t\in [T]$, the maximal signal $c^{(t)}$ updates as: 
\begin{align*} 
c^{(t+1)} -c^{(t)} \leq  \Theta(\eta) \left(\alpha \nu_1^{(t)} \min\{\kappa, (c^{(t)})^2\alpha^2\}+    \frac{\beta^3}{\alpha^2} \nu_2^{(t)}\right).
\end{align*}
\end{lemma}
\begin{proof}[Proof of \autoref{lem:lateiter}] From the signal update given by \autoref{lem:signupdate}, we know that:
\begin{align} \label{eq:kfckewkw}
    c^{(t+1)}&=c^{(t)}+\frac{3\eta\alpha}{N}\sum_{i\in\mathcal{Z}_1} (\alpha c^{(t)})^2\ell_i^{(t)} +   3\eta\beta^3 \nu_2^{(t)}  (c^{(t)})^2.
\end{align}

To obtain the desired result, we need to prove for $i\in\mathcal{Z}_1$:
\begin{align}\label{eq:kfecreedw}
(   \alpha c^{(t)})^2\ell_i^{(t)} &\leq\Theta(1)\min\{\kappa,\alpha^2 (c^{(t)})^2\}.
\end{align}
Indeed, we remark that: 
\begin{align}\label{eq:gradalphac}
   (\alpha c^{(t)})^2 \ell_i^{(t)} &= \frac{\alpha^3 (c^{(t)})^2 }{1+\exp\left( \alpha^3\sum_{s=1}^m(c_{s}^{(t)})^3+\Xi_{i}^{(t)}\right)}.
\end{align}
By using \autoref{indh:noiseGD} and \autoref{indh:signGD}, \eqref{eq:gradalphac} is bounded as: 
\begin{align}\label{eq:frcejzn}
   (\alpha c^{(t)})^2\ell_i^{(t)} &=  \frac{\alpha^3 (c^{(t)})^2  }{1+\exp\left( \alpha^3(c^{(t)})^3+\alpha^3\sum_{s\neq r_{\max}}(c_s^{(t)})^3+\Xi_{i}^{(t)}\right)}\nonumber\\
    &\leq \frac{\alpha^3 (c^{(t)})^2  }{1+\exp\left( \alpha^3(c^{(t)})^3 -\tilde{O}(m\alpha^3\sigma_0^3) - \tilde{O}(mP(\sigma\sigma_0\sqrt{d})^3)\right)}\nonumber\\ 
    &= \frac{\Theta(\alpha) (\alpha c^{(t)})^2}{1+\exp( (\alpha c^{(t)})^3)}.
\end{align}
Using \autoref{remark}, the sigmoid term in \eqref{eq:frcejzn} becomes small when $\alpha c^{(t)}\geq \kappa^{1/3}$. To summarize, we have: 
\begin{align}\label{eq:alphalictcd}
    (  \alpha c^{(t)})^2 \ell_i^{(t)}&=\begin{cases}
                                    0       & \text{if }\alpha c^{(t)}\geq \kappa^{1/3}\\
                                      (\alpha c^{(t)})^2     \ell_i^{(t)}  & \text{otherwise}
                                       \end{cases}.
\end{align}
\eqref{eq:alphalictcd} therefore implies $ (\alpha c^{(t)})^2 \ell_i^{(t)} \leq \Theta(1)\min\{\kappa^{2/3},(\alpha c^{(t)})^2\}$ which implies \eqref{eq:kfecreedw}.

Besides, we use \autoref{indh:signGD} to bound $(c^{(t)})^2$ in the right-hand side of \eqref{eq:kfckewkw}.
\end{proof}

We now show that once $\nu_1^{(t)}$ is small, the amount of learnt signal  is controlled by $\nu_2^{(t)}$ . 

\lemctbdGD*

\begin{proof}[Proof of \autoref{eq:ct_update}]

Let $\tau \in [T_0,T)$. From \autoref{lem:lateiter}, we know that: 
\begin{align}\label{eq:crcfrne}
    c^{(\tau+1)} -c^{(\tau)} \leq \Theta(\eta) \left(\alpha \nu_1^{(\tau)} \min\{\kappa, (c^{(\tau)})^2\alpha^2\}+    \frac{\beta^3}{\alpha^2} \nu_2^{(\tau)}(c^{(\tau)})^2\right)
\end{align}
Let $t\in [\tau,T)$. We now sum up \eqref{eq:crcfrne} for $\tau=T_0,\dots,t$ and obtain: 
\begin{align}\label{eq:crcfrnvfedccedfvf}
  c^{(t+1)} \leq c^{(T_0)} + \Theta(\eta\alpha) \sum_{\tau=T_0}^t\nu_1^{(\tau)} \min\{\kappa, (c^{(\tau)})^2\alpha^2\}+    \frac{\Theta(\eta\beta^3)}{\alpha^2}\sum_{\tau=T_0}^t \nu_2^{(\tau)}.
\end{align}
We now plug the bound on $\nu_1^{(t)}$ from \autoref{lem:Z1derivative} in \eqref{eq:crcfrnvfedccedfvf}. This implies:
\begin{align}\label{eq:crcfrnvfezedzedcce}
  c^{(t+1)} \leq c^{(T_0)} +   \sum_{\tau=T_0}^t  \frac{\tilde{O}(1)}{\tau-T_0+1}+ \tilde{O}(\eta\beta^3)\left(1+\frac{1}{\alpha^2}\right)    \sum_{\tau=T_0}^t \nu_2^{(\tau)}.
\end{align}

Plugging $\sum_{\tau}1/\tau \leq \tilde{O}(1)$ and $c^{(T_0)}\leq \tilde{O}(1/\alpha)$ (\autoref{indh:signGD}) in \eqref{eq:crcfrnvfezedzedcce}, we obtain: 
\begin{align*}
    c^{(t+1)} \leq \frac{\tilde{O}(1)}{\alpha}+\frac{\tilde{O}(\eta\beta^3)}{\alpha^2}\sum_{\tau=T_0}^t \nu_2^{(\tau)}.
\end{align*}
 \end{proof}

\subsection{Memorization process of  GD}\label{sec:mem_gd}

\autoref{lem:Z1derivative} shows that after $T_0$ iterations, the gradient is controlled by $\nu_2^{(t)}$. In this section, we show that this yields the GD model to memorize.

\subsubsection{Memorizing $\mathcal{Z}_2$ ($t\in [0,T_1] $)}

Using \autoref{lem:signupdate}, we simplify the noise update. 

\begin{lemma}[Noise update]\label{lem:noisupdate}  Let all $t\geq 0$, $i\in [N]$, $j\in [P]\backslash \{ P(\bm{X}_i)\}$ and $r\in[m]$. Then, with probability at least $1-o(1)$, the noise update \eqref{eq:GD_noise} is bounded as
\begin{align}\label{eq:noisudpd}
    \left| y_i\Xi_{i,j,r}^{(t+1)} - y_i\Xi_{i,j,r}^{(t)}  -\frac{\tilde{\Theta}(\eta\sigma^2 d)}{N}\ell_i^{(t)}(\Xi_{i,j,r}^{(t)})^2  \right|\leq \frac{\tilde{\Theta}(\eta\sigma^2\sqrt{d})}{N}\sum_{a=1}^N\ell_a^{(t)}\sum_{ k\neq P((\bm{X}_a) } (\Xi_{a,k,r}^{(t)})^2 .
\end{align}
\end{lemma}
\begin{proof}[Proof of \autoref{lem:noisupdate}]
   Let $i\in [N]$, $j\in [P]\backslash \{ P(\bm{X}_i)\}$ and $r\in [m]$. From \autoref{lem:noisegrad}, we know that the noise update satisfies:
   \begin{equation}\label{eq:bdyuXIii}
   \begin{split}
       y_i\Xi_{i,j,r}^{(t+1)}&=  y_i\Xi_{i,j,r}^{(t)} + \frac{3\eta}{N}\ell_i^{(t)}(\Xi_{i,j,r}^{(t)})^2 \|\bm{X}_i[j]\|_2^2 +\frac{3\eta}{N}\ell_i^{(t)}\sum_{\substack{k\neq P(X_i)\\ k\neq j}} (\Xi_{i,k,r}^{(t)})^2 \langle \bm{X}_i[k], \bm{X}_i[j]\rangle\\
       &+\frac{3\eta}{N}\sum_{a\neq i}\ell_a^{(t)}\sum_{ k\neq P(\bm{X}_a) } (\Xi_{a,k,r}^{(t)})^2\langle \bm{X}_a[k], \bm{X}_i[j]\rangle.
   \end{split}
   \end{equation}
We now apply \autoref{thm:hgh_prob_gauss} and \autoref{prop:dotprodGauss} to respectively bound $\|\bm{X}_i[j]\|_2^2$ and $\langle \bm{X}_a[k],\bm{X}_i[j]\rangle$ in \eqref{eq:bdyuXIii} and  obtain the desired result.
\end{proof}


In the next lemma, we further simplify the noise update from \autoref{lem:noisupdate}. 

\begin{lemma}[Sum of noise updates]\label{lem:sumnoisupdate} Let $i\in\mathcal{Z}_2$, $j\in [P]\backslash\{P(\bm{X}_i)\}$ and $r\in [m].$ Let $\mathfrak{T}=\tilde{\Theta}\left(\frac{ P\sigma^2\sqrt{d}}{\eta\beta^3\hat{\mu} } \right)$. For $t\leq \mathfrak{T}$, the noise update satisfies: 
  \begin{align}\label{eq:noisudpdindc}
   \left| y_i \Xi_{i,j,r}^{(t+1)}- y_i \Xi_{i,j,r}^{(0)}-\frac{\Tilde{\Theta}(\eta\sigma^2 d)}{N}\sum_{\tau=0}^t \widehat{\ell}^{(\tau)}(\Xi_i^{(\tau)}) (\Xi_{i,j,r}^{(\tau)})^2\right| \leq \Tilde{O}\left(\frac{P\sigma^2\sqrt{d}}{\alpha} \right).
\end{align}
\end{lemma}
\begin{proof}[Proof of \autoref{lem:sumnoisupdate}] Let $i\in\mathcal{Z}_2$, $j\in [P]\backslash\{P(\bm{X}_i)\}$ and $r\in [m]$. Our starting point is \autoref{lem:indhnoisesum} which states that: 
 \begin{align}\label{eq:rnfejddze}
     \left| y_i \Xi_{i,j,r}^{(t)}- y_i \Xi_{i,j,r}^{(0)} -\frac{\eta\Tilde{\Theta}(\sigma^2 d)}{N}\sum_{\tau=0}^{t-1} \ell_i^{(\tau)}(\Xi_{i,j,r}^{(\tau)})^2\right| &\leq \Tilde{O}\left(\frac{P\sigma^2\sqrt{d}}{\alpha }   \right)+    \tilde{O}\left(\frac{\eta\beta^3}{\alpha  } \right)  \sum_{j=0}^{t }\nu_2^{(j)}.
 \end{align}
 Since $t\leq \mathfrak{T}=\tilde{\Theta}\left(\frac{ P\sigma^2\sqrt{d}}{\eta\beta^3\hat{\mu} } \right)$, we bound the second sum term in \eqref{eq:rnfejddze} as: 
 \begin{align}\label{eq:nrejn}
     \tilde{O}\left(\frac{\eta\beta^3}{\alpha  } \right)  \sum_{j=0}^{t }\nu_2^{(j)}&\leq  \tilde{O}\left(\frac{\eta\beta^3}{\alpha  } \right) \hat{\mu}t\leq  \tilde{O}\left(\frac{\eta\beta^3\hat{\mu}\mathfrak{T}}{\alpha  } \right) \leq\tilde{O}\left(\frac{P\sigma^2\sqrt{d}}{\alpha} \right). 
 \end{align}
From \eqref{eq:nrejn}, we deduce that
 \begin{align}\label{eq:vfece}
     \left| y_i \Xi_{i,j,r}^{(t)}- y_i \Xi_{i,j,r}^{(0)} -\frac{\eta\Tilde{\Theta}(\sigma^2 d)}{N}\sum_{\tau=0}^{t-1} \ell_i^{(\tau)}(\Xi_{i,j,r}^{(\tau)})^2\right| &\leq \Tilde{O}\left(\frac{P\sigma^2\sqrt{d}}{\alpha }   \right).
 \end{align}
Lastly, we know from \autoref{eq:derZ2_bd} that $\ell_i^{(\tau)}=\Theta(1)\widehat{\ell}^{(t)}(\Xi_i^{(\tau)})$. Plugging this in \eqref{eq:vfece} yields the desired result.
\end{proof}

 Since $\bm{w}_r^{(0)}\sim\mathcal{N}(0,\sigma_0^2\mathbf{I}_d)$ with $\sigma_0$ small, the sigmoid terms $\widehat{\ell}^{(t)}(\Xi_i^{(t)})$  in  the noise  update are large at early iterations. After a certain time $T_1>0$, there exist an index $s\in[m]$ such that $\Xi_{i,j,s}^{(t)}$ becomes large and $\widehat{\ell}^{(t)}(\Xi_i^{(t)})$ eventually becomes small. We therefore simplify \eqref{eq:noisudpdindc} for $t\in[0,T_1].$

\begin{lemma}[Noise update at early iterations]\label{lem:noiseearly} Let $i\in\mathcal{Z}_2$ and $j\in [P]\backslash \{P(\bm{X}_i)\}$. Let $T_1>0$ be the time where there exists  $s\in[m]$ such that $\Xi_{i,j,s}^{(t)}\geq \tilde{\Omega}(1)$. Then, for $t\in [0,T_1]$ and for all $r\in[m]$, the noise update is simplified as:
 \begin{align}\label{eq:noiseeealy}
     \left| y_i \Xi_{i,j,r}^{(t+1)} - y_i \Xi_{i,j,r}^{(0)}-\frac{\Tilde{\Theta}(\eta\sigma^2 d)}{N}\sum_{\tau=0}^t   (\Xi_{i,j,r}^{(\tau)})^2 \right|\leq  \Tilde{O}\left(\frac{P\sigma^2\sqrt{d}}{\alpha} \right).
\end{align}
\end{lemma}
\begin{proof}[Proof of \autoref{lem:noiseearly}] Let  $t\leq T_1$. We assume for now that $T_1\leq\mathfrak{T}=\tilde{\Theta}\left(\frac{ P\sigma^2\sqrt{d}}{\eta\beta^3\hat{\mu} } \right)$ and will check this hypothesis in the proof of \autoref{lem:noise_dominates}.  Let $i\in\mathcal{Z}_2$, $j\in [P]\backslash \{P(\bm{X}_i)\}$ and $r\in [m]$. From
 \autoref{lem:sumnoisupdate}, we know that   
\begin{align} \label{eq:lwbdXijr111}
 \left| y_i \Xi_{i,j,r}^{(t+1)} - y_i \Xi_{i,j,r}^{(0)}-\frac{\Tilde{\Theta}(\eta\sigma^2 d)}{N}\sum_{\tau=0}^t \widehat{\ell}^{(\tau)}(\Xi_i^{(\tau)}) (\Xi_{i,j,r}^{(\tau)})^2 \right|\leq  \Tilde{O}\left(\frac{P\sigma^2\sqrt{d}}{\alpha} \right).
\end{align}



From \autoref{remark}, we know that $\widehat{\ell}^{(\tau)}(\Xi_i^{(\tau)})$ is small when $\Xi_i^{(\tau)}\geq \kappa \geq\tilde{\Omega}(1).$ To have this condition, it is sufficient that there exists an index $s\in[m]$ such that $y_i\Xi_{i,j,s}^{(\tau)}\geq \tilde{\Omega}(1).$ Indeed, by using \autoref{indh:noiseGD}, we have:
\begin{align*}
    \Xi_i^{(t)}&=(y_i\Xi_{i,j,s}^{(t)})^3+\sum_{s=1}^m\sum_{k\neq P(\bm{X}_i)} (y_i\Xi_{i,k,s}^{(t)})^3 \geq \tilde{\Omega}(1)-\tilde{O}(mP(\sigma\sigma_0\sqrt{d})^3)\geq \tilde{\Omega}(1).
\end{align*}
Therefore, for $t\in[0,T_1]$ and $\tau\leq t$, we have $\widehat{\ell}^{(\tau)}(\Xi_i^{(\tau)})=\Theta(1)$. In this case,  the noise update \eqref{eq:lwbdXijr111} is: 
\begin{align} \label{eq:lwfvve}
  \left| y_i \Xi_{i,j,r}^{(t+1)} - y_i \Xi_{i,j,r}^{(0)}-\frac{\Tilde{\Theta}(\eta\sigma^2 d)}{N}\sum_{\tau=0}^t   (\Xi_{i,j,r}^{(\tau)})^2 \right|\leq  \Tilde{O}\left(\frac{P\sigma^2\sqrt{d}}{\alpha} \right).
\end{align}
\end{proof}

\autoref{lem:noiseearly} indicates that $y_i\Xi_{i,j,r}^{(t)}$ is \textit{not non-decreasing} but overall, this quantity gets large over time. We now want to determine the time $T_1$ where one of the $y_i\Xi_{i,j,r}^{(t)}$ becomes large.

\begin{lemma}\label{lem:timet1tmp}
Let $i\in\mathcal{Z}_2$, $j\in[P]\backslash\{P(\bm{X}_i\}$ and $T_1=\tilde{\Theta}\left(\frac{N}{\sigma_0\sigma\sqrt{d}\sigma^2d}\right).$ After $T_1$ iterations, there exists $s\in[m]$ such that $\Xi_{i,j,s}^{(t)}\geq\tilde{\Omega}(1).$
\end{lemma}
\begin{proof}[Proof of \autoref{lem:timet1tmp}] 
\autoref{lem:noiseearly} indicates that the noise iterate satisfies for $t\in [0,T_1]$:
\begin{equation}
    \begin{aligned} \label{eq:yXibds2}
    \begin{cases}
         y_i\Xi_{i,j,r}^{(t)}\geq  y_i\Xi_{i,j,r}^{(0)}+A\sum_{\tau=0}^{t-1}(\Xi_{i,j,r}^{(\tau)})^2 -C\\
         y_i\Xi_{i,j,r}^{(t)}\leq  y_i\Xi_{i,j,r}^{(0)}+A\sum_{\tau=0}^{t-1}(\Xi_{i,j,r}^{(\tau)})^2 + C
    \end{cases}
    \end{aligned},
\end{equation}
 where $A,C>0$ are constants defined as
 \begin{align}\label{eq:constants_ABC}
     A=\frac{\tilde{\Theta}(\eta\sigma^2d)}{N},\quad 
     C=\tilde{O}\left(\frac{P\sigma^2\sqrt{d}}{\alpha }\right).
 \end{align}
To find $T_1$, we apply the Tensor Power method (\autoref{lem:pow_method_sum}) to \eqref{eq:yXibds2}.  We initialize the weights $\bm{w}_r^{(0)} \sim \mathcal{N}(0,\sigma_0^2 \mathbf{I}_d)$ and $\bm{X}_i[j]\sim\mathcal{N}(0,\sigma^2\mathbf{I}_d)$. Therefore, we have 
$\mathbb{P}[y_i\Xi_{i,j,r}^{(0)}\geq 0] =1/2 $. Since all the $w_r^{(0)}$'s are i.i.d. so do the $\Xi_{i,j,r}^{(0)}$'s. Therefore, the probability that at least one of the $\Xi_{i,j,r}^{(0)}$ is non-negative is $1-(1/2)^{m}=1-o(1).$ We thus conclude that with high probability, there exist an index  $r\in [m]$ such that $y_i\Xi_{i,j,r}^{(0)}\geq \Omega(\sigma \sigma_0 \sqrt{d}) \geq \Omega(C)$. In what follows, we focus on such index $r$.

Setting the constants $A,C$ as in \eqref{eq:constants_ABC} and $\upsilon=\tilde{O}(1)$, the time $T_1$ obtained with the Tensor Power method is 
\begin{align}
T_1&=\frac{21N}{\tilde{\Theta}(\eta\sigma^2d) y_i\Xi_{i,j,r}^{(0)}}+\frac{8N  }{\tilde{\Theta}(\eta\sigma^2d)(y_i\Xi_{i,j,r}^{(0)})}\left\lceil \frac{\log\left(\frac{\tilde{O}(1)}{y_i\Xi_{i,j,r}^{(0)}}\right)}{\log(2)}\right\rceil.\nonumber
\end{align}

We thus obtain $T_1=\tilde{\Theta}\left(\frac{N }{\eta\sigma^2d\sigma_0\sigma\sqrt{d} }  \right).$ We indeed verify that $T_1\leq \mathfrak{T}$ since $\tilde{\Theta}\left(\frac{N }{\eta\sigma^2d\sigma_0\sigma\sqrt{d} }  \right)\ll \tilde{\Theta}\left(\frac{ NP\sigma^2\sqrt{d}}{\eta\beta^3\hat{\mu} } \right).$

\end{proof}

\lemxibdGD*

\begin{proof}[Proof of \autoref{lem:noise_dominates}] The simplified \eqref{eq:GD_noise} update is obtained from \autoref{lem:sumnoisupdate}. Besides, we know that $T_1$ is the first time where there exists $s\in[m]$ such that $\Xi_{i,j,s}^{(t)}\geq\tilde{\Omega}(1).$ As explained in the proof of \autoref{lem:noiseearly},  $\Xi_{i,k,s}^{(t)}\geq\tilde{\Omega}(1)$ implies that $\Xi_i^{(t)}\geq \tilde{\Omega}(1).$ We can therefore apply \autoref{lem:timet1tmp} to obtain the aimed result.
 
 \end{proof}

\subsubsection{Late stages of memorization $t\in [T_1,T]$: convergence to a minimum}

We proved in the previous section that after $T_1$ iterations, the amount of noise memorized by the GD model significantly increases. We want to show that after this phase, $\nu_2^{(t)}$ is well-controlled.

\begin{lemma}[Bound on $\mathcal{Z}_2$ derivative at late iterations]\label{eq:njefecjn}
Let $T_1=\tilde{\Theta}\left(\frac{N}{\sigma_0\sigma\sqrt{d}\sigma^2d}\right)$. For $t\in [T_1,T]$, 
we have $\sum_{\tau=T_1}^t\nu_2^{(\tau)}\leq \tilde{O}\left( \frac{1}{\eta \sigma_0} \right).$
\end{lemma}

\begin{proof}[Proof of \autoref{eq:njefecjn} ]
In \autoref{lem:timet1tmp}, we proved that after $T_1$ iterations, for all $i\in\mathcal{Z}_2$ and $j\in [P]\backslash  \{P(\bm{X}_i)\} $, there exists $s\in[m]$ such that $y_i\Xi_{i,j,s}^{(t)}\geq \tilde{\Omega}(1).$ Therefore, for $t\in [T_1,T]$, there exists $s\in [m]$ such that the noise update (from \autoref{lem:indhnoisesum}) satisfies: 
\begin{equation}\label{eq:sumlitaa}
        \begin{aligned}
  \sum_{\tau=T_1}^t\nu_2^{(\tau)} 
 &\leq \tilde{O}\left(\frac{1}{\eta\sigma^2 d}\right)  \sum_{i\in\mathcal{Z}_2}y_i(\Xi_{i,j,s}^{(t+1)}-  \Xi_{i,j,s}^{(T_1)})\\
  &+ \Tilde{O}\left(\frac{P}{\alpha\eta\sqrt{d}} \right)+ \tilde{O}\left(\frac{\beta^3 }{\alpha \sigma^2d} \right) \sum_{j=T_1}^{t-1}\nu_2^{(j)} .
    \end{aligned}
    \end{equation}
On the other hand, from \autoref{lem:indhnoisesum}, we know that for all $r\in [m]$:
\begin{equation}\label{eq:refdc} 
        \begin{aligned}
    \sum_{i\in\mathcal{Z}_2} y_i (\Xi_{i,j,r}^{(t+1)}-  \Xi_{i,j,r}^{(T_1)})  &\leq\frac{\eta\Tilde{\Theta}(\sigma^2 d)}{N}\sum_{\tau=T_1}^{t-1}\sum_{i\in\mathcal{Z}_2} \ell_i^{(\tau)}(\Xi_{i,j,r}^{(\tau)})^2\\
     &+\Tilde{O}\left(\frac{P\sigma^2\sqrt{d}}{\alpha}  \right)+    \tilde{O}\left(\frac{\eta\beta^3 }{\alpha} \right) \sum_{j=T_1}^{t-1} \nu_2^{(j)}.
    \end{aligned}
    \end{equation}
Combining \eqref{eq:sumlitaa} and \eqref{eq:refdc} yields: 
\begin{equation}\label{eq:sumlitaa2}
        \begin{aligned}
 \sum_{\tau=T_1}^t\nu_2^{(\tau)}&\leq    \frac{\tilde{O}(1)}{N}\sum_{\tau=T_1}^{t-1}\sum_{i\in\mathcal{Z}_2} \ell_i^{(\tau)}(\Xi_{i,j,s}^{(\tau)})^2+    \tilde{O}\left(\frac{ \beta^3 }{\alpha \sigma^2 d} \right) \sum_{j=T_1}^{t-1}\nu_2^{(j)}  \\
  &+ \Tilde{O}\left(\frac{P}{\eta\alpha\sqrt{d}}  \right) \\
    \end{aligned}
    \end{equation}
Again, because $ \tilde{O}\left(\frac{ \beta^3 }{\alpha \sigma^2 d} \right) \ll 1$, we further simplify \eqref{eq:sumlitaa2}:
\begin{equation}\label{eq:sumlitaa3}
        \begin{aligned}
  \sum_{\tau=T_1}^t \nu_2^{(\tau)} &\leq    \frac{\tilde{O}(1)}{N}\sum_{\tau=T_1}^{t-1}\sum_{i\in\mathcal{Z}_2} \ell_i^{(\tau)}(\Xi_{i,j,s}^{(\tau)})^2      + \Tilde{O}\left(\frac{P}{\eta\alpha\sqrt{d}}  \right).\\
    \end{aligned}
    \end{equation}
We apply \autoref{eq:derZ2_bd} to bound $\ell_i^{(\tau)}$ on the right-hand side of \eqref{eq:sumlitaa3} and get 
\begin{equation}\label{eq:sumlitaa4}
        \begin{aligned}
   \sum_{\tau=T_1}^t \nu_2^{(\tau)}  &\leq    \frac{\tilde{O}(1)}{N}\sum_{\tau=T_1}^{t-1}\sum_{i\in\mathcal{Z}_2} \widehat{\ell}^{(\tau)}(\Xi_i^{(\tau)})(\Xi_{i,j,s}^{(\tau)})^2      + \Tilde{O}\left(\frac{P}{\eta\alpha\sqrt{d}}  \right).
    \end{aligned}
    \end{equation}
We add $\sum_{j\neq P(\bm{X}_i)}\sum_{r\neq s} \widehat{\ell}^{(\tau)}(\Xi_i^{(\tau)})(\Xi_{i,j,r}^{(\tau)})^2\geq 0$ to the right-hand side of  \eqref{eq:sumlitaa4} and obtain:
\begin{equation}\label{eq:sumlitaa5}
        \begin{aligned}
  \frac{1}{N} \sum_{\tau=T_1}^t \sum_{i\in\mathcal{Z}_2}\ell_i^{(\tau)}  &\leq    \frac{\tilde{O}(1)}{N }\sum_{\tau=T_1}^{t-1}\sum_{i\in\mathcal{Z}_2}\sum_{r=1}^m\sum_{j\neq P(X_i)} \widehat{\ell}^{(\tau)}(\Xi_i^{(\tau)})(\Xi_{i,j,r}^{(\tau)})^2      + \Tilde{O}\left(\frac{P}{\eta\alpha\sqrt{d}}  \right).\\
    \end{aligned}
    \end{equation}
 Moreover, by applying \autoref{lem:masterlogsigmdelta} to \eqref{eq:sumlitaa5}, we have:
 \begin{equation}\label{eq:sumlitaa5522}
        \begin{aligned}
  \frac{1}{N} \sum_{\tau=T_1}^t \sum_{i\in\mathcal{Z}_2}\ell_i^{(\tau)}  &\leq    \frac{\tilde{O}(1)}{N}\sum_{\tau=T_1}^{t-1}\sum_{i\in\mathcal{Z}_2}  \widehat{\mathcal{L}}^{(\tau)}(\Xi_i^{(\tau)})       + \Tilde{O}\left(\frac{P}{\eta\alpha\sqrt{d}}  \right).\\
    \end{aligned}
    \end{equation}
 We now apply \autoref{thm:convrate} to bound the loss in \eqref{eq:sumlitaa5522}. 
\begin{align}
    \frac{1}{N}\sum_{\tau=T_1}^t \sum_{i\in\mathcal{Z}_2} \ell_i^{(\tau)} & \leq \frac{\tilde{O}(1)}{\eta} \sum_{\tau=T_1}^{t} \frac{1}{\tau-T_1+1}     +   \Tilde{O}\left(\frac{P}{\alpha\sqrt{d}} \right)\leq \tilde{O}\left(\frac{1}{\eta}\right) +\tilde{O}\left(\frac{P}{\eta\alpha\sqrt{d}}\right)\leq \tilde{O}\left(\frac{1}{\eta}\right),\label{eq:pekzde}
\end{align}
where we used in \eqref{eq:pekzde} $\sum_{\tau=T_1+1}^{t}1/\tau \leq \tilde{O}(1)$ and $P/\alpha=\tilde{O}(1)$.

\end{proof}

  Using \autoref{eq:njefecjn}, we can obtain a bound on the sum over time of $\mathcal{Z}_2$ derivatives .

 \lemnutwo*
\begin{proof}[Proof of \autoref{lem:Z2derivative} ]

We know that: 
\begin{align}\label{eq:nu2T1}
    \sum_{j=0}^{T_1-1}\nu_2^{(j)}\leq \hat{\mu}T_1.
\end{align}
Combining the bound on $\sum_{j=T_1}^T\nu_2^{(j)}$ from \autoref{eq:njefecjn} and \eqref{eq:nu2T1} yields:
\begin{align}
 \sum_{j=0}^{T}\nu_2^{(j)}&= \sum_{j=0}^{T_1-1}\nu_2^{(j)}+ \sum_{j=T_1}^{T}\nu_2^{(j)}\leq \tilde{\Theta}\left(\frac{\hat{\mu}N}{\eta\sigma_0\sigma\sqrt{d}\sigma^2d} \right)+\tilde{O}\left(\frac{1}{\eta} \right)\leq \tilde{O}\left( \frac{1}{\eta\sigma_0}\right).
\end{align}

\end{proof}

We have thus a control on the sum over time of $\nu_2^{(t)}$. We can make use of \autoref{eq:ct_update} to get the final control on the signal iterate $c^{(t)}.$

\lemctfinGD*

\begin{proof}[Proof of \autoref{lem:signal_final}]
Let $t\in [T].$ From \autoref{eq:ct_update}, we know that the signal is bounded as 
\begin{align}
    c^{(t)}& \leq \tilde{O}(1/\alpha) +  \tilde{O}(\eta\beta^3/\alpha^2) \sum_{\tau=T_0}^{t-1}\nu_2^{(\tau)}. \label{eq:ctc0f}
\end{align}
We plug the bound from \autoref{lem:Z2derivative} to bound the last term in the right-hand side of \eqref{eq:ctc0f}.

\end{proof}
 
 \subsection{Proof of \autoref{thm:GD}}\label{sec:gdthmprf}

We proved that the weights learnt by GD satisfy for $r\in[m]$
\begin{align}\label{eq:wrtttGD}
    \bm{w}_r^{(T)} = c_r^{(T)}\bm{w^*} +\bm{v}_r^{(T)},
\end{align}
where for all $r\in[m],$ $c_r^{(T)}\leq \tilde{O}(1/\alpha)$ (\autoref{lem:signal_final}) and $v_r^{(T)}\in \mathrm{span}(\bm{X}_i[j])\subset \mathrm{span}(\bm{w^*})^{\perp}.$ By \autoref{lem:noise_dominates}, since $\Xi_{i}^{(t)}\geq \tilde{\Omega}(1)$, we have $\|\bm{v}_r^{(T)}\|_2\geq 1.$  We are now ready to prove the generalization achieved by GD and stated in  \autoref{thm:GD}.

\GDmain*

\begin{proof}[Proof of \autoref{thm:GD}] 
 We now bound the training and test error achieved by GD at time $T.$

\textbf{Train error.} \autoref{thm:convrate} provides a convergence bound on the training loss.
\begin{align}\label{eq:lossWTGD}
    \widehat{L}(\bm{W}^{(T)})&\leq \frac{\tilde{\Theta}(1)}{\eta (T-T_0+1)}.
\end{align}
Plugging $T\geq \mathrm{poly}(d)N/\eta$ and $\mu =\Theta(1/N)$ in \eqref{eq:lossWTGD} yields:
\begin{align}
     \widehat{L}(\bm{W}^{(T)})&\leq\tilde{O}\left( \frac{1}{\mathrm{poly}(d)N}\right)\leq \tilde{O}\left( \frac{\mu}{\mathrm{poly}(d)}\right).
\end{align}

\textbf{Test error.} Let $(X,y)$ be a datapoint. We remind that $\bm{X}=(\bm{X}[1],\dots,\bm{X}[P])$ where $\bm{X}[P(\bm{X})] =\theta y\bm{w^*}$ and $\bm{X}[j]\sim\mathcal{N}(0,\sigma^2\mathbf{I}_d)$ for $j\in[P]\backslash \{P(\bm{X})\}.$  We bound the test error as follows: 
\begin{align}
   \mathscr{L}(f_{\bm{W}^{(T)}})&= \mathbb{E}_{\substack{\mathcal{Z}\sim\mathcal{D}\\
    (\bm{X},y)\sim \mathcal{Z}}}[\mathbf{1}_{yf_{\bm{W}^{(T)}}(X)<0}] \nonumber \\
    &=\mathbb{E}_{ 
    (\bm{X},y)\sim \mathcal{Z}_1}[\mathbf{1}_{yf_{\bm{W}^{(T)}}(\bm{X})<0}]\mathbb{P}[\mathcal{Z}_1] + \mathbb{E}_{ 
    (\bm{X},y)\sim \mathcal{Z}_2}[\mathbf{1}_{yf_{\bm{W}^T}(\bm{X})<0}]\mathbb{P}[\mathcal{Z}_2]\nonumber\\
    &=(1-\hat{\mu})\mathbb{P}[yf_{\bm{W}^{(T)}}(\bm{X})<0|(\bm{X},y)\sim\mathcal{Z}_1]\label{eq:yfGDtmp111}\\
    &+\hat{\mu}\mathbb{P}[yf_{\bm{W}^{(T)}}(\bm{X})<0|(\bm{X},y)\sim\mathcal{Z}_2].\label{eq:yfGDtmp}
\end{align}
We now want to compute the probability terms in \eqref{eq:yfGDtmp111} and \eqref{eq:yfGDtmp}. We remind that $(\bm{X},y)\sim \mathcal{Z}_1,$ $yf_{\bm{W}^{(T)}}(\bm{X})$ is given by 
\begin{align}
    yf_{\bm{W}^{(T)}}(\bm{X})  &= y\sum_{s=1}^m\sum_{j=1}^P \langle w_s^{(T)}, \bm{X}[j]\rangle^3 \nonumber\\
    &= \alpha^3 \sum_{s=1}^m (c_s^{(T)})^3+y\sum_{s=1}^m\sum_{j\neq P(\bm{X})} \langle  \bm{v}_s^{(T)}, \bm{X}[j]\rangle^3.\label{eq:yf1stGD}
\end{align}
We now apply \autoref{lem:signal_final} in \eqref{eq:yf1stGD} and obtain: 
\begin{align}\label{eq:bdfYGD}
    yf_{\bm{W}^{(T)}}(X)\leq \tilde{O}(1)+ y\sum_{s=1}^m\sum_{j\neq P(X)} \langle  \bm{v}_s^{(T)}, \bm{X}[j]\rangle^3.
\end{align}
Let $(\bm{X},y)\sim\mathcal{Z}_2$. Similarly, by applying \autoref{lem:signal_final}, $yf_{\bm{W}^{(T)}}(\bm{X})$ is bounded as: 
\begin{align}
    yf_{\bm{W}^{(T)}}(\bm{X})\leq \tilde{O}((\beta/\alpha)^3)+y\sum_{s=1}^m\sum_{j\neq P(X)} \langle  \bm{v}_s^{(T)}, \bm{X}[j]\rangle^3.
\end{align}

Therefore, using \eqref{eq:bdfYGDM}, we upper bound the test error \eqref{eq:yf} as:
\begin{equation}
    \begin{aligned}\label{eq:tmp_class_err_GD}
    \mathscr{L}(f_{\bm{W}^{(T)}})&\geq (1-\hat{\mu})\mathbb{P}\left[ y\sum_{s=1}^m\sum_{j\neq P(\bm{X})} \langle  \bm{v}_s^{(T)}, \bm{X}[j]\rangle^3\leq -\tilde{\Omega}\left(1\right) \right]\\
    &+ \hat{\mu}\mathbb{P}\left[ y\sum_{s=1}^m\sum_{j\neq P(\bm{X})} \langle  \bm{v}_s^{(T)}, \bm{X}[j]\rangle^3\leq -\tilde{\Omega}((\beta/\alpha)^3) \right]\\
    &\geq \hat{\mu}\mathbb{P}\left[ y\sum_{s=1}^m\sum_{j\neq P(\bm{X})} \langle  \bm{v}_s^{(T)}, \bm{X}[j]\rangle^3\leq -\tilde{\Omega}((\beta/\alpha)^3) \right].
\end{aligned}
\end{equation}
Since $y$ is taken uniformly from $\{-1,1\},$ we further simplify \eqref{eq:tmp_class_err_GD} as: 
    \begin{align}
   \mathscr{L}(f_{\bm{W}^{(T)}})   &\geq  \frac{\hat{\mu}}{2}\mathbb{P}\left[\left| \sum_{s=1}^m\sum_{j\neq P(\bm{X})} \langle  \bm{v}_s^{(T)}, \bm{X}[j]\rangle^3\right|\geq \tilde{\Omega}((\beta/\alpha)^3) \right] .\label{eq:tmp_class_err_GD1}
\end{align}
We know that $\tilde{\Theta}(\beta^3)=\tilde{\Theta}(\sigma^3)$. Therefore, we now apply \autoref{prop:vr_xi} to bound \eqref{eq:tmp_class_err_GD1} and finally obtain:
    \begin{align}
   \mathscr{L}(f_{\bm{W}^{(T)}})   &\geq  \frac{\hat{\mu}}{2}\mathbb{P}\left[\left| \sum_{s=1}^m\sum_{j\neq P(\bm{X})} \langle  \bm{v}_s^{(T)}, \bm{X}[j]\rangle^3\right|\geq \tilde{\Omega}((\beta/\alpha)^3) \right] \geq \frac{\hat{\mu}}{2}\left(1-\frac{\tilde{O}(d)}{2^d}\right)\geq \tilde{\Omega}(\mu) . 
\end{align}
 
\end{proof}
 
\subsection{Proof of the GD induction hypotheses}\label{sec:indhGDprf}

To prove \autoref{thm:GD}, we used the induction hypotheses stated in \autoref{sec:indh}. The goal of this section is to prove them for $t+1.$ 

\begin{proof}[Proof of \autoref{indh:noiseGD}] We prove here the main hypotheses we made on the noise when using GD.

\paragraph{GD Noise for $i\in\mathcal{Z}_2.$}  Let $i\in\mathcal{Z}_2,$ $j\in [P]\backslash\{P(\bm{X}_i)\}$ and $r\in[m].$ We know that for $t \in[T]$, $y_i\Xi_{i,j,r}^{(t)}\leq \tilde{O}(1).$ Let's prove the result for $t+1.$ From \autoref{lem:indhnoisesumlyytateiter}, we have: 
\begin{equation}
    \begin{aligned}
    &\left| y_i(\Xi_{i,j,r}^{(t+1)} - \Xi_{i,j,r}^{(0)})  -\frac{\tilde{\Theta}(\eta\sigma^2 d)}{N}\sum_{\tau=0}^{t}\ell_i^{(\tau)}\min\{\kappa,(\Xi_{i,j,r}^{(\tau)})^2\}  \right| \\
    &\leq  \tilde{O}\left(\frac{P\sigma^2\sqrt{d}}{\alpha}\right) +   \tilde{O}\left(\frac{\eta\beta^3}{\alpha^2  } \right)  \sum_{\tau=0}^{t}\nu_2^{(\tau)} .\label{eq:nferjcn}
\end{aligned}
\end{equation}
Let's start with the upper bound $y_i\Xi_{i,j,r}^{(t)}$ for $i\in\mathcal{Z}_2$. Using \autoref{lem:ncejncejzd}, \autoref{lem:Z2derivative} and \autoref{indh:noiseGD}, we deduce from \eqref{eq:nferjcn} that: 
\begin{equation}
\begin{aligned}
    y_i\Xi_{i,j,r}^{(t+1)}&\leq \tilde{O}(1)+\tilde{O}(\sigma^2d)+\tilde{O}\left(\frac{P\sigma^2\sqrt{d}}{\alpha}\right)+\tilde{O}\left(\frac{\beta^3}{\alpha^2\sigma_0  }\right)\leq \tilde{O}(1),
\end{aligned}
\end{equation}
which proves the induction hypothesis for $t+1.$ Regarding the lower bound, using \autoref{indh:noiseGD} and \autoref{lem:Z2derivative}, we deduce from \eqref{eq:nferjcn} that:
\begin{align}
     y_i\Xi_{i,j,r}^{(t+1)}&\geq -\tilde{O}(\sigma\sigma_0\sqrt{d})-\tilde{O}\left(\frac{P\sigma^2\sqrt{d}}{\alpha}\right)-\tilde{O}\left(\frac{\beta^3}{\alpha^2\sigma_0  }\right)\geq -\tilde{O}(\sigma\sigma_0\sqrt{d}),
\end{align}
which proves the induction hypothesis for $t+1.$

\paragraph{GD Noise for $i\in\mathcal{Z}_1.$} Let $i\in\mathcal{Z}_1.$ We know that for $t \in[T]$, $y_i\Xi_{i,j,r}^{(t)}\leq \tilde{O}(\sigma\sigma_0\sqrt{d}).$ Let's prove the result for $t+1.$ Using \autoref{lem:noisegrad}, we know that the \eqref{eq:GD_noise} update is:
\begin{equation}
    \begin{split} \label{eq:iter_noiseZ1dcvrf}
        y_i\Xi_{i,j,r}^{(t+1)}&\leq y_i\Xi_{i,j,r}^{(0)} +\frac{\tilde{\Theta}(\eta\sigma^2d)}{N}\sum_{\tau=0}^t\ell_i^{(\tau)}(\Xi_{i,j,r}^{(\tau)})^2 \\
    &+\frac{\tilde{\Theta}(\eta\sigma^2\sqrt{d})}{N}\sum_{a\in\mathcal{Z}_1}\sum_{k\neq P(\bm{X}_a)} \sum_{\tau=0}^t\ell_a^{(\tau)}(\Xi_{a,k,r}^{(\tau)})^2\\
    &+\frac{\tilde{\Theta}(\eta\sigma^2\sqrt{d})}{N}\sum_{a\in\mathcal{Z}_2}\sum_{k\neq P((\bm{X}_a)}\sum_{\tau=0}^t \ell_a^{(\tau)}(\Xi_{a,k,r}^{(\tau)})^2.
    \end{split}
\end{equation}
Using \autoref{indh:noiseGD}, we bound $y_i\Xi_{i,j,r}^{(0)}$ and $(\Xi_{a,k,r}^{(\tau)})^2$ in \eqref{eq:iter_noiseZ1dcvrf}. We obtain: 
\begin{equation}
    \begin{split} \label{eq:iter_noiseZ1bretdcgfbvrf}
        y_i\Xi_{i,j,r}^{(t+1)}&\leq \tilde{O}(\sigma\sigma_0\sqrt{d}) +\frac{\tilde{\Theta}(\eta\sigma_0^2\sigma^4d^2)}{N}\sum_{\tau=0}^t\ell_i^{(\tau)}  \\
    &+ \tilde{\Theta}(\eta P\sigma_0^2\sigma^4d^{3/2})    \sum_{\tau=0}^t \nu_1^{(\tau)}  \\
    &+\frac{\tilde{\Theta}(\eta\sigma^2\sqrt{d})}{N}\sum_{a\in\mathcal{Z}_2}\sum_{k\neq P((\bm{X}_a)}\sum_{\tau=0}^t \ell_a^{(\tau)}(\Xi_{a,k,r}^{(\tau)})^2
    \end{split}
\end{equation}
Now, we apply  \autoref{prop:bd_nu1sum} to bound $\nu_1^{(\tau)}$ and $\ell_i^{(\tau)}/N$ in \eqref{eq:iter_noiseZ1bretdcgfbvrf}. 
\begin{equation}
    \begin{split} \label{eq:iter_noiseZ1bretdcrefrgfbvrf}
        y_i\Xi_{i,j,r}^{(t+1)}&\leq \tilde{O}(\sigma\sigma_0\sqrt{d}) +  \tilde{O}\left(\frac{\sigma_0^2\sigma^4d^2}{  \alpha}\right) +\tilde{O}\left(\frac{\eta\sigma_0^2\sigma^4d^2\beta^3}{\alpha}\right)\sum_{\tau=0}^t \nu_2^{(\tau)}  \\
    &+ \tilde{O}\left(\frac{P\sigma_0^2\sigma^4d^{3/2}}{  \alpha}\right) +\tilde{O}\left(\frac{\eta P\sigma_0^2\sigma^4d^{3/2}\beta^3}{\alpha}\right)\sum_{\tau=0}^t \nu_2^{(\tau)} \\
    &+\frac{\tilde{\Theta}(\eta\sigma^2\sqrt{d})}{N}\sum_{a\in\mathcal{Z}_2}\sum_{k\neq P((\bm{X}_a)}\sum_{\tau=0}^t \ell_a^{(\tau)}(\Xi_{a,k,r}^{(\tau)})^2
    \end{split}
\end{equation}
We now apply \autoref{lem:indhnoisesumlyytateiter} and \autoref{lem:Z2derivative} to bound the derivative terms in \eqref{eq:iter_noiseZ1bretdcrefrgfbvrf}. 
\begin{equation*}
    \begin{split}
        y_i\Xi_{i,j,r}^{(t+1)}&\leq \tilde{O}(\sigma\sigma_0\sqrt{d})+  \tilde{O}\left(\frac{\sigma_0^2\sigma^4d^2}{  \alpha}\right) +\tilde{O}\left(\frac{\sigma_0\sigma^4d^2\beta^3}{\alpha}\right)   \\
    &+ \tilde{O}\left(\frac{P\sigma_0^2\sigma^4d^{3/2}}{  \alpha}\right) +\tilde{O}\left(\frac{ P\sigma_0\sigma^4d^{3/2}\beta^3}{\alpha}\right)  \\
    &+ \tilde{O}(  P\sigma^2\sqrt{d})  \\
    &\leq \tilde{O}(\sigma\sigma_0\sqrt{d}),
    \end{split}
\end{equation*}
which proves the induction hypothesis for $t+1.$ Now, let's prove that $y_i\Xi_{i,j,r}^{(t+1)}\geq -\tilde{O}(\sigma\sigma_0\sqrt{d}).$ Similarly to above, the \eqref{eq:GD_noise} update is bounded as:
\begin{equation}
    \begin{split} \label{eq:iter_ncdszcroisvreZ1}
        y_i\Xi_{i,j,r}^{(t+1)}&\geq y_i\Xi_{i,j,r}^{(0)} +\frac{\tilde{\Theta}(\eta\sigma^2d)}{N}\sum_{\tau=0}^t\ell_i^{(\tau)}(\Xi_{i,j,r}^{(\tau)})^2 \\
    &-\frac{\tilde{\Theta}(\eta\sigma^2\sqrt{d})}{N}\sum_{a\in\mathcal{Z}_1}\sum_{k\neq P((\bm{X}_a)} \sum_{\tau=0}^t\ell_a^{(\tau)}(\Xi_{a,k,r}^{(\tau)})^2\\
    &-\frac{\tilde{\Theta}(\eta\sigma^2\sqrt{d})}{N}\sum_{a\in\mathcal{Z}_2}\sum_{k\neq P((\bm{X}_a)}\sum_{\tau=0}^t \ell_a^{(\tau)}(\Xi_{a,k,r}^{(\tau)})^2.
    \end{split}
\end{equation}
Using the same type of reasoning as for the upper bound, one can show that \eqref{eq:iter_ncdszcroisvreZ1} yields: 
\begin{equation}
    \begin{split} \label{eq:iter_noiseZ1bretdcrefrgfbvrfreref}
        y_i\Xi_{i,j,r}^{(t+1)}&\geq -\tilde{O}(\sigma\sigma_0\sqrt{d})  - \tilde{O}\left(\frac{\sigma_0\sigma^4d^2\beta^3}{\alpha}\right)   \\
    &- \tilde{O}\left(\frac{P\sigma_0^2\sigma^4d^{3/2}}{  \alpha}\right) +\tilde{O}\left(\frac{ P\sigma_0\sigma^4d^{3/2}\beta^3}{\alpha}\right)  \\
    &- \tilde{O}(  P\sigma^2\sqrt{d})  \\
    &\geq -\tilde{O}(\sigma\sigma_0\sqrt{d}).
    \end{split}
\end{equation}
\eqref{eq:iter_noiseZ1bretdcrefrgfbvrfreref} shows the induction hypothesis for $t+1.$

\end{proof}

\begin{proof}[Proof of \autoref{indh:signGD}] We prove the induction hypotheses for the signal $c_r^{(t)}.$

\paragraph{Proof of $c_r^{(t+1)}\geq -\tilde{O}(\sigma_0)$.} We know that with high probability, $c_r^{(0)}\geq -\tilde{O}(\sigma_0)$. By \autoref{lem:signupdate}, $c_r^{(t)}$ is a non-decreasing sequence and therefore, we always have $c_r^{(t)}\geq-\tilde{O}(\sigma_0). $

\paragraph{Proof of $c_r^{(t+1)}\leq \tilde{O}(1/\alpha)$.} Using the same proof as the one for \autoref{lem:signal_final}, we get $c^{(t+1)}\leq\tilde{O}(1/\alpha).$ Besides, $c_r^{(t+1)}\leq c^{(t+1)}$ which implies the induction hypothesis for $t+1.$

\end{proof}

\begin{proof}[Proof of \autoref{indh:maxnoisesign}]  $\alpha\min\{1, (c^{(t)})^2\alpha^2\}\geq (\Xi_{i,j,r}^{(t)})^2$ is true for all $t.$ Indeed, we proved in \autoref{lem:increase_signalGD} that after $T_0$ iterations, $c^{(t)}\geq \tilde{\Omega}(1/\alpha).$ Moreover, we proved \autoref{indh:noiseGD} claiming that $|\Xi_{i,j,r}^{(t)}|\leq \tilde{O}(1)$  
Therefore, we have $\alpha\min\{1, (c^{(t)})^2\alpha^2\}\geq (\Xi_{i,j,r}^{(t)})^2$.

\end{proof}

\section{Learning with GD+M}\label{sec:app_GDM}

In this section, we prove the Lemmas in \autoref{sec:GDM} and \autoref{thm:GDM}.

\subsection{Learning signal with GD+M}

To track the amount of signal learnt by GD, we make use of the following update.

\begin{lemma}[Signal momentum]\label{lem:signupdateM} For all $t\geq 0$ and $r\in[m]$, the signal momentum in \eqref{eq:GDM_signal} is equal to: 
\begin{align*} 
    \mathcal{G}_r^{(t+1)} &=\gamma\mathcal{G}_r^{(t)}-3(1-\gamma)\left( \alpha^3\nu_1^{(t)} +\beta^3\nu_2^{(t)}\right)(c_r^{(t)})^2.
\end{align*}
We can further simplify this update as:
\begin{align*} 
\mathcal{G}_r^{(t+1)} &=\gamma\mathcal{G}_r^{(t)}-\Theta(1-\gamma)\left( \alpha^3(1-\hat{\mu})\widehat{\ell}^{(t)}(\alpha)+\beta^3\hat{\mu}\widehat{\ell}^{(t)}(\beta)\right)(c_r^{(t)})^2.
\end{align*}
\end{lemma}
\begin{proof}[Proof of \autoref{lem:signupdateM}] By definition of the momentum update, we have: $\bm{g}_r^{(t+1)}=\gamma \bm{g}_r^{(t)}+(1-\gamma) \nabla_{\bm{w}_r}\widehat{L}(\bm{W}^{(t)}).$ We project this update onto $\bm{w^*}$ and use \autoref{lem:signgrad} to get:
\begin{align}\label{eq:gefrcb}
    \mathcal{G}_r^{(t+1)} &=\gamma\mathcal{G}_r^{(t)}-3(1-\gamma)\left( \alpha^3\nu_1^{(t)} +\beta^3\nu_2^{(t)}\right)(c_r^{(t)})^2.
\end{align}
By applying \autoref{lem:dervdGDM}, we have $\nu_1^{(t)}=\Theta(1-\hat{\mu})\widehat{\ell}^{(t)}(\alpha)$ and $\nu_2^{(t)}=\Theta(\hat{\mu})\widehat{\ell}^{(t)}(\beta).$ Plugging this observation in \eqref{eq:gefrcb} yields the desired result.
\end{proof}


\subsubsection{Early stages of the learning process $t\in [0,\mathcal{T}_0]$: learning $\mathcal{Z}_1$ data}

Similarly to GD, since we initialize $\bm{w}_r^{(0)}\sim\mathcal{N}(0,\sigma_0^2\mathbf{I}_d)$ with $\sigma_0$ small, the sigmoid terms $\widehat{\ell}^{(t)}(\alpha)$ and $\widehat{\ell}^{(t)}(\beta)$ in the momentum are large at early iterations. As $c_r^{(t)}$ is non-decreasing, $\widehat{\ell}^{(t)}(\alpha)$ eventually becomes small at a time $\mathcal{T}_0>0.$ We therefore simplify the signal momentum update for $t\in [0,\mathcal{T}_0].$
\begin{lemma}[Signal momentum at early iterations]\label{lem:signalmomearly} Let $\mathcal{T}_0>0$ the time where there exists $s\in[m]$ such that $c_s^{(t)}\geq \tilde{\Omega}(1/\alpha).$ Then, for $t\in [0,\mathcal{T}_0]$ and $r\in [m]$, the signal momentum is simplified as: 
\begin{align}
    \mathcal{G}_r^{(t+1)}&=\gamma\mathcal{G}_r^{(t)}-\Theta(\alpha^3(1-\gamma)) (c_r^{(t)})^2.
\end{align}
\end{lemma}
\begin{proof}[Proof of \autoref{lem:signalmomearly}]
From \autoref{lem:signupdateM}, we can simplify the momentum update as: 
\begin{align}\label{eq:momgrla}
-\hat{\mu}\widehat{\ell}^{(t)}(\beta) (c_r^{(t)})^2   \leq \mathcal{G}_r^{(t+1)}-\gamma\mathcal{G}_r^{(t)}  +\Theta(1-\gamma)\alpha^3(1-\hat{\mu})\widehat{\ell}^{(t)}(\alpha)(c_r^{(t)})^2\leq 0.
\end{align}

For $t\in [0,\mathcal{T}_0]$, we know that for all $s\in[m]$, we have $c_{s}^{(t)}\leq \frac{\kappa}{m^{1/3}\alpha}$ . Thus, we have:
\begin{align}\label{eq:hulfceff}
  \frac{1}{1+\exp(\tilde{\Omega}(1))} \leq  \widehat{\ell}^{(t)}(\alpha) = \frac{1}{1+\exp\left( \sum_{s=1}^m \alpha^3 (c_s^{(t)})^3 \right)}\leq 1.
\end{align}
By \autoref{remark}, we know that the sigmoid is small only when we have $\frac{1}{1+\exp(\tilde{\Omega}(1))} $. From \eqref{eq:hulfceff}, we have
\begin{align}\label{eq:hbvceecrezr}
    \widehat{\ell}^{(t)}(\alpha)=\Theta(1).
\end{align}
Besides, we have: 
\begin{align}\label{eq:encfrj}
   \frac{1}{1+\exp(\tilde{\Omega}(1))}  \leq \frac{1}{1+\exp(\tilde{\Omega}(\beta^3/\alpha^3))} \leq  \widehat{\ell}^{(t)}(\beta) = \frac{1}{1+\exp\left( \sum_{s=1}^m \beta^3 (c_s^{(t)})^3 \right)}\leq 1.
\end{align}
From \eqref{eq:encfrj}, we have: 
\begin{align}\label{eq:lbeta1}
    \widehat{\ell}^{(t)}(\beta)=\Theta(1).
\end{align}
Plugging \eqref{eq:hbvceecrezr} and \eqref{eq:lbeta1} in \eqref{eq:momgrla} yields the desired result.
\end{proof}


We now prove \autoref{lem:increase_signalGDM} that quantifies the signal learnt by GD when $\nu_1^{(t)}$ is non-zero.

\ctalphGDM*

\begin{proof}[Proof of \autoref{lem:increase_signalGDM}]
By \autoref{lem:signalmomearly}, the signal update for $t\in [0,\mathcal{T}_0]$ satisfies: 
\begin{align}\label{eq:rjcnzc"z}
    \begin{cases}
    \mathcal{G}_r^{(t+1)}&=\gamma\mathcal{G}_r^{(t)}-(1-\gamma)\tilde{\Theta}(\alpha^3) (c_r^{(t)})^2\\
    c_r^{(t+1)}&=c_r^{(t)}-\eta \mathcal{G}_r^{(t+1)}
\end{cases}.
\end{align}


As $c_r^{(t)}$ is non-decreasing, it will eventually reach $\tilde{\Omega}(1/\alpha).$ We can use the arguments as in the proof of \autoref{lem:increase_signalGD} to argue that there exists an index $r$ such that $c_r^{(t)}>0.$ Among all the possible indices, we focus on $r=r_{\max}$, where $r_{\max}=\mathrm{argmax}_{r\in [m]} c_r^{(0)}.$\\

To find $\mathcal{T}_0$, we apply the Tensor Power Method (\autoref{lem:pow_method_GDM}) to \eqref{eq:rjcnzc"z}. Setting $\upsilon=\tilde{O}(1/\alpha)$ in \autoref{lem:pow_method_GDM}, we deduce that the time $\mathcal{T}_0$ is 
\begin{align*}
    \mathcal{T}_0= \frac{1}{1-\gamma}\left\lceil \frac{\log(\tilde{O}(1/\alpha) )}{\log(1+\delta)}\right\rceil+\frac{1+\delta}{\eta(1-e^{-1})\alpha^3 c^{(0)}},
\end{align*}
where $\delta\in (0,1)$.
 
\end{proof}


\subsubsection{Late stages of learning process $t\in [\mathcal{T}_0,T]$: learning $\mathcal{Z}_2$ data}

We now show that contrary to GD, GD+M still has a large momentum in the $\bm{w^*}$ direction. In other words, we want to show that $-\mathcal{G}^{(t)}$ is still large after $\mathcal{T}_0$ iterations. Given that the small margin and large margin data share the same feature $\bm{w^*}$, this large momentum helps to learn $\mathcal{Z}_2$.

Before proving such result, we need some intermediate lemmas.

\begin{lemma}\label{eq:GTbound}
Let $\mathcal{T}>0$ the time such that $ -\mathcal{G}^{(\mathcal{T})} \leq   \tilde{O}(\sqrt{1-\gamma}/\alpha).$ Then, for all $t'\leq \mathcal{T}$, we have: 
\begin{align*}
            -\mathcal{G}^{(t')}\leq \frac{\tilde{O}(\sqrt{1-\gamma}) }{ \alpha \gamma^{\mathcal{T}-t'}} .
\end{align*}
\end{lemma}
\begin{proof}[Proof of \autoref{eq:GTbound}]
Using the momentum update rule, we know that:
  \begin{align}\label{eq:Gtmtpr}
     - \mathcal{G}^{(\mathcal{T})}
      &=  - \gamma^{\mathcal{T}-t'} \mathcal{G}^{(t')}-(1-\gamma)\sum_{\tau=t'}^{\mathcal{T}-1}\gamma^{\mathcal{T}-\tau}\mathscr{G}^{(\tau)}.
  \end{align}
Since $-\mathscr{G}^{(\tau)}>0$ for all $\tau\geq 0$, \eqref{eq:Gtmtpr} implies $-\gamma^{\mathcal{T}-t'} \mathcal{G}^{(t')}\leq -\mathcal{G}^{(\mathcal{T})}$. Using $ -\mathcal{G}^{(\mathcal{T})} \leq  \tilde{O}(\sqrt{1-\gamma}/\alpha),$ we obtain the aimed result.
\end{proof}

\begin{lemma}\label{lem:cTGTbound}
Let $\mathcal{T}_0$ be the first iteration where $c^{(t)}>\tilde{\Omega}(1/\alpha)$.  Assume that $ \mathcal{G}^{(\mathcal{T}_0)} \leq  \tilde{O}(\sqrt{1-\gamma}/\alpha)$  Then, for all $t \in \left[ \mathcal{T}_0-\frac{1}{\sqrt{1-\gamma}}, \mathcal{T}_0 \right]$, we have: 
  \begin{align*}
       c^{(t)}  &\geq 0.5 \tilde{\Omega}(1/\alpha) . 
  \end{align*}

\end{lemma}
\begin{proof}[Proof of \autoref{lem:cTGTbound}]
Let's define $t':=\mathcal{T}_0-\frac{1}{\sqrt{1-\gamma}}$.  We start by summing the GD+M update \eqref{eq:GDM_signal} for $\tau=t',\dots,\mathcal{T}_0$ to get
  \begin{align}\label{eq:sign_bd_M1}
       c^{(\mathcal{T}_0)}  &= c^{(t')} -\eta \sum_{\tau=t'}^{\mathcal{T}_0-1} \mathcal{G}^{(\tau)}.
   \end{align}      
  Applying \autoref{eq:GTbound} to bound the momentum gradient, we further bound \eqref{eq:sign_bd_M1} to get: 
  \begin{align}
       c^{(t')}   &=c^{(\mathcal{T}_0)} -\eta \sum_{\tau=t'}^{\mathcal{T}_0-1} \mathcal{G}^{(\tau)}\nonumber\\
       &\geq c^{(\mathcal{T}_0)} -\frac{\eta\tilde{O}(\sqrt{1-\gamma})}{\alpha} \sum_{\tau=t'}^{\mathcal{T}_0-1}\frac{1}{\gamma^{\mathcal{T}_0-\tau}}\nonumber\\
       &=c^{(\mathcal{T}_0)} -   \frac{\eta\tilde{O}(\sqrt{1-\gamma)} }{\alpha}\sum_{j=1}^{\mathcal{T}_0-t'}\gamma^{-j}\nonumber\\
       &=c^{(\mathcal{T}_0)} -   \frac{\eta\tilde{O}(\sqrt{1-\gamma})}{\alpha} \frac{1-\gamma^{\mathcal{T}_0-t'}}{1-\gamma}.\label{eq:tmp_c}
  \end{align}
  We now use the fact that $\mathcal{T}_0-t'=\frac{1}{\sqrt{1-\gamma}}$ in \eqref{eq:tmp_c} to get: 
  \begin{align}
       c^{(t')}  &\geq c^{(\mathcal{T}_0)} -   \tilde{O}(\eta) \frac{1-\gamma^{\frac{1}{\sqrt{1-\gamma}}}}{\sqrt{1-\gamma}\alpha}.\label{eq:tmp_c1}
  \end{align}
  Since $\gamma=1-\varepsilon$ with $\varepsilon\ll 1$, we linearize the right-hand side in \eqref{eq:tmp_c1} to obtain:   \begin{align}
       c^{(t')}  &\geq c^{(\mathcal{T}_0)} -   \tilde{O}(\eta) \frac{1-(1-\varepsilon)^{\frac{1}{\sqrt{\varepsilon}}}}{\sqrt{\epsilon}\alpha}\nonumber\\
       &= c^{(\mathcal{T}_0)} -   \tilde{O}(\eta) \frac{1-(1-\varepsilon)^{\frac{1}{\sqrt{\varepsilon}}}}{\sqrt{\epsilon}\alpha} \nonumber\\
       &=c^{(\mathcal{T}_0)} - \frac{\tilde{O}(\eta)}{\alpha}.\label{eq:tmp_c2}
  \end{align}

Given our choice of $\eta$, we therefore conclude that $c^{(t')} \geq 0.5\tilde{\Omega}(1/\alpha).$  
\end{proof}

Using \autoref{lem:cTGTbound}, we can therefore show that once we learn $\mathcal{Z}_1,$ $\mathcal{G}^{(t)}$ still stays large.

\grdlrg*

\begin{proof}[Proof of \autoref{lem:gradlarge}]
  By contradiction, let's assume that $ -\mathcal{G}^{(\mathcal{T}_0)} \leq  \tilde{O}(\sqrt{1-\gamma} /\alpha).$ Let's define $t':=\mathcal{T}_0-\frac{1}{\sqrt{1-\gamma}}$. 
 Since $-\mathcal{G}^{(t')}\geq 0$, $-\mathcal{G}^{(\mathcal{T}_0)}$ is bounded as:
  \begin{align}
   - \mathcal{G}^{(\mathcal{T}_0)} 
    &\geq \Theta(1-\gamma) \alpha^3 \sum_{\tau=t'}^{\mathcal{T}_0-1} \gamma^{\mathcal{T}_0-1-\tau}     (c_{r}^{(\tau)})^2.\label{eq:Gtmtpr2}
  \end{align}
 Using \autoref{lem:cTGTbound}, we bound $(c_{r}^{(\tau)})$ in \eqref{eq:Gtmtpr2} and get: 
   \begin{align}
  -  \mathcal{G}^{(\mathcal{T}_0)} 
    &\geq(1-\gamma) \tilde{\Omega}(\alpha) \sum_{\tau=t'}^{\mathcal{T}_0-1} \gamma^{\mathcal{T}_0-1-\tau}  \nonumber\\
    &= (1-\gamma) \tilde{\Omega}(\alpha) \sum_{\tau=0}^{\mathcal{T}_0-1-t'} \gamma^{j}  \nonumber\\
    &= \tilde{\Omega}(\alpha) (1-\gamma^{\mathcal{T}_0-t'}) .\nonumber\\
    &=\tilde{\Omega}(\alpha) (1-\gamma^{1/\sqrt{1-\gamma}})\label{eq:Gtmtpr4}
  \end{align}
Since $\gamma=1-\varepsilon$ with $\epsilon\ll1,$ we have $(1-\gamma^{1/\sqrt{1-\gamma}})\geq\sqrt{1-\gamma}$. Therefore, we proved that $ -\mathcal{G}^{(\mathcal{T}_0)}\geq  \tilde{\Omega}(\alpha)\sqrt{1-\gamma}$ which is a contradiction.
\end{proof}

Since the signal momentum is large (\autoref{lem:gradlarge}), we want to argue that GD+M keeps learning the feature to eventually have a large signal.

\ctlargGDM*

\begin{proof}[Proof of \autoref{lem:ct_large_M}]
Let $\mathcal{T}_1\in[T]$ such that $\mathcal{T}_0<\mathcal{T}_1.$ From the signal momentum update, we deduce: 
\begin{align}\label{eq:tmp_Gttm}
    -\mathcal{G}^{(\mathcal{T}_1)}\geq -\gamma^{\mathcal{T}_1-\mathcal{T}_0}\mathcal{G}^{(\mathcal{T}_0)}+ \sum_{\tau=\mathcal{T}_0}^{\mathcal{T}_1}\gamma^{\mathcal{T}_1-\tau }(\mathscr{G}^{(\tau)})^2\geq -\gamma^{\mathcal{T}_1-\mathcal{T}_0}\mathcal{G}^{(\mathcal{T}_0)}.
\end{align}
We now apply \autoref{lem:gradlarge} to bound  $-\mathcal{G}^{(\mathcal{T}_1)}$ in \eqref{eq:tmp_Gttm} and get: 
 \begin{align}
         -\mathcal{G}^{(\mathcal{T}_1)}\geq \gamma^{\mathcal{T}_1-\mathcal{T}_0}\tilde{\Omega}(\sqrt{1-\gamma}/\alpha).
 \end{align}
 We would like to find the time $\mathcal{T}_1$ such that $\gamma^{\mathcal{T}_1-\mathcal{T}_0}$ is a constant factor $a\leq 1$ i.e. such that 
 \begin{align}
     \gamma^{\mathcal{T}_1-\mathcal{T}_0} = a\iff \mathcal{T}_1-\mathcal{T}_0 = \frac{-\log(a)}{-\log(\gamma)}\leq \frac{\log(a)}{1-\gamma},
 \end{align}
where we used the fact that $\log(x)\leq x-1$ for $x>0$ in the last inequality. Therefore, we proved that $\mathcal{T}_1=\mathcal{T}_0+\tilde{O}(\frac{1}{1-\gamma})$ and
\begin{align}\label{eq:GT0bd}
    -\mathcal{G}^{(\mathcal{T}_1)}\geq - a \mathcal{G}^{(\mathcal{T}_0)}= \tilde{\Omega}\left( \frac{1}{\sqrt{1-\gamma} \alpha}\right).
\end{align}
From \eqref{eq:GDM_signal} update rule, we know that $c^{(\mathcal{T}_1)}=c^{(\mathcal{T}_1-1)}-\eta\mathcal{G}^{(\mathcal{T}_1)}$. Using successively $c^{(\mathcal{T}_1-1)}\geq 0$, \eqref{eq:GT0bd} and $\eta=\tilde{\Theta}(1)$, we obtain: 
\begin{align}\label{eq:ct1fbeje}
    c^{(\mathcal{T}_1)}\geq -\eta\mathcal{G}^{(\mathcal{T}_1)}\geq \tilde{\Omega}\left( \frac{\eta}{\sqrt{1-\gamma} \alpha}\right)= \tilde{\Omega}\left( \frac{1}{\sqrt{1-\gamma} \alpha}\right).
\end{align}

Let $t\in (\mathcal{T}_1,T]$. Using \eqref{eq:GDM_signal} update rule, we have
\begin{align}
    c^{(t)}&=c^{(\mathcal{T}_1)}-\eta\sum_{\tau=\mathcal{T}_1}^{t} \mathcal{G}^{(\tau)}\nonumber\\
    &\geq c^{(\mathcal{T}_1)},\label{eq:fejkjererfrf}
\end{align}
where we used the fact that $-\mathcal{G}^{(\tau)}\geq 0$ in \eqref{eq:fejkjererfrf}. Plugging \eqref{eq:ct1fbeje} in \eqref{eq:fejkjererfrf} yields the desired bound.

\end{proof}

\subsection{GD+M does not memorize}

\autoref{lem:ct_large_M} implies that after $\mathcal{T}_1$ iterations, the learnt signal is very large. We would like to show that this implies that the full derivative quickly decreases (\autoref{lem:ZderivativeM}) which implies that the GD+M cannot memorize (\autoref{lem:noise_GDM1}). Before proving \autoref{lem:ZderivativeM}, we need an auxiliary lemma that connects the signal momentum and the full derivative $\nu^{(t)}$.

\begin{lemma}[Bound on signal momentum]\label{lem:sign_mom} For $t\in [\mathcal{T}_1,T]$, the signal momentum is bounded as
  \begin{align*}
     &-\mathcal{G}^{(t + 1)} \geq -\gamma \mathcal{G}^{(t)} +  (1 - \gamma)\Omega \left( \nu^{(t)} \beta^2 \right)  
  \end{align*}
\end{lemma}

\begin{proof}[Proof of \autoref{lem:sign_mom}] From \autoref{lem:signupdateM} we know that the signal momentum is equal to
    \begin{equation}
        \begin{aligned}\label{eq:mom1_}
        -\mathcal{G}^{(t+1)}&= -\gamma \mathcal{G}^{(t)}+3(1-\gamma)\left(\alpha^3\nu_1^{(t)}   + \beta^3 \nu_2^{(t)}    \right)(c^{(t)})^2. 
        \end{aligned}
    \end{equation}

Since $\beta\leq \alpha,$ \eqref{eq:mom1_} becomes
     \begin{equation}
        \begin{aligned}\label{eq:mom2_}
        -\mathcal{G}^{(t+1)}&\geq -\gamma \mathcal{G}^{(t)}+\Theta(1)(1-\gamma)\beta^3  \nu^{(t)}  (c^{(t)})^2.  
        \end{aligned}
    \end{equation}
We finally apply \autoref{lem:ct_large_M} to bound $c^{(t)}$ in \eqref{eq:mom2_} to obtain the desired result.
\end{proof}

We now present the proof of \autoref{lem:ZderivativeM}.

\nubd*

\begin{proof}[Proof of \autoref{lem:ZderivativeM} ]   \autoref{lem:sign_mom} provides an upper bound on $\nu^{(t)}$ since: 
\begin{align}\label{eq:nuGDMbdtmp}
    \nu^{(t)}\leq \tilde{O}\left(\frac{1}{(1-\gamma)\beta }\right)(\mathcal{G}^{(t + 1)}-\gamma\mathcal{G}^{(t )}).
\end{align}
We now would like to give a convergence rate on the iterates $\mathcal{G}^{(t+1)}-\gamma\mathcal{G}^{(t)}.$ Since \autoref{thm:convrate_GDM} gives a rate on the loss function, we  connect the momentum increment with a loss term. Applying \autoref{lem:signupdateM}, we have: 
 \begin{align}
    \mathcal{G}^{(t+1)}-\gamma\mathcal{G}^{(t)}&\leq \sum_{r=1}^m |\mathcal{G}_r^{(t+1)}-\gamma\mathcal{G}_r^{(t)}|\nonumber\\
    &=\Theta(1)(1-\gamma)\sum_{r=1}^m\left((1-\hat{\mu})\alpha^3  \widehat{\ell}^{(t)}(\alpha)   + \hat{\mu}\beta^3 \widehat{\ell}^{(t)}(\beta)    \right)(c_r^{(t)})^2\label{eq:Gdiff}.
\end{align}
We  now show  that for $t\in[\mathcal{T}_1,T]$, we have:
\begin{align}\label{eq:losstalphaderiv}
    (1-\hat{\mu})\alpha^3  \widehat{\ell}^{(t)}(\alpha)\leq \hat{\mu}\beta^3 \widehat{\ell}^{(t)}(\beta)   .
\end{align}
Indeed, by using \autoref{lem:ct_large_M} and \autoref{indh:signal_mom}, we have:
\begin{align}
    (1-\hat{\mu})\alpha^3  \widehat{\ell}^{(t)}(\alpha)&\leq \frac{(1-\hat{\mu}) \alpha^3 }{1+\exp(\tilde{\Omega}(\alpha^3/\beta^3))},\label{eq:efihheiec}\\
    \hat{\mu}\beta^3  \widehat{\ell}^{(t)}(\beta)&\geq \frac{\hat{\mu}\beta^3}{1+\exp(\tilde{O}(1))}.\label{eq:everewcdwe}
\end{align}
Thus, combining \eqref{eq:efihheiec} and \eqref{eq:everewcdwe} yields:
\begin{align}
    \frac{(1-\hat{\mu})\alpha^3  \widehat{\ell}^{(t)}(\alpha)}{\hat{\mu}\beta^3  \widehat{\ell}^{(t)}(\beta)}\leq \frac{(1-\hat{\mu})\alpha^3}{ \hat{\mu}\beta^3} \frac{1+\exp(\tilde{O}(1))}{1+\exp(\tilde{\Omega}(\alpha^3/\beta^3))}.\label{eq:kncerekcrw}
\end{align}
Given our choice of  $\alpha$, $\beta$ and $\hat{\mu}$, we finally bound \eqref{eq:kncerekcrw} as:
\begin{align}
     \frac{(1-\hat{\mu})\alpha^3  \widehat{\ell}^{(t)}(\alpha)}{\hat{\mu}\beta^3  \widehat{\ell}^{(t)}(\beta)}\leq 1.
\end{align}
Therefore, plugging \eqref{eq:losstalphaderiv} in
\eqref{eq:Gdiff} yields:
\begin{align}
    \mathcal{G}^{(t+1)}-\gamma\mathcal{G}^{(t)}&\leq 2\Theta(1-\gamma)\sum_{r=1}^m \hat{\mu}\beta^3 \widehat{\ell}^{(t)}(\beta)  (c_r^{(t)})^2\label{eq:Gdifffvdvfoej}.
\end{align}
We now apply \autoref{lem:masterlogsigmdelta} to link \eqref{eq:Gdifffvdvfoej} with a loss term. By \autoref{lem:ct_large_M} and \autoref{indh:signal_mom}, we have: 
\begin{align}\label{eq:hifccheic}
  \tilde{\Omega}(1) \leq\tilde{\Omega}(1)-m\tilde{O}(\sigma_0)  & \leq \sum_{r=1}^m \beta c_r^{(t)}\leq \tilde{O}(m)\leq \tilde{O}(1).
\end{align}
Therefore, applying \autoref{lem:masterlogsigmdelta} in  \eqref{eq:Gdifffvdvfoej} gives:
\begin{align}
    \mathcal{G}^{(t+1)}-\gamma\mathcal{G}^{(t)}&\leq 40\hat{\mu}\Theta(1-\gamma) \frac{m \beta e^{m \tilde{O}(\sigma_0)}}{\tilde{\Omega}(1)} \widehat{\mathcal{L}}^{(t)}(\beta)\leq \tilde{O}(\beta)\hat{\mu}(1-\gamma)\widehat{\mathcal{L}}^{(t)}(\beta).\label{eq:hfeihrirew}
\end{align}
Thus, plugging \eqref{eq:hfeihrirew} in \eqref{eq:nuGDMbdtmp} yields:
\begin{align}\label{eq:ocjdjoejoew}
    \nu^{(t)}\leq \tilde{O}(1)\hat{\mu}\widehat{\mathcal{L}}^{(t)}(\beta).
\end{align}
 
We finally apply \autoref{thm:convrate_GDM} to bound the loss term in \eqref{eq:ocjdjoejoew}  and get the desired result.
 \end{proof}

After $\mathcal{T}_1$ iterations, the gradient is now very small and the noise component learnt by GD+M stays very small.

\lemnoisGDM*

\begin{proof}[Proof of \autoref{lem:noise_GDM1} ] This Lemma is intended to prove \autoref{indh:xi_mom}. At time $t=0,$ we have $|\Xi_{i,j,r}^{(0)}|\leq \tilde{O}(\sigma\sigma_0\sqrt{d})$ by \autoref{prop:dotprodGauss}. Assume that \autoref{indh:xi_mom} is true for $t\in [\mathcal{T}_1,T).$ Now, let's prove this induction hypothesis for time $t+1.$ For $s\in[\mathcal{T}_1,t]$, we remind that \eqref{eq:GDM_noise} update rule is
\begin{align}\label{eq:noiseupdGDM}
    \Xi_{i,j,r}^{(s+1)}&=\Xi_{i,j,r}^{(s)}-\eta G_{i,j,r}^{(s+1)}.
\end{align}
We sum up \eqref{eq:noiseupdGDM} for $s=\mathcal{T}_1,\dots,t$ and obtain:
\begin{align}\label{eq:xiijrtnoisemom}
    \Xi_{i,j,r}^{(t+1)}&=\Xi_{i,j,r}^{(\mathcal{T}_1)}-\eta \sum_{s=\mathcal{T}_1}^t G_{i,j,r}^{(s+1)}.
\end{align}
We apply the triangle inequality in \eqref{eq:xiijrtnoisemom} and obtain:
\begin{align}\label{eq:evfeeaw}
    |\Xi_{i,j,r}^{(t+1)}|&\leq|\Xi_{i,j,r}^{(\mathcal{T}_1)}|+ \eta  \sum_{s=\mathcal{T}_1}^t |G_{i,j,r}^{(s+1)}| .
\end{align}
We now use \autoref{indh:xi_mom} to bound $|\Xi_{i,j,r}^{(\mathcal{T}_1)}|$ in \eqref{eq:evfeeaw}:
\begin{align}\label{eq:efiviihh vb1}
    |\Xi_{i,j,r}^{(t+1)}|&\leq \tilde{O}(\sigma\sigma_0\sqrt{d}) + \eta  \sum_{s=\mathcal{T}_1}^t |G_{i,j,r}^{(s+1)}|.
\end{align}
We now plug the bound on $\sum_{s=\mathcal{T}_1}^t |G_{i,j,r}^{(s+1)}|$ given by \autoref{lem:noisegradlate} and obtain:
\begin{align}\label{eq:efiviihh vb}
    |\Xi_{i,j,r}^{(t+1)}|&\leq \tilde{O}(\sigma\sigma_0\sqrt{d}) +    \tilde{O}(\sigma^4\sigma_0^2d^{2}) \left(\eta\mathcal{T}_1 +\frac{1}{\beta}\right).
\end{align}

Given the values of $\mathcal{T}_1,$ $\eta$, $\gamma$ and $\beta$, we can deduce that 
\begin{align}\label{eq:ierihafce}
    \tilde{O}(\sigma^4\sigma_0^2d^{2})\left(\eta\mathcal{T}_1+\frac{1}{\beta}\right) \leq \sigma\sigma_0\sqrt{d}.
\end{align}
Plugging \eqref{eq:ierihafce} in \eqref{eq:efiviihh vb} proves the induction hypothesis for $t+1.$

\end{proof}


\subsection{Proof of \autoref{thm:GDM}}

We proved that the weights learnt by GD+M satisfy for $r\in[m]$
\begin{align}\label{eq:wrtttGDM}
    \bm{w}_r^{(T)} = c_r^{(T)}\bm{w^*} +\bm{v}_r^{(T)},
\end{align}
where at least one of the $c_r^{(T)}\geq \tilde{\Omega}(1/\beta)$ (\autoref{lem:ct_large_M}) and $\bm{v}_r^{(T)}\in \mathrm{span}(\bm{X}_i[j])\subset \mathrm{span}(\bm{w^*})^{\perp}.$ By \autoref{lem:noise_GDM1}, since $\Xi_{i,j,r}^{(t)}\leq \tilde{O}(\sigma_0)$, we have $\|\bm{v}_r^{(T)}\|_2\leq 1.$ We are now ready to prove the generalization achieved by GD+M and stated in  \autoref{thm:GDM}.

\GDMmain* 

\begin{proof}[Proof of \autoref{thm:GDM}] We now bound the training and test error achieved by GD+M at time $T.$

\textbf{Train error.} \autoref{thm:convrate_GDM} provides a convergence bound on the fake loss. Indeed, we know that: 
\begin{align}\label{eq:efjcefrn}
     (1-\hat{\mu})\widehat{\mathcal{L}}^{(T)}(\alpha) + \hat{\mu}\widehat{\mathcal{L}}^{(T)}(\beta)\leq \tilde{O}\left( \frac{1}{\eta (T-\mathcal{T}_1+1)}\right).
\end{align}
Using \autoref{lem:logfzd} along with \autoref{indh:xi_mom}, we  lower bound the loss term in \eqref{eq:efjcefrn} by the true loss.
\begin{align}\label{eq:feiceewd}
   \Theta(1)\widehat{L}(W^{(T)}) \leq (1-\hat{\mu})\widehat{\mathcal{L}}^{(T)}(\alpha) + \hat{\mu}\widehat{\mathcal{L}}^{(T)}(\beta).
\end{align}
Combining \eqref{eq:efjcefrn} and \eqref{eq:feiceewd}, we obtain a bound on the training loss. 
\begin{align}\label{eq:lossWTGDM}
    \widehat{L}(W^{(T)})&\leq \frac{\tilde{O}(1)}{\eta (T-\mathcal{T}_1+1)}.
\end{align}
Plugging $T\geq \mathrm{poly}(d)N/\eta$ and $\mu =\Theta(1/N)$ in \eqref{eq:lossWTGDM} yields:
\begin{align}
     \widehat{L}(W^{(T)})&\leq\tilde{O}\left( \frac{1}{\mathrm{poly}(d)N}\right)\leq \tilde{O}\left( \frac{\mu}{\mathrm{poly}(d)}\right).
\end{align}

\textbf{Test error.} Let $(\bm{X},y)$ be a datapoint. We remind that $\bm{X}=(\bm{X}[1],\dots,\bm{X}[P])$ where $\bm{X}[P(\bm{X})] =\theta y\bm{w^*}$ and $\bm{X}[j]\sim\mathcal{N}(0,\sigma^2\mathbf{I}_d)$ for $j\in[P]\backslash \{P(\bm{X})\}.$  We bound the test error as follows: 
\begin{align}
   \mathscr{L}(f_{\bm{W}^T})&= \mathbb{E}_{\substack{\mathcal{Z}\sim\mathcal{D}\\
    (\bm{X},y)\sim \mathcal{Z}}}[\mathbf{1}_{yf_{\bm{W}^T}(X)<0}] \nonumber \\
    &=\mathbb{E}_{ 
    (\bm{X},y)\sim \mathcal{Z}_1}[\mathbf{1}_{yf_{\bm{W}^T}(\bm{X})<0}]\mathbb{P}[\mathcal{Z}_1] + \mathbb{E}_{ 
    (\bm{X},y)\sim \mathcal{Z}_2}[\mathbf{1}_{yf_{\bm{W}^T}(X)<0}]\mathbb{P}[\mathcal{Z}_2]\nonumber\\
    &=(1-\hat{\mu})\mathbb{P}[yf_{\bm{W}^T}(\bm{X})<0|(\bm{X},y)\sim\mathcal{Z}_1]\label{eq:yferfkre}\\ &+\hat{\mu}\mathbb{P}[yf_{\bm{W}^T}(\bm{X})<0|(\bm{X},y)\sim\mathcal{Z}_2].\label{eq:yf}
\end{align}
We now want to compute the probability terms in \eqref{eq:yferfkre} and \eqref{eq:yf}. We remind that $yf_{\bm{W}^T}(\bm{X})$ is given by 
\begin{align}
    yf_{W^T}(\bm{X})  &= y\sum_{s=1}^m\sum_{j=1}^P \langle \bm{w}_s^{(T)}, \bm{X}[j]\rangle^3 \nonumber\\
    &= \theta^3 \sum_{s=1}^m (c_s^{(T)})^3+y\sum_{s=1}^m\sum_{j\neq P(\bm{X})} \langle  \bm{v}_s^{(T)}, \bm{X}[j]\rangle^3\nonumber\\
    &\geq \theta^3 (c^{(T)})^3 + y\sum_{s=1}^m\sum_{j\neq P(\bm{X})} \langle  \bm{v}_s^{(T)}, \bm{X}[j]\rangle^3.\label{eq:yf1st}
\end{align}
We now apply \autoref{lem:ct_large_M}, \eqref{eq:yf1st} is finally bounded as: 
\begin{align}\label{eq:bdfYGDM}
    yf(\bm{X})\geq \Omega\left(\frac{\theta^3}{\beta^3} \right)+ y\sum_{s=1}^m\sum_{j\neq P(\bm{X})} \langle  \bm{v}_s^{(T)}, \bm{X}[j]\rangle^3.
\end{align}

Therefore, using \eqref{eq:bdfYGDM}, we upper bound the test error \eqref{eq:yf} as:
\begin{equation}
    \begin{aligned}\label{eq:tmp_class_err_GDM}
    \mathscr{L}(f_{\bm{W}^T})&\leq (1-\hat{\mu})\mathbb{P}\left[ y\sum_{s=1}^m\sum_{j\neq P(\bm{X})} \langle  \bm{v}_s^{(T)}, \bm{X}[j]\rangle^3\leq -\tilde{\Omega}\left(\frac{\alpha^3}{\beta^3}\right) \right]\\
    &+ \hat{\mu}\mathbb{P}\left[ y\sum_{s=1}^m\sum_{j\neq P(\bm{X})} \langle  \bm{v}_s^{(T)}, \bm{X}[j]\rangle^3\leq -\tilde{\Omega}(1) \right].
\end{aligned}
\end{equation}
Since $y$ is uniformly sampled from $\{-1,1\},$ we further simplify \eqref{eq:tmp_class_err_GDM} as: 
\begin{equation}
    \begin{aligned}
   \mathscr{L}(f_{\bm{W}^T})
   &\leq \frac{1-\hat{\mu}}{2}\mathbb{P}\left[ \sum_{s=1}^m\sum_{j\neq P(X)} \langle  \bm{v}_s^{(T)}, \bm{X}[j]\rangle^3\leq -\tilde{\Omega}\left(\frac{\alpha^3}{\beta^3}\right) \right]\\
   &+\frac{1-\hat{\mu}}{2}\mathbb{P}\left[ \sum_{s=1}^m\sum_{j\neq P(X)} \langle  \bm{v}_s^{(T)}, \bm{X}[j]\rangle^3\geq \tilde{\Omega}\left(\frac{\alpha^3}{\beta^3}\right) \right] \\
    &+ \frac{\hat{\mu}}{2}\mathbb{P}\left[ \sum_{s=1}^m\sum_{j\neq P(\bm{X})} \langle  \bm{v}_s^{(T)},\bm{X}[j]\rangle^3\leq -\tilde{\Omega}(1) \right]\\
    &+\frac{\hat{\mu}}{2}\mathbb{P}\left[ \sum_{s=1}^m\sum_{j\neq P(X)} \langle  \bm{v}_s^{(T)}, \bm{X}[j]\rangle^3\geq \tilde{\Omega}(1) \right] .\label{eq:tmp_class_err_GDM1}
\end{aligned}
\end{equation}

We know that $\langle \bm{v}_s^{(T)},\bm{X}[j]\rangle \sim \mathcal{N}(0,\|\bm{v}_s^{(T)}\|_2^2\sigma^2)$. Therefore, $\langle \bm{v}_s^{(T)},\bm{X}[j]\rangle^3$ is the cube of a centered Gaussian.
This random variable is symmetric. By \autoref{lem:sym_rv}, we know that $\sum_{s=1}^m\sum_{j\neq P(\bm{X})} \langle  \bm{v}_s^{(T)}, \bm{X}[j]\rangle^3$ is also symmetric. Therefore, we simplify \eqref{eq:tmp_class_err_GDM1} as: 
\begin{equation}
    \begin{aligned}\label{eq:tmp_class_err_GDM12}
     \mathscr{L}(f_{\bm{W}^T}) &\leq (1-\hat{\mu})\mathbb{P}\left[ \sum_{s=1}^m\sum_{j\neq P(\bm{X})} \langle  \bm{v}_s^{(T)}, X[j]\rangle^3\geq \tilde{\Omega}\left(\frac{\alpha^3}{\beta^3}\right)\right]\\
    &+  \hat{\mu}  \mathbb{P}\left[ \sum_{s=1}^m\sum_{j\neq P(\bm{X})} \langle  \bm{v}_s^{(T)}, \bm{X}[j]\rangle^3\geq \tilde{\Omega}(1) \right] .
\end{aligned}
\end{equation}

From \autoref{lem:sumsubG2}, we know that $\sum_{s=1}^m\sum_{j\neq p} \langle  \bm{v}_s^{(T)}, \bm{X}[j]\rangle^3$ is $\sigma^3\sqrt{P-1}\sqrt{\sum_{s=1}^m\|\bm{v}_s^{(T)}\|_2^6}$-subGaussian. Therefore, by applying \autoref{eq:subGhighbd}, \eqref{eq:tmp_class_err_GDM12} is further bounded by: 
\begin{equation}
    \begin{aligned}\label{eq:tmp_class_err_GDM122}
    \mathscr{L}(f_{\bm{W}^T})&\leq 2(1-\mu) \exp\left(-\tilde{\Omega}\left(\frac{\alpha^6}{\beta^6}\right)\frac{1}{\sigma^6\sum_{s=1}^m\|\bm{v}_s^{(T)}\|_2^6}\right)\\
    &+  2\mu\exp\left(-\frac{\tilde{\Omega}(1)}{\sigma^6\sum_{s=1}^m\|\bm{v}_s^{(T)}\|_2^6}\right)    .
\end{aligned}
\end{equation}
Using the fact that $\|\bm{v}_s^{(T)}\|_2\leq 1$ in \eqref{eq:tmp_class_err_GDM122} finally yields: 
\begin{equation}
    \begin{aligned}\label{eq:finGDM}
    \mathscr{L}(f_{\bm{W}^T})&\leq 2(1-\mu) \exp\left(-\tilde{\Omega}\left(\frac{\alpha^6}{\beta^6\sigma^6}\right)\right)+  2\mu\exp\left(-\tilde{\Omega}\left(\frac{1}{\sigma^6 }\right)\right)    .
\end{aligned}
\end{equation}
Since $\exp(-\alpha^6/(\beta^6\sigma^6))\leq \mu/\mathrm{poly}(d)$ and $\exp(-\tilde{\Omega}(1/\sigma^6))\leq 1/\mathrm{poly(d)}$ , we obtain the desired result.

\end{proof}

\subsection{Proof of the GD+M induction hypotheses}

\begin{proof}[Proof of \autoref{indh:signal_mom}] We prove the induction hypotheses for the signal $c_r^{(t)}.$

\paragraph{Proof of $c_r^{(t+1)}\geq -\tilde{O}(\sigma_0)$.} We know that with high probability, $c_r^{(0)}\geq -\tilde{O}(\sigma_0)$. By \autoref{lem:signupdate}, $c_r^{(t)}$ is a non-decreasing sequence and therefore, we always have $c_r^{(t)}\geq-\tilde{O}(\sigma_0). $

\paragraph{Proof of $c_r^{(t+1)}\leq \tilde{O}(1/\beta)$.} Assume that \autoref{indh:xi_mom} is true for $t\in [\mathcal{T}_1,T).$ Now, let's prove this induction hypothesis for time $t+1.$ For $\tau\in[t]$, we remind that \eqref{eq:GDM_signal} update rule is
\begin{align}\label{eq:noiseupdGDM12}
    c_{r}^{(\tau+1)}&=c_{r}^{(\tau)}-\eta \mathcal{G}_{r}^{(\tau+1)}.
\end{align}
We sum up \eqref{eq:noiseupdGDM12} for $\tau=\mathcal{T}_1,\dots,t$ and obtain:
\begin{align}\label{eq:xifvdfvds ijrtnoisemom}
   c_{r}^{(t+1)}&=c_{r}^{(\mathcal{T}_1)}-\eta \sum_{\tau=\mathcal{T}_1}^t \mathcal{G}_{r}^{(\tau+1)}.
\end{align}
We apply the triangle inequality in \eqref{eq:xifvdfvds ijrtnoisemom} and obtain:
\begin{align}\label{eq:evfeekdsvsdvkaw}
    |c_{r}^{(t+1)}|&\leq|c_{r}^{(\mathcal{T}_1)}|+ \eta  \sum_{\tau=\mathcal{T}_1}^t |\mathcal{G}_{r}^{(\tau+1)}| .
\end{align}
We now use \autoref{indh:signal_mom} to bound $|c_{r}^{(\mathcal{T}_1)}|$ in \eqref{eq:evfeekdsvsdvkaw}:
\begin{align}\label{eq:efiviihh vbsddcds}
    |c_{r}^{(t+1)}|&\leq \tilde{O}(1/\beta) + \eta  \sum_{\tau=\mathcal{T}_1}^t |\mathcal{G}_{r}^{(\tau+1)}|.
\end{align}
We now plug the bound on $\sum_{\tau=\mathcal{T}_1}^t |\mathcal{G}_{r}^{(\tau+1)}|$ given by \autoref{lem:latestgmome}. We have:
\begin{align}
    |c_{r}^{(t+1)}|&\leq \tilde{O}(1/\beta) +\tilde{O}(\eta\alpha\mathcal{T}_0 ) + \tilde{O}(\eta\hat{\mu}\beta\mathcal{T}_1) +\tilde{O}(1)
    \leq \tilde{O}(1/\beta),
    \end{align}
where we used $\tilde{O}(\eta\alpha\mathcal{T}_0 ) + \tilde{O}(\eta\hat{\mu}\beta\mathcal{T}_1) +\tilde{O}(1)\leq 1/\beta.$ This proves the induction hypothesis for $t+1.$
\end{proof}

\section{Extension to $\lambda > 0$}\label{sec:ext}

Now we discuss how to extend the result to $\lambda > 0$. In our result, since $\lambda = \frac{1}{N \mathrm{poly}(d)}$, we know that before $T = \tilde{\Theta}\left(\frac{1}{\eta \lambda} \right)$ iterations, the weight decay would not affect the learning process and we can show everything similarly. 

After iteration $T$, by Lemma~\ref{thm:convrate} and Lemma~\ref{thm:convrate_GDM}, we know that for GD:
$$\nu^{(t)} \leq \tilde{O} \left( \lambda \right)$$

and for GD + M:
$$\nu^{(t)} \leq \tilde{O} \left( \lambda/(\beta^2) \right)$$

For GD, we just need to maintain that $c^{(t)} = \tilde{O}(1/\alpha)$ and $\Xi_i^{(t)} = \tilde{\Omega}(1)$. To see this, we know that if $c^{(t)} = \tilde{\Omega}(1/\alpha)$, then 
$$c^{(t + 1)} \leq (1 - \eta \lambda) c^{(t)} + \eta \tilde{O}\left( \nu^{(t)} \frac{\beta^3}{\alpha^2} \right) \leq c^{(t)}$$

To show that $\Xi_i^{(t)} = \tilde{\Omega}(1)$, assuming that $\Xi_i^{(t)} = 1/\mathrm{polylog}(d)$, we know that
$$ \Xi_i^{(t + 1)} \geq (1 - \eta \lambda )  \Xi_i^{(t)}  + \tilde{\Omega}\left(\eta \frac{1}{N} \right) \geq \Xi_i^{(t)}  + \tilde{\Omega}\left(\eta \frac{1}{N} \right) $$

Similarly, for GD + M, since $\nu^{(t)} \leq \tilde{O} \left( \lambda/(\beta^2) \right)$, we know that 
$$\nabla \hat{L}(W^{(t)} ) \leq  \tilde{O} \left( \lambda \alpha^3 /(\beta^2) \right)$$

This implies that
$$\|W^{(t + 1)}  - W^{(t)} \|_2 \leq  \tilde{O} \left(\eta \lambda \alpha^3 /(\beta^2) \right)$$

We need to show that $c^{(t)} = \tilde{\Omega}(1/\beta)$ and all $|\Xi_{i, j, r}^{(t)} | \leq \tilde{O}(\sigma_0 \sigma \sqrt{d})$. To see this, we know that when $c^{(t)} = \Theta \left( \frac{1}{\beta} \right)$, we know that $c^{(t - t_0)} =  \Theta \left( \frac{1}{\beta} \right)$ for every $t_0 \leq \frac{1}{\gamma}$. This implies that
$$c^{(t + 1)} \geq c^{(t)} - O\left(\eta \lambda \frac{1}{\beta} \right)+ \Omega\left( \frac{\eta}{N} \beta \right) \geq  c^{(t)}+ \Omega\left( \frac{\eta}{N} \beta \right) $$

On the other hand, for $\Xi_{i, j, r}^{(t)} $ we know that:
$$|\Xi_{i, j, r}^{(t + 1)}  |\leq  (1 - \eta \lambda) |\Xi_{i, j, r}^{(t)}|  + \tilde{O}\left(\eta \nu^{(t)} \sigma_0^2 (\sigma \sqrt{d})^2 \right) \leq \tilde{O}(\sigma_0 \sigma \sqrt{d}) $$
\section{Technical lemmas for GD}

This section presents the technical lemmas needed in \autoref{sec:app_GD}. These lemmas mainly consists in different rewritings of GD.
 
\subsection{Rewriting derivatives}

Using \autoref{indh:noiseGD} and \autoref{indh:signGD}, we rewrite the sigmoid terms $\ell_i^{(t)}$ when using GD.

\begin{lemma}[$\mathcal{Z}_1$ derivative]\label{eq:derZ1_bd}
Let $i\in\mathcal{Z}_1.$ We have $\ell_i^{(t)}= \Theta(1)\widehat{\ell}^{(t)}(\alpha).$
\end{lemma}

\begin{proof}[Proof of \autoref{eq:derZ1_bd}]
 Let $i\in\mathcal{Z}_1$. Using \autoref{indh:noiseGD}, we bound $\ell_i^{(t)}$ as
\begin{align}
    \frac{1}{1+\exp\left( \alpha^3 \sum_{s=1}^m (c_s^{(t)})^{3} +\tilde{O}((\sigma\sigma_0\sqrt{d})^3) \right)} &\leq \ell_i^{(t)} \leq \frac{1}{1+\exp\left( \alpha^3 \sum_{s=1}^m (c_s^{(t)})^{3} -\tilde{O}((\sigma\sigma_0\sqrt{d})^3) \right)}  \nonumber\\
   \iff e^{-\tilde{O}((\sigma\sigma_0\sqrt{d})^3)}\widehat{\ell}^{(t)}(\alpha)&\leq \ell_i^{(t)}\leq e^{\tilde{O}((\sigma\sigma_0\sqrt{d})^3)}\widehat{\ell}^{(t)}(\alpha).\label{eq:elliZ1_bd}
\end{align}
\eqref{eq:elliZ1_bd} yields the aimed result.
\end{proof}

\begin{lemma}[$\mathcal{Z}_2$ derivative]\label{eq:derZ2_bd}
Let $i\in\mathcal{Z}_2.$ We have $\ell_i^{(t)}=\Theta(1)\widehat{\ell}^{(t)}(\Xi_i^{(t)}).$
\end{lemma}

\begin{proof}
 Let $i\in\mathcal{Z}_2$. Using \autoref{indh:signGD}, we bound $\ell_i^{(t)}$ as
\begin{align}
    \frac{1}{1+\exp\left( \tilde{O}(\beta^3/\alpha^3)+\Xi_{i}^{(t)}\right)}&\leq\ell_i^{(t)}\leq \frac{1}{1+\exp\left( -\tilde{O}(\beta^3\sigma_0^3) +\Xi_{i}^{(t)}\right)}\nonumber\\
     \iff e^{-\tilde{O}(\beta^3/\alpha^3)}\widehat{\ell}^{(t)}(\Xi_i^{(t)})&\leq \ell_i^{(t)}\leq e^{\tilde{O}(\beta^3\sigma_0^3)}\widehat{\ell}^{(t)}(\Xi_i^{(t)}).\label{eq:elliZ2_bd} 
\end{align}
 \eqref{eq:elliZ2_bd} yields the aimed result.
\end{proof}

\subsection{Signal  lemmas}

In this section, we present a lemma that bounds the sum over time of the GD increment. 

\begin{lemma}\label{prop:bd_nu1sum}  
Let $t,\mathscr{T}\in [T]$ such that $\mathscr{T}<t.$ Then, the $\mathcal{Z}_1$ derivative is bounded as:
\begin{align*}
  \sum_{\tau=\mathscr{T}}^{t}\nu_1^{(\tau)}\min\{\kappa,\alpha^2 (c^{(\tau)})^2\}&\leq \tilde{O}\left(\frac{1}{\eta  \alpha^2}\right) +\tilde{O}\left(\frac{\beta^3}{\alpha^2}\right)\sum_{\tau=\mathscr{T}}^t \nu_2^{(\tau)}.
\end{align*}
\end{lemma}
\begin{proof}[Proof of \autoref{prop:bd_nu1sum}] From \autoref{lem:lateiter}, we know that:
\begin{align}\label{eqctlilt2}
    c^{(t+1)}&\geq c^{(t)} +  \tilde{\Theta}(\eta \alpha  )\nu_1^{(t)}\min\{\kappa,\alpha^2 (c^{(t)})^2\}
\end{align}
Let $\mathscr{T},t\in [T]$ such that $\mathscr{T}<t.$ We now sum up \eqref{eqctlilt2} for $\tau=\mathscr{T},\dots,t$ and get: 
\begin{align}\label{eqctlilt3}
    \sum_{\tau=\mathscr{T}}^{t}\nu_1^{(\tau)}\min\{\kappa,\alpha^2 (c^{(\tau)})^2\}&\leq \tilde{O}\left(\frac{1}{\eta  \alpha}\right)(c^{(t+1)}-c^{(\mathscr{T})}).
\end{align}
We now consider two cases.  
\paragraph{Case 1: $t< T_0.$} By definition, we know that $c^{(t)}\leq \tilde{O}(1/\alpha).$ Therefore, \eqref{eqctlilt3} yields:
\begin{align}\label{eqctlilt4}
    \sum_{\tau=\mathscr{T}}^{t}\nu_1^{(\tau)}\min\{\kappa,\alpha^2 (c^{(\tau)})^2\}&\leq \tilde{O}\left(\frac{1}{\eta \alpha^2}\right).
\end{align}
\paragraph{Case 2: $t\in [T_0,T].$} We distinguish two subcases.
\begin{itemize}
    \item[--] \textbf{Subcase 1}: $\mathscr{T}<T_0.$ From \autoref{eq:ct_update}, we know that:
    \begin{align}\label{eq:jrfnedne}
        c^{(t+1)}&\leq \tilde{O}(1/\alpha)+\tilde{O}(\eta\beta^3/\alpha^2)\sum_{\tau=T_0}^t \nu_2^{(\tau)}.
    \end{align}
    We can further bound \eqref{eq:jrfnedne} as: 
     \begin{align}\label{eq:jrfnvfeceredne}
        c^{(t+1)}&\leq \tilde{O}(1/\alpha)+\tilde{O}(\eta\beta^3/\alpha^2)\sum_{\tau=\mathscr{T}}^t \nu_2^{(\tau)},
    \end{align}
    which combined with \eqref{eqctlilt3} implies: 
    \begin{align}\label{eq:frnenc}
        \sum_{\tau=\mathscr{T}}^{t}\nu_1^{(\tau)}\min\{\kappa,\alpha^2 (c^{(\tau)})^2\}&\leq \tilde{O}\left(\frac{1}{\eta  \alpha^2}\right) +\tilde{O}\left(\frac{\beta^3}{\alpha^2}\right)\sum_{\tau=\mathscr{T}}^t \nu_2^{(\tau)}
    \end{align}
     \item[--] \textbf{Subcase 2}: $\mathscr{T}>T_0.$ From \autoref{eq:ct_update}, we know that:
      \begin{align}\label{eq:jrerferc}
        c^{(t+1)}&\leq \tilde{O}(1/\alpha)+\tilde{O}(\eta\beta^3/\alpha^2)\sum_{\tau=\mathcal{T}}^t \nu_2^{(\tau)},
    \end{align}
     which combined with \eqref{eqctlilt3} yields \eqref{eq:frnenc}.
    
\end{itemize}

We therefore managed to prove that in all the cases, \eqref{eq:frnenc} holds.

\end{proof}

\subsection{Noise lemmas}

In this section, we present the technical lemmas needed in \autoref{sec:mem_gd}. The following lemma bounds the projection of the GD increment on the noise. 

\begin{lemma}\label{lem:indhnoisesum}
 Let $i\in [N]$, $j\in [P]\backslash \{ P(\bm{X}_i)\}$ and $r\in [m]$. Let $\mathscr{T},t\in [T]$ such that $\mathscr{T} < t.$ Then, the noise update \eqref{eq:GD_noise} satisfies
    \begin{equation*}
   \begin{split}
      \left| y_i(\Xi_{i,j,r}^{(t)} - \Xi_{i,j,r}^{(\mathscr{T})})  -\frac{\tilde{\Theta}(\eta\sigma^2 d)}{N}\sum_{\tau=\mathscr{T}}^{t-1}\ell_i^{(\tau)}(\Xi_{i,j,r}^{(\tau)})^2 \right| &\leq  \tilde{O}\left(\frac{P\sigma^2\sqrt{d}}{\alpha}\right) +   \tilde{O}\left(\frac{\eta\beta^3}{\alpha  } \right)  \sum_{j=\mathscr{T}}^{t-1}\nu_2^{(j)} .
   \end{split}
   \end{equation*}
\end{lemma}
\begin{proof}[Proof of \autoref{lem:indhnoisesum}]   Let $i\in [N]$, $j\in [P]\backslash \{ P(\bm{X}_i)\}$ and $r\in [m]$. We set up the following induction hypothesis: 
  \begin{equation}
\begin{split}\label{eq:noiseindhypoth}
    & \left| y_i\Xi_{i,j,r}^{(t)} - y_i\Xi_{i,j,r}^{(\mathscr{T})}  -\frac{\tilde{\Theta}(\eta\sigma^2 d)}{N}\sum_{\tau=\mathscr{T}}^{t-1}\ell_i^{(\tau)}(\Xi_{i,j,r}^{(\tau)})^2 \right|\\
    &\leq  
     \Tilde{O}\left(\frac{P\sigma^2\sqrt{d}}{\alpha} \left(1+\frac{\alpha}{\sigma^2 d}+\frac{\alpha \eta}{N}\right)\sum_{\tau=0}^{t-1-\mathscr{T}} \frac{P^{\tau}}{d^{\tau/2}} \right)\\
     &+    \tilde{O}\left(\frac{\eta\beta^3}{\alpha^2  } \right) \sum_{\tau=0}^{t-1-\mathscr{T}} \frac{P^{\tau }}{d^{\tau/2}}\sum_{j=\mathscr{T}}^{t-\tau}\nu_2^{(j)} ,
\end{split}
\end{equation}

Let's first show this hypothesis for $t=\mathscr{T}.$ From \autoref{lem:noisupdate}, we have:
\begin{equation}\label{eq:noisupfnjdeqn}
    \begin{split}
       &\left|y_i (\Xi_{i,j,r}^{(\mathscr{T}+1)} -  \Xi_{i,j,r}^{(\mathscr{T})}) -\frac{\Tilde{\Theta}(\eta\sigma^2 d)}{N} \ell_i^{(\mathscr{T})}(\Xi_{i,j,r}^{(\mathscr{T})})^2\right|\\
       &\leq   \frac{\Tilde{\Theta}(\eta\sigma^2 \sqrt{d})}{N} \sum_{a\in \mathcal{Z}_2} \sum_{k\neq P(\bm{X}_a)} \ell_a^{(\mathscr{T})}(\Xi_{a,k,r}^{(\mathscr{T})})^2  \\
     &+   \frac{\Tilde{\Theta}(\eta\sigma^2 \sqrt{d})}{N} \sum_{a\in \mathcal{Z}_1} \sum_{k\neq P(\bm{X}_a)} \ell_a^{(\mathscr{T})}(\Xi_{a,k,r}^{(\mathscr{T})})^2.
    \end{split}
\end{equation}
Now, we apply \autoref{indh:maxnoisesign} to bound $(\Xi_{a,k,r}^{(\mathscr{T})})^2$ in \eqref{eq:noisupfnjdeqn} and obtain: 
\begin{equation}
    \begin{split}
      & \left|y_i \Xi_{i,j,r}^{(\mathscr{T}+1)} -  y_i \Xi_{i,j,r}^{(\mathscr{T})} -\frac{\Tilde{\Theta}(\eta\sigma^2 d)}{N} \ell_i^{(\mathscr{T})}(\Xi_{i,j,r}^{(\mathscr{T})})^2\right|\\
      &\leq  \Tilde{\Theta}(\eta P\sigma^2 \sqrt{d})  \nu_2^{(\mathscr{T})}\min\{\kappa,(c^{(\mathscr{T})})^2\alpha^2\}\alpha \\
    & + \Tilde{\Theta}(\eta P\sigma^2\sqrt{d})\nu_1^{(\mathscr{T})} \min\{\kappa,(c^{(\mathscr{T})})^2\alpha^2\}\alpha.\label{eq:yixiixide}
    \end{split}
\end{equation}
We successively apply \autoref{prop:bd_nu1sum}, use $\nu_2^{(\mathscr{T})}\min\{\kappa,(c^{(\mathscr{T})})^2\alpha^2\}\alpha\leq \hat{\mu} \tilde{O}(1)\leq \tilde{O}(\hat{\mu})$ and $\hat{\mu}=\Theta(1/N)$ in \eqref{eq:yixiixide} to finally obtain: 
\begin{align*}
    \left|y_i \Xi_{i,j,r}^{(\mathscr{T}+1)}- y_i \Xi_{i,j,r}^{(\mathscr{T})} -\frac{\Tilde{\Theta}(\eta\sigma^2 d)}{N}\ell_i^{(\mathscr{T})}(\Xi_{i,j,r}^{(\mathscr{T})})^2\right|&\leq  \tilde{O}\left(\frac{P\sigma^2\sqrt{d}}{\alpha}\left(1+\frac{\eta\alpha}{N}\right)  \right).
\end{align*}
Therefore, the induction hypothesis is verified for $t=\mathscr{T}.$ Now, assume \eqref{eq:noiseindhypoth} for $t.$ Let's prove the result for $t+1.$ We start by summing up the noise update from \autoref{lem:noisupdate} for $\tau=\mathscr{T},\dots,t$ which yields: 
\begin{equation}
\begin{split}\label{eq:noiseupd20}
    &\left| y_i(\Xi_{i,j,r}^{(t+1)} -  \Xi_{i,j,r}^{(\mathscr{T})} ) -\frac{\tilde{\Theta}(\eta\sigma^2 d)}{N}\sum_{\tau=\mathscr{T}}^{t}\ell_i^{(\tau)}(\Xi_{i,j,r}^{(\tau)})^2 \right|\\
    &\leq 
    \frac{\tilde{\Theta}(\eta\sigma^2\sqrt{d})}{N}\sum_{\tau=\mathscr{T}}^{t-1}\sum_{a\in\mathcal{Z}_2}\ell_a^{(\tau)}\sum_{ k\neq P(\bm{X}_a) } (\Xi_{a,k,r}^{(\tau)})^2\\
    &+ \frac{\tilde{\Theta}(\eta\sigma^2\sqrt{d})}{N} \sum_{a\in\mathcal{Z}_2}\ell_a^{(t)}\sum_{ k\neq P(\bm{X}_a) } (\Xi_{a,k,r}^{(t)})^2\\
    &+ \frac{\tilde{\Theta}(\eta\sigma^2\sqrt{d})}{N}\sum_{\tau=\mathscr{T}}^{t}\sum_{a\in\mathcal{Z}_1}\ell_a^{(\tau)}\sum_{ k\neq P(\bm{X}_a) } (\Xi_{a,k,r}^{(\tau)})^2
\end{split}
\end{equation}
We apply \autoref{indh:maxnoisesign} to bound $(\Xi_{a,k,r}^{(t)})^2$ in \eqref{eq:noiseupd20} and obtain: 
\begin{equation}
\begin{split}\label{eq:noiseupd201}
   & \left| y_i(\Xi_{i,j,r}^{(t+1)} -  \Xi_{i,j,r}^{(\mathscr{T})} ) -\frac{\tilde{\Theta}(\eta\sigma^2 d)}{N}\sum_{\tau=\mathscr{T}}^{t}\ell_i^{(\tau)}(\Xi_{i,j,r}^{(\tau)})^2 \right|\\
   &\leq 
    \frac{\tilde{\Theta}(\eta\sigma^2\sqrt{d})}{N}\sum_{\tau=\mathscr{T}}^{t-1}\sum_{a\in\mathcal{Z}_2}\ell_a^{(\tau)}\sum_{ k\neq P(\bm{X}_a) } (\Xi_{a,k,r}^{(\tau)})^2\\
    &+  \tilde{\Theta}(\eta P\sigma^2\sqrt{d})   \nu_2^{(t)} \alpha 
    \min\{\kappa,(c^{(t)})^2\alpha^2\}  \\
    &+ \tilde{\Theta}(\eta P\sigma^2\sqrt{d}) \sum_{\tau=\mathscr{T}}^{t}\nu_1^{(\tau)}\alpha \min\{\kappa,(c^{(\tau)})^2\alpha^2\}
\end{split}
\end{equation}
Similarly to above, we apply \autoref{prop:bd_nu1sum} to bound $\sum_{\tau=0}^{t}\nu_1^{(\tau)}\alpha\min\{\kappa,(c^{(\tau)})^2\alpha^2\} $. We also use $\nu_2^{(t)}\alpha\min\{\kappa,(c^{(t)})^2\alpha^2\}\leq \tilde{O}(\hat{\mu})$ and $\hat{\mu}=\Theta(1/N)$ in \eqref{eq:noiseupd201}  and obtain: 
 \begin{equation}
\begin{split}\label{eq:noiseupd202}
   & \left| y_i(\Xi_{i,j,r}^{(t+1)} - \Xi_{i,j,r}^{(\mathscr{T})}) -\frac{\tilde{\Theta}(\eta\sigma^2 d)}{N}\sum_{\tau=\mathscr{T}}^{t}\ell_i^{(\tau)}(\Xi_{i,j,r}^{(\tau)})^2 \right|\\
   &\leq 
    \frac{\tilde{\Theta}(\eta\sigma^2\sqrt{d})}{N}\sum_{\tau=\mathscr{T}}^{t-1}\sum_{a\in\mathcal{Z}_2}\ell_a^{(\tau)}\sum_{ k\neq P(\bm{X}_a) } (\Xi_{a,k,r}^{(\tau)})^2\\
    &+  \tilde{O}\left(\frac{P\sigma^2\sqrt{d}}{\alpha}\left(1+\frac{\eta\alpha}{N}\right) \right)\\
    &+ \tilde{O}\left(\frac{\eta\beta^3}{\alpha^2} \right) \sum_{j=\mathscr{T}}^{t}\nu_2^{(j)}. 
\end{split}
\end{equation}
To bound the first term in the right-hand side of \eqref{eq:noiseupd202}, we use the induction hypothesis \eqref{eq:noiseindhypoth}. Plugging this inequality in \eqref{eq:noiseupd202} yields: 
 \begin{equation}
\begin{split}\label{eq:noiseupd203}
   & \left| y_i(\Xi_{i,j,r}^{(t+1)} - \Xi_{i,j,r}^{(\mathscr{T})}) -\frac{\tilde{\Theta}(\eta\sigma^2 d)}{N}\sum_{\tau=\mathscr{T}}^{t}\ell_i^{(\tau)}(\Xi_{i,j,r}^{(\tau)})^2 \right|\\
   &\leq 
  \frac{1}{\sqrt{d}}\sum_{a\in\mathcal{Z}_2}\sum_{k\neq P(X_k)}y_a(\Xi_{a,k,r}^{(t)}-\Xi_{a,k,r}^{(\mathscr{T})})\\
  &+\Tilde{O}\left(\frac{P^2\sigma^2}{\alpha\sqrt{d}} \left(1+\frac{\alpha}{\sigma^2 d}+\frac{\alpha \eta}{N}\right)\sum_{\tau=0}^{t-1-\mathscr{T}} \frac{P^{\tau}}{d^{\tau/2}} \right) \\
  &+ \frac{P}{\sqrt{d}} \tilde{O}\left(\frac{\eta\beta^3}{\alpha^2  } \right) \sum_{\tau=0}^{t-1-\mathscr{T}} \frac{P^{\tau }}{d^{\tau/2}}\sum_{j=\mathscr{T}}^{t-1-\tau}\nu_2^{(j)}\\
    &+  \tilde{O}\left(\frac{P\sigma^2\sqrt{d}}{\alpha}\left(1+\frac{\eta\alpha}{N}\right) \right)\\
   &+ \tilde{O}\left(\frac{\eta\beta^3}{\alpha^2 } \right) \sum_{j=\mathscr{T}}^{t} \nu_2^{(j)}.
\end{split}
\end{equation}
 Now, we apply \autoref{indh:noiseGD} to have $y_a(\Xi_{a,k,r}^{(t)}-\Xi_{a,k,r}^{(0)})\leq \tilde{O}(1)$ in \eqref{eq:noiseupd203} and therefore,
  \begin{equation}
\begin{split}\label{eq:noiseupd204}
   & \left| y_i\Xi_{i,j,r}^{(t+1)} - y_i\Xi_{i,j,r}^{(\mathscr{T})}  -\frac{\tilde{\Theta}(\eta\sigma^2 d)}{N}\sum_{\tau=\mathscr{T}}^{t}\ell_i^{(\tau)}(\Xi_{i,j,r}^{(\tau)})^2 \right|\\
   &\leq 
  \frac{\tilde{O}(P)}{\sqrt{d}} +\Tilde{O}\left(\frac{P\sigma^2\sqrt{d}}{\alpha} \left(1+\frac{\alpha}{\sigma^2 d}+\frac{\alpha \eta}{N}\right)\sum_{\tau=1}^{t-\mathscr{T}} \frac{P^{\tau}}{d^{\tau/2}} \right) \\
  &+       \tilde{O}\left(\frac{\eta\beta^3}{\alpha^2 } \right) \sum_{\tau=1}^{t-\mathscr{T}} \frac{P^{\tau }}{d^{\tau/2}}\sum_{j=\mathscr{T}}^{t-\tau}\nu_2^{(j)}\\
    &+  \tilde{O}\left(\frac{P\sigma^2\sqrt{d}}{\alpha}\left(1+\frac{\eta\alpha}{N}\right) \right)\\
     &+ \tilde{O}\left(\frac{\eta\beta^3}{\alpha^2  } \right) \sum_{j=\mathscr{T}}^{t} \nu_2^{(j)}.
\end{split}
\end{equation}
By rearranging the terms, we finally have: 
  \begin{equation}
\begin{split}\label{eq:bhfdbqej}
    & \left| y_i\Xi_{i,j,r}^{(t+1)} - y_i\Xi_{i,j,r}^{(\mathscr{T})}  -\frac{\tilde{\Theta}(\eta\sigma^2 d)}{N}\sum_{\tau=\mathscr{T}}^{t}\ell_i^{(\tau)}(\Xi_{i,j,r}^{(\tau)})^2 \right|\\
    &\leq  
     \Tilde{O}\left(\frac{P\sigma^2\sqrt{d}}{\alpha} \left(1+\frac{\alpha}{\sigma^2 d}+\frac{\alpha \eta}{N}\right)\sum_{\tau=0}^{t-\mathscr{T}} \frac{P^{\tau}}{d^{\tau/2}} \right)\\
     &+    \tilde{O}\left(\frac{\eta\beta^3}{\alpha^2  } \right) \sum_{\tau=0}^{t-\mathscr{T}} \frac{P^{\tau }}{d^{\tau/2}}\sum_{j=\mathscr{T}}^{t-\tau}\nu_2^{(j)} ,
\end{split}
\end{equation}
which proves the induction hypothesis for $t+1.$

Now, let's simplify the sum terms in \eqref{eq:noiseindhypoth}. Since $P\ll\sqrt{d}$, by definition of a geometric sequence, we have: 
\begin{align}\label{eq:sumgeo}
    \sum_{\tau=0}^{t-\mathscr{T}} \frac{P^{\tau}}{d^{\tau/2}}&\leq \frac{1}{1-\frac{P}{\sqrt{d}}}= \Theta(1).
\end{align}
Plugging \eqref{eq:sumgeo} in \eqref{eq:noiseindhypoth} yields 
  \begin{equation}
\begin{split}\label{eq:gbfjdebqnjn}
     \left| y_i(\Xi_{i,j,r}^{(t)} - \Xi_{i,j,r}^{(\mathscr{T})})  -\frac{\tilde{\Theta}(\eta\sigma^2 d)}{N}\sum_{\tau=\mathscr{T}}^{t}\ell_i^{(\tau)}(\Xi_{i,j,r}^{(\tau)})^2 \right| &\leq  \tilde{O}\left(\frac{P\sigma^2\sqrt{d}}{\alpha}\right)\\
     &+   \tilde{O}\left(\frac{\eta\beta^3}{\alpha^2  } \right) \sum_{\tau=0}^{t-1-\mathscr{T}} \frac{P^{\tau }}{d^{\tau/2}}\sum_{j=\mathscr{T}}^{t-1-\tau}\nu_2^{(j)} .
\end{split}
\end{equation}
Now, let's simplify the second sum term in \eqref{eq:gbfjdebqnjn}. Indeed, we have: 
\begin{align}\label{eq:sufcdsnfesd}
     \sum_{\tau=0}^{t-1-\mathscr{T}} \frac{P^{\tau }}{d^{\tau/2}}\sum_{j=\mathscr{T}}^{t-1-\tau}\nu_2^{(j)} &\leq  \sum_{\tau=0}^{t-1-\mathscr{T}} \frac{P^{\tau }}{d^{\tau/2}}\sum_{j=\mathscr{T}}^{t-1}\nu_2^{(j)} \leq \Theta(1) \sum_{j=\mathscr{T}}^{t-1}\nu_2^{(j)} ,
\end{align}
where we used \eqref{eq:sumgeo} in the last inequality. Plugging \eqref{eq:sufcdsnfesd} in \eqref{eq:gbfjdebqnjn} gives the final result.
\end{proof}

After $T_1$ iterations, we prove with \autoref{lem:noise_dominates} that for $i\in\mathcal{Z}_2$ and $j\in[P]\backslash\{P(\bm{X}_i)\}$, there exists $r\in[m]$ such that $\Xi_{i,j,r}^{(\tau)}$ is large. This implies that $(\Xi_{i,j,r}^{(\tau)})^2\ell_i^{(\tau)}(\Xi_i^{(\tau)}) $ stays well controlled. We therefore rewrite  \autoref{lem:indhnoisesum} to take this into account.

 \begin{lemma}\label{lem:indhnoisesumlyytateiter}
 Let $i\in [N]$, $j\in [P]\backslash \{ P(\bm{X}_i)\}$ and $r\in [m]$. Let $\mathscr{T},t\in [T]$ such that $\mathscr{T} < t.$ Then, the noise update \eqref{eq:GD_noise} satisfies
    \begin{equation*}
   \begin{split}
      \left| y_i(\Xi_{i,j,r}^{(t)} - \Xi_{i,j,r}^{(\mathscr{T})})  -\frac{\tilde{\Theta}(\eta\sigma^2 d)}{N}\sum_{\tau=\mathscr{T}}^{t-1}\ell_i^{(\tau)}\min\{\kappa,(\Xi_{i,j,r}^{(\tau)})^2\} \right| &\leq  \tilde{O}\left(\frac{P\sigma^2\sqrt{d}}{\alpha}\right) +   \tilde{O}\left(\frac{\eta\beta^3}{\alpha^2  } \right)  \sum_{j=\mathscr{T}}^{t-1}\nu_2^{(j)} .
   \end{split}
   \end{equation*}
\end{lemma}
\begin{proof}[Proof of \autoref{lem:indhnoisesumlyytateiter}]
From \autoref{lem:indhnoisesum}, we know that 
\begin{align}\label{eq:rf,cdez,dzd}
     \left| y_i(\Xi_{i,j,r}^{(t)} - \Xi_{i,j,r}^{(\mathscr{T})})  -\frac{\tilde{\Theta}(\eta\sigma^2 d)}{N}\sum_{\tau=\mathscr{T}}^{t-1}\ell_i^{(\tau)} (\Xi_{i,j,r}^{(\tau)})^2  \right| &\leq  \tilde{O}\left(\frac{P\sigma^2\sqrt{d}}{\alpha}\right) +   \tilde{O}\left(\frac{\eta\beta^3}{\alpha^2  } \right)  \sum_{j=\mathscr{T}}^{t-1}\nu_2^{(j)}.
\end{align}
Using \autoref{remark}, we know that a sufficient condition to have $\widehat{\ell}^{(\tau)}(\Xi_i^{(t)}$ is $(\Xi_{i,j,r}^{(\tau)})^2\geq\kappa\geq \tilde{\Omega}(1).$ Therefore, we can replace $\widehat{\ell}^{(t)}(\Xi_i^{(t)})(\Xi_{i,j,r}^{(\tau)})^2=\min\{\kappa, (\Xi_{i,j,r}^{(\tau)})^2\}.$ Plugging this equality in \eqref{eq:rf,cdez,dzd} yields the aimed result.
\end{proof}

\begin{lemma}\label{lem:ncejncejzd}
Let $T_1=\tilde{O}\left(\frac{N}{\sigma_0\sigma\sqrt{d}\sigma^2d}\right)$. For $t\in [T_1,T]$, we have $\frac{1}{N}\sum_{\tau=0}^t  \sum_{i\in\mathcal{Z}_2}\ell_i^{(\tau)} \min \{\kappa, (\Xi_{i,j,r}^{(\tau)})^2\}\leq \tilde{O}\left(\frac{1}{\eta} \right).$
\end{lemma} 
\begin{proof}[Proof of \autoref{lem:ncejncejzd}]
From \autoref{eq:njefecjn}, we know that:
\begin{align}
    \sum_{\tau=T_1}^t\nu_2^{(\tau)}\leq \tilde{O}\left( \frac{1}{\eta \sigma_0} \right).
\end{align}
On the other hand we know from \autoref{lem:indhnoisesumlyytateiter} that:
\begin{equation}
    \begin{aligned}\label{eq:sumliZ223nd}
    \frac{\Tilde{\Theta}(\eta\sigma^2 d)}{N}\sum_{\tau=0}^{T_1-1} \sum_{i\in\mathcal{Z}_2} \ell_i^{(\tau)} \min \{\kappa, (\Xi_{i,j,r}^{(\tau)})^2\} &\leq y_i (\Xi_{i,j,r}^{(T_1)}-  \Xi_{i,j,r}^{(0)})+ \Tilde{O}\left(\frac{P\sigma^2\sqrt{d}}{\alpha} \right)\\
    &+   \tilde{O}\left(\frac{\eta\hat{\mu}\beta^3}{\alpha  } \right) T_1.
\end{aligned}
\end{equation}
Besides, we have: $\tilde{O}\left(\frac{\eta\hat{\mu}\beta^3}{\alpha  } \right) T_1\leq\Tilde{O}\left(\frac{P\sigma^2\sqrt{d}}{\alpha} \right).$ 
Plugging this inequality yields
\begin{equation}
    \begin{aligned}\label{eq:sumlierfercfZ223nd}
    \frac{\Tilde{\Theta}(\eta\sigma^2 d)}{N}\sum_{\tau=0}^{T_1-1} \sum_{i\in\mathcal{Z}_2} \ell_i^{(\tau)} \min \{\kappa, (\Xi_{i,j,r}^{(\tau)})^2\} &\leq y_i (\Xi_{i,j,r}^{(T_1)}-  \Xi_{i,j,r}^{(0)})+ \Tilde{O}\left(\frac{P\sigma^2\sqrt{d}}{\alpha} \right).
\end{aligned}
\end{equation}

By applying \autoref{indh:noiseGD}, \eqref{eq:sumlierfercfZ223nd} is  eventually bounded as: 
\begin{align}\label{eq:sumliZ2sdccds23nd}
    \frac{1}{N}\sum_{\tau=0}^t \sum_{i\in\mathcal{Z}_2} \ell_i^{(\tau)} \min \{\kappa, (\Xi_{i,j,r}^{(\tau)})^2\}&\leq \tilde{O}\left(\frac{1}{\eta \sigma^2 d} \right)+ \Tilde{O}\left(\frac{P}{\eta\alpha \sqrt{d}} \right)\leq \tilde{O}\left(\frac{1}{\eta} \right).
\end{align}

By combining \eqref{eq:pekzde} and \eqref{eq:sumliZ2sdccds23nd} we deduce that for all $j\in [P]\backslash\{P(\bm{X}_i)\}$ and $r\in [m]$:
\begin{equation}
    \begin{aligned}
  \frac{1}{N}\sum_{\tau=0}^t  \sum_{i\in\mathcal{Z}_2}\ell_i^{(\tau)} \min \{\kappa, (\Xi_{i,j,r}^{(\tau)})^2\}&=   \frac{1}{N} \sum_{\tau=0}^{T_1} \sum_{i\in\mathcal{Z}_2}\ell_i^{(\tau)} \min \{\kappa, (\Xi_{i,j,r}^{(\tau)})^2\}\\
  &+\frac{1}{N}\sum_{\tau=T_1}^{t}\sum_{i\in\mathcal{Z}_2}\ell_i^{(\tau)} \min \{\kappa, (\Xi_{i,j,r}^{(\tau)})^2\}\\
   &\leq   \tilde{O}\left(\frac{1}{\eta}\right).
\end{aligned}
\end{equation}
\end{proof}

\subsection{Convergence rate of the training loss using GD}\label{sec:cvrateGD}

In this section, we prove that when using GD, the training loss converges sublinearly in our setting.

\subsubsection{Convergence after learning $\mathcal{Z}_1$ ($t\in[T_0,T]$)}

\begin{lemma}[Convergence rate of the $\mathcal{Z}_1$ loss]\label{thm:convratez1}
Let $t\in[T_0,T]$. Run GD  with learning rate $\eta$ for $t$ iterations. Then, the $\mathcal{Z}_1$ loss sublinearly converges to zero as:
    \begin{align*}
        (1-\hat{\mu})\widehat{\mathcal{L}}^{(t)}(\alpha)\leq \frac{\tilde{O}(1)}{\eta\alpha^2 (t-T_0+1)}.
    \end{align*}
\end{lemma}

\begin{proof}[Proof of \autoref{thm:convratez1}] Let $t\in[T_0,T].$ From \autoref{lem:signupdate}, we know that the signal update is lower bounded as: 
\begin{align}\label{eq:efvervrt}
    c^{(t+1)}\geq c^{(t)}+\Theta(\eta\alpha)(1-\hat{\mu})\widehat{\ell}^{(t)}(\alpha)(\alpha c^{(t)})^2.
\end{align}
From \autoref{lem:increase_signalGD}, we know that $c^{(t)}\geq \tilde{\Omega}(1/\alpha)$. Thus, we simplify \eqref{eq:efvervrt} as:
\begin{align}
    c^{(t+1)}\geq c^{(t)}+\tilde{\Omega}(\eta\alpha)(1-\hat{\mu})\widehat{\ell}^{(t)}(\alpha).
\end{align}
Since $\alpha^3\sum_{r=1}^m (c_r^{(t)})^3 \geq \tilde{\Omega}(1/\alpha)-m\tilde{O}(\sigma_0)\geq \tilde{\Omega}(1/\alpha)>0$, we can apply \autoref{lem:logsigm2} and obtain:
\begin{align}\label{eq:frcjrejdewed3}
    c^{(t+1)}\geq c^{(t)}+\tilde{\Omega}(\eta\alpha)(1-\hat{\mu})\widehat{\mathcal{L}}^{(t)}(\alpha).
\end{align}
Let's now assume by contradiction that for $t\in[T_0,T]$, we have:
\begin{align}\label{eq:idwenkfvfevfevewkn}
    (1-\hat{\mu}) \widehat{\mathcal{L}}^{(t)}(\alpha)> \frac{\tilde{\Omega}(1)}{\eta\alpha^2 (t-T_0+1)}.
\end{align}
From the \eqref{eq:GDM_signal} update, we know that $c_r^{(\tau)}$ is a non-decreasing sequence which implies that $\sum_{r=1}^m (\alpha c_r^{(\tau)})^3$ is also non-decreasing. Since $x\mapsto \log(1+\exp(-x))$ is non-increasing, this implies that for $s\leq t$, we have:
\begin{align}\label{eq:ferjefriwjefefcfer}
  \frac{\tilde{\Omega}(1)}{\eta\alpha^2 (t-T_0+1)}  < (1-\hat{\mu}) \widehat{\mathcal{L}}^{(t)}(\alpha) \leq (1-\hat{\mu}) \widehat{\mathcal{L}}^{(s)}(\alpha).
\end{align}
Plugging \eqref{eq:ferjefriwjefefcfer} in the update \eqref{eq:frcjrejdewed3} yields for $s\in[T_0,t]$:
\begin{align}\label{eq:icfhrewihwc}
    c^{(s+1)}> c^{(s)}+\frac{\tilde{\Omega}(1)}{ \alpha(t-T_0+1)}.
\end{align}
Let $t\in[T_0,T]$. We now sum \eqref{eq:icfhrewihwc} for $s=T_0,\dots,t$ and obtain:
\begin{align}\label{eq:iejeiced}
    c^{(t+1)}> c^{(T_0)}+\frac{\tilde{\Omega}(1)(t-T_0+1)}{\alpha(t-T_0+1)}> \frac{\tilde{\Omega}(1)}{\alpha},
\end{align}
where we used the fact that $c^{(T_0)}\geq \tilde{\Omega}(1/\alpha)> 0$ (\autoref{lem:increase_signalGD}) in the last inequality.  Therefore, we have for $t\in[T_0,T],$ $c^{(t)}\geq \tilde{\Omega}(1/\alpha)> 0$. Let's now show that \eqref{eq:iejeiced} implies  a contradiction. Indeed, we have:
\begin{align}
   &\eta \alpha^2(t-T_0+1)  (1-\hat{\mu}) \widehat{\mathcal{L}}^{(t)}(\alpha)\nonumber\\
   \leq&  \eta \alpha^2T(1-\hat{\mu})\log\left(1+\exp(-(\alpha c^{(t)})^3 - \sum_{r\neq r_{\max}}(\alpha c_r^{(t)})^3\right) \nonumber\\
   \leq&  \eta \alpha^2T(1-\hat{\mu})\log\left(1+\exp(-\tilde{\Omega}(1)\right), \label{eq:efevvoekmeifriw3jiepr}
\end{align}
where we used $\sum_{r\neq r_{\max}} (c_r^{(t)})^3\geq -m\tilde{O}(\sigma_0^3)$ along with \eqref{eq:iejeiced} in \eqref{eq:efevvoekmeifriw3jiepr}. We now apply \autoref{lem:logsigm2} in \eqref{eq:efevvoekmeifriw3jiepr} and obtain:
\begin{align}\label{eq;frrikriknnkn}
   \eta \alpha^2(t-T_0+1)  (1-\hat{\mu}) \widehat{\mathcal{L}}^{(t)}(\alpha)\leq \frac{(1-\hat{\mu})  \eta \alpha^2T}{1+\exp(\tilde{\Omega}(1))}.
\end{align}
Given the values of $T,\eta,\alpha,\hat{\mu}$, we finally have:
\begin{align}
     \eta \alpha^2(t-(T_0-1))  (1-\hat{\mu}) \widehat{\mathcal{L}}^{(t)}(\alpha) <\tilde{O}(1),
\end{align}
which contradicts \eqref{eq:idwenkfvfevfevewkn}.
\end{proof}

\subsubsection{Convergence at late stages ($t\in[T_1,T]$)}

\begin{lemma}[Convergence rate of the loss]\label{thm:convrate}
Let $t\in[T_1,T]$. Run GD  with learning rate $\eta \in (0,1/L)$ for $t$ iterations. Then, the loss sublinearly converges to zero as:
    \begin{align*}
        \widehat{L}(\bm{W}^{(t)})\leq \frac{\tilde{O}(1)}{\eta (t-T_1+1)}.
    \end{align*}

\end{lemma}
\begin{proof}[Proof of \autoref{thm:convrate}] We first apply the classical descent lemma for smooth functions (\autoref{lem:desc}). Since $\widehat{L}(W)$ is smooth,  we have:
\begin{align}\label{eq:appl_smoothL}
    \widehat{L}(\bm{W}^{(t+1)})\leq \widehat{L}(\bm{W}^{(t)}) - \frac{\eta}{2}\|\nabla\widehat{L}(\bm{W}^{(t)})\|_2^2 =\widehat{L}(\bm{W}^{(t)}) - \frac{\eta}{2} \sum_{r=1}^m \|\nabla_{\bm{w}_r}\widehat{L}(\bm{W}^{(t)})\|_2^2.
\end{align}
   \autoref{eq:norm_grads} provides a lower bound on the gradient. We plug it in \eqref{eq:appl_smoothL} and get:
\begin{align}\label{eq:feknvrktr}
 \widehat{L}(\bm{W}^{(t+1)})\leq \widehat{L}(\bm{W}^{(t)}) - \tilde{\Omega}(\eta) \widehat{L}(\bm{W}^{(t)})^2.
\end{align}
Applying \autoref{lem:sublinear} to \eqref{eq:feknvrktr} yields the aimed result.
\end{proof}

\subsubsection{Auxiliary lemmas for the proof of \autoref{thm:convrate}}

To obtain the convergence rate in \autoref{thm:convrate}, we used the  following auxiliary lemma.

\begin{lemma}[Bound on the gradient for GD]\label{eq:norm_grads} 
Let $t\in[T_1,T]$. Run GD for $t$ iterations. Then, the norm of gradient is lower bounded as follows:
     \begin{align*} 
     \sum_{r=1}^m \|\nabla_{\bm{w}_r}\widehat{L}(\bm{W}^{(t)})\|_2^2&\geq \tilde{\Omega}(1)  \widehat{L}(\bm{W}^{(t)})^2.
\end{align*} 

\end{lemma} 
\begin{proof}[Proof of \autoref{eq:norm_grads}]
 Let $t\in[T_1,T]$. To obtain the lower bound, we project the gradient on the the signal and on the noise.
 
 \paragraph{Projection on the signal.} Since $\|w^*\|_2=1$, we lower bound $\|\nabla_{\bm{w}_r}\widehat{L}(\bm{W}^{(t)})\|_2^2$ as 
\begin{align}\label{eq:w*sqbd}
    \|\nabla_{\bm{w}_r}\widehat{L}(\bm{W}^{(t)})\|_2^2\geq \langle \nabla_{\bm{w}_r}\widehat{L}(\bm{W}^{(t)}),w^*\rangle^2=(\mathscr{G}_r^{(t)})^2.
\end{align}
By successively applying  \autoref{lem:signgrad} and \autoref{eq:derZ1_bd},  $(\mathscr{G}_r^{(t)})^2$ is   lower bounded as
\begin{align}\label{eq:gradsignal}
     (\mathscr{G}_r^{(t)})^2 \geq   
     \left(\frac{\alpha^3}{N}\sum_{i\in \mathcal{Z}_1}\ell_i^{(t)}(c_r^{(t)})^2\right)^2\geq \Omega(1)\left( \alpha^3 (1-\hat{\mu}) \widehat{\ell}^{(t)}(\alpha)(c_r^{(t)})^2\right)^2.
\end{align}
Combining \eqref{eq:w*sqbd} and  \eqref{eq:gradsignal} yields:
\begin{align}\label{eq:gradLsignGD}
    \|\nabla_{\bm{w}_r}\widehat{L}(\bm{W}^{(t)})\|_2^2\geq  \Omega(1) \left(\alpha^3(1-\hat{\mu})\widehat{\ell}^{(t)}(\alpha)(c_r^{(t)})^2\right)^2.
\end{align}

\paragraph{Projection on the noise.} For a fixed $i\in\mathcal{Z}_2$ and $j\in [P]\backslash \{P(\bm{X}_i)\}$, we know that $\|\nabla_{\bm{w}_r}\widehat{L}(\bm{W}^{(t)})\|_2^2$ is lower bounded as
 \begin{align}\label{eq:Xijsqbd}
    \|\nabla_{\bm{w}_r}\widehat{L}(\bm{W}^{(t)})\|_2^2\geq \left\langle \nabla_{\bm{w}_r}\widehat{L}(\bm{W}^{(t)}),\frac{\frac{1}{N}\sum_{i\in\mathcal{Z}_2} \sum_{j\neq P(\bm{X}_i)} \bm{X}_i[j]}{\|\frac{1}{N}\sum_{i\in\mathcal{Z}_2} \sum_{j\neq P(\bm{X}_i)}\bm{X}_i[j]\|_2} \right\rangle^2=(\mathrm{G}_{r}^{(t)})^2.
\end{align}
On the other hand, by \autoref{prop:normgradXij}, we lower bound $\mathrm{G}_{r}^{(t)}$ term with probability $1-o(1)$ as: 
\begin{align}\label{eq:sum_noisenorm_grad}
    (\mathrm{G}_{r}^{(t)})^2&\geq \left( \frac{\tilde{\Omega}(\sigma\sqrt{d})}{N}\sum_{i\in\mathcal{Z}_2}\sum_{j\neq P(\bm{X}_i)}\ell_i^{(t)} (\Xi_{i,j,r}^{(t)})^2 - \frac{\tilde{O}(\sigma)}{N}\sum_{i\in\mathcal{Z}_1}\sum_{j\neq P(\bm{X}_i)} \ell_i^{(t)} (\Xi_{i,j,r}^{(t)})^2 \right)^2
\end{align}

\paragraph{Gathering the bounds.}

Combining \eqref{eq:w*sqbd}, \eqref{eq:Xijsqbd}, \eqref{eq:gradsignal} and \eqref{eq:sum_noisenorm_grad} and using $2a^2+2b^2\geq (a+b)^2,$  we thus bound $\|\nabla_{\bm{w}_r}\widehat{L}(\bm{W}^{(t)})\|_2^2$ as: 
\begin{equation}
    \begin{split}\label{eq:masternablaL}
        \|\nabla_{\bm{w}_r}\widehat{L}(\bm{W}^{(t)})\|_2^2\geq & \left(\frac{\alpha+\tilde{O}(\sigma)}{N}\sum_{i\in \mathcal{Z}_1}\ell_i^{(t)}\alpha^2(c_r^{(t)})^2\right.\\
        &\left.+\frac{\tilde{\Omega}(\sigma\sqrt{d})}{N}\sum_{i\in\mathcal{Z}_2}\sum_{j\neq P(\bm{X}_i)}\ell_i^{(t)} (\Xi_{i,j,r}^{(t)})^2\right.\\
     &\left.- \frac{\tilde{O}(\sigma)}{N}\sum_{i\in\mathcal{Z}_1}\sum_{j\neq P(\bm{X}_i)} \ell_i^{(t)} \left((\alpha^2 (c_r^{(t)})^2+ (\Xi_{i,j,r}^{(t)})^2\right) \right)^2.
    \end{split}
\end{equation}

We now sum up \eqref{eq:masternablaL} for $r=1,\dots,m$ and apply Cauchy-Schwarz inequality  to get: 
 \begin{equation}
    \begin{split}\label{eq:sum_grads_r}
       \sum_{r=1}^m \|\nabla_{\bm{w}_r}\widehat{L}(\bm{W}^{(t)})\|_2^2 &\geq \frac{1}{m} \left( \frac{\alpha+\tilde{O}(\sigma)}{N} \sum_{r=1}^m\ell_i^{(t)}(\alpha)\alpha^2(c_r^{(t)})^2\right.\\
        &\left.+\frac{\tilde{\Omega}(\sigma\sqrt{d})}{N}\sum_{i\in\mathcal{Z}_2}\sum_{r=1}^m\sum_{j\neq P(\bm{X}_i)}\ell_i^{(t)} (\Xi_{i,j,r}^{(t)})^2\right.\\
     &\left.- \frac{\tilde{O}(\sigma)}{N}\sum_{i\in\mathcal{Z}_1}\sum_{r=1}^m\sum_{j\neq P(\bm{X}_i)} \ell_i^{(t)} \left((\alpha^2 (c_r^{(t)})^2+ (\Xi_{i,j,r}^{(t)})^2\right) \right)^2.
    \end{split}
\end{equation}
We apply \autoref{eq:derZ1_bd} to further lower bound \eqref{eq:sum_grads_r} and get: 
 \begin{equation}
    \begin{split}\label{eq:sum_grads_rejcfjhv}
       \sum_{r=1}^m \|\nabla_{\bm{w}_r}\widehat{L}(\bm{W}^{(t)})\|_2^2&\geq  \Omega\left(\frac{1}{m}\right) \left((\alpha+\tilde{O}(\sigma))  (1-\hat{\mu})\sum_{r=1}^m\widehat{\ell}^{(t)}(\alpha)\alpha^2(c_r^{(t)})^2\right.\\
        &\left.+\frac{\tilde{\Omega}(\sigma\sqrt{d})}{N}\sum_{i\in\mathcal{Z}_2}\sum_{r=1}^m\sum_{j\neq P(\bm{X}_i)}\ell_i^{(t)} (\Xi_{i,j,r}^{(t)})^2\right.\\
     &\left.- \frac{\tilde{O}(\sigma)}{N}\sum_{i\in\mathcal{Z}_1}\sum_{r=1}^m\sum_{j\neq P(\bm{X}_i)} \ell_i^{(t)} \left((\alpha^2 (c_r^{(t)})^2+ (\Xi_{i,j,r}^{(t)})^2\right) \right)^2.
    \end{split}
\end{equation}

\paragraph{Bound the gradient terms by the loss.} 
Using \autoref{eq:gdcvde3}, \autoref{eq:btbttr} and \autoref{lem:ytnthnt} we have:
\begin{align}
    (\alpha+\tilde{O}(\sigma))  (1-\hat{\mu})\sum_{r=1}^m\widehat{\ell}^{(t)}(\alpha)\alpha^2(c_r^{(t)})^2 &\geq  \tilde{\Omega}(\alpha+\tilde{O}(\sigma)) \widehat{\mathcal{L}}^{(t)}(\alpha),\label{eq:nvnrrv}\\
    \frac{\tilde{O}(\sigma)}{N}\sum_{i\in\mathcal{Z}_1}\sum_{r=1}^m\sum_{j\neq P(\bm{X}_i)} \ell_i^{(t)} \left((\alpha^2 (c_r^{(t)})^2+ (\Xi_{i,j,r}^{(t)})^2\right)&\leq \tilde{O}(\sigma) (1-\hat{\mu}) \widehat{\mathcal{L}}^{(t)}(\alpha),\label{eq:jvbrreercv}\\
    \frac{\tilde{\Omega}(\sigma\sqrt{d})}{N}\sum_{i\in\mathcal{Z}_2}\sum_{r=1}^m\sum_{j\neq P(\bm{X}_i)}\ell_i^{(t)} (\Xi_{i,j,r}^{(t)})^2&\geq \frac{\tilde{\Omega}(\sigma\sqrt{d})  }{N    }\sum_{i\in\mathcal{Z}_2}\widehat{\mathcal{L}}^{(t)}(\Xi_i^{(t)}).\label{eq:jbfecjevef}
\end{align}
Plugging \eqref{eq:nvnrrv}, \eqref{eq:jvbrreercv} and \eqref{eq:jbfecjevef} in \eqref{eq:sum_grads_rejcfjhv} yields:
\begin{align}
       \sum_{r=1}^m \|\nabla_{\bm{w}_r}\widehat{L}(\bm{W}^{(t)})\|_2^2&\geq  \Omega\left(\frac{1}{m}\right) \left((\alpha+\tilde{O}(\sigma))  (1-\hat{\mu})\widehat{\mathcal{L}}^{(t)}(\alpha)\right.\nonumber\\
        &\left.+\frac{\tilde{\Omega}(\sigma\sqrt{d})}{N} \sum_{i\in\mathcal{Z}_2}\widehat{\mathcal{L}}^{(t)}(\Xi_i^{(t)}) - (1-\hat{\mu})\tilde{O}(\sigma)\widehat{\mathcal{L}}^{(t)}(\alpha)\right)^2\nonumber\\
        &\geq \tilde{\Omega}(1) \left(  (1-\hat{\mu})\widehat{\mathcal{L}}^{(t)}(\alpha)+\frac{1}{N} \sum_{i\in\mathcal{Z}_2}\widehat{\mathcal{L}}^{(t)}(\Xi_i^{(t)})  \right)^2,\label{eq:ppekere}
\end{align}
Finally, we use \autoref{lem:lossbdejverte} and lower bound \eqref{eq:ppekere} by  $\widehat{L}(\bm{W}^{(t)})^2$. This gives the aimed result.

\end{proof}

We now present auxiliary lemmas that  link the gradient terms with their corresponding loss. 
\begin{lemma} \label{eq:gdcvde3}
Let $t\in[T_1,T].$ Run GD for $t$ iterations. Then, we have:
\begin{align*}
\sum_{r=1}^m\widehat{\ell}^{(t)}(\alpha)\alpha^2(c_r^{(t)})^2 \geq \tilde{\Omega}(1)\widehat{\mathcal{L}}^{(t)}(\alpha).
\end{align*}
\end{lemma}
\begin{proof}[Proof of \autoref{eq:gdcvde3}] In order to bound $\sum_{r=1}^m\widehat{\ell}^{(t)}(\alpha)\alpha^2(c_r^{(t)})^2$, we apply \autoref{lem:masterlogsigmdelta}. We first verify that the conditions of the lemma are met.  From  \autoref{lem:increase_signalGD} we know that for $t\in[T_0,T]$, we have $c^{(t)}\geq \tilde{\Omega}(1/\alpha)$. Along with \autoref{indh:noiseGD}, this implies that 
\begin{align}\label{eq:fvekn}
 \tilde{\Omega}(1) \leq  \tilde{\Omega}(1)-m\tilde{O}(\alpha\sigma_0) \leq \sum_{r=1}^m \alpha c_r^{(t)}  \leq \tilde{O}(\alpha)m \leq  \tilde{O}(1).
\end{align}
Therefore, we can apply \autoref{lem:masterlogsigmdelta} and get the lower bound:
\begin{align}\label{eq:lderlwbd}
    \sum_{r=1}^m\widehat{\ell}^{(t)}(\alpha)(\alpha c_r^{(t)})^2\geq  \frac{0.05e^{-m\tilde{O}(\sigma_0)}}{\tilde{O}(1)\left(1+\frac{m^2\tilde{O}(\sigma^2\sigma_0^2d)}{\tilde{\Omega}(1)^2}\right)}\log\left(1+e^{-\sum_{r=1}^m  (\alpha c_r^{(t)})^3}\right)\geq \tilde{\Omega}(1)\widehat{\mathcal{L}}^{(t)}(\alpha).
\end{align}
\end{proof}

\begin{lemma} \label{eq:btbttr}
Let $t\in[T_1,T].$ Run GD for $t$ iterations. Then, we have:
\begin{align*}
\frac{1}{N}\sum_{i\in\mathcal{Z}_1}\sum_{r=1}^m\sum_{j\neq P(\bm{X}_i)} \ell_i^{(t)} \left((\alpha^2 (c_r^{(t)})^2+ (\Xi_{i,j,r}^{(t)})^2\right)&\leq  \tilde{O}(1)(1-\hat{\mu}) \widehat{\mathcal{L}}^{(t)}(\alpha).
\end{align*}
\end{lemma}
\begin{proof}[Proof of \autoref{eq:btbttr}]
 We again verify that the conditions of \autoref{lem:masterlogsigmdelta} are met. By using \autoref{indh:noiseGD}, \autoref{indh:signGD} and \autoref{lem:increase_signalGD}, we have:
\begin{equation}
    \begin{aligned}
     \sum_{r=1}^m \alpha c_r^{(t)} +\sum_{r=1}^m\sum_{j\neq P(\bm{X}_i)} y_i\Xi_{i,j,r}^{(t)}&\leq  m\tilde{O}(\alpha)+mP\tilde{O}(\sigma\sigma_0\sqrt{d})\leq \tilde{O}(1),\\
  \sum_{r=1}^m \alpha c_r^{(t)} +\sum_{r=1}^m\sum_{j\neq P(\bm{X}_i)} y_i\Xi_{i,j,r}^{(t)} &\geq \tilde{\Omega}(1)-m\tilde{O}(\alpha\sigma_0)\geq \tilde{\Omega}(1).
\end{aligned}
\end{equation}

By applying \autoref{lem:masterlogsigmdelta}, we have:
\begin{align}
    &\frac{1}{N}\sum_{i\in\mathcal{Z}_1}\sum_{r=1}^m\sum_{j\neq P(\bm{X}_i)} \ell_i^{(t)} \left((\alpha^2 (c_r^{(t)})^2+ (\Xi_{i,j,r}^{(t)})^2\right)\nonumber\\
    &\leq \frac{m   e^{m \tilde{O}(\sigma_0)}}{\tilde{\Omega}(1)N}\sum_{i\in\mathcal{Z}_1} \log\left(1+\exp\left(-\sum_{r=1}^m \alpha^3 (c_r^{(t)})^3 -\Xi_i^{(t)}\right) \right)\nonumber\\
    &\leq \frac{\tilde{O}(1)}{N}\sum_{i\in\mathcal{Z}_1} \log\left(1+\exp\left(-\sum_{r=1}^m \alpha^3 (c_r^{(t)})^3 -\Xi_i^{(t)}\right) \right).\label{eq:fjenfnejwfe}
\end{align}
Lastly, we want to link the loss term in \eqref{eq:fjenfnejwfe} with $\widehat{\mathcal{L}}^{(t)}(\alpha)$. By applying \autoref{indh:noiseGD} and \autoref{lem:logfzd} in \eqref{eq:fjenfnejwfe}, we finally get:
\begin{equation}
    \begin{aligned}
    \frac{1}{N}\sum_{i\in\mathcal{Z}_1}\sum_{r=1}^m\sum_{j\neq P(\bm{X}_i)} \ell_i^{(t)} \left((\alpha^2 (c_r^{(t)})^2+ (\Xi_{i,j,r}^{(t)})^2\right)&\leq   (1-\hat{\mu}) (1+e^{\tilde{O}((\sigma\sigma_0\sqrt{d})^3)})\widehat{\mathcal{L}}^{(t)}(\alpha)\\
    &\leq (1-\hat{\mu}) \widehat{\mathcal{L}}^{(t)}(\alpha).\label{eq:knvenkve}
\end{aligned}
\end{equation}
Combining \eqref{eq:fjenfnejwfe} and \eqref{eq:knvenkve} yields the aimed result.
\end{proof}

\begin{lemma} \label{lem:ytnthnt}
Let $t\in[T_1,T].$ Run GD for $t$ iterations. Then, we have:
\begin{align*}
\frac{1}{N}\sum_{i\in\mathcal{Z}_2}\sum_{r=1}^m\sum_{j\neq P(\bm{X}_i)}\ell_i^{(t)} (\Xi_{i,j,r}^{(t)})^2&\geq \frac{\tilde{\Omega}(1) }{N    }\sum_{i\in\mathcal{Z}_2}\widehat{\mathcal{L}}^{(t)}(\Xi_i^{(t)}).
\end{align*}
\end{lemma}
\begin{proof}[Proof of \autoref{lem:ytnthnt}]  We again verify that the conditions of \autoref{lem:masterlogsigmdelta} are met. Using \autoref{indh:noiseGD}, \autoref{indh:signGD} and  \autoref{lem:noise_dominates}, we have:
\begin{equation}
    \begin{aligned}
    &\sum_{r=1}^m \beta c_r^{(t)} +\sum_{r=1}^m\sum_{j\neq P(\bm{X}_i)} y_i\Xi_{i,j,r}^{(t)}\leq m\tilde{O}(\beta)+ mP\tilde{O}(1)\leq \tilde{O}(1)\\
    & \sum_{r=1}^m \beta c_r^{(t)} +\sum_{r=1}^m\sum_{j\neq P(\bm{X}_i)} y_i\Xi_{i,j,r}^{(t)}\geq \tilde{\Omega}(1)-m\tilde{O}(\sigma_0)-mP\tilde{O}(\sigma_0\sigma\sqrt{d})\geq \tilde{\Omega}(1).
\end{aligned}
\end{equation}
By applying \autoref{lem:masterlogsigmdelta}, we have:
\begin{equation}
    \begin{aligned}\label{eq:neknfce}
    &\frac{1}{N}\sum_{i\in\mathcal{Z}_2}\sum_{r=1}^m\sum_{j\neq P(\bm{X}_i)}\ell_i^{(t)} (\Xi_{i,j,r}^{(t)})^2\\
    \geq& \frac{0.05  e^{-m \tilde{O}(\sigma\sigma_0\sqrt{d})}}{N\tilde{O}(1) \left(1+\frac{m^2(\sigma\sigma_0\sqrt{d})^2}{\tilde{\Omega}(1)}\right)}\sum_{i\in\mathcal{Z}_2} \log\left(1+\exp\left(-\sum_{r=1}^m \beta^3 (c_r^{(t)})^3 -\Xi_i^{(t)}\right) \right)\\
    \geq &\frac{\tilde{\Omega}(1)}{N    }\sum_{i\in\mathcal{Z}_2} \log\left(1+\exp\left(-\sum_{r=1}^m \beta^3 (c_r^{(t)})^3 -\Xi_i^{(t)}\right) \right).
\end{aligned}
\end{equation}
Lastly, we want to link the loss term in \eqref{eq:neknfce} with $\widehat{\mathcal{L}}^{(t)}(\Xi_i^{(t)})$. By applying \autoref{indh:noiseGD} and \autoref{lem:logfzd} in \eqref{eq:neknfce}, we finally get:

\begin{equation}
    \begin{aligned}\label{eq:vfcwscwds}
    \frac{\tilde{\Omega}(1)}{N}\sum_{i\in\mathcal{Z}_2}\sum_{r=1}^m\sum_{j\neq P(\bm{X}_i)}\ell_i^{(t)} (\Xi_{i,j,r}^{(t)})^2 &\geq \frac{\tilde{\Omega}(1)e^{-m\tilde{O}(\beta^3)}  }{N    }\sum_{i\in\mathcal{Z}_2}\widehat{\mathcal{L}}^{(t)}(\Xi_i^{(t)})\\
    &\geq\frac{ \tilde{\Omega}(1) }{N    }\sum_{i\in\mathcal{Z}_2}\widehat{\mathcal{L}}^{(t)}(\Xi_i^{(t)}).
    \end{aligned}
\end{equation}
Combining \eqref{eq:neknfce} and \eqref{eq:vfcwscwds} yields the aimed result.
\end{proof}

\begin{lemma}\label{lem:lossbdejverte}
Let  $t\in[0,T]$ Run GD for for $t$ iterations. Then, we have:
\begin{align}
     (1-\hat{\mu}) \widehat{\mathcal{L}}^{(t)}(\alpha)+ \frac{1}{N} \sum_{i\in\mathcal{Z}_2}\widehat{\mathcal{L}}^{(t)}(\Xi_i^{(t)})\geq \Theta(1) \widehat{L}(\bm{W}^{(t)}) .
\end{align}
\end{lemma}
\begin{proof}[Proof of \autoref{lem:lossbdejverte}]

 we need to lower bound $\widehat{\mathcal{L}}^{(t)}(\alpha)$. By successively applying \autoref{lem:logfzd} and \autoref{indh:noiseGD}, we obtain:
\begin{align}
   (1-\hat{\mu}) \widehat{\mathcal{L}}^{(t)}(\alpha)&=\frac{1}{N}\sum_{i\in\mathcal{Z}_1} \frac{1+e^{-\Xi_i^{(t)}}}{1+e^{-\Xi_i^{(t)}}} \log\left(1+\exp\left(-\sum_{r=1}^m (\alpha c_r^{(t)})^3\right)\right)\nonumber\\
   &\geq \frac{1}{N}\sum_{i\in\mathcal{Z}_1} \frac{1}{1+e^{-\Xi_i^{(t)}}} \log\left(1+\exp\left(-\sum_{r=1}^m (\alpha c_r^{(t)})^3\right) -\Xi_i^{(t)}\right)\nonumber\\
   &\geq \frac{\widehat{L}_{\mathcal{Z}_1}(\bm{W}^{(t)})}{1+e^{\tilde{O}((\sigma\sigma_0\sqrt{d})^3)}}\nonumber \\
   &\geq \Theta(1)\widehat{L}_{\mathcal{Z}_1}(\bm{W}^{(t)}).\label{eq:necercnevew}
\end{align}

By successively applying \autoref{lem:logfzd} and \autoref{indh:noiseGD}, we obtain:
\begin{align}
   \frac{1}{N} \sum_{i\in\mathcal{Z}_2}\widehat{\mathcal{L}}^{(t)}(\Xi_i^{(t)}) &=\frac{1}{N}\sum_{i\in\mathcal{Z}_2} \frac{1+e^{-\sum_{r=1}^m (\beta c_r^{(t)})^3}}{1+e^{-\sum_{r=1}^m (\beta c_r^{(t)})^3}} \log\left(1+\exp\left(-\Xi_i^{(t)}\right)\right)\nonumber\\
   &\geq \frac{1}{N}\sum_{i\in\mathcal{Z}_2} \frac{1}{1+e^{-\sum_{r=1}^m (\beta c_r^{(t)})^3}} \log\left(1+\exp\left(-\sum_{r=1}^m (\beta c_r^{(t)})^3\right) -\Xi_i^{(t)}\right)\nonumber\\
   &\geq \frac{\widehat{L}_{\mathcal{Z}_2}(\bm{W}^{(t)})}{1+e^{\tilde{O}((\beta\sigma_0)^3)}}\nonumber \\
   &\geq \Theta(1)\widehat{L}_{\mathcal{Z}_2}(\bm{W}^{(t)}).\label{eq:evever3erg}
\end{align}
Combining \eqref{eq:necercnevew} and   \eqref{eq:evever3erg} yields the aimed result.

\end{proof}

Lastly, to obtain \autoref{thm:convrate}, we need to bound $G_r^{(t)}$ which is given by the next lemma.

\begin{lemma}[Gradient on the normalized noise]\label{prop:normgradXij}
For $r\in [m]$, the gradient of the loss $\widehat{L}(\bm{W}^{(t)})$ projected on the normalized noise $\bm{\rchi}$ satisfies with probability $1-o(1)$ for $r\in [m]$:
\begin{equation*}
    \begin{aligned}
      -\mathrm{G}_{r}^{(t)}\geq \frac{\tilde{\Theta}(\sigma\sqrt{d})}{N} \sum_{i\in\mathcal{Z}_2} \ell_i^{(t)} \sum_{j\neq P(\bm{X}_i)}(\Xi_{i,j,r}^{(t)})^2  - \frac{\tilde{O}(\sigma)}{N}\sum_{i\in\mathcal{Z}_1}\sum_{j\neq P(\bm{X}_i)} \ell_i^{(t)}   (\Xi_{i,j,r}^{(t)})^2.
    \end{aligned}
\end{equation*}
\end{lemma}
\begin{proof}[Proof of \autoref{prop:normgradXij}] 
Projecting the gradient (given by \autoref{lem:grad}) on $\bm{\rchi}$ yields: \begin{equation}
\begin{split}\label{eq:Grbd1}
    -\mathrm{G}_{r}^{(t)}&= \frac{3}{N^2}\sum_{i\in \mathcal{Z}_2}\sum_{j\neq P(\bm{X}_i)} \ell_i^{(t)}(\Xi_{i,j,r}^{(t)})^2 \frac{\|\bm{X}_i[j]\|_2^2}{\|\frac{1}{N}\sum_{b\in \mathcal{Z}_2}\sum_{l\neq P(\bm{X}_i)} \bm{X}_b[l]\|_2}\\
    &+\frac{3}{N^2}\sum_{i\in\mathcal{Z}_2}\ell_i^{(t)}\sum_{j\neq P(\bm{X}_i)}\sum_{\substack{k\neq P(\bm{X}_i)\\ k\neq j}} (\Xi_{i,k,r}^{(t)})^2 \left\langle \bm{X}_i[k], \frac{\bm{X}_i[j]}{\|\frac{1}{N}\sum_{b\in \mathcal{Z}_2}\sum_{l\neq P(\bm{X}_i)} \bm{X}_b[l]\|_2}\right\rangle\\
    &+\frac{3}{N^2}\sum_{i\in\mathcal{Z}_2}\sum_{\substack{a\in\mathcal{Z}_2\\a\neq i}}\ell_a^{(t)}\sum_{k\neq P(\bm{X}_a)} (\Xi_{a,k,r}^{(t)})^2 \sum_{j\neq P(\bm{X}_i)} \left\langle \bm{X}_a[k], \frac{\bm{X}_i[j]}{\|\frac{1}{N}\sum_{b\in \mathcal{Z}_2}\sum_{l\neq P(\bm{X}_i)} \bm{X}_b[l]\|_2}\right\rangle\\
    &+\frac{3}{N}\sum_{a\in\mathcal{Z}_1} \sum_{k\neq P(\bm{X}_a)}\ell_a^{(t)} (\Xi_{a,k,r}^{(t)})^2  \left\langle \bm{X}_a[k], \frac{\frac{1}{N}\sum_{i\in\mathcal{Z}_2}\sum_{j\neq P(\bm{X}_i)}\bm{X}_i[j]}{\|\frac{1}{N}\sum_{b\in \mathcal{Z}_2}\sum_{l\neq P(\bm{X}_i)} \bm{X}_b[l]\|_2}\right\rangle.
\end{split}
\end{equation}
We further bound \eqref{eq:Grbd1} as: 
\begin{equation}
\begin{split}\label{eq:Grbd2}
    &\left|\mathrm{G}_{r}^{(t)}+ \frac{3}{N^2}\sum_{i\in \mathcal{Z}_2}\sum_{j\neq P(\bm{X}_i)} \ell_i^{(t)}(\Xi_{i,j,r}^{(t)})^2 \frac{\|\bm{X}_i[j]\|_2^2}{\|\frac{1}{N}\sum_{b\in \mathcal{Z}_2}\sum_{l\neq P(\bm{X}_i)} \bm{X}_b[l]\|_2}\right.\\
    &\left.-\frac{3}{N^2}\sum_{i\in\mathcal{Z}_2}\sum_{a\in\mathcal{Z}_2}\ell_a^{(t)}\sum_{j\neq P(\bm{X}_i)}\sum_{k\neq P(\bm{X}_a) } (\Xi_{a,k,r}^{(t)})^2 \left|\left\langle \bm{X}_a[k], \frac{\bm{X}_i[j]}{\|\frac{1}{N}\sum_{b\in \mathcal{Z}_2}\sum_{l\neq P(\bm{X}_i)} \bm{X}_b[l]\|_2}\right\rangle\right|\right|\\
    &\leq \frac{3}{N}\sum_{a\in\mathcal{Z}_1} \sum_{k\neq P(\bm{X}_a)}\ell_a^{(t)} (\Xi_{a,k,r}^{(t)})^2  \left|\left\langle \bm{X}_a[k], \frac{\frac{1}{N}\sum_{i\in\mathcal{Z}_2}\sum_{j\neq P(\bm{X}_i)}\bm{X}_i[j]}{\|\frac{1}{N}\sum_{b\in \mathcal{Z}_2}\sum_{l\neq P(\bm{X}_i)} \bm{X}_b[l]\|_2}\right\rangle\right|.
\end{split}
\end{equation}
Since $\frac{\frac{1}{N}\sum_{i\in\mathcal{Z}_2}\sum_{j\neq P(\bm{X}_i)}\bm{X}_i[j]}{\|\frac{1}{N}\sum_{b\in \mathcal{Z}_2}\sum_{l\neq P(\bm{X}_i)} \bm{X}_b[l]\|_2}$ is a unit Gaussian vector, using \autoref{prop:dotprodGaussunit}, we bound the right-hand side of \eqref{eq:Grbd2} with probability $1-o(1)$, as: 
\begin{equation}
\begin{split}\label{eq:Grbd3}
    &\left|\mathrm{G}_{r}^{(t)}+ \frac{3}{N^2}\sum_{i\in \mathcal{Z}_2}\sum_{j\neq P(\bm{X}_i)} \ell_i^{(t)}(\Xi_{i,j,r}^{(t)})^2 \frac{\|\bm{X}_i[j]\|_2^2}{\|\frac{1}{N}\sum_{b\in \mathcal{Z}_2}\sum_{l\neq P(\bm{X}_i)} \bm{X}_b[l]\|_2}\right.\\
    &\left.-\frac{3}{N^2}\sum_{i\in\mathcal{Z}_2}\sum_{a\in\mathcal{Z}_2}\ell_a^{(t)}\sum_{j\neq P(\bm{X}_i)}\sum_{k\neq P(\bm{X}_a) } (\Xi_{a,k,r}^{(t)})^2 \left|\left\langle \bm{X}_a[k], \frac{\bm{X}_i[j]}{\|\frac{1}{N}\sum_{b\in \mathcal{Z}_2}\sum_{l\neq P(\bm{X}_i)} \bm{X}_b[l]\|_2}\right\rangle\right|\right|\\
    &\leq \frac{ \sigma}{N}\sum_{a\in\mathcal{Z}_1} \sum_{k\neq P(\bm{X}_a)}\ell_a^{(t)} (\Xi_{a,k,r}^{(t)})^2  .
\end{split}
\end{equation}
Now, using Lemma \autoref{lem:normdotprofefsdved} , we can further lower bound the left-hand side of \eqref{eq:Grbd3} as: 
\begin{equation}
\begin{split}\label{eq:Grbd4}
    &\left|\mathrm{G}_{r}^{(t)}+ \frac{3}{N^2}\sum_{i\in \mathcal{Z}_2}\sum_{j\neq P(\bm{X}_i)} \ell_i^{(t)}(\Xi_{i,j,r}^{(t)})^2 \frac{\|\bm{X}_i[j]\|_2^2}{\|\frac{1}{N}\sum_{b\in \mathcal{Z}_2}\sum_{l\neq P(\bm{X}_i)} \bm{X}_b[l]\|_2}\right.\\
    &\left.-\frac{\tilde{\Theta}( P)}{\sqrt{d} N^2} \sum_{a\in\mathcal{Z}_2}\ell_a^{(t)} \sum_{k\neq P(\bm{X}_a) } (\Xi_{a,k,r}^{(t)})^2\frac{\|\bm{X}_a[k]\|_2^2}{\|\frac{1}{N}\sum_{b\in \mathcal{Z}_2}\sum_{l\neq P(\bm{X}_i)} \bm{X}_b[l]\|_2}  \right|\\
    &\leq \frac{\sigma}{N}\sum_{a\in\mathcal{Z}_1} \sum_{k\neq P(\bm{X}_a)}\ell_a^{(t)} (\Xi_{a,k,r}^{(t)})^2  .
\end{split}
\end{equation}
Rewriting \eqref{eq:Grbd4} yields: 
\begin{equation}
\begin{split}\label{eq:Grbd4vfefvr}
    &\left|\mathrm{G}_{r}^{(t)}+ \frac{\Theta(1)}{N^2}\sum_{i\in \mathcal{Z}_2}\sum_{j\neq P(\bm{X}_i)} \ell_i^{(t)}(\Xi_{i,j,r}^{(t)})^2 \frac{\|\bm{X}_i[j]\|_2^2}{\|\frac{1}{N}\sum_{b\in \mathcal{Z}_2}\sum_{l\neq P(\bm{X}_i)} \bm{X}_b[l]\|_2}\right|\\
    &\leq \frac{\sigma}{N}\sum_{a\in\mathcal{Z}_1} \sum_{k\neq P(\bm{X}_a)}\ell_a^{(t)} (\Xi_{a,k,r}^{(t)})^2  .
\end{split}
\end{equation}
Remark that  $ \frac{1}{N}\sum_{b\in \mathcal{Z}_2}\sum_{l\neq P(\bm{X}_i)} \bm{X}_b[l] \sim \mathcal{N}(0, \frac{\hat{\mu}P}{N}\sigma^2)$. By applying \autoref{lem:rationorm}, we have: 
\begin{align}
    \frac{1}{N}\frac{\|\bm{X}_i[j]\|_2^2}{\|\frac{1}{N}\sum_{b\in \mathcal{Z}_2}\sum_{l\neq P(\bm{X}_i)} \bm{X}_b[l]\|_2}&= \frac{1}{N} \tilde{\Theta}\left(\sigma\sqrt{\frac{dN}{\hat{\mu}P}} \right) = \tilde{\Theta}\left(\sigma\sqrt{\frac{d}{\hat{\mu}NP}} \right) =  \tilde{\Theta} (\sigma\sqrt{d } ),\label{eq:Grbd5}
\end{align}
where we used $P=\tilde{\Theta}(1)$ and $\hat{\mu}N=\tilde{\Theta}(1)$ in the last equality of \eqref{eq:Grbd5}. Plugging this in \eqref{eq:Grbd4vfefvr} yields the desired result.
\end{proof}

\newpage

\

\section{Auxiliary lemmas for GD+M}\label{sec:aux_heifva}

This section presents the auxiliary lemmas needed in \autoref{sec:app_GDM}. 

\subsection{Rewriting derivatives}

\begin{lemma}[Derivatives for GD+M]\label{lem:dervdGDM}
Let $i\in\mathcal{Z}_k$,  for $k\in\{1,2\}.$ Then, $\ell_i^{(t)}= \Theta(1)\widehat{\ell}^{(t)}(\theta)$.
\end{lemma}
\begin{proof}[Proof]
Let $i\in[N].$ Using \autoref{indh:xi_mom}, we have: 
\begin{align*}
    \ell_i^{(t)}&= \mathrm{sigmoid}\left( -\theta^3\sum_{s=1}^m (c_s^{(t)})^3- \sum_{s=1}^m \sum_{j\neq P(\bm{X}_i)} (\Xi_{i,j,s}^{(t)})^3 \right).
\end{align*}
Therefore, we deduce that: 
\begin{align*}
     e^{-\tilde{O}((\sigma\sigma_0\sqrt{d})^3)}\widehat{\ell}^{(t)}(\theta)    \leq\ell_i^{(t)}\leq  e^{\tilde{O}((\sigma\sigma_0\sqrt{d})^3)}\widehat{\ell}^{(t)}(\theta)
\end{align*}
which yields the aimed result.
\end{proof}

\subsection{Signal lemmas}

In this section, we present the auxiliary lemmas needed to prove  \autoref{indh:signal_mom}. We first rewrite the \eqref{eq:GDM_signal} update to take into account the case where the signal $c^{(\tau)}$ becomes large.
\begin{lemma}[Rewriting signal momentum]\label{lem:lateitermomrfejowfr}
For $t\in [T]$, the maximal signal momentum $\mathcal{
G}^{(t)}$ is  bounded as: 
\begin{align*} 
 \mathcal{G}^{(t+1)}&\leq \Theta(1-\gamma)\sum_{\tau=0}^{t} \gamma^{t-\tau}\left(\alpha \nu_1^{(\tau)}\min\{\kappa,  (\alpha c^{(\tau)})^2\} +    \beta \nu_2^{(\tau)}\min\{\kappa, (\beta c^{(\tau)})^2\} \right).
\end{align*}
\end{lemma}
\begin{proof}[Proof of \autoref{lem:lateitermomrfejowfr}] Let $t\in[T]$. Using the signal momentum given by \autoref{lem:signupdateM}, we know that:
\begin{align}\label{eq:kfckewkwjbbjbj}
     \mathcal{G}^{(t+1)}&= \Theta(1-\gamma)\sum_{\tau=0}^{t} \gamma^{t-\tau}\left(\frac{\alpha}{N}\sum_{i\in\mathcal{Z}_1} (\alpha c^{(\tau)})^2\ell_i^{(\tau)} +   \frac{\beta}{N}\sum_{i=1}^N   (\beta c^{(\tau)})^2\ell_i^{(\tau)}\right).
\end{align}

To obtain the desired result, we need to prove for $i\in\mathcal{Z}_1$:
\begin{align}\label{eq:kfecredscsdedsedw}
(   \alpha c^{(t)})^2\ell_i^{(\tau)} &\leq\Theta(1)\min\{\kappa,(\alpha c^{(\tau)})^2\}\ell_i^{(\tau)}.
\end{align}
Indeed, we remark that: 
\begin{align}\label{eq:grafgdgrdalphac}
   (\alpha c^{(\tau)})^2 \ell_i^{(\tau)} &= \frac{\alpha^2 (c^{(\tau)})^2 }{1+\exp\left( \alpha^3\sum_{s=1}^m(c_{s}^{(\tau)})^3+\Xi_{i}^{(\tau)}\right)}.
\end{align}
By using \autoref{indh:xi_mom} and \autoref{indh:signal_mom}, \eqref{eq:grafgdgrdalphac} is bounded as: 
\begin{align}\label{eq:frcedffdvvfsjzn}
   (\alpha c^{(\tau)})^2\ell_i^{(\tau)} &=  \frac{\alpha^3 (c^{(\tau)})^2  }{1+\exp\left( \alpha^2(c^{(\tau)})^3+\alpha^3\sum_{s\neq r_{\max}}(c_s^{(\tau)})^3+\Xi_{i}^{(\tau)}\right)}\nonumber\\
    &\leq \frac{\alpha^2 (c^{(\tau)})^2  }{1+\exp\left( \alpha^3(c^{(\tau)})^3 -\tilde{O}(m\alpha^3\sigma_0^3) - \tilde{O}(mP(\sigma\sigma_0\sqrt{d})^3)\right)}\nonumber\\ 
    &= \frac{\Theta(1) (\alpha c^{(\tau)})^2}{1+\exp( (\alpha c^{(\tau)})^3)}.
\end{align}
Using \autoref{remark}, the sigmoid term in \eqref{eq:frcedffdvvfsjzn} becomes small when $\alpha c^{(\tau)}\geq \kappa^{1/3}$. To summarize, we have: 
\begin{align}\label{eq:alperwhalictcdvwvrw}
    (  \alpha c^{(\tau)})^2 \ell_i^{(\tau)}&=\begin{cases}
                                    0       & \text{if }\alpha c^{(\tau)}\geq \kappa^{1/3}\\
                                      (\alpha c^{(\tau)})^2     \ell_i^{(\tau)}  & \text{otherwise}
                                       \end{cases}.
\end{align}
\eqref{eq:alperwhalictcdvwvrw} therefore implies $ (\alpha c^{(t)})^2 \ell_i^{(t)} \leq \Theta(1)\min\{\kappa^{2/3},(\alpha c^{(t)})^2\}\ell_i^{(t)}$ which implies \eqref{eq:kfecredscsdedsedw}.

A similar reasoning implies for $i\in\mathcal{Z}_2$: 
\begin{align}\label{eq:kfecredscsdedsedvsww}
(   \beta c^{(t)})^2\ell_i^{(\tau)} &\leq\Theta(1)\min\{\kappa,\beta^2 (c^{(t)})^2\}\ell_i^{(\tau)}.
\end{align}
Plugging \eqref{eq:kfecredscsdedsedw} and \eqref{eq:kfecredscsdedsedvsww} in \eqref{eq:kfckewkwjbbjbj} yields the aimed result.
\end{proof}

We proved in \autoref{lem:increase_signalGDM} that after $\mathcal{T}_0$ iterations, the signal $c^{(t)}\geq \tilde{\Omega}(1/\alpha)$ which makes $\nu_1^{(t)}$ small. Besides, in \autoref{lem:ct_large_M}, we show that after $\mathcal{T}_1$ iterations, the signal $c^{(t)}\geq \tilde{\Omega}(1/\beta)$ which makes $\nu_2^{(t)}$ small. We use these two facts to bound the sum over time of signal momentum.
\begin{lemma}[Sum of signal momentum at late stages]\label{lem:latestgmome}
For $t\in[\mathcal{T}_1,T)$, the sum of maximal signal momentum is  bounded as:
\begin{align}
   \sum_{s=\mathcal{T}_1}^t  |\mathcal{
G}^{(s+1)}|\leq \tilde{O}(\alpha\mathcal{T}_0 ) + \tilde{O}(\hat{\mu}\beta\mathcal{T}_1) +\frac{\tilde{O}(1)}{\eta}.
\end{align}
\end{lemma}
\begin{proof}[Proof of \autoref{lem:latestgmome}] Let $s\in[\mathcal{T}_1,T]$. From  \autoref{lem:lateitermomrfejowfr}, the signal momentum is bounded as:
\begin{equation}
    \begin{aligned} \label{eq:kneavnaerrev}
 |\mathcal{G}^{(s+1)}|&\leq \Theta(1-\gamma)\sum_{\tau=0}^{\mathcal{T}_0-1} \gamma^{s-\tau} \alpha \nu_1^{(\tau)}\min\{\kappa, (\alpha c^{(\tau)})^2\}\\
 &+\Theta(1-\gamma)\sum_{\tau=\mathcal{T}_0}^{s} \gamma^{s-\tau} \alpha \nu_1^{(\tau)}\min\{\kappa, (\alpha c^{(\tau)})^2\}\\
 &+  \Theta(1-\gamma)\sum_{\tau=0}^{\mathcal{T}_1-1} \gamma^{s-\tau}  \beta \nu_2^{(\tau)}\min\{\kappa,  (\beta c^{(\tau)})^2\}\\
  &+\Theta(1-\gamma)\sum_{\tau=\mathcal{T}_1}^{s} \gamma^{s-\tau}  \beta \nu_2^{(\tau)}\min\{\kappa,  (\beta c^{(\tau)})^2\}.
\end{aligned}
\end{equation}
We know that for $\tau\geq \mathcal{T}_0,$ $c^{(\tau)}\geq\tilde{\Omega}(1/\alpha)$ and for $\tau\geq\mathcal{T}_1,$ $c^{(\tau)}\geq\tilde{\Omega}(1/\beta)$. Plugging these two facts and using $\nu_1^{(\tau)}\leq 1-\hat{\mu}$ and $\nu_2^{(\tau)}\leq \hat{\mu}$ in \eqref{eq:kneavnaerrev} leads to:
\begin{equation}
   \begin{aligned}\label{eqL:reikkaca}
    \mathcal{G}^{(s+1)}&\leq (1-\hat{\mu})\alpha \tilde{O}(1-\gamma)\sum_{\tau=0}^{\mathcal{T}_0-1} \gamma^{s-\tau}  +\alpha\tilde{O}(1-\gamma)\sum_{\tau=\mathcal{T}_0}^{s} \gamma^{s-\tau}  \nu_1^{(\tau)} \\
    &+\hat{\mu}\beta \tilde{O}(1-\gamma)\sum_{\tau=0}^{\mathcal{T}_1-1} \gamma^{s-\tau} + \beta \tilde{O}(1-\gamma)\sum_{\tau=\mathcal{T}_1}^{s} \gamma^{s-\tau}  \nu_2^{(\tau)} 
\end{aligned} 
\end{equation}
For $\tau\in[\mathcal{T}_0-1]$, we have $\gamma^{s-\tau}\leq\gamma^{s-\mathcal{T}_0+1}$ and for $\tau\in[\mathcal{T}_1-1]$, $\gamma^{s-\tau}\leq\gamma^{s-\mathcal{T}_1+1}$.   From \autoref{lem:nu1bdwefew} and \autoref{lem:ZderivativeM}, we can bound $\nu_1^{(\tau)}$ and $\nu_2^{(\tau)}$. Therefore, \eqref{eqL:reikkaca} is further bounded as: 
\begin{equation}\label{eq:efejoajvav}
   \begin{aligned}
    \mathcal{G}^{(s+1)}&\leq (1-\hat{\mu})\mathcal{T}_0 \alpha \tilde{O}(1-\gamma) \gamma^{s-\mathcal{T}_0+1}  +\frac{  \tilde{O}(1-\gamma)}{\eta}\sum_{\tau=1}^{s-\mathcal{T}_0+1} \frac{\gamma^{s-\mathcal{T}_0+1-\tau}}{\tau }    \\
    &+\hat{\mu}\mathcal{T}_1\beta \tilde{O}(1-\gamma)  \gamma^{s-\mathcal{T}_1+1} + \frac{ \tilde{O}(1-\gamma)}{\eta}\sum_{\tau=1}^{s-\mathcal{T}_1+1} \frac{\gamma^{s-\mathcal{T}_1+1-\tau}}{\tau}  
\end{aligned} 
\end{equation}
We now use \autoref{lem:hfehiaeier} to bound the sum terms in \eqref{eq:efejoajvav}. We have:
\begin{equation}\label{eq:jwjivrejwvie}
   \begin{aligned}
    \mathcal{G}^{(s+1)}&\leq (1-\hat{\mu})\mathcal{T}_0 \alpha \tilde{O}(1-\gamma) \gamma^{s-\mathcal{T}_0+1}+\hat{\mu}\mathcal{T}_1\beta \tilde{O}(1-\gamma)  \gamma^{s-\mathcal{T}_1+1}\\ & +\frac{  \tilde{O}(1-\gamma)}{\eta} \left(\gamma^{s-\mathcal{T}_0}+ \gamma^{(s-\mathcal{T}_0+1)/2}\log\left( \frac{s-\mathcal{T}_0+1}{2}\right) + \frac{1}{1-\gamma} \frac{2}{s-\mathcal{T}_0+1} \right)   \\
    & + \frac{ \tilde{O}(1-\gamma)}{\eta}\left(\gamma^{s-\mathcal{T}_1}+ \gamma^{(s-\mathcal{T}_1+1)/2}\log\left( \frac{s-\mathcal{T}_1+1}{2}\right) + \frac{1}{1-\gamma} \frac{2}{s-\mathcal{T}_1+1} \right).
\end{aligned} 
\end{equation}
We now sum \eqref{eq:jwjivrejwvie} for $s=\mathcal{T}_1,\dots,t$. Using the geometric sum inequality $\sum_{s}\gamma^s\leq 1/(1-\gamma)$ and obtain:
\begin{equation}
    \begin{aligned}\label{eq:ojfeojer}
   \sum_{s=\mathcal{T}_1}^t \mathcal{G}^{(s+1)} &\leq \tilde{O}(\mathcal{T}_0 \alpha) + \tilde{O}(\hat{\mu}\beta\mathcal{T}_1)\\
   &+\frac{\tilde{O}(1)}{\eta}\left(1+(1-\gamma)\log(t)\sum_{s=\mathcal{T}_1}^t(\sqrt{\gamma})^{s-\mathcal{T}_0+1} +\sum_{s=\mathcal{T}_1}^t\frac{2}{s-\mathcal{T}_0+1} \right)\\
   &+\frac{\tilde{O}(1)}{\eta}\left(1+(1-\gamma)\log(t)\sum_{s=\mathcal{T}_1}^t(\sqrt{\gamma})^{s-\mathcal{T}_1+1} +\sum_{s=\mathcal{T}_1}^t\frac{2}{s-\mathcal{T}_1+1} \right)
\end{aligned}
\end{equation}
We plug $\sum_s \sqrt{\gamma}^s\leq 1/(1-\sqrt{\gamma})$ and $\sum_{s=1}^{t-\mathcal{T}_1+1}1/s\leq \log(t)+1$ in \eqref{eq:ojfeojer}. This yields the desired result.
\end{proof}

 \subsection{Noise lemmas}
 
 In this section, we present the technical lemmas to prove \autoref{lem:noise_GDM1}.
 
\begin{lemma}[Bound on noise momentum]\label{prop:mom_gradnoise} Run GD+M on the loss function $\widehat{L}(\bm{W}).$ Let $i\in[N]$, $j\in [P]\backslash\{P(\bm{X}_i)\}$. At a time $t$, the noise momentum is bounded with probability $1-o(1)$ as:  
    \begin{equation*}
    \begin{aligned} 
 &      \left|-G_{i,j,r}^{(t+1)} +\gamma G_{i,j,r}^{(t)} \right|\leq (1-\gamma)  \tilde{O}(\sigma^4\sigma_0^2d^{2}) \nu^{(t)}.
    \end{aligned}
\end{equation*}
\end{lemma} 
\begin{proof}[Proof of \autoref{prop:mom_gradnoise}]
 Let $i\in[N]$ and $j\in [P]\backslash \{P(\bm{X}_i)\}$. 
Combining the \eqref{eq:GDM_noise} update rule and \autoref{lem:noisegrad} to get the noise gradient $\texttt{G}_{i,j,r}^{(t)}$, we obtain
\begin{equation}
    \begin{aligned}\label{eq:diffmoms1}
&\left|-G_{i,j,r}^{(t+1)} +\gamma G_{i,j,r}^{(t)} \right|\\
&\leq\frac{3(1-\gamma)}{N}  \ell_i^{(t)}   (\Xi_{i,j,r}^{(t)})^2 \|\bm{X}_i[j]\|_2^2+\left|\frac{3(1-\gamma)}{N}\sum_{a=1}^N \ell_a^{(t)}\sum_{k\neq P(\bm{X}_a)} (\Xi_{a,k,r}^{(t)})^2\langle \bm{X}_{a}[k],\bm{X}_{i}[j]\rangle\right|.
\end{aligned}
\end{equation}
Using \autoref{thm:hgh_prob_gauss} and \autoref{prop:dotprodGauss}, \eqref{eq:diffmoms1} becomes  with probability $1-o(1),$
\begin{equation}
    \begin{aligned}\label{eq:diffmomvfeercs1}
&\left|-G_{i,j,r}^{(t+1)} +\gamma G_{i,j,r}^{(t)} \right|\\
&\leq\frac{(1-\gamma)\tilde{\Theta}(\sigma^2d)}{N}  \ell_i^{(t)}   (\Xi_{i,j,r}^{(t)})^2 +\frac{(1-\gamma)\tilde{\Theta}(\sigma^2\sqrt{d})}{N}\sum_{a=1}^N \ell_a^{(t)}\sum_{k\neq P(\bm{X}_a)} (\Xi_{a,k,r}^{(t)})^2.
\end{aligned}
\end{equation}

Using $\ell_i^{(t)}/N \leq \nu^{(t)},$ \autoref{indh:xi_mom}, we upper bound the first term in \eqref{eq:diffmomvfeercs1} to get:
\begin{equation}
    \begin{aligned}\label{eq:diffmoms2}
&\left|-G_{i,j,r}^{(t+1)} +\gamma G_{i,j,r}^{(t)} \right|\\
&\leq(1-\gamma) \tilde{O}(\sigma^4\sigma_0^2d^{2})  \nu^{(t)} +\frac{(1-\gamma)\tilde{\Theta}(\sigma^2\sqrt{d})}{N}\sum_{a=1}^N \ell_a^{(t)}\sum_{k\neq P(\bm{X}_a)} (\Xi_{a,k,r}^{(t)})^2.
\end{aligned}
\end{equation}
We upper bound the second term in \eqref{eq:diffmoms2} by again using \autoref{indh:xi_mom}: 
    \begin{equation}
    \begin{aligned} \label{eq:Gdiffnoise}
 &      \left|-G_{i,j,r}^{(t+1)} +\gamma G_{i,j,r}^{(t)} \right|\leq (1-\gamma)\left( \tilde{O}(\sigma^4\sigma_0^2d^{2})  + \tilde{O}(P\sigma_0^2\sigma^4d^{3/2})\right)\nu^{(t)}.
    \end{aligned}
\end{equation}
By using $P\leq \tilde{O}(1)$ and thus, $\tilde{O}(P\sigma_0^2\sigma^4d^{3/2})\leq \tilde{O}(\sigma^4\sigma_0^2d^{2})$  in \eqref{eq:Gdiffnoise}, we obtain the desired result.
\end{proof}

\begin{lemma}\label{lem:momentum_orac_noise2}
Let $t\in[T]$. The noise momentum is bounded as
\begin{align*}
    |G_{i,j,r}^{(t+1)} |\leq   (1-\gamma)\tilde{O}(\sigma^4\sigma_0^2d^{2}) \sum_{\tau=0}^{t} \gamma^{t-1-\tau}\nu^{(\tau)}.
\end{align*}
\end{lemma}
\begin{proof}[Proof of \autoref{lem:momentum_orac_noise2}] Let $\tau \in [T].$ From \autoref{prop:mom_gradnoise}, we know that: 
\begin{align}\label{eq:Gnoisetbd2}
  |G_{i,j,r}^{(\tau+1)}|  \leq |\gamma G_{i,j,r}^{(\tau)} |+ (1-\gamma) \tilde{O}(\sigma^4\sigma_0^2d^{2}) \nu^{(\tau)}.
\end{align}
We unravel the recursion \eqref{eq:Gnoisetbd2} rule for $\tau=0,\dots,t$ and obtain:
\begin{align*} 
  |G_{i,j,r}^{(t+1)}|&\leq   (1-\gamma)\tilde{O}(\sigma^4\sigma_0^2d^{2}) \sum_{\tau=0}^{t} \gamma^{t-\tau} \nu^{(\tau)}.
\end{align*}
\end{proof}

\begin{lemma}[Noise momentum at late stages]\label{lem:noisegradlate}
For $t\in[\mathcal{T}_1,T)$, the sum of noise momentum is bounded as:
\begin{align*}
  \sum_{s=\mathcal{T}_1}^t|G_{i,j,r}^{(s+1)} | \leq 
  \tilde{O}(\sigma^4\sigma_0^2d^{2}) \left(\mathcal{T}_1 +\frac{1}{\eta\beta}\right).
\end{align*}
\end{lemma} 
\begin{proof}[Proof of \autoref{lem:noisegradlate}]
Let $s\in[\mathcal{T}_1,T)$. We first apply \autoref{lem:momentum_orac_noise2} and obtain:
\begin{align}\label{eq:fierhwsdhewaq}
    |G_{i,j,r}^{(s+1)} |\leq   (1-\gamma)\tilde{O}(\sigma^4\sigma_0^2d^{2}) \left(\sum_{\tau=0}^{\mathcal{T}_1-1} \gamma^{s-\tau}\nu^{(\tau)}+\sum_{\tau=\mathcal{T}_1}^{t} \gamma^{s-\tau}\nu^{(\tau)}\right).
\end{align}
 Using the bound from \autoref{lem:ZderivativeM},  \eqref{eq:fierhwsdhewaq} becomes
\begin{align}\label{eq;fveivenisves}
     |G_{i,j,r}^{(s+1)} |\leq   (1-\gamma)\tilde{O}(\sigma^4\sigma_0^2d^{2}) \left(\sum_{\tau=0}^{\mathcal{T}_1-1} \gamma^{s-\tau} +\sum_{\tau=\mathcal{T}_1}^{s}\frac{\gamma^{s-\tau}}{\eta\beta(\tau-\mathcal{T}_1+1)}\right)
\end{align}
For $\tau\in[0,\mathcal{T}_1-1]$, we have $\gamma^{s-1-\tau}\leq \gamma^{s-\mathcal{T}_1+1}$. Plugging these two bounds in \eqref{eq;fveivenisves} implies:
\begin{align}\label{eq:joreoejraare}
    |G_{i,j,r}^{(s+1)} |\leq   (1-\gamma)\tilde{O}(\sigma^4\sigma_0^2d^{2}) \left( \mathcal{T}_1 \gamma^{s-\mathcal{T}_1+1} +\frac{1}{\eta\beta}\sum_{\tau=1}^{s-\mathcal{T}_1+1}\frac{\gamma^{s-\mathcal{T}_1+1-\tau}}{\tau}\right).
\end{align}
We now use \autoref{lem:hfehiaeier} to bound the sum terms in \eqref{eq:joreoejraare}. We have:
\begin{equation}\label{eq:ojerjrvere}
    \begin{aligned}
     &|G_{i,j,r}^{(s+1)} |\\
     &\leq   (1-\gamma)\tilde{O}(\sigma^4\sigma_0^2d^{2}) \mathcal{T}_1 \gamma^{s-\mathcal{T}_1+1} \\
     &+\frac{1-\gamma}{\eta\beta}\tilde{O}(\sigma^4\sigma_0^2d^{2})\left( \gamma^{s-\mathcal{T}_1}+ \gamma^{(s-\mathcal{T}_1+1)/2}\log\left( \frac{s-\mathcal{T}_1+1}{2}\right) +\frac{1}{1-\gamma} \frac{2}{s-\mathcal{T}_1+1}\right).
\end{aligned}
\end{equation}
We now sum \eqref{eq:ojerjrvere} for $s=\mathcal{T}_1,\dots,t$. Using the geometric sum inequality $\sum_s\gamma^s\leq 1/(1-\gamma),$ we obtain:
\begin{equation}\label{eq:oejjoverjorvefwer}
    \begin{aligned}
    \sum_{s=\mathcal{T}_1}^t |G_{i,j,r}^{(s+1)} |\leq   \tilde{O}(\sigma^4\sigma_0^2d^{2}) \mathcal{T}_1  +\frac{\tilde{O}(\sigma^4\sigma_0^2d^{2})}{\eta\beta}\left( \log\left(t\right) + \sum_{s=\mathcal{T}_1}^t\frac{2}{s-\mathcal{T}_1+1}\right).
\end{aligned}
\end{equation} 
We finally use the harmonic series inequality $\sum_{s=1}^{t-\mathcal{T}_1}1/s\leq 1+\log(t)$ in \eqref{eq:oejjoverjorvefwer} to obtain the desired result.
\end{proof}


\subsection{Convergence rate of the training loss using GD+M}\label{sec:convrategd}

In this section, we prove that when using GD+M, the training loss converges sublinearly in our setting. 

\subsubsection{Convergence after learning $\mathcal{Z}_1$ $(t\in[\mathcal{T}_0,T])$} 

\begin{lemma}\label{thm:convratet0alpha_GDM} For $t\in [\mathcal{T}_0, T]$ Using GD+M  with learning rate $\eta$, the loss sublinearly converges to zero as
\begin{align}
    (1-\hat{\mu})\widehat{\mathcal{L}}^{(t)}(\alpha) 
    \leq \tilde{O}\left( \frac{1}{\eta \alpha^2 ( t - \mathcal{T}_0+1)}\right).
\end{align}
\end{lemma}
\begin{proof}[Proof of \autoref{thm:convrate_GDM}]

Let $t\in[\mathcal{T}_0,T].$ Using \autoref{lem:z1dervefvev}, we bound the signal momentum as:
\begin{align}
    -\mathcal{G}^{(t)}&\geq \Theta(1-\gamma)\alpha\sum_{s=\mathcal{T}_0}^t \gamma^{t-s}\widehat{\ell}^{(s)}(\alpha)(\alpha c^{(s)})^2\nonumber\\
    &\geq (1-\hat{\mu})  \Theta(1-\gamma)\alpha(\alpha c^{(t)})^2\widehat{\ell}^{(t)}(\alpha)\sum_{s=\mathcal{T}_0}^t\gamma^{t-s}\nonumber\\
    &\geq (1-\hat{\mu}) \Theta(1)\alpha (\alpha c^{(t)})^2\widehat{\ell}^{(t)}(\alpha) .\label{eq:vkneikre12}
\end{align}
From \autoref{lem:increase_signalGDM}, we know that $c^{(t)}\geq\tilde{\Omega}(1/\alpha).$ Thus, we simplify \eqref{eq:vkneikre12} as:
\begin{align}\label{eq:refojrvres}
     -\mathcal{G}^{(t)}\geq (1-\hat{\mu}) \tilde{\Omega}(\alpha)\widehat{\ell}^{(t)}(\alpha).
\end{align}
We now plug \eqref{eq:refojrvres} in the signal update \eqref{eq:GDM_signal}. 
\begin{align}\label{eq:ctfinalseqveknevkne}
    c^{(t+1)}\geq c^{(t)}+\tilde{\Omega}(\eta\alpha)(1-\hat{\mu}) \widehat{\ell}^{(t)}(\alpha) .
\end{align}
We now apply \autoref{lem:logsigm2} to lower bound \eqref{eq:ctfinalseqveknevkne} by loss terms. We have:
\begin{align}\label{eq:kfeckvevdecnwffwfw}
    c^{(t+1)}\geq c^{(t)}+\tilde{\Omega}(\eta\alpha)(1-\hat{\mu}) \widehat{\mathcal{L}}^{(t)}(\alpha) .
\end{align}
Let's now assume by contradiction that for $t\in[\mathcal{T}_0,T]$, we have:
\begin{align}\label{eq:idwenfdfwekewkn}
    (1-\hat{\mu}) \widehat{\mathcal{L}}^{(t)}(\alpha)  > \frac{\tilde{\Omega}(1)}{\eta\alpha^2 (t-\mathcal{T}_0+1)}.
\end{align}
From the \eqref{eq:GDM_signal} update, we know that $c_r^{(\tau)}$ is a non-decreasing sequence which implies that $\sum_{r=1}^m (\alpha c_r^{(\tau)})^3$ is also non-decreasing for $\tau\in[T]$. Since $x\mapsto \log(1+\exp(-x))$ is non-increasing, this implies that for $s\leq t$, we have:
\begin{align}\label{eq:ferjefriwjrfwe3frc}
  \frac{\tilde{\Omega}(1)}{\eta\alpha^2 (t-\mathcal{T}_0+1)}  < (1-\hat{\mu}) \widehat{\mathcal{L}}^{(t)}(\alpha) \leq (1-\hat{\mu}) \widehat{\mathcal{L}}^{(s)}(\alpha) .
\end{align}
Plugging \eqref{eq:ferjefriwjrfwe3frc} in the update \eqref{eq:idwenfdfwekewkn} yields for $s\in[\mathcal{T}_0,t]$:
\begin{align}\label{eq:ojeojekjjcrwed}
     c^{(s+1)}> c^{(s)}+\frac{\tilde{\Omega}(1)}{\alpha (t-\mathcal{T}_0+1)}
\end{align}
We now sum \eqref{eq:ojeojekjjcrwed} for $s=\mathcal{T}_0,\dots,t$ and obtain:
\begin{align}\label{eq;ceijcrwccwjrcewd}
    c^{(t+1)}> c^{(\mathcal{T}_{0})}+\frac{\tilde{\Omega}(1)(t-\mathcal{T}_0+1)}{\alpha(t-\mathcal{T}_0+1)}> \frac{\tilde{\Omega}(1)}{\alpha},
\end{align}
where we used the fact that $c^{(\mathcal{T}_{0})}\geq \tilde{\Omega}(1/\alpha)>0$ (\autoref{lem:gradlarge}) in the last inequality. Thus, from \autoref{lem:increase_signalGDM} and \eqref{eq;ceijcrwccwjrcewd}, we have for $t\in[\mathcal{T}_0,T]$, $c^{(t)}\geq \tilde{\Omega}(1/\alpha).$ Let's now show that this leads to a contradiction. Indeed, for $t\in[\mathcal{T}_0,T]$, we have:
\begin{align}
   \eta \alpha^2(t-\mathcal{T}_0+1)(1-\hat{\mu}) \widehat{\mathcal{L}}^{(t)}(\alpha)
   &\leq \eta \alpha^2T(1-\hat{\mu})\log\left(1+\exp(-\tilde{\Omega}(1)\right) , \label{eq:ifheihreirhefevef}
\end{align}
where we used $c^{(t)}\geq \tilde{\Omega}(1/\alpha)$ in \eqref{eq:ifheihreirhefevef}. We now apply \autoref{lem:logsigm2} in \eqref{eq:ifheihreirhefevef} and obtain:
\begin{align}\label{eq;frriecwkri}
   \eta \alpha^2(t-\mathcal{T}_0+1)(1-\hat{\mu}) \widehat{\mathcal{L}}^{(t)}(\alpha)\leq \frac{(1-\hat{\mu})  \eta \alpha^2T}{1+\exp(\tilde{\Omega}(1))} .
\end{align}
Given the values of $\alpha,\eta,T$, we finally have:
\begin{align}
     \eta \alpha^2(t-\mathcal{T}_0+1)(1-\hat{\mu}) \widehat{\mathcal{L}}^{(t)}(\alpha)\leq \tilde{O}(1),
\end{align}
which contradicts \eqref{eq:idwenfdfwekewkn}.
\end{proof}

We now link the bound on the loss to the derivative $\nu_1^{(t)}.$

\begin{lemma}\label{lem:nu1bdwefew}
For $t\in[\mathcal{T}_0,T]$, we have $\nu_1^{(t)}\leq    \tilde{O}\left(\frac{1}{\eta (t-\mathcal{T}_0+1) \alpha} \right)$.
\end{lemma}
\begin{proof}[Proof of \autoref{lem:nu1bdwefew}] The proof is similar to the one of \autoref{lem:ZderivativeM}.

\end{proof}
 
\subsubsection{Convergence at late stages $(t\in[\mathcal{T}_1,T])$} 

\begin{lemma}[Convergence rate of the loss]\label{thm:convrate_GDM} For $t\in [\mathcal{T}_1, T]$ Using GD+M  with learning rate $\eta >0$, the loss sublinearly converges to zero as
\begin{align}
    (1-\hat{\mu})\widehat{\mathcal{L}}^{(t)}(\alpha) + \hat{\mu}\widehat{\mathcal{L}}^{(t)}(\beta)
    \leq \tilde{O}\left( \frac{1}{\eta \beta^2 ( t - \mathcal{T}_1+1)}\right).
\end{align}
\end{lemma}
\begin{proof}[Proof of \autoref{thm:convrate_GDM}]

Let $t\in[\mathcal{T}_1,T].$ From \autoref{lem:grdsigndecr}, we know that the signal gradient is bounded as $-\mathscr{G}^{(t)}\geq -\mathscr{G}^{(s)}$ for $s\in [\mathcal{T}_1, t].$
\begin{align}
    -\mathcal{G}^{(t)}&=-\gamma^{t-\mathcal{T}_1}\mathcal{G}^{(\mathcal{T}_1)}-(1-\gamma)\sum_{s=\mathcal{T}_1}^t \gamma^{t-s}\mathscr{G}^{(s)}\nonumber\\
    &\geq -(1-\gamma)\sum_{s=\mathcal{T}_1}^t \gamma^{t-s}\mathscr{G}^{(s)}\nonumber\\
    &\geq -(1-\gamma)\mathscr{G}^{(t)}\sum_{s=\mathcal{T}_1}^t\gamma^{t-s}\nonumber\\
    &=-\Theta(1)\mathscr{G}^{(t)}.\label{eq:vkneikre}
\end{align}

From \autoref{lem:signgrad}, the signal gradient is:
\begin{align}\label{eq:gradMtmpwdwde}
    -\mathscr{G}^{(t)}= \Theta(1)\left(\alpha^3\widehat{\ell}^{(t)}(\alpha)+ \beta^3\widehat{\ell}^{(t)}(\beta)\right)(c^{(t)})^2. 
\end{align}
From \autoref{lem:ct_large_M}, we know that $c^{(t)} \geq \tilde{\Omega}(1/\beta)$. Thus, we simplify \eqref{eq:gradMtmpwdwde} as:
\begin{align}\label{eq:gradbefehvi}
     -\mathscr{G}^{(t)}\geq \tilde{\Omega}(\beta)\left((1-\hat{\mu}) \widehat{\ell}^{(t)}(\alpha)+ \hat{\mu}\beta\widehat{\ell}^{(t)}(\beta)\right). 
\end{align}
By combining \eqref{eq:vkneikre} and  \eqref{eq:gradbefehvi}, we finally obtain: 
\begin{align} \label{eq:gbdtmpere}
    -\mathcal{G}^{(t)}\geq \tilde{\Omega}(\beta)\left((1-\hat{\mu}) \widehat{\ell}^{(t)}(\alpha)+ \hat{\mu} \widehat{\ell}^{(t)}(\beta)\right).
\end{align}

We now plug \eqref{eq:gbdtmpere} in the signal update \eqref{eq:GDM_signal}. 
\begin{align}\label{eq:ctfinalseq}
    c^{(t+1)}\geq c^{(t)}+\tilde{\Omega}(\eta\beta)\left((1-\hat{\mu}) \widehat{\ell}^{(t)}(\alpha)+ \hat{\mu} \widehat{\ell}^{(t)}(\beta)\right).
\end{align}
We now apply \autoref{lem:logsigm2} to lower bound \eqref{eq:ctfinalseq} by loss terms. We have:
\begin{align}\label{eq:kfeckvevcnfw}
    c^{(t+1)}\geq c^{(t)}+\tilde{\Omega}(\eta\beta)\left((1-\hat{\mu}) \widehat{\mathcal{L}}^{(t)}(\alpha)+ \hat{\mu} \widehat{\mathcal{L}}^{(t)}(\beta)\right).
\end{align}
Let's now assume by contradiction that for $t\in[\mathcal{T}_1,T]$, we have:
\begin{align}\label{eq:idwenkewkn}
    (1-\hat{\mu}) \widehat{\mathcal{L}}^{(t)}(\alpha)+ \hat{\mu} \widehat{\mathcal{L}}^{(t)}(\beta) > \frac{\tilde{\Omega}(1)}{\eta\beta^2 (t-\mathcal{T}_1+1)}.
\end{align}
From the \eqref{eq:GDM_signal} update, we know that $c_r^{(\tau)}$ is a non-decreasing sequence which implies that $\sum_{r=1}^m (\theta c_r^{(\tau)})^3$ is also non-decreasing for $\tau\in[T]$. Since $x\mapsto \log(1+\exp(-x))$ is non-increasing, this implies that for $s\leq t$, we have:
\begin{align}\label{eq:ferjefriwjc}
  \frac{\tilde{\Omega}(1)}{\eta\beta^2 (t-\mathcal{T}_1+1)}  < (1-\hat{\mu}) \widehat{\mathcal{L}}^{(t)}(\alpha)+ \hat{\mu} \widehat{\mathcal{L}}^{(t)}(\beta) \leq (1-\hat{\mu}) \widehat{\mathcal{L}}^{(s)}(\alpha)+ \hat{\mu} \widehat{\mathcal{L}}^{(s)}(\beta).
\end{align}
Plugging \eqref{eq:ferjefriwjc} in the update \eqref{eq:kfeckvevcnfw} yields for $s\in[\mathcal{T}_1,t]$:
\begin{align}\label{eq:ojeojecrwed}
     c^{(s+1)}> c^{(s)}+\frac{\tilde{\Omega}(1)}{\beta (t-\mathcal{T}_1+1)}
\end{align}
We now sum \eqref{eq:ojeojecrwed} for $s=\mathcal{T}_1,\dots,t$ and obtain:
\begin{align}\label{eq;ceijcrjrcewd}
    c^{(t+1)}> c^{(\mathcal{T}_{1})}+\frac{\tilde{\Omega}(1)(t-\mathcal{T}_1+1)}{\beta(t-\mathcal{T}_1+1)}> \frac{\tilde{\Omega}(1)}{\beta},
\end{align}
where we used the fact that $c^{(\mathcal{T}_{1})}\geq \tilde{\Omega}(1/\beta)>0$ (\autoref{lem:gradlarge}) in the last inequality. Thus, from \autoref{lem:gradlarge} and \eqref{eq;ceijcrjrcewd}, we have for $t\in[\mathcal{T}_1,T]$, $c^{(t)}\geq \tilde{\Omega}(1/\beta).$ Let's now show that this leads to a contradiction. Indeed, for $t\in[\mathcal{T}_1,T]$, we have: 
\begin{align}
   &\eta \beta^2(t-\mathcal{T}_1+1)  \left((1-\hat{\mu}) \widehat{\mathcal{L}}^{(t)}(\alpha)+ \hat{\mu} \widehat{\mathcal{L}}^{(t)}(\beta)\right)\nonumber\\
   \leq&  \eta \beta^2T\left((1-\hat{\mu})\log\left(1+\exp(-(\alpha c^{(t)})^3 - \sum_{r\neq r_{\max}}(\alpha c_r^{(t)})^3\right)\right.\nonumber\\
   &\left.\qquad +\hat{\mu}\log\left(1+\exp(-(\beta c^{(t)})^3 - \sum_{r\neq r_{\max}}(\beta c_r^{(t)})^3\right)\right)\nonumber\\
   \leq&  \eta \beta^2T\left((1-\hat{\mu})\log\left(1+\exp(-\tilde{\Omega}(\alpha^3/\beta^3)\right)+\hat{\mu}\log\left(1+\exp(-\tilde{\Omega}(1)\right) \right), \label{eq:ifheihreirhef}
\end{align}
where we used $\sum_{r\neq r_{\max}} (c_r^{(t)})^3\geq -m\tilde{O}(\sigma_0^3)$ and $c^{(t)}\geq \tilde{\Omega}(1/\beta)$ in \eqref{eq:ifheihreirhef}. We now apply \autoref{lem:logsigm2} in \eqref{eq:ifheihreirhef} and obtain:
\begin{align}\label{eq;frrikri}
    \hspace{-.3cm}\eta \beta^2(t-\mathcal{T}_1+1)  \left((1-\hat{\mu}) \widehat{\mathcal{L}}^{(t)}(\alpha)+ \hat{\mu} \widehat{\mathcal{L}}^{(t)}(\beta)\right)\leq \frac{(1-\hat{\mu})  \eta \beta^2T}{1+\exp(\tilde{\Omega}(\alpha^3/\beta^3))}+\frac{\hat{\mu} \eta \beta^2T}{1+\exp(\tilde{\Omega}(1))}.
\end{align}
Given the values of $\alpha,\beta,\eta,T,\hat{\mu}$, we finally have:
\begin{align}
     &\eta \beta^2(t-\mathcal{T}_1+1)  \left((1-\hat{\mu}) \widehat{\mathcal{L}}^{(t)}(\alpha)+ \hat{\mu} \widehat{\mathcal{L}}^{(t)}(\beta)\right)\leq \tilde{O}(1),
\end{align}
which contradicts \eqref{eq:idwenkewkn}.
\end{proof}

\subsubsection{Auxiliary lemmas}

We now provide an auxiliary lemma needed to obtain \eqref{thm:convrate_GDM}. 
\begin{lemma}\label{lem:grdsigndecr}
Let $t\in[\mathcal{T}_1,T].$ Then, the signal gradient decreases i.e. $-\mathscr{G}^{(s)}\geq-\mathscr{G}^{(t)} $ for $s\in [\mathcal{T}_1, t].$
\end{lemma}
\begin{proof}[Proof of \autoref{lem:grdsigndecr}] From \autoref{lem:signgrad}, we know that 
\begin{align}\label{eq:vfejve}
    -\mathscr{G}^{(t)}= \Theta(1)\left(\alpha^3\widehat{\ell}^{(t)}(\alpha)+ \beta^3\widehat{\ell}^{(t)}(\beta)\right)(c^{(t)})^2.
\end{align}
Since $c_r^{(t)}\geq -\tilde{O}(\sigma_0)$, we bound \eqref{eq:vfejve} as:
\begin{align}
    -\mathscr{G}^{(t)}\leq \Theta(1)\left(\alpha^3\mathfrak{S}((\alpha c^{(t)})^3)+ \beta^3\mathfrak{S}((\beta c^{(t)})^3)\right)(c^{(t)})^2.
\end{align}
The function $x\mapsto x^2\mathfrak{S}(x^3)$ is non-increasing for $x\geq 1.$ Since $c^{(t)}\geq \tilde{\Omega}(1/\beta)$, we have:
\begin{align}
     -\mathscr{G}^{(t)}\leq \Theta(1)\left(\alpha^3\mathfrak{S}((\alpha c^{(\mathcal{T}_1)})^3)+ \beta^3\mathfrak{S}((\beta c^{(\mathcal{T}_1)})^3)\right)(c^{(\mathcal{T}_1)})^2= -\mathscr{G}^{(\mathcal{T}_1)}.
\end{align}
\end{proof}
\begin{lemma}\label{lem:z1dervefvev} Let $t\in[\mathcal{T}_0,T].$ Then, the signal $\mathcal{Z}_1$ gradient decreases i.e. $\widehat{\ell}^{(s)}(\alpha)(\alpha c^{(s)})^2\geq \widehat{\ell}^{(t)}(\alpha)(\alpha c^{(t)})^2$ for $s\in [\mathcal{T}_0, t].$
\end{lemma}
\begin{proof}[Proof of \autoref{lem:z1dervefvev}] The proof is similar to the one of \autoref{lem:grdsigndecr}.

\end{proof}

\section{Useful lemmas}

In this section, we provide the probabilistic and optimization lemmas and the main inequalities used above.

\subsection{Probabilistic lemmas}

In this section, we introduce the probabilistic lemmas used in the proof. 


\subsubsection{High-probability bounds}\label{sec:high_prob}

\begin{lemma}\label{lem:sym_rv}
The sum of of symmetric random variables is  symmetric. 
\end{lemma}

\begin{lemma}[Sum of sub-Gaussians \citep{vershynin2018high}]\label{lem:sumsubG} Let $\sigma_1,\sigma_2>0.$ Let $X$ and $Y$ respetively be $\sigma_1$- and $\sigma_2$-subGaussian random variables. Then, $X+Y$ is $\sqrt{\sigma_1+\sigma_2}$-subGaussian random variable. 
\end{lemma}

\begin{lemma}[High probability bound subGaussian \citep{vershynin2018high}]\label{eq:subGhighbd} Let $t>0.$ Let $X$ be a $\sigma$-subGaussian random variable. Then, we have: 
\begin{align*}
    \mathbb{P}\left[|X|>t\right]\leq 2e^{-\frac{t^2}{2\sigma^2}}.
\end{align*}

\end{lemma} 

\begin{theorem}[Concentration of Lipschitz functions of Gaussian variables \citep{wainwright2019high}]\label{thm:lipschitz} Let $X_1,\dots,X_N$ be $N$ i.i.d. random variables such that $X_i\sim\mathcal{N}(0,\sigma^2)$ and $X:=(X_1,\dots,X_n).$ Let $f\colon \mathbb{R}^d\rightarrow\mathbb{R}$ be $L$-Lipschitz with respect to the Euclidean norm. Then, 
\begin{align}
    \mathbb{P}[|f(X)-\mathbb{E}[f(X)]|\geq t]\leq 2  e^{-\frac{t^2}{2L}}.
\end{align}

\end{theorem}
\begin{lemma}[Expectation of Gaussian vector \citep{wainwright2019high}]\label{thm:gauss_vector} Let $X\in\mathbb{R}^d$ be a Gaussian vector such that $X\sim \mathcal{N}(0, \sigma^2\mathbf{I}).$ Then, its expectation is equal to $\mathbb{E}[\|X\|_2] = \Theta(\sigma\sqrt{d}).$
 
\end{lemma} 
\begin{lemma}[High-probability bound on squared norm of Gaussian]\label{thm:hgh_prob_gauss} Let $\bm{X}\in\mathbb{R}^d$ be a Gaussian vector such that $X\sim \mathcal{N}(0, \sigma^2\mathbf{I}_d).$ Then, with probability at least $1-o(1)$, we have $ \|\bm{X}\|_2^2 = \Theta(\sigma^2d).$
\end{lemma}
\begin{proof}[Proof of \autoref{thm:hgh_prob_gauss}. ]
We know that the $\|\cdot\|_2$ is $1$-Lipschitz and by applying \autoref{thm:lipschitz}, we therefore have::
\begin{align}\label{eq:lip_norm2}
    \mathbb{P}\left[\left|\|\bm{X} \|_2-\mathbb{E}[\|\bm{X} \|_2 ]\right|>\epsilon\right]&\leq \exp\left( -\frac{\epsilon^2}{2\sigma^2}\right).
\end{align}
By rewriting \eqref{eq:lip_norm2} and using \autoref{thm:gauss_vector}, we have with probability $1-\delta,$
\begin{align}
  & 
  \Theta(\sigma\sqrt{d})-\sigma\sqrt{2\log\left(\frac{1}{\delta}\right)}\leq\|\bm{X} \|_2\leq  \Theta(\sigma\sqrt{d})+\sigma\sqrt{2\log\left(\frac{1}{\delta}\right)}.\label{eq:almost_fin_sqgauss}
\end{align}
By squaring \eqref{eq:almost_fin_sqgauss} and using $(a+b)^2\leq a^2+b^2$, we obtain the aimed result.
\end{proof}

\begin{lemma}[Precise bound  on squared norm of Gaussian]\label{cor:sqnomg} Let $\bm{X}\in\mathbb{R}^d$ be a Gaussian vector such that $X\sim \mathcal{N}(0, \sigma^2\mathbf{I}_d).$ Then, we have: 
\begin{align*}
    \mathbb{P}\left[ \|X\|_2\in \left[\frac{1}{2}\sigma\sqrt{d}, \frac{3}{2}\sigma\sqrt{d} \right]\right] \geq 1 -e^{-d/8}.
\end{align*}
\end{lemma}
\begin{proof}[Proof of \autoref{cor:sqnomg}] We know that the $\|\cdot\|_2$ is $1$-Lipschitz and by applying \autoref{thm:lipschitz}, we therefore have:
\begin{align}\label{eq:lip_nefrerform2}
    \mathbb{P}\left[\left|\|\bm{X} \|_2-\mathbb{E}[\|\bm{X} \|_2 ]\right|>\epsilon\right]&\leq \exp\left( -\frac{\epsilon^2}{2\sigma^2}\right).
\end{align}
We use \autoref{thm:gauss_vector} and set $\epsilon=\frac{\sigma\sqrt{d}}{2}$ in \eqref{eq:lip_nefrerform2} to finally get: 
\begin{align*}
    \mathbb{P}\left[\left|\|\bm{X} \|_2-\mathbb{E}[\|\bm{X} \|_2 ]\right|>\frac{\sigma\sqrt{d}}{2}\right]&\leq \exp\left( -\frac{d}{8}\right).
\end{align*}
\end{proof}


\begin{lemma}[High-probability bound on dot-product of  Gaussians]\label{prop:dotprodGauss} Let $X$ and $Y$ be two independent Gaussian vectors in $\mathbb{R}^d$ such that $\bm{X},\bm{Y}$ independent and $\bm{X}\sim \mathcal{N}(0,\sigma^2 \mathbf{I})$ and $\bm{Y}\sim \mathcal{N}(0,\sigma_0^2 \mathbf{I}_d)$. Assume that $\sigma\sigma_0\leq 1/d.$ Then, with probability $1-o(1)$, we have:
\begin{align*}
    |\langle \bm{X},\bm{Y}\rangle|&\leq \tilde{O}(\sigma\sigma_0\sqrt{d}).
\end{align*}
\end{lemma}
\begin{proof}[Proof of \autoref{prop:dotprodGauss} ] Let's define $Z:=\langle \bm{X},\bm{Y}\rangle.$ We first remark that $Z$ is a sub-exponential random variable. Indeed, the generating moment function is: 
\begin{align*}
  M_Z(t)=  \mathbb{E}[e^{t\langle X,Y\rangle}]=\frac{1}{(1-\sigma^2\sigma_0^2t^2)^{d/2}}=e^{-\frac{d}{2}\log(1-\sigma^2\sigma_0^2t^2)}\leq e^{\frac{d\sigma^2\sigma_0^2t^2}{2}}, \qquad \text{for }t\leq \frac{1}{\sigma\sigma_0}.
\end{align*}
where we used $\log(1-x)\geq -x$ for $x<1$ in the last inequality. Therefore, by definition of a sub-exponential variable, we have: 
\begin{align}\label{eq:sub_exp_def}
    \mathbb{P}\left[ |Z-\mathbb{E}[Z]| >\epsilon\right]&\leq \begin{cases}
                             2e^{-\frac{\epsilon^2}{2d\sigma^2\sigma_0^2}} & \text{for }0\leq \epsilon\leq d\sigma\sigma_0\\
                             2e^{-\frac{\epsilon}{2\sigma\sigma_0}} & \text{for }\epsilon\geq d\sigma\sigma_0
                           \end{cases}.
\end{align}
 Since $\sigma^2d\leq 1$ and $\epsilon\in[0,1],$ \eqref{eq:sub_exp_def} is bounded as: 
 \begin{align}\label{eq:proba_XY_tmp}
     \mathbb{P}\left[ |Z-\mathbb{E}[Z]| >\epsilon\right]\leq 2e^{-\frac{\epsilon^2}{2d\sigma^2\sigma_0^2}} .
 \end{align}
 We know that $\mathbb{E}[Z]= M'(0)= \left(d(1-\sigma^2\sigma_0^2t^2)^{-\frac{d}{2}-1}\sigma^2\sigma_0^2t\right)(0)=0.$ By plugging this expectation in \eqref{eq:proba_XY_tmp}, we have with probability $1-\delta,$
 \begin{align*}
       |\langle \bm{X},\bm{Y}\rangle|&\leq \sigma\sigma_0\sqrt{2d\log\left(\frac{2}{\delta}\right)}.
\end{align*}
 

\end{proof}

\begin{lemma}[High-probability bound on dot-product of Gaussians]\label{prop:dotprodGaussunit} Let $\bm{X}$ and $\bm{Y}$ be two independent Gaussian vectors in $\mathbb{R}^d$ such that $\bm{X},\bm{Y}\sim \mathcal{N}(0,\sigma^2 \mathbf{I}_d).$   Then, with probability $1-\delta,$ we have: 
\begin{align*}
    \left|\left\langle \frac{\bm{X}}{\|\bm{X}\|_2},\bm{Y}\right\rangle\right|&\leq \tilde{O}(\sigma).
\end{align*}
\end{lemma}
\begin{proof}[Proof of \autoref{prop:dotprodGauss} ] Let $\bm{U}:=\bm{X}/\|\bm{X}\|_2$ and $Z:=\langle \bm{U},\bm{Y}\rangle.$ We know that the pdf of $\bm{U}$ in polar coordinates is $f_{\bm{U}}(\theta)=\frac{\Gamma(d/2)}{2\pi^{d/2}}.$ Therefore, the generating moment function of $Z$ is:
\begin{align}
    &M_Z(t)= \int_{\mathbb{S}^{d-1}}\int_{\mathbb{R}^d} e^{t\langle \bm{u},\bm{y}\rangle}f_{\bm{U}}(\bm{u})f_{\bm{Y}}(\bm{y})d\bm{u}d\bm{y}\nonumber\\
    &=\frac{\Gamma(d/2)}{2\pi^{d/2}(2\pi\sigma^2)^{d/2}} \int_{\mathbb{S}^{d-1}}\int_{\mathbb{R}^d} e^{t\langle \bm{u},\bm{y}\rangle} e^{-\frac{\|\bm{y}\|_2^2}{2\sigma^2}}d\bm{y}d\bm{u}\nonumber\\
    &=\frac{\Gamma(d/2)}{2\pi^{d/2}(2\pi\sigma^2)^{d/2}} \int_{\mathbb{S}^{d-1}}\int_{\mathbb{R}^d}  e^{-\frac{\|\bm{y}-t\sigma^2 \bm{u}\|_2^2}{2\sigma^2}} e^{\frac{t^2\sigma^2\|\bm{u}\|_2^2}{2}}d\bm{y}d\bm{u}\nonumber\\    &=\frac{\Gamma(d/2)}{2\pi^{d/2}(2\pi\sigma^2)^{d/2}} \int_{\mathbb{S}^{d-1}}    e^{\frac{\sigma^2t^2\|\bm{u}\|_2^2}{2}}d\bm{u}\nonumber\\
    &=\frac{\Gamma(d/2)}{2\pi^{d/2}(2\pi\sigma^2)^{d/2}} \int_{\mathbb{S}^{d-1}}    e^{ \frac{\sigma^2t^2}{2}}d\bm{u}\nonumber\\
    &=e^{ \frac{\sigma^2t^2}{2}}.\label{eq:Mz}
\end{align}
\eqref{eq:Mz} indicates that $Z$ is a sub-Gaussian random variable of parameter $\sigma$. By definition, it satisfies
\begin{align}\label{eq:PsubG}
    \mathbb{P}[ | Z | >\epsilon ]\leq 2e^{-\frac{\epsilon^2}{2\sigma^2}}.
\end{align}
 Setting $\delta=2e^{-\frac{\epsilon^2}{2\sigma^2}}$ in \eqref{eq:PsubG} yields that we have with probability $1-\delta,$
 \begin{align*}
    \left|\left\langle \frac{\bm{X}}{\|\bm{X}\|_2},\bm{Y}\right\rangle\right|&\leq \sqrt{2\log\left(\frac{2}{\delta}\right)}.
\end{align*}
 
\end{proof}

\begin{lemma}[High probability bound for ratio of norms]\label{lem:rationorm} Let $\bm{X}_1,\dots,\bm{X}_n$ i.i.d. vectors from $\mathcal{N}(0,\sigma^2\mathbf{I}).$ Then, with probability $1-o(1)$, we have: 
\begin{align}
    \frac{\|\bm{X}_1\|_2^2}{\|\sum_{i=1}^n \bm{X}_i\|_2} = \tilde{\Theta}\left(\sigma\sqrt{\frac{d}{n}} \right).
\end{align}
\end{lemma}
\begin{proof}[Proof of \autoref{lem:rationorm}] We know that for $\bm{X}_1\sim \mathcal{N}(0,\sigma^2 d)$, we have: 
\begin{align}\label{eq:hgprobaX1norm2}
    \mathbb{P}\left[\|\bm{X}_1\|_2^2\in \left[\frac{\sigma^2 d}{4},\frac{9\sigma^2d}{4}\right]\right]\leq e^{-d/8}.
\end{align}
Therefore, using the law of total probability and \eqref{eq:hgprobaX1norm2}, we have: 
\begin{align}
    \mathbb{P}\left[ \frac{\|\bm{X}_1\|_2^2}{\|\sum_{i=1}^n \bm{X}_i\|_2}>t \right]&= \mathbb{P}\left[ \frac{\|\bm{X}_1\|_2^2}{\|\sum_{i=1}^n \bm{X}_i\|_2}>t \;\middle| \;\|\bm{X}_1\|_2^2 >\frac{9\sigma^2d}{4}\right]\mathbb{P}\left[\|\bm{X}_1\|_2^2 >\frac{9\sigma^2d}{4}\right]\nonumber\\
    &+\mathbb{P}\left[ \frac{\|\bm{X}_1\|_2^2}{\|\sum_{i=1}^n \bm{X}_i\|_2}>t \;\middle| \;\|\bm{X}_1\|_2^2 <\frac{9\sigma^2d}{4}\right]\mathbb{P}\left[\|\bm{X}_1\|_2^2 <\frac{9\sigma^2d}{4}\right]\nonumber\\
    &\leq e^{-d/8}+\mathbb{P}\left[ \frac{\|\bm{X}_1\|_2^2}{\|\sum_{i=1}^n \bm{X}_i\|_2}>t \;\middle| \;\|\bm{X}_1\|_2^2 <\frac{9\sigma^2d}{4}\right].\label{eq:tmrnjvs}
\end{align}
Now, we can further bound \eqref{eq:tmrnjvs} as: 
\begin{align}
    \mathbb{P}\left[ \frac{\|\bm{X}_1\|_2^2}{\|\sum_{i=1}^n X_i\|_2}>t \right]&\leq e^{-d/8}+\mathbb{P}\left[ \frac{9\sigma^2d}{4t} >\|\sum_{i=1}^n \bm{X}_i\|_2  \right].\label{eq:tmrnjvsferf}
\end{align}
Since $\sum_{i=1}^n \bm{X}_i\sim\mathcal{N}(0,n\sigma^2\mathbf{I}_d)$, we also have 
\begin{align}\label{eq:hgprobaX1norm2122}
    \mathbb{P}\left[\|\sum_{i=1}^n \bm{X}_i\|_2\in \left[\frac{\sigma\sqrt{nd}}{2},\frac{3\sigma\sqrt{nd}}{2}\right]\right]\leq e^{-d/8}.
\end{align}
Therefore by setting $t=\frac{3\sigma}{2}\sqrt{\frac{d}{n}}$, we obtain: \begin{align}
    \mathbb{P}\left[ \frac{\|\bm{X}_1\|_2^2}{\|\sum_{i=1}^n \bm{X}_i\|_2}>\frac{3\sigma}{2}\sqrt{\frac{d}{n}} \right] \leq 2 e^{-d/8}.
\end{align}
Doing the similar reasoning for the lower bound yields: 
\begin{align}
    \mathbb{P}\left[ \frac{\|\bm{X}_1\|_2^2}{\|\sum_{i=1}^n \bm{X}_i\|_2}<\frac{\sigma}{2}\sqrt{\frac{d}{n}} \right] \leq 2 e^{-d/8}.
\end{align}
\end{proof}

\begin{lemma}[High probability bound norms vs dot product]\label{lem:normdotprofefsdved} Let $\bm{X}_1,\dots,\bm{X}_n$ i.i.d. vectors from $\mathcal{N}(0,\sigma^2\mathbf{I}_d).$ Then, with probability $1-o(1),$ we have: 
\begin{align}
     \frac{\sqrt{d}|\langle \bm{X}_1,\bm{X}_2\rangle|}{\|\sum_{i=1}^N \bm{X}_i\|_2}\leq  \frac{\|\bm{X}_1\|_2^2}{\|\sum_{i=1}^N \bm{X}_i\|_2}.
\end{align}
\end{lemma}
\begin{proof}[Proof of \autoref{lem:normdotprofefsdved}] To show the result, it's enough to upper bound the following probability: 
\begin{align}
    \mathbb{P}\left[\|\bm{X}_1\|_2^2>\sqrt{d}|\langle \bm{X}_1,\bm{X}_2\rangle|\right].
\end{align}
By using the law of total probability we have: 
\begin{align}
     &\mathbb{P}\left[\|\bm{X}_1\|_2^2>\sqrt{d}|\langle \bm{X}_1,\bm{X}_2\rangle|\right]\nonumber\\
     =& \mathbb{P}\left[\|\bm{X}_1\|_2^2>\sqrt{d}|\langle \bm{X}_1,\bm{X}_2\rangle|\; \middle| \; \|\bm{X}_1\|_2^2\in \left[\frac{\sigma^2 d}{2},\frac{9\sigma^2d}{4}\right]\right]\mathbb{P}\left[\|\bm{X}_1\|_2^2\in \left[\frac{\sigma^2 d}{2},\frac{9\sigma^2}{4}\right]\right]\nonumber\\
     +&\mathbb{P}\left[\|\bm{X}_1\|_2^2>\sqrt{d}|\langle \bm{X}_1,\bm{X}_2\rangle|\; \middle| \; \|\bm{X}_1\|_2^2\not\in \left[\frac{\sigma^2 d}{2},\frac{9\sigma^2d}{4}\right]\right]\mathbb{P}\left[\|\bm{X}_1\|_2^2\not\in \left[\frac{\sigma^2 d}{2},\frac{9\sigma^2}{4}\right]\right]\nonumber\\
     \leq & \mathbb{P}\left[\|\bm{X}_1\|_2^2>\sqrt{d}|\langle \bm{X}_1,\bm{X}_2\rangle|\; \middle| \; \|\bm{X}_1\|_2^2\in \left[\frac{\sigma^2 d}{2},\frac{9\sigma^2d}{4}\right]\right] +e^{-d/8},\label{esqerfrfn}
\end{align}
where we used \autoref{cor:sqnomg} in \eqref{esqerfrfn}. Using \autoref{cor:sqnomg} again, we can simplify \eqref{esqerfrfn} as: 
\begin{align*}
    \mathbb{P}\left[\|\bm{X}_1\|_2^2>\sqrt{d}|\langle \bm{X}_1,\bm{X}_2\rangle|\right]& \leq \mathbb{P}\left[\frac{9\sigma^2\sqrt{d}}{4}> |\langle \bm{X}_1,\bm{X}_2\rangle|\right]+e^{-d/8}\\
     &\leq 2e^{-d/8}.
\end{align*}
\end{proof}

\subsubsection{Anti-concentration of Gaussian polynomials}\label{sec:anti_poly}

\begin{theorem}[Anti-concentration of Gaussian polynomials \citep{carbery2001distributional,Lovett2010AnEP}] \label{prop:anticoncent} 
Let $P(x)=P(x_1,\dots,x_n)$ be a degree $d$ polynomial and $x_1,\dots,x_n$ be i.i.d. Gaussian univariate random variables. Then, the following holds for all $d,n$.
\begin{align*}
    \mathbb{P}\left[|P(x)|\leq \epsilon \mathrm{Var}[P(x)]^{1/2}\right]\leq O(d) \epsilon^{1/d}.
\end{align*}
\end{theorem}

\begin{lemma}[Gaussians and Hermite]\label{prop:hermite_meanvargauss} 
Let $\mathcal{P}(x_1,\dots,x_P)=\sum_{k=1}^d \sum_{\mathcal{I}\subset [P]: |\mathcal{I}|=k} c_{\mathcal{I}} \prod_{i\in \mathcal{I}}x_i$ be a degree $d$ polynomial where $x_1,\dots,x_P\overset{i.i.d.}{\sim}\mathcal{N}(0,\sigma^2)$  and $c_{\mathcal{I}}\in\mathbb{R}$.\\
Let $\mathcal{H}(x)=\sum_{e\in \mathbb{N}^P:|e|\leq d}c_{e}^{H}\prod_{i=1}^P H_{e_i}(x_i)$ be the corresponding Hermite polynomial to $\mathcal{P}$ where $\{H_{e_k}\}_{k=1}^d$ is the Hermite polynomial basis. Then, the variance of $P$ is given by $\mathrm{Var}[P(x)^2]=\sum_{e} |c_e^{H}|^2.$
\end{lemma}

\begin{lemma}\label{prop:vr_xi}
Let $\{\bm{v}_r\}_{r=1}^m$ be vectors in $\mathbb{R}^d $ such that there exist a unit norm vector $\bm{x}$ that satisfies $|\sum_{r=1}^m \langle \bm{v}_r,\bm{x} \rangle^3|\geq 1.$ Then, for  $\bm{\xi}_1,\dots,\bm{\xi}_{k} \sim \mathcal{N} (0,   \sigma^2 \mathbf{I}_d)$ i.i.d., we have: 
\begin{align*}
\mathbb{P}\left[\left|  \sum_{j=1}^{P}\sum_{r=1}^m \langle \bm{v}_r, \bm{\xi}_j \rangle^3 \right|\geq \tilde{\Omega}(\sigma^3)\right]\geq 1- \frac{O(d)}{2^{1/d}}.
\end{align*}
\end{lemma}
\begin{proof}[Proof of \autoref{prop:vr_xi}] Let $\xi_1,\dots,\xi_j\sim \mathcal{N} (0,   \sigma^2 \mathbf{I}_d)$ i.i.d. We decompose $\bm{\xi}_j$ as $\bm{\xi}_j=\tilde{a}_j\bm{x}+\bm{b}_j$ where $\bm{b}_j$ is an independent Gaussian on the orthogonal complement of $\bm{x}$ and $\tilde{a}_j\sim\mathcal{N}(0,\sigma^2).$ Finally, we rewrite $\tilde{a}_j$ as $\tilde{a}_j=\sigma a_j$ where $a_j\sim\mathcal{N}(0,1).$ Therefore, we can rewrite $\sum_{j=1}^P\sum_{r=1}^m \langle \bm{v}_r,\bm{\xi}_j\rangle^3$ as a polynomial $\mathcal{P}(a_1,\dots,a_P)$ defined as:
\begin{equation}
    \begin{aligned}
    \mathcal{P}(a_1,\dots,a_P)&=\sigma^3\sum_{j=1}^Pa_j^3\left(\sum_{r=1}^m \langle \bm{v}_r,\bm{x}\rangle^3\right)+3\sigma^2\sum_{j=1}^Pa_j^2\left(\sum_{r=1}^m \langle \bm{v}_r,\bm{x}\rangle^2\langle \bm{v}_r,\bm{b}_j\rangle\right)\\
    &+3\sigma\sum_{j=1}^Pa_j\left(\sum_{r=1}^m\langle \bm{v}_r,\bm{x}\rangle\langle \bm{v}_r,\bm{b}_j\rangle^2\right)+\sum_{j=1}^P\sum_{r=1}^m  \langle \bm{v}_r,\bm{b}_j\rangle^3.
\end{aligned}
\end{equation}
We now compute the mean and variance of $  \mathcal{P}(a_1,\dots,a_P).$ Those quantities are obtained through the corresponding Hermite polynomial of $P$ as stated in \autoref{prop:hermite_meanvargauss}. Let $\mathcal{H}(x)$ be an Hermite polynomial of degree 3. Since the Hermite basis is given by $H_0(x)=1,$ $H_{e_1}(x)=x,$ $H_{e_2}(x)=x^2-1$ and $H_{e_3}(x)=x^3-3x$, for $\alpha_j,\beta_j,\gamma_j,\delta_j\in\mathbb{R}$, we have: 
\begin{align}
    \mathcal{H}(a_1,\dots, a_P)&= \sum_{j=1}^P \alpha_jH_{e_3}(a_j)+\sum_{j=1}^P\beta_j H_{e_2}(a_j)+\gamma \sum_{j=1}^P H_{e_1}(a_j)+\delta \sum_{j=1}^P H_{e_0}(a_j)\nonumber\\
    &=\sum_{j=1}^P\alpha_j(a_j^3-3a_j)+\sum_{j=1}^P\beta_j(a_j^2-1)+ \sum_{j=1}^P\gamma_j a_j+\sum_{j=1}^P\delta_j\nonumber\\
    &=\sum_{j=1}^P \alpha _ja_j^3 + \sum_{j=1}^P\beta_j a_j^2+ \sum_{j=1}^P (\gamma_j -3\alpha_j) a_j+\sum_{j=1}^P(\delta_j-\beta_j).
\end{align}
Since the decomposition of a polynomial in the monomial basis is unique, we can equate the coefficients of $H$ and $P$ and obtain:
\begin{align}
    \begin{cases}
    \alpha_j = \sigma^3 \sum_{r=1}^m \langle\bm{v}_r,\bm{x}\rangle^3\\
    \beta_j = 3\sigma^2\sum_{r=1}^m \langle \bm{v}_r,\bm{x}\rangle^2\langle \bm{v}_r,\bm{b}_j\rangle\\
    \gamma_j  = 3\sigma\sum_{r=1}^m\langle \bm{v}_r,\bm{x}\rangle\langle v_r,b_j\rangle^2 + 3\sigma^3 \sum_{r=1}^m \langle \bm{v}_r,\bm{x}\rangle^3\\
    \delta_j=\sum_{r=1}^m  \langle \bm{v}_r,\bm{b}_j\rangle^3+ 3\sigma^2\sum_{r=1}^m \langle \bm{v}_r,\bm{x}\rangle^2\langle \bm{v}_r,\bm{b}_j\rangle
    \end{cases}.
\end{align}
By applying \autoref{prop:hermite_meanvargauss}, we get that $\mathrm{Var}[P(a)]=\sum_{j=1}^P\alpha_j^2+\sum_{j=1}^P\beta_j^2+\sum_{j=1}^P\gamma_j^2\geq \sum_{j=1}^P\alpha_j^2.$ By using this lower bound on the variance, the fact that $|\sum_{r=1}^m \langle \bm{v}_r, \bm{x} \rangle^3|\geq 1$ and \autoref{prop:anticoncent}, we obtain
\begin{align}\label{eq:proba_tmp1}
\mathbb{P}\left[\left|  \sum_{j=1}^{P}\sum_{r=1}^m \langle \bm{v}_r, \bm{\xi}_j \rangle^3 \right|\geq \epsilon \sigma^3\right]\geq 1-O(d) \epsilon^{1/d}
\end{align}
Setting $\epsilon=1/2$ in \eqref{eq:proba_tmp1} yields the desired result. 
 

\end{proof}

\subsubsection{Properties of the cube of a Gaussian}
\label{sec:cube_gauss}

\begin{lemma}\label{lem:Xcubedist}
Let $X\sim\mathcal{N}(0,\sigma^2)$. Then, $X^3$ is $\sigma^3$-subGaussian. 
\end{lemma} 
 \begin{proof}[Proof of \autoref{lem:Xcubedist}]
 By definition of the moment generating function, we have: 
\begin{align*}
    M_{X^3}(t)&=\sum_{i=0}^{\infty} \frac{t^i E[X^{3i}]}{i!}=\sum_{k=0}^{\infty} \frac{t^{2k}\sigma^{6k}(2k-1)!! }{(2k)!}=\sum_{k=0}^{\infty}\frac{t^{2k}\sigma^{6k} }{2^k k!} = e^{\frac{t^2\sigma^6}{2}}.
\end{align*}
\end{proof}

\begin{lemma}\label{lem:sumsubG2}
 Let $(\bm{X}[1],\dots,\bm{X}[P-1])$ be i.i.d. random variables such that $\bm{X}[j]\sim\mathcal{N}(0,\sigma^2\mathbf{I}_d).$ Let $(\bm{w}_{1},\dots,\bm{w}_m)$ be fixed vectors such that $w_r\in\mathbb{R}^d.$ Therefore,
 \begin{align*}
     \sum_{s=1}^m\sum_{j=1}^{P-1} \langle \bm{w}_s, \bm{X}[j]\rangle^3 \text{ is } \textstyle(\sigma^3\sqrt{P-1}\sqrt{\sum_{s=1}^m\|\bm{w}_s\|_2^6})-\text{subGaussian.}
 \end{align*}
  
\end{lemma}

\begin{proof}
We know that $\langle \bm{w}_s,\bm{X}[j]\rangle \sim \mathcal{N}(0,\|\bm{w}_s\|_2^2\sigma^2)$. Therefore, $\langle \bm{w}_s,\bm{X}[j]\rangle^3$ is the cube of a centered Gaussian. From \autoref{lem:Xcubedist}, $\langle \bm{w}_s,\bm{X}[j]\rangle^3$ is $\sigma^3\|\bm{w}_s\|_2^3$-subGaussian. Using \autoref{lem:sumsubG}, we deduce that $\sum_{j=1}^{P-1} \langle \bm{w}_s, \bm{X}[j]\rangle^3$ is $\sqrt{P}\sigma^3\|\bm{w}_s\|_2^3$-subGaussian. Applying again \autoref{lem:sumsubG}, we finally obtain that     $\sum_{s=1}^m\sum_{j=1}^{P-1} \langle \bm{w}_s, \bm{X}[j]\rangle^3$ is $\sigma^3\sqrt{P-1}\sqrt{\sum_{s=1}^m\|\bm{w}_s\|_2^6}$-subGaussian. 
\end{proof}

\subsection{Tensor Power Method Bound}

In this subsection we establish a lemma for comparing the growth speed of two sequences of updates of the form $z^{(t+1)}=z^{(t)}+\eta C^{(t)}(z^{(t)})^{2}$.  This technique is reminiscent of the classical analysis of the growth of eigenvalues on the (incremental) tensor power method of degree $2$ and is stated in full generality in \citep{allen2020towards}.

\subsubsection{Bounds for GD}

\begin{lemma}\label{lem:pow_method}
Let $\{z^{(t)}\}_{t=0}^T$ be a positive sequence defined by the following recursions
\begin{align*}
    \begin{cases}
        z^{(t+1)}\geq z^{(t)} +m(z^{(t)})^2\\ 
         z^{(t+1)}\leq z^{(t)} + M (z^{(t)})^2 
    \end{cases},
\end{align*}
where $z^{(0)}>0$ is the initialization and $m,M>0$.
Let $\upsilon>0$ such that $z^{(0)}\leq \upsilon.$ Then, the time $t_0$ such that $z_t\geq \upsilon$ for all $t\geq t_0$ is: 
\begin{align*}
    t_0= \frac{3}{mz^{(0)}}+\frac{8M}{m}\left\lceil \frac{\log(\upsilon/z_0)}{\log(2)}\right\rceil.
\end{align*}

%
%
\end{lemma}
\begin{proof}[Proof of \autoref{lem:pow_method}]
 Let $n\in\mathbb{N}^*$. Let $T_n$ be the time where $z^{(t)}\geq 2^n z^{(0)}$. This time exists because $z^{(t)}$ is a non-decreasing sequence. We want to find an upper bound on this time. We start with the case $n=1.$ By summing the recursion, we have:
 \begin{align}\label{eq:fierihvei}
     z^{(T_1)}\geq z^{(0)}+m\sum_{s=0}^{T_1-1} (z^{(s)})^2.
 \end{align}
 We use the fact that $z^{(s)}\geq z^{(0))}$ in \eqref{eq:fierihvei} and obtain:
 \begin{align}\label{eq:hvefihedw}
     T_1 \leq \frac{z^{(T_1)} - z^{(0)}}{m (z^{(0)})^2}.
 \end{align}
 Now, we want to bound $z^{(T_1)} - z^{(0)}$. Using again the recursion and $z^{(T_1-1)}\leq 2z^{(0)}$, we have:
 \begin{align}\label{eq:hifeihrhie}
     z^{(T_1)}\leq z^{(T_1-1)}+M(z^{(T_1-1)})^2\leq 2z^{(0)} + 4M(z^{(0)})^2.
 \end{align}
 Combining \eqref{eq:hvefihedw} and \eqref{eq:hifeihrhie}, we get a bound on $T_1.$
 \begin{align}\label{eq:bdtdewojewd}
     T_1\leq\frac{ 1}{m (z^{(0)})}+\frac{4M}{m}.
 \end{align}
 Now, let's find a bound for $T_n$. Starting from the recursion and using the fact that $z^{(s)}\geq 2^{n-1}z^{(0)}$ for $s\geq T_{n-1}$ we have:
 \begin{align}\label{eq:vbjfwcsd}
     z^{(T_n)}\geq z^{(T_{n-1})} +  m\sum_{s=T_{n-1}}^{T_{n}-1}(z^{(s)})^2\geq z^{(T_{n-1})} +(2^{n-1})^2m(z^{(0)})^2 (T_{n}-T_{n-1}).
 \end{align}
 On the other hand, by using $z^{(T_n-1)}\leq 2^nz^{(0)}$ we upper bound $z^{(T_n)}$ as follows.
 \begin{align}
     z^{(T_n)}&\leq z^{(T_n-1)}+M(z^{(T_n-1)})^2\leq 2^n z^{(0)} + M2^{2n}(z^{(0)})^2.
 \end{align}
 Besides, we know that $z^{(T_{n-1})}\geq 2^{n-1}z^{(0)}$. Therefore, we upper bound $z^{(T_n)}-z^{(T_{n-1})}$ as
 \begin{align}\label{eq:cneknecfr}
    z^{(T_n)}-z^{(T_{n-1})}\leq  2^{n-1} z^{(0)} + M2^{2n}(z^{(0)})^2.
 \end{align}
 Combining \eqref{eq:vbjfwcsd} and \eqref{eq:cneknecfr} yields:
 \begin{align}\label{eq:cfedaefjnve}
     T_{n}\leq T_{n-1} +\frac{1 }{2^{n-1}m(z^{(0)})} + \frac{4M }{ m }.
 \end{align}
 We now sum \eqref{eq:cfedaefjnve} for $n=2,\dots,n$, use \eqref{eq:bdtdewojewd} and obtain:
 \begin{align}\label{eq:nfceknenvnervvre}
     T_n\leq T_1+ \frac{2}{mz^{(0)}} + \frac{4Mn}{m}\leq \frac{3}{mz^{(0)}}+\frac{4M(n+1)}{m}\leq \frac{3}{mz^{(0)}}+\frac{8Mn}{m}.
 \end{align}
Lastly, we know that $n$ satisfies $2^nz^{(0)}\geq \upsilon$ this implies that we can set $n=\left\lceil \frac{\log(\upsilon/z_0)}{\log(2)}\right\rceil $ in 
 \eqref{eq:nfceknenvnervvre}.
\end{proof}

\begin{lemma}\label{lem:pow_method_sum}

Let $\{z^{(t)}\}_{t=0}^T$ be a positive sequence defined by the following recursion
\begin{align}
    \begin{cases}\label{eq:ztrecinit}
        z^{(t)}\geq z^{(0)} + A \sum_{s=0}^{t-1}(z^{(s)})^2 - C\\
         z^{(t)}\leq z^{(0)} + A \sum_{s=0}^{t-1}(z^{(s)})^2 +C
    \end{cases},
\end{align}
where $A,C>0$ and $z^{(0)}>0$ is the initialization. Assume that $C\leq z^{(0)}/2.$  Let $\upsilon>0$ such that $z^{(0)}\leq \upsilon.$  Then, the time $t_0$ such that $z^{(t)}\geq \upsilon$ is upper bounded as: 
\begin{align*}
    t_0=8\left\lceil \frac{\log(\upsilon/z_0)}{\log(2)}\right\rceil  +\frac{21}{(z^{(0)})A}.
\end{align*}
\end{lemma}

\begin{proof}[Proof of \autoref{lem:pow_method_sum}]  
Let $n\in\mathbb{N}^*.$  Let $T_n$ be the time where $z^{(t)}\geq 2^{n-1} z^{(0)}$. We want to upper bound this time. We start with the case $n=1.$ We have:
\begin{align}\label{eq:zt1}
    z^{(T_1)}\geq z^{(0)} + A \sum_{s=0}^{T_1-1}(z^{(s)})^2 - C
\end{align}
 By assumption, we know that $C\leq z^{(0)}/2.$ This implies that for all $z^{(t)}\geq z^{(0)}/2$ for all $t\geq0.$ Plugging this in \eqref{eq:zt1} yields:
 \begin{align}\label{eq:zt1wddwwed}
    z^{(T_1)}\geq z^{(0)} + \frac{A}{4} T_1(z^{(0)})^2 - C
\end{align}
From \eqref{eq:zt1wddwwed}, we deduce that:
\begin{align}\label{eq:inittmph1}
    T_1\leq  4\frac{z^{(T_1)}-z^{(0)}+C}{A(z^{(0)})^2}.
\end{align}
Now, we want to upper bound $z^{(T_1)}-z^{(0)}$. Using \eqref{eq:ztrecinit}, we deduce that: 
\begin{align}
    \begin{cases}\label{eq:ztrecuse1}
        z^{(T_1)}\geq z^{(0)} + A \sum_{s=0}^{T_1-1}(z^{(s)})^2 - C\\
         z^{(T_1-1)}\leq z^{(0)} + A \sum_{s=0}^{T_1-2}(z^{(s)})^2 +C
    \end{cases}.
\end{align}
Combining the two equations in \eqref{eq:ztrecuse1} yields
\begin{align}\label{eq:zT1diff}
    z^{(T_1)}-z^{(T_1-1)}\leq A (z^{(T_1-1)})^2 +2C.
\end{align}
Since $T_1$ is the first time where $z^{(T_1)}\geq z^{(0)}$, we have $z^{(T_1-1)}\leq z^{(0)}$. Plugging this in \eqref{eq:zT1diff} leads to:
\begin{align}\label{eq:zT1bd11}
    z^{(T_1)}\leq z^{(0)} + A(z^{(0)})^2 +2C.
\end{align}
Finally, using \eqref{eq:zT1bd11} in \eqref{eq:inittmph1} and $C=o(z^{(0)})$ gives an upper bound on $T_1.$
\begin{align}\label{eq:resH1111fin}
    T_1\leq  4+ \frac{3C}{A(z^{(0)})^2} \leq 4+ \frac{3}{A(z^{(0)})}  .
\end{align}
%
%
Now, let's find a bound for $T_n$. Starting from the recursion, we have:
 \begin{align}\label{eq:vbjfwcsvfeeved}
 \begin{cases}
    z^{(T_n)}\geq z^{(0)} + A \sum_{s=0}^{T_n-1}(z^{(s)})^2 - C  \\
    z^{(T_{n-1})}\leq z^{(0)} + A \sum_{s=0}^{T_{n-1}-1}(z^{(s)})^2 + C 
 \end{cases}.
 \end{align}
We substract the two equations in \eqref{eq:vbjfwcsvfeeved}, use $z^{(s)}\geq 2^{n-2}$ for $s\geq T_{n-1}$ and obtain:
\begin{align}\label{eq:knecnfnkrvwefew}
    z^{(T_n)}-z^{(T_{n-1})}\geq A\sum_{s=T_{n-1}}^{T_n-1}(z^{(s)})^2 -2C\geq 2^{2(n-2)}(z^{(0)})^2A(T_n-T_{n-1})-2C.
\end{align}
On the other hand, from the recursion, we have the following inequalities:
 \begin{align}\label{eq:vbjfwcsvfeervervd}
 \begin{cases}
    z^{(T_n)}\leq z^{(0)} + A \sum_{s=0}^{T_n-1}(z^{(s)})^2 -C  \\
    z^{(T_{n}-1)}\geq z^{(0)} + A \sum_{s=0}^{T_{n}-2}(z^{(s)})^2 - C 
 \end{cases}.
 \end{align}
We substract the two equations in \eqref{eq:vbjfwcsvfeervervd}, use $z^{(T_n-1)}\leq 2^{n-1}z^{(0)}$ and upper bound $z^{(T_n)}$ as follows.
 \begin{align}
     z^{(T_n)}&\leq z^{(T_n-1)}+ A(z^{(T_n-1)})^2+2C\leq 2^{n-1}z^{(0)} + 2^{2(n-1)} A(z^{(0)})^2+2C.
 \end{align}
Besides, we know that $z^{(T_{n-1})}\geq 2^{n-2}z^{(0)}$. Therefore, we upper bound $z^{(T_n)}-z^{(T_{n-1})}$ as
 \begin{align}\label{eq:ncdscnjdw}
     z^{(T_n)}-z^{(T_{n-1})}&\leq  2^{n-2}z^{(0)} + 2^{2(n-1)} A(z^{(0)})^2+2C.
 \end{align}
Combining \eqref{eq:knecnfnkrvwefew} and \eqref{eq:ncdscnjdw} yields:
\begin{align}\label{eq:fneknfkenr}
    T_n&\leq T_{n-1}+ 4+\frac{1}{2^{(n-2)}(z^{(0)})A} + \frac{4C}{2^{2(n-2)}(z^{(0)})^2A}
\end{align}
 We now sum \eqref{eq:fneknfkenr} for $n=2,\dots,n$, use $C=o(z^{(0)})$ and then \eqref{eq:resH1111fin} to obtain:
 \begin{align}\label{eq:nfceweecwceknenvnervvre}
     T_n&\leq T_{1}+ 4n+\frac{2}{(z^{(0)})A} + \frac{16C}{(z^{(0)})^2A}\leq T_{1}+ 4n+\frac{18}{(z^{(0)})A} \leq 4(n+1)  +\frac{21}{(z^{(0)})A}.
 \end{align}
 Lastly, we know that $n$ satisfies $2^nz^{(0)}\geq \upsilon$ this implies that we can set $n=\left\lceil \frac{\log(\upsilon/z_0)}{\log(2)}\right\rceil $ in 
 \eqref{eq:nfceweecwceknenvnervvre}.
\end{proof}

\subsubsection{Bounds for GD+M}

\begin{lemma}[Tensor Power Method for momentum]\label{lem:pow_method_GDM}
Let $\gamma\in (0,1).$ Let $\{c^{(t)}\}_{t\geq 0}$ and $\{\mathcal{G}^{(t)}\}$ be positive sequences defined by the following recursions
\begin{align*}
\begin{cases}
     \mathcal{G}^{(t+1)}=\gamma \mathcal{G}^{(t)} - \alpha^3 (c^{(t)})^2,\\
     c^{(t+1)}=c^{(t)}-\eta \mathcal{G}^{(t+1)}
\end{cases},
\end{align*}
and respectively initialized by $z^{(0)}\geq 0$ and $\mathcal{G}^{(0)}=0$.
Let $\upsilon\in\mathbb{R}$ such that $z^{(0)}\leq \upsilon.$ Then, the time $t_0$ such that $z^{(t)}\geq \upsilon$ is: 
\begin{align*}
    t_0= \frac{1}{1-\gamma}\left\lceil \frac{\log(\upsilon )}{\log(1+\delta)}\right\rceil+\frac{1+\delta}{\eta(1-e^{-1})\alpha^3 c^{(0)}},
\end{align*}
where $\delta \in (0,1).$
\end{lemma}
\begin{proof}[Proof of \autoref{lem:pow_method_GDM}] Let $\delta \in (0,1).$ We want to prove the following induction hypotheses:
\begin{enumerate}
    \item After $T_n= \frac{n}{1-\gamma} + \sum_{j=0}^{n-2}\frac{\delta(\delta+1)^{j}}{\eta(1-e^{-1})\alpha^3 c^{(0)}\sum_{\tau=0}^{j} e^{-(j-\tau)} (1+\delta)^{2\tau}}$ iterations, we have: 
    \begin{align}\tag{TPM-1}\label{eq:indhtmp1}
             -\mathcal{G}^{(T_n)}\geq (1-e^{-1})\alpha^3 (c^{(0)})^2\sum_{\tau=0}^{n-1} e^{-(n-1-\tau)} (1+\delta)^{2\tau}.
    \end{align}
    \item After $T_n'=  \frac{n}{1-\gamma}+\sum_{j=0}^{n-1}\frac{\delta(\delta+1)^{j}}{\eta(1-e^{-1})\alpha^3 c^{(0)}\sum_{\tau=0}^{j} e^{-(j-\tau)} (1+\delta)^{2\tau}}$, we have: 
    \begin{align}\tag{TPM-2}\label{eq:indhtmp2}
        c^{(T_n')}&\geq (1+\delta)^{n} c^{(0)}.
    \end{align}
\end{enumerate}
Let's first prove \eqref{eq:indhtmp1} and \eqref{eq:indhtmp2} for $n=1.$ First, by using the momentum update, we have: 
\begin{align}\label{eq:momentuefnjvsfnm}
     -   \mathcal{G}^{(T_1)}&= (1-\gamma)\alpha^3\sum_{\tau=0}^{T_1-1} \gamma^{T_1-1-\tau} (c^{(\tau)})^2 \geq \alpha^3(1-\gamma^{T_0})(c^{(0)})^2.
    \end{align}
Setting $T_1=1/(1-\gamma)$ and using $\gamma=1-\varepsilon$, we have $1-\gamma^{\frac{1}{1-\gamma}}=1-\exp(\log(1-\varepsilon)/\varepsilon)=1-e^{-1}.$ Plugging this in \eqref{eq:momentuefnjvsfnm} yields \eqref{eq:indhtmp1} for $n=1.$ 

Regarding \eqref{eq:indhtmp2}, we use the iterate update to have: 
 \begin{align}
     c^{(T_1')}&=c^{(T_1)}-\eta\sum_{\tau=T_1}^{T_1'-1}  \mathcal{G}^{(\tau)}\nonumber\\
     &\geq c^{(0)}  +\eta\alpha^3(1-e^{-1})(c^{(0)})^2(T_1'-T_1),\label{eq:ccTbfech}
 \end{align}
where we used $c^{(T_1)}\geq c^{(0)}$ and \eqref{eq:momentuefnjvsfnm} to obtain \eqref{eq:ccTbfech}. Since $T_1'+1$ is the first time where $c^{(t)}\geq (1+\delta)c^{(0)},$ we further simplify \eqref{eq:ccTbfech} to obtain: 
\begin{align}\label{eq:T0prim}
    T_1'&=T_1 +\frac{\delta}{\eta\alpha^3(1-e^{-1})c^{(0)}} = \frac{1}{1-\gamma}+\frac{\delta}{\eta\alpha^3(1-e^{-1})c^{(0)}}.
\end{align}
We therefore obtained \eqref{eq:indhtmp2} for $n=1.$ Let's now assume \eqref{eq:indhtmp1} and \eqref{eq:indhtmp2} for $n$. We now want to prove these induction hypotheses for $n+1.$ First, by using the momentum update, we have: 
\begin{align}\label{eq:momentuefnjvfeesfnm}
     -   \mathcal{G}^{(T_{n+1})}&=- \gamma^{T_{n+1}-T_n'} \mathcal{G}^{(T_{n}')}+ (1-\gamma)\alpha^3\sum_{\tau=T_n'}^{T_{n+1}-1} \gamma^{T_{n+1}-1-\tau} (c^{(\tau)})^2.
    \end{align}
From \eqref{eq:indhtmp2} for $n$, we know that $c^{(t)}\geq (1+\delta)^{n}c^{(0)}$ for $t>T_n'$. Therefore, \eqref{eq:momentuefnjvfeesfnm} becomes:
\begin{align}\label{eq:momenfeesf}
     -   \mathcal{G}^{(T_{n+1})}&=- \gamma^{T_{n+1}-T_n'} \mathcal{G}^{(T_{n}')}+  \alpha^3(1-\gamma^{T_{n+1}-T_n'}) (1+\delta)^{2n}(c^{(0)})^2.
    \end{align}
 From \eqref{eq:indhtmp1}, we know that $ -\mathcal{G}^{(T_n')}\geq (1-e^{-1})\alpha^3 (c^{(0)})^2\sum_{\tau=0}^{n-1} e^{-(n-1-\tau)} (1+\delta)^{2\tau}$  for $t\geq T_n.$ Therefore, we simplify \eqref{eq:momenfeesf} as: 
 \begin{equation}
      \begin{aligned}\label{eq:momenfeefedcscdsf}
     -   \mathcal{G}^{(T_{n+1})}&\geq \gamma^{T_{n+1}-T_n'}(1-e^{-1})\alpha^3 (c^{(0)})^2\sum_{\tau=0}^{n-1} e^{-(n-1-\tau)} (1+\delta)^{2\tau}\\
     &+ \alpha^3(1-\gamma^{T_{n+1}-T_n'}) (1+\delta)^{2n}(c^{(0)})^2.
    \end{aligned}
 \end{equation}
When we set $T_{n+1}$ as in \eqref{eq:indhtmp1}, we have $T_{n+1}-T_n'=\frac{1}{1-\gamma}.$ Moreover, since $\gamma=1-\varepsilon$, we have $\gamma^{\frac{1}{1-\gamma}}=e^{-1}$. Using these two observations,  \eqref{eq:momenfeefedcscdsf} is thus equal to:
 \begin{align}
     -   \mathcal{G}^{(T_{n+1})}&\geq (1-e^{-1})\alpha^3 (c^{(0)})^2\sum_{\tau=0}^{n-1} e^{-(n-\tau)} (1+\delta)^{2\tau}\nonumber\\
     &+  \alpha^3(1-e^{-1}) (1+\delta)^{2n}(c^{(0)})^2\nonumber\\
     &=(1-e^{-1})\alpha^3 (c^{(0)})^2\sum_{\tau=0}^{n} e^{-(n-\tau)} (1+\delta)^{2\tau}.\label{eq:momenfeefedcscdsffec}
    \end{align}
 We therefore proved \eqref{eq:indhtmp1} for $n+1.$ Now, let's prove \eqref{eq:indhtmp2}. We use the iterates update and obtain: 
 \begin{align}
     c^{(T_{n+1}')}&=c^{(T_{n+1})}-\eta\sum_{\tau=T_{n+1}}^{T_{n+1}'-1}  \mathcal{G}^{(\tau)}\nonumber\\
     &\geq (\delta+1)^{n} c^{(0)} +\eta(1-e^{-1})\alpha^3 (c^{(0)})^2\sum_{\tau=0}^{n} e^{-(n-\tau)} (1+\delta)^{2\tau}(T_{n+1}-T_{n+1}'),\label{eq:ccTbfecfecdh}
 \end{align}
where we used $c^{(T_{n+1})}\geq (\delta+1)^{n} c^{(0)}$ and \eqref{eq:momenfeefedcscdsffec} in the last inequality. Since $T_{n+1}'+1$ is the first time where $c^{(t)}\geq (1+\delta)^{n+1}c^{(0)},$ we further simplify \eqref{eq:ccTbfecfecdh} to obtain: 
\begin{align}\label{eq:T0pricfdcsdm}
    T_{n+1}'&=T_{n+1} +\frac{\delta(\delta+1)^{n-1}}{\eta(1-e^{-1})\alpha^3 (c^{(0)})^2\sum_{\tau=0}^{n} e^{-(n-\tau)} (1+\delta)^{2\tau}}\nonumber\\
    &= \frac{n+1}{1-\gamma}+\sum_{j=0}^{n-1}\frac{\delta(\delta+1)^{j}}{\eta(1-e^{-1})\alpha^3 c^{(0)}\sum_{\tau=0}^{j} e^{-(j-\tau)} (1+\delta)^{2\tau}}\nonumber\\
    &+\frac{\delta(\delta+1)^{n}}{\eta(1-e^{-1})\alpha^3 (c^{(0)})^2\sum_{\tau=0}^{n} e^{-(n-\tau)} (1+\delta)^{2\tau}}\nonumber\\
    &= \frac{n+1}{1-\gamma}+ \sum_{j=0}^{n}\frac{\delta(\delta+1)^{j}}{\eta(1-e^{-1})\alpha^3 c^{(0)}\sum_{\tau=0}^{j} e^{-(j-\tau)} (1+\delta)^{2\tau}}.
\end{align}
We therefore proved \eqref{eq:indhtmp2} for $n+1.$

Let's now obtain an upper bound on $T_n'.$ We have:
\begin{align}
    T_n'&\leq \frac{n}{1-\gamma}+\frac{\delta }{\eta(1-e^{-1})\alpha^3 c^{(0)} }\sum_{j=0}^{n-1}\frac{1 }{  (1+\delta)^{j}}\nonumber\\
    &\leq  \frac{n}{1-\gamma}+\frac{1+\delta}{\eta(1-e^{-1})\alpha^3 c^{(0)}}:=\mathscr{T}_n.
\end{align}
Finally, we choose $n$ such that $(1+\delta)^n\geq \upsilon$ or equivalently, $n=\left\lceil \frac{\log(\upsilon )}{\log(1+\delta)}\right\rceil$. Plugging this choice in $\mathscr{T}_n$ yields the desired bound.
\end{proof}

\subsection{Optimization lemmas}

\begin{definition}[Smooth function] Let $f\colon \mathbb{R}^{n\times d}\rightarrow \mathbb{R}$. $f$ is $\beta$-smooth if $\|\nabla f(\bm{X})-\nabla f(\bm{Y})\|_2\leq \beta\|\bm{X}-\bm{Y}\|_2,$ for all $X,Y\in\mathbb{R}^{n\times d}.$ A consequence of the smoothness is the inequality:
\begin{align*}
    f(\bm{X})\leq f(\bm{Y})+\langle \nabla f(\bm{Y}),\bm{X}-\bm{Y}\rangle +\frac{L}{2}\|\bm{X}-\bm{Y}\|_2^2, \quad \text{for all }\bm{X},\bm{Y}\in\mathbb{R}^{n\times d}.
\end{align*}
\end{definition}

\begin{lemma}[Descent lemma for GD]\label{lem:desc} Let $f\colon \mathbb{R}^{n\times d}\rightarrow \mathbb{R}$ be a $\beta$-smooth function. Let $\bm{W}^{(t+1)} \in \mathbb{R}^{n\times d}$ be an iterate of $GD$ with learning rate $\eta\in (0,1/L).$ Then, we have
\begin{align*}
    f(\bm{W}^{(t+1)})\leq f(\bm{W}^{(t)}) -\frac{\eta}{2}\|\nabla f(\bm{W}^{(t)})\|_2^2.
\end{align*}

\end{lemma}

\begin{proof}[Proof of \autoref{lem:desc}] By applying the definition of smooth functions and the GD update, 
we have:
\begin{align}
    f(\bm{W}^{(t+1)})&\leq f(\bm{W}^{(t)}) + \langle \nabla f(\bm{W}^{(t)}),\bm{W}^{(t+1)}-\bm{W}^{(t)}\rangle +\frac{L}{2}\|\bm{W}^{(t+1)}-\bm{W}^{(t)}\|_2^2\nonumber\\
    &= f(\bm{W}^{(t)})  -\eta \|\nabla f(\bm{W}^{(t)}) \|_2^2 + \frac{L\eta^2}{2}\|\nabla f(\bm{W}^{(t)}) \|_2^2.\label{eq:desc_lem}
\end{align}
Setting $\eta <1/L$ in \eqref{eq:desc_lem} leads to the expected result.

\end{proof}

\begin{lemma}[Sublinear convergence]\label{lem:sublinear} Let $\mathscr{T}\geq 0$. Let $(x_t)_{t> \mathscr{T}}$ be a non-negative sequence that satisfies the recursion: 
$x^{(t+1)}\leq x^{(t)} - A (x^{(t)})^2,$
for $A>0.$ Then, it is bounded at a time $t>\mathscr{T}$ as
\begin{align}
    x^{(t)}\leq \frac{1}{A(t-\mathscr{T})}.
\end{align}
\end{lemma}
\begin{proof}[Proof of \autoref{lem:sublinear}] Let $\tau\in (\mathscr{T},t]$.  By multiplying each side of the recursion by $(x^{(\tau)}x^{(\tau+1)})^{-1}$, we get: 
\begin{align}\label{eq:Axx1}
\frac{Ax^{(\tau)}}{x^{(\tau+1)}}\leq \frac{1}{x^{(\tau+1)}} -  \frac{1}{x^{(\tau)}}.
\end{align}
Besides, the update rule indicates that $x^{(\tau)}$ is non-increasing i.e. $x^{(\tau+1)}\leq x^{(\tau)}.$ Using this fact in \eqref{eq:Axx1} yields: 
\begin{align}\label{eq:Axx2}
A\leq \frac{1}{x^{(\tau+1)}} -  \frac{1}{x^{(\tau)}}.
\end{align}
Now, we sum up \eqref{eq:Axx2} for $\tau=\mathscr{T},\dots,t-1$ and obtain:
\begin{align}\label{eq:Axx3}
    A(t -\mathscr{T})\leq \frac{1}{x^{(t)}}-\frac{1}{x^{(\mathscr{T})}}\leq \frac{1}{x^{(t)}}.
\end{align}
Inverting \eqref{eq:Axx3} yields the expected result.
\end{proof}

\subsection{Other useful lemmas}

\subsubsection{Logarithmic inequalities}

\begin{lemma}[Connection between derivative and loss]\label{lem:masterlogsigmdelta}
 Let $a_1,\dots,a_m\in \mathbb{R}$ such that $-\delta\leq a_{i}\leq A$ where $A,\delta>0$.  Assume that $\sum_{i=1}^m a_i\in (C_-,C_+)$, where $C_+,C_->0$.
Then, the following inequality holds: 
\begin{equation*}
    \begin{aligned}
    & \frac{0.05e^{-6m A^2\delta}}{C_+\left(1+\frac{m^2\delta^2}{C_-^2}\right)}\log\left(1+e^{-\sum_{i=1}^m a_i^3}\right)
    \leq    \frac{\sum_{i=1}^m a_i^2}{1+\exp(\sum_{i=1}^m a_i^3)}
     \leq \frac{20m  e^{6m A^2\delta}}{C_-} \log\left(1+e^{-\sum_{i=1}^m a_i^3}\right).
\end{aligned}
\end{equation*}
\end{lemma}
\begin{proof}[Proof of \autoref{lem:masterlogsigmdelta}]
We apply \autoref{lem:logsigm} to the sequence $a_i+\delta$ and obtain: 
\begin{equation}\label{eq:logsummrefr}
    \begin{aligned}
     \frac{0.1}{C+}\log\left(1+\exp\left(-\sum_{i=1}^m (a_i+\delta)^3\right)\right)&\leq    \frac{\sum_{i=1}^m (a_i+\delta)^2}{1+\exp(\sum_{i=1}^m (a_i+\delta)^3)}\\
     &\leq \frac{10m}{C_-} \log\left(1+\exp\left(-\sum_{i=1}^m (a_i+\delta)^3\right)\right).
\end{aligned}
\end{equation}
We apply \autoref{lem:logfzd} to further simplify \eqref{eq:logsummrefr}. \begin{equation}\label{eq:logsummrefrds}
    \begin{aligned}
    & \frac{0.1e^{-\sum_{i=1}^m(3a_i^2\delta+3a_i\delta^2+\delta^3)}}{C+}\log\left(1+\exp\left(-\sum_{i=1}^m a_i^3\right)\right)\\
    &\leq    \frac{\sum_{i=1}^m (a_i+\delta)^2}{1+\exp(\sum_{i=1}^m (a_i+\delta)^3)}\\
     &\leq \frac{10m(1+e^{-\sum_{i=1}^m(3a_i^2\delta+3a_i\delta^2+\delta^3)})}{C_-} \log\left(1+\exp\left(-\sum_{i=1}^m a_i^3\right)\right).
\end{aligned}
\end{equation} 
We remark that the term inside the exponential in \eqref{eq:logsummrefrds} can be bounded as: 
\begin{align}\label{eq:gegev}
  0 \leq 2\sum_{i=1}^m a_i^2\delta \leq \sum_{i=1}^m(3a_i^2\delta-2\delta^3)\leq  \sum_{i=1}^m(3a_i^2\delta+3a_i\delta^2+\delta^3)\leq 6\sum_{i=1}^ma_i^2\delta \leq 6A^2m\delta.
\end{align}
Plugging \eqref{eq:gegev} in \eqref{eq:logsummrefrds} yields: 
\begin{equation}\label{eq:evrrds}
    \begin{aligned}
    & \frac{0.1e^{-6m A^2\delta}}{C+}\log\left(1+\exp\left(-\sum_{i=1}^m a_i^3\right)\right)\\
    &\leq    \frac{\sum_{i=1}^m (a_i+\delta)^2}{1+\exp(\sum_{i=1}^m (a_i+\delta)^3)}\\
     &\leq \frac{20m}{C_-} \log\left(1+\exp\left(-\sum_{i=1}^m a_i^3\right)\right).
\end{aligned}
\end{equation} 
Lastly, we need to bound the term in the middle in \eqref{eq:evrrds}. On one hand, we have: 
\begin{align}\label{eq:nckwknce}
  \hspace{-.4cm}\sum_{i=1}^m (a_i+\delta)^2&= 2\sum_{i=1}^m a_i^2+2m\delta^2 \leq 2\left(1+\frac{m^2\delta^2}{\left(\sum_{i=1}^m a_i\right)^2}\right) \sum_{i=1}^m a_i^2 \leq 2\left(1+\frac{m^2\delta^2}{C_-^2}\right)\sum_{i=1}^m a_i^2 .  
\end{align}
Besides, since $x\mapsto x^3$ is non-decreasing, we have the following lower bound:
\begin{align}\label{eq:njcowdccejoe}
    \sum_{i=1}^m (a_i+\delta)^3 \geq \sum_{i=1}^m a_i^3.
\end{align}
Combining \eqref{eq:nckwknce} and \eqref{eq:njcowdccejoe} yields:
\begin{align}\label{eq:uppbeccihve}
    \frac{\sum_{i=1}^m (a_i+\delta)^2}{1+\exp(\sum_{i=1}^m (a_i+\delta)^3)}\leq 2\left(1+\frac{m^2\delta^2}{C_-^2}\right)\frac{\sum_{i=1}^m a_i^2}{1+\exp(\sum_{i=1}^m a_i^3)}.
\end{align}
On the other hand, we have: 
\begin{align}\label{eq:nckwkncrerecr}
  \sum_{i=1}^m (a_i+\delta)^2&\geq \sum_{i=1}^m a_i^2+ 2\delta \sum_{i=1}^m a_i \geq \sum_{i=1}^m a_i^2+2\delta C_-\geq \sum_{i=1}^m a_i^2.
\end{align}
Besides, using \eqref{eq:gegev}, we have:
\begin{align}\label{eq:bjefbebj}
    \sum_{i=1}^m (a_i+\delta)^3\leq \sum_{i=1}^m a_i^3 +  6A^2m\delta.
\end{align}
Thus, using \eqref{eq:nckwkncrerecr} and \eqref{eq:bjefbebj} yields:
\begin{align}\label{eq:lwbdfrfdcjfev}
    \frac{\sum_{i=1}^m (a_i+\delta)^2}{1+\exp(\sum_{i=1}^m (a_i+\delta)^3)}\geq \frac{e^{-6m A^2\delta}\sum_{i=1}^m a_i^2}{1+\exp(\sum_{i=1}^m a_i^3)}.
\end{align}
Finally, we obtain the desired result by combining \eqref{eq:evrrds}, \eqref{eq:uppbeccihve} and  \eqref{eq:lwbdfrfdcjfev}.

\end{proof}

\begin{lemma}[Connection between derivative and loss for positive sequences]\label{lem:logsigm}
 Let $a_1,\dots,a_m\in \mathbb{R}$ such that $a_{i}\geq 0$.   Assume that $\sum_{i=1}^m a_i\in (C_-,C_+)$, where $C_+,C_->0.$ 
Then, the following inequality holds: 
\begin{align*}
 \frac{0.1}{C+}\log\left(1+\exp\left(-\sum_{i=1}^m a_i^3\right)\right)\leq    \frac{\sum_{i=1}^m a_i^2}{1+\exp(\sum_{i=1}^m a_i^3)}\leq \frac{10m}{C_-} \log\left(1+\exp\left(-\sum_{i=1}^m a_i^3\right)\right).
\end{align*}
\end{lemma}

\begin{proof}[Proof of \autoref{lem:logsigm}]
We first remark that: 
\begin{align}
    \frac{\sum_{i=1}^m a_i^2}{1+\exp(\sum_{i=1}^m a_i^3)}&=\frac{\left(\sum_{i=1}^m a_i^2\right)\left(\sum_{j=1}^m a_j\right)}{\left(1+\exp(\sum_{i=1}^m a_i^3)\right)\left(\sum_{j=1}^m a_j\right)}\nonumber\\
    &=\frac{ \sum_{i=1}^m a_i^3 +\sum_{i=1}^m \sum_{j\neq i} a_i^2 a_j }{\left(1+\exp(\sum_{i=1}^m a_i^3)\right)\left(\sum_{j=1}^m a_j\right)}.\label{eq:194}
\end{align}

\paragraph{Upper bound.} We upper bound \eqref{eq:194} by successively applying $\sum_{i=1}^n a_i>C_-$ and $a_i>0$ for all $i$:
\begin{align}
    \frac{\sum_{i=1}^m a_i^2}{1+\exp(\sum_{i=1}^m a_i^3)}&\leq  \frac{\sum_{i=1}^m a_i^3 +\sum_{i=1}^m \sum_{j\neq i} a_i^2 a_j }{C_-\left(1+\exp(\sum_{i=1}^m a_i^3)\right)}\nonumber\\
    &\leq \frac{ \sum_{i=1}^m a_i^3 +\sum_{i=1}^m \sum_{j=1}^m a_i^2 a_j }{C_{-}\left(1+\exp(\sum_{i=1}^m a_i^3)\right)}\label{eq:195}
\end{align}

where we used  $a_i>0$ for all $i$ in \eqref{eq:194}. By applying the rearrangement inequality to \eqref{eq:195}, we obtain: 
\begin{align}\label{eq:196}
    \frac{\sum_{i=1}^m a_i^2}{1+\exp(\sum_{i=1}^m a_i^3)}&\leq \frac{m}{C_-}\frac{\sum_{i=1}^m a_i^3}{1+\exp(\sum_{i=1}^m a_i^3)}.
\end{align}
We obtain the final bound by applying \autoref{lem:logsigm2} to \eqref{eq:196}.

\paragraph{Lower bound.} We lower bound \eqref{eq:194} by using $\sum_{i=1}^n a_i\leq C_+$ and $\sum_{i=1}^m \sum_{j\neq i} a_i^2 a_j$: 
\begin{align}
    \frac{\sum_{i=1}^m a_i^2}{1+\exp(\sum_{i=1}^m a_i^3)}&\geq   \frac{\sum_{i=1}^m a_i^3 +\sum_{i=1}^m \sum_{j\neq i} a_i^2 a_j }{C_+\left(1+\exp(\sum_{i=1}^m a_i^3)\right)}\nonumber\\ 
    &\geq \frac{ \sum_{i=1}^m a_i^3}{C_{+}\left(1+\exp(\sum_{i=1}^m a_i^3)\right)}.\label{eq:199}
\end{align}
We obtain the final bound by applying \autoref{lem:logsigm2} to \eqref{eq:199}.
\end{proof}

\begin{lemma}[Connection between derivative and loss]\label{lem:logsigm2}
Let $x>0.$ Then, we have: 
\begin{align}
 0.1\log(1+\exp(-x)) \leq \mathfrak
 {S}(x)\leq 10\log(1+\exp(-x))
\end{align}
\end{lemma}

\begin{lemma}\label{lem:sumxtde,} Let $(x^{(t)})_{t\geq 0}$ be a non-negative sequence. Let $A>0.$ Assume that $\sum_{\tau=0}^T x^{(\tau)}\leq A.$ Then, there exists a time $\mathscr{T}\in [T]$ such that $x^{(\mathscr{T})}\leq A/T.$
\end{lemma}
\begin{proof}[Proof of \autoref{lem:sumxtde,}] Assume by contradiction that for all $\tau \in [T]$, $x^{(\tau)}>A/T$. By summing up $x^{\tau}$, we obtain $\sum_{\tau=0}^T x^{(\tau)}>A.$ This contradicts the assumption that $\sum_{\tau=0}^T x^{(\tau)}\leq A.$

\end{proof}

\begin{lemma}[Log inequalities]\label{lem:logfzd} Let $x,y>0.$   Then, the following inequalities holds: 
\begin{enumerate}
    \item Assume that $y\leq x.$ We have:
\begin{align*}
    \log(1+xy)\leq(1+y) \log(1+x).
\end{align*}
\item Assume $y<1$. We have: 
\begin{align*}
    y\log(1+x) \leq \log(1+xy).
\end{align*}
\end{enumerate}

\end{lemma}
\begin{proof}[Proof of \autoref{lem:logfzd}] We first remark that: 
\begin{align}
    \log(1+xy)-\log(1+x)&= \log\left( \frac{1+xy}{1+x}\right)\nonumber\\
    &=\log\left( 1+ \frac{x(y-1)}{1+x}\right).\label{eq:log1xyx}
\end{align}

From \eqref{eq:log1xyx}, we deduce an upper bound as: 
\begin{align}\label{eq:vrjrjt}
    \log(1+xy)-\log(1+x) \leq \log\left( 1+ \frac{x(y+1)}{1+x}\right).
\end{align}

Successively using the inequalities $\log(1+x)\leq x$ and $\frac{x}{1+x}\leq \log(1+x)$ for $x>-1$ in \eqref{eq:vrjrjt} yields:
\begin{align*}
\log(1+xy)-\log(1+x)&\leq(1+y) \frac{x}{1+x}\leq (1+y)\log(1+x).
\end{align*}
This proves item 1 of the Lemma. Let's now prove item 2.  Using $a^z\leq 1+(a-1)z$ for $z\in (0,1)$ and $a\geq 1$, we know that:
\begin{align}\label{eq:ncren}
    (1+x)^y \leq 1+xy.
\end{align}
Since $ \log$ is non-decreasing, applying $ \log$ to \eqref{eq:ncren} proves item 2.

\end{proof}

In \autoref{sec:aux_heifva}, we need to bound the sum $\sum_{s=1}^t\frac{\gamma^{t-s}}{s }$ for $\gamma<1.$ We derive such bound here.
\begin{lemma}\label{lem:hfehiaeier}
Let $t\geq 1$. Then, we have:
\begin{align*}
     \sum_{s=1}^t \frac{\gamma^{t-s} }{s}\leq \gamma^{t-1}+ \gamma^{t/2}\log\left( \frac{t}{2}\right) + \frac{2}{t}\frac{1}{1-\gamma}. 
\end{align*}
\end{lemma}
\begin{proof}[Proof of \autoref{lem:hfehiaeier}] Let $t=1$. Then, we have: 
\begin{align}
    \sum_{s=1}^t \frac{\gamma^{t-s} }{s}=1\leq \gamma^{0}+\gamma^{1/2}\log\left(\frac{1}{2}\right) + \frac{2}{1-\gamma},
\end{align}
given our choice of $\gamma$. Let $t\geq 2.$  We split the sum in two parts as  as follows.
\begin{align}
    \sum_{s=1}^t \frac{\gamma^{t-s} }{s}-\gamma^{t-1}&=   \sum_{s=2}^t \frac{\gamma^{t-s} }{s}\nonumber\\
    &=   \sum_{s=2}^{\lfloor t/2\rfloor} \frac{\gamma^{t-s} }{s} + \sum_{s=\lfloor t/2\rfloor+1}^{t} \frac{\gamma^{t-s} }{s}\nonumber\\
    &\leq \gamma^{t-\lfloor t/2\rfloor} \sum_{s=2}^{\lfloor t/2\rfloor} \frac{1 }{s}+\frac{1}{\lfloor t/2\rfloor+1}\sum_{s=\lfloor t/2\rfloor+1}^{t}\gamma^{t-s}\nonumber\\
    &\leq \gamma^{t/2} \sum_{s=2}^{\lfloor t/2\rfloor} \frac{1 }{s}+\frac{2}{t}\sum_{u=0}^{t-\lfloor t/2\rfloor-1}\gamma^{u}\\
    &\leq  \gamma^{t/2}\log\left( \frac{t}{2}\right) + \frac{2}{t}\frac{1}{1-\gamma},\label{eq:rejojirvqe}
\end{align}
 where we used the harmonic series inequality $\sum_{s=2}^{\mathscr{T}}1/s\leq \log(\mathscr{T})$, $\sum_{u=0}^{\mathscr{T}}\gamma^u\leq 1/(1-\gamma)$ and $\lfloor t/2\rfloor\leq t/2$ in \eqref{eq:rejojirvqe}.

\end{proof}

\end{document}